\documentclass{article} %
\usepackage[export]{adjustbox}

\usepackage{caption}
\usepackage{subcaption}
\usepackage{etoolbox}
\newcommand{\arxiv}[1]{\iftoggle{icml}{}{#1}}
\newcommand{\icml}[1]{\iftoggle{icml}{#1}{}}
\newtoggle{icml}
\global\togglefalse{icml}

\icml{
\PassOptionsToPackage{dvipsnames}{xcolor} 
\usepackage[accepted]{icml2024}
}

\newcommand{\loose}{\looseness=-1}

\usepackage[utf8]{inputenc} %
\usepackage[T1]{fontenc}    %
\usepackage{url}            %
\usepackage{booktabs}       %
\usepackage{amsfonts}       %
\usepackage{nicefrac}       %
\usepackage{microtype}      %

\usepackage{tocloft}            %

\usepackage{enumitem}

\usepackage{breakcites}

\newtoggle{draft}
\togglefalse{draft}

\usepackage{mathrsfs}

\usepackage{algorithm}
\usepackage{verbatim}
\usepackage[noend]{algpseudocode}
\newcommand{\multiline}[1]{\parbox[t]{\dimexpr\linewidth-\algorithmicindent}{#1}}

\usepackage{multicol}

\usepackage{colortbl}

\usepackage{setspace}

\usepackage{transparent}

\usepackage{inconsolata}
\usepackage[scaled=.90]{helvet}
\usepackage{xspace}

\usepackage{pifont}

\arxiv{
\usepackage[letterpaper, left=1in, right=1in, top=1in, bottom=1in]{geometry}
\PassOptionsToPackage{hypertexnames=false}{hyperref}  %
\usepackage{parskip}
\usepackage[dvipsnames]{xcolor}
\usepackage[colorlinks=true, linkcolor=blue!70!black, citecolor=blue!70!black,urlcolor=black,breaklinks=true]{hyperref}
}

\icml{
\usepackage[colorlinks=true, linkcolor=blue!70!black, citecolor=blue!70!black,urlcolor=blue,breaklinks=true]{hyperref}
}

\usepackage{microtype}
\usepackage{hhline}

\makeatletter
\newcommand{\neutralize}[1]{\expandafter\let\csname c@#1\endcsname\count@}
\makeatother

\usepackage{algorithm}

\arxiv{
\usepackage{natbib}
\bibliographystyle{plainnat}
\bibpunct{(}{)}{;}{a}{,}{,}
}

\usepackage{amsthm}
\usepackage{mathtools}
\usepackage{amsmath}
\usepackage{bbm}
\usepackage{amsfonts}
\usepackage{amssymb}

\usepackage{xpatch}

\usepackage{thmtools}
\usepackage{thm-restate}
\declaretheorem[name=Theorem,parent=section]{theorem}
\declaretheorem[name=Lemma,parent=section]{lemma}
\declaretheorem[name=Assumption, parent=section]{assumption}
\declaretheorem[name=Condition, parent=section]{condition}
\declaretheorem[qed=$\triangleleft$,name=Example,style=definition, parent=section]{example}
\declaretheorem[name=Remark,style=definition, parent=section]{remark}
\declaretheorem[name=Proposition, parent=section]{proposition}

\makeatletter
  \renewenvironment{proof}[1][Proof]%
  {%
   \par\noindent{\bfseries\upshape {#1.}\ }%
  }%
  {\qed\newline}
  \makeatother

\theoremstyle{definition}  %

\newtheorem{corollary}{Corollary}[section]

\theoremstyle{plain}
\newtheorem{definition}{Definition}[section]

\xpatchcmd{\proof}{\itshape}{\normalfont\proofnameformat}{}{}
\newcommand{\proofnameformat}{\bfseries}

\usepackage[nameinlink,capitalize]{cleveref}

\newcommand{\pref}[1]{\cref{#1}}
\newcommand{\pfref}[1]{Proof of \pref{#1}}
\newcommand{\savehyperref}[2]{\texorpdfstring{\hyperref[#1]{#2}}{#2}}

\renewcommand{\eqref}[1]{\texorpdfstring{\hyperref[#1]{(\ref*{#1})}}{(\ref*{#1})}}

\crefformat{equation}{#2Eq. (#1)#3}
\Crefformat{equation}{#2Eq. (#1)#3}

\Crefformat{figure}{#2Figure #1#3}
\Crefname{assumption}{Assumption}{Assumptions}
\Crefformat{assumption}{#2Assumption #1#3}
\Crefname{subsubsection}{Section}{Sections}
\crefformat{subsubsection}{#2Section #1#3}
\Crefformat{subsubsection}{#2Section #1#3}

\usepackage{crossreftools}
\newcommand{\creftitle}[1]{\crtcref{#1}}

\usepackage{xparse}

\ExplSyntaxOn
\DeclareDocumentCommand{\XDeclarePairedDelimiter}{mm}
 {
  \__egreg_delimiter_clear_keys: %
  \keys_set:nn { egreg/delimiters } { #2 }
  \use:x %
   {
    \exp_not:n {\NewDocumentCommand{#1}{sO{}m} }
     {
      \exp_not:n { \IfBooleanTF{##1} }
       {
        \exp_not:N \egreg_paired_delimiter_expand:nnnn
         { \exp_not:V \l_egreg_delimiter_left_tl }
         { \exp_not:V \l_egreg_delimiter_right_tl }
         { \exp_not:n { ##3 } }
         { \exp_not:V \l_egreg_delimiter_subscript_tl }
       }
       {
        \exp_not:N \egreg_paired_delimiter_fixed:nnnnn 
         { \exp_not:n { ##2 } }
         { \exp_not:V \l_egreg_delimiter_left_tl }
         { \exp_not:V \l_egreg_delimiter_right_tl }
         { \exp_not:n { ##3 } }
         { \exp_not:V \l_egreg_delimiter_subscript_tl }
       }
     }
   }
 }

\keys_define:nn { egreg/delimiters }
 {
  left      .tl_set:N = \l_egreg_delimiter_left_tl,
  right     .tl_set:N = \l_egreg_delimiter_right_tl,
  subscript .tl_set:N = \l_egreg_delimiter_subscript_tl,
 }

\cs_new_protected:Npn \__egreg_delimiter_clear_keys:
 {
  \keys_set:nn { egreg/delimiters } { left=.,right=.,subscript={} }
 }

\cs_new_protected:Npn \egreg_paired_delimiter_expand:nnnn #1 #2 #3 #4
 {%
  \mathopen{}
  \mathclose\c_group_begin_token
   \left#1
   #3
   \group_insert_after:N \c_group_end_token
   \right#2
   \tl_if_empty:nF {#4} { \c_math_subscript_token {#4} }
 }
\cs_new_protected:Npn \egreg_paired_delimiter_fixed:nnnnn #1 #2 #3 #4 #5
 {
  \mathopen{#1#2}#4\mathclose{#1#3}
  \tl_if_empty:nF {#5} { \c_math_subscript_token {#5} }
 }
\ExplSyntaxOff

\XDeclarePairedDelimiter{\supnorm}{
  left=\lVert,
  right=\rVert,
  subscript=\infty
  }

\DeclarePairedDelimiter{\abs}{\lvert}{\rvert} %
\DeclarePairedDelimiter{\brk}{[}{]}
\DeclarePairedDelimiter{\crl}{\{}{\}}
\DeclarePairedDelimiter{\prn}{(}{)}
\DeclarePairedDelimiter{\nrm}{\|}{\|}
\DeclarePairedDelimiter{\tri}{\langle}{\rangle}

\DeclareMathOperator{\En}{\mathbb{E}}
\DeclareMathOperator{\Enbar}{\widebar{\mathbb{E}}}

\newcommand{\Eh}{\wh{\bbE}}

\DeclareMathOperator*{\argmin}{arg\,min} %
\DeclareMathOperator*{\argmax}{arg\,max}

\newcommand{\wt}[1]{\widetilde{#1}}
\newcommand{\wh}[1]{\widehat{#1}}
\newcommand{\wb}[1]{\widebar{#1}}

\def\ddefloop#1{\ifx\ddefloop#1\else\ddef{#1}\expandafter\ddefloop\fi}
\def\ddef#1{\expandafter\def\csname bb#1\endcsname{\ensuremath{\mathbb{#1}}}}
\ddefloop ABCDEFGHIJKLMNOPQRSTUVWXYZ\ddefloop
\def\ddefloop#1{\ifx\ddefloop#1\else\ddef{#1}\expandafter\ddefloop\fi}
\def\ddef#1{\expandafter\def\csname b#1\endcsname{\ensuremath{\mathbf{#1}}}}
\ddefloop ABCDEFGHIJKLMNOPQRSTUVWXYZ\ddefloop
\def\ddef#1{\expandafter\def\csname sf#1\endcsname{\ensuremath{\mathsf{#1}}}}
\ddefloop ABCDEFGHIJKLMNOPQRSTUVWXYZ\ddefloop
\def\ddef#1{\expandafter\def\csname c#1\endcsname{\ensuremath{\mathcal{#1}}}}
\ddefloop ABCDEFGHIJKLMNOPQRSTUVWXYZ\ddefloop
\def\ddef#1{\expandafter\def\csname h#1\endcsname{\ensuremath{\widehat{#1}}}}
\ddefloop ABCDEFGHIJKLMNOPQRSTUVWXYZ\ddefloop
\def\ddef#1{\expandafter\def\csname hc#1\endcsname{\ensuremath{\widehat{\mathcal{#1}}}}}
\ddefloop ABCDEFGHIJKLMNOPQRSTUVWXYZ\ddefloop
\def\ddef#1{\expandafter\def\csname t#1\endcsname{\ensuremath{\widetilde{#1}}}}
\ddefloop ABCDEFGHIJKLMNOPQRSTUVWXYZ\ddefloop
\def\ddef#1{\expandafter\def\csname tc#1\endcsname{\ensuremath{\widetilde{\mathcal{#1}}}}}
\ddefloop ABCDEFGHIJKLMNOPQRSTUVWXYZ\ddefloop
\def\ddefloop#1{\ifx\ddefloop#1\else\ddef{#1}\expandafter\ddefloop\fi}
\def\ddef#1{\expandafter\def\csname scr#1\endcsname{\ensuremath{\mathscr{#1}}}}
\ddefloop ABCDEFGHIJKLMNOPQRSTUVWXYZ\ddefloop

\newcommand{\ind}{\mathbbm{1}}    %

\newcommand{\eps}{\epsilon}
\newcommand{\veps}{\varepsilon}

\newcommand{\ldef}{\vcentcolon=}
\newcommand{\rdef}{=\vcentcolon}

\newcommand{\desone}{\savehyperref{item:des1}{(1)}\xspace}
\newcommand{\destwo}{\savehyperref{item:des2}{(2)}\xspace}
\newcommand{\desthree}{\savehyperref{item:des3}{(3)}\xspace}

\newcommand{\npsdp}{n_{\mathsf{psdp}}}
\newcommand{\Nweight}{N_{\mathsf{weight}}}
\newcommand{\Npsdp}{N_{\mathsf{psdp}}}
\newcommand{\Nopt}{N_{\mathsf{opt}}}

\newcommand{\Pibar}{\wb{\Pi}}

\newcommand{\Pbar}{\wb{P}}
\newcommand{\bbPbar}{\wb{\bbP}}
\newcommand{\term}{\mathfrak{t}}

\newcommand{\cAbar}{\wb{\cA}}
\newcommand{\cXbar}{\wb{\cX}}

\newcommand{\PibarRNS}{\Pibar_{\mathrm{RNS}}}

\newcommand{\epsopt}{\eps_{\mathsf{opt}}}
\newcommand{\deltaopt}{\delta_{\mathsf{opt}}}
\newcommand{\epsweight}{\eps_{w}}
\newcommand{\deltaweight}{\delta_{w}}

\newcommand{\nweight}{n_{\mathsf{weight}}}

\newcommand{\estimateweight}{\texttt{EstimateWeightFunction}\xspace}
\newcommand{\policyopt}{\texttt{PolicyOptimization}\xspace}

\newcommand{\pCovbar}{\wb{\obj}_{\mathrm{push};h,\veps}}
\newcommand{\pCovbarell}{\wb{\obj}_{\mathrm{push};\ell,\veps}}

\newcommand{\dest}{d_{\mathsf{est}}}
\newcommand{\Best}{B_{\mathsf{est}}}

\newcommand{\vepspac}{\eps}        %

\newcommand{\decopac}[1][\gamma]{\pdec_{#1}}

\newcommand{\psdp}{\texttt{PSDP}\xspace}
\newcommand{\npg}{\texttt{NPG}\xspace}
\newcommand{\fqi}{\texttt{FQI}\xspace}
\newcommand{\reinforce}{\texttt{REINFORCE}\xspace}
\newcommand{\maxent}{\texttt{MaxEnt}\xspace}
\newcommand{\mountaincar}{\texttt{MountainCar}\xspace}
\newcommand{\pendulum}{\texttt{Pendulum}\xspace}

\newcommand{\EstOn}{\mathrm{\mathbf{Est}}^{\mathsf{on}}_{\mathsf{H}}(T)}

\newcommand{\EstOfft}[1][T]{\mathrm{\mathbf{Est}}^{\mathsf{off}}_{\mathsf{H}}(#1)}
\newcommand{\EstProbOn}{\mathrm{\mathbf{Est}}_{\mathsf{H}}^{\mathsf{on}}(\cM,T,\delta)}
\newcommand{\EstProbOff}{\mathrm{\mathbf{Est}}_{\mathsf{H}}^{\mathsf{off}}(\cM,T,\delta)}

\newcommand{\Ccov}{C_{\mathsf{cov}}}

\newcommand{\mainalg}{\texttt{CODEX}\xspace}  %

\newcommand{\mainalgr}{\texttt{CODEX.R}\xspace} %

\newcommand{\mfalg}{\texttt{CODEX.W}\xspace}  %

\newcommand{\mle}{\texttt{MLE}\xspace}  %

\newcommand{\obj}{\Psi}

\newcommand{\cov}{\mathsf{Cov}}

\newcommand{\Cov}{\obj_{h,\veps}}
\newcommand{\CovM}[1][M]{\obj_{h,\veps}^{\sss{#1}}}
\newcommand{\CovMveps}[2][M]{\obj_{h,#2}^{\sss{#1}}}
\newcommand{\CovOpt}{\cov_{h,\veps}}
\newcommand{\CovOptM}[1][M]{\cov_{h,\veps}^{\sss{#1}}}

\newcommand{\CovOptMmax}[1][M]{\cov_{\veps}^{\sss{#1}}}
\newcommand{\CovOptMmaxzero}[1][M]{\cov_{0}^{\sss{#1}}}

\newcommand{\ICovM}[1][M]{\obj_{\infty;h,\veps}^{\sss{#1}}}

\newcommand{\pCovM}[1][M]{\obj_{\mathsf{push};h,\veps}^{\sss{#1}}}

\newcommand{\pCovOptM}[1][M]{\cov_{\mathsf{push};h,\veps}^{\sss{#1}}}

\newcommand{\Cinf}{C_{\infty}}

\newcommand{\ConeM}[1][M]{C_{\mathsf{avg}}^{\sss{#1}}}

\newcommand{\CinfM}[1][M]{C_{\infty}^{\sss{#1}}}
\newcommand{\CpushM}[1][M]{C_{\mathsf{push}}^{\sss{#1}}}
\newcommand{\CfeatM}[1][M]{C_{\phi}^{\sss{#1}}}

\newcommand{\ConehM}[1][M]{C_{\mathsf{avg};h}^{\sss{#1}}}

\newcommand{\CinfhM}[1][M]{C_{\infty;h}^{\sss{#1}}}
\newcommand{\CpushhM}[1][M]{C_{\mathsf{push};h}^{\sss{#1}}}
\newcommand{\CfeathM}[1][M]{C_{\phi;h}^{\sss{#1}}}
\newcommand{\CfeathoM}[1][M]{C_{\phi;h-1}^{\sss{#1}}}

\newcommand{\mainobj}{$L_1$-Coverage\xspace}
\newcommand{\mainopt}{$L_1$-Coverability\xspace}
\newcommand{\linfcov}{$L_{\infty}$-Coverability\xspace}
\newcommand{\pushforwardcov}{Pushforward Coverability\xspace}

\newcommand{\phistar}{\phi^{\star}}

\newcommand{\Dhelsnu}[2]{D^{2}_{\mathsf{H},\nu}\prn*{#1,#2}}

\newcommand{\wstar}{w^{\star}}

\newcommand{\dtil}{\wt{d}}
\newcommand{\dcheck}{\check{d}}

\newcommand{\sM}{\sss{M}}
\newcommand{\sMstar}{\sss{\Mstar}}
\newcommand{\sMbar}{\sss{\Mbar}}
\newcommand{\sMhat}{\sss{\Mhat}}

\newcommand{\pimhat}{\pi\subs{\Mhat}}

\newcommand{\Jm}[1][M]{J\sups{#1}}
\newcommand{\Jmhat}[1][\Mhat]{J\sups{#1}}
\newcommand{\Jmstar}[1][\Mstar]{J\sups{#1}}

  \newcommand{\sfrak}{\mathfrak{s}}

\newcommand{\rhs}{right-hand side\xspace}

\newcommand{\pdec}{\normalfont{\textsf{p-dec}}}

\newcommand{\Emhat}[1][\pi]{\En^{\sss{\Mhat},#1}}

\renewcommand{\c}{\mathrm{c}}

\newcommand{\fmp}{f\sups{M'}}

\newcommand{\errm}[1][M]{\mathrm{err}\sups{#1}}

\renewcommand{\emptyset}{\varnothing}

\newcommand{\filt}{\mathscr{F}}

\newcommand{\hist}{\mathfrak{H}}

\newcommand{\act}{\pi}

\newcommand{\obs}{o}

\newcommand{\etd}{\textsf{E\protect\scalebox{1.04}{2}D}\xspace}

\newcommand{\M}[1]{^{{\scriptscriptstyle M}}}  %
 
\newcommand{\sups}[1]{^{{\scriptscriptstyle#1}}}
\newcommand{\subs}[1]{_{{\scriptscriptstyle#1}}}

\newcommand{\sss}[1]{{\scriptscriptstyle#1}}

\newcommand{\fm}[1][M]{f\sups{#1}}
\newcommand{\pim}[1][M]{\pi_{\sss{#1}}}

\newcommand{\fmbar}{f\sups{\Mbar}}

\newcommand{\pimstar}{\pi\subs{\Mstar}}

\newcommand{\pistar}{\pi_{\star}}

\newcommand{\pihat}{\wh{\pi}}

\newcommand{\Mbar}{\wb{M}}

\newcommand{\Rm}[1][M]{R\sups{#1}}

\newcommand{\Pm}[1][M]{P\sups{#1}}
\newcommand{\Pmstar}[1][\Mstar]{P\sups{#1}}
\newcommand{\Rmstar}[1][\Mstar]{R\sups{#1}}
\newcommand{\Pmhat}[1][\Mhat]{P\sups{#1}}
\newcommand{\Rmhat}[1][\Mhat]{R\sups{#1}}

\newcommand{\Pmbar}{P\sups{\Mbar}}
\newcommand{\Rmbar}{R\sups{\Mbar}}

\newcommand{\PiNS}{\Pi_{\mathsf{ns}}} %
\newcommand{\PiRNS}{\Pi_{\mathsf{rns}}} %
\newcommand{\PiGen}{\Pi_{\mathrm{RNS}}} %

\newcommand{\dm}[2]{d^{\sss{#1},#2}}
\newcommand{\dbar}{\bar{d}}

  \newcommand{\AlgEst}{\mathrm{\mathbf{Alg}}_{\mathsf{Est}}}

\newcommand{\Mhat}{\wh{M}}
\newcommand{\Mstar}{M^{\star}}

\icml{\newcommand{\algcommentlight}[1]{\textcolor{blue!70!black}{\transparent{0.5}\scriptsize{\texttt{\textbf{//\hspace{2pt}#1}}}}}}
\arxiv{\newcommand{\algcommentlight}[1]{\textcolor{blue!70!black}{\transparent{0.5}\footnotesize{\texttt{\textbf{//\hspace{2pt}#1}}}}}}

\newcommand{\trn}{\top}

\newcommand{\approxleq}{\lesssim}

\renewcommand{\ind}[1]{^{{\scriptscriptstyle#1}}}

\newcommand{\bigoh}{O}
\newcommand{\bigoht}{\wt{O}}
\newcommand{\bigom}{\Omega}

\newcommand{\indic}{\mathbb{I}}

\newcommand{\poly}{\mathrm{poly}}

\newcommand{\Dkl}[2]{D_{\mathsf{KL}}\prn*{#1\,\|\,#2}}

\newcommand{\Dhel}[2]{D_{\mathsf{H}}\prn*{#1,#2}}

\newcommand{\Dhels}[2]{D^{2}_{\mathsf{H}}\prn*{#1,#2}}

\newcommand{\DhelX}[3]{D_{\mathsf{H}}\prn[#1]{#2,#3}}
\newcommand{\DhelsX}[3]{D^{2}_{\mathsf{H}}\prn[#1]{#2,#3}}

\newcommand{\Vf}{V}
\newcommand{\Qf}{Q}

\newcommand{\Qhat}{\wh{\Qf}}
\newcommand{\Qbar}{\widebar{Q}}
\newcommand{\Vhat}{\wh{\Vf}}

\newcommand{\unif}{\mathrm{Unif}}

\newcommand{\piunif}{\pi_{\mathrm{unif}}}

\newcommand{\supp}{\mathrm{supp}}

\newcommand{\mathand}{\quad\text{and}\quad}

\def\multiset#1#2{\ensuremath{\left(\kern-.3em\left(\genfrac{}{}{0pt}{}{#1}{#2}\right)\kern-.3em\right)}}

\renewcommand{\emptyset}{\varnothing}

\newcommand{\what}{\wh{w}}

\input{widebar}

\usepackage[suppress]{color-edits}
\addauthor{df}{ForestGreen}
\addauthor{ak}{BurntOrange}
\addauthor{pa}{Cyan}
\addauthor{cut}{Purple}
\newcommand{\dfc}[1]{\dfcomment{#1}}
\newcommand{\akc}[1]{\akcomment{#1}}
\newcommand{\cut}[1]{\cutcomment{#1}}
 
\arxiv{
\usepackage[final]{showlabels}

}

\makeatletter
\let\OldStatex\Statex
\renewcommand{\Statex}[1][3]{%
  \setlength\@tempdima{\algorithmicindent}%
  \OldStatex\hskip\dimexpr#1\@tempdima\relax}
\makeatother

 \usepackage{accents}

\addtocontents{toc}{\protect\setcounter{tocdepth}{0}}

\let\oldparagraph\paragraph
\renewcommand{\paragraph}[1]{\oldparagraph{#1.}}

\newcommand{\fakepar}[1]{\arxiv{\paragraph{#1}}\icml{\noindent\textbf{#1.}}}

\usepackage{yfonts}

\arxiv{
    \title{Scalable Online Exploration via Coverability}
   \author{Philip Amortila\thanks{Authors listed in alphabetical order.} 
\\
\normalsize
\href{mailto:philipa4@illinois.edu}{\texttt{philipa4@illinois.edu}}
\and
Dylan J. Foster$^*$
\\
\normalsize
\href{mailto:dylanfoster@microsoft.com}{\texttt{dylanfoster@microsoft.com}}
\and
Akshay Krishnamurthy$^*$
\\
\normalsize
\href{mailto:akshaykr@microsoft.com}{\texttt{akshaykr@microsoft.com}}
}
\date{}
}

\icml{
\icmltitlerunning{Scalable Online Exploration via Coverability}
}

\begin{document}

\arxiv{\maketitle

}
\icml{
\twocolumn[
\icmltitle{Scalable Online Exploration via Coverability}

\icmlsetsymbol{equal}{*}

\begin{icmlauthorlist}
\icmlauthor{Philip Amortila}{equal,yyy}
\icmlauthor{Dylan J. Foster}{equal,comp}
\icmlauthor{Akshay Krishnamurthy}{equal,comp}
\end{icmlauthorlist}

\icmlaffiliation{yyy}{University of Illinois, Urbana-Champaign.}
\icmlaffiliation{comp}{Microsoft Research. Full version appears at {\color{red} \href{https://arxiv.org/abs/2403.06571}{[arXiv:2403.06571]}}}

\vskip 0.3in
]
\printAffiliationsAndNotice{}  %
}

\begin{abstract}
Exploration is a major challenge in reinforcement learning,
especially for high-dimensional domains
that require
function approximation. %
We propose \emph{exploration objectives}---policy optimization
  objectives that enable downstream maximization of any reward function---as a conceptual framework to systematize the study of exploration.
\arxiv{Within this framework, we}\icml{We} introduce a new
  objective, \emph{\mainobj}, which generalizes previous exploration
  schemes and supports three fundamental desiderata:\loose

\icml{\begin{enumerate}[leftmargin=*]}
  \arxiv{\begin{enumerate}}
\item \emph{Intrinsic complexity control.} \mainobj is associated with
  a structural parameter, \emph{\mainopt},
  which reflects the
      intrinsic statistical difficulty of the underlying MDP, subsuming Block and Low-Rank MDPs.

\item \akedit{\emph{Efficient planning.} For a known MDP, optimizing \mainobj efficiently reduces to standard policy optimization, 
  allowing flexible integration with off-the-shelf methods such as policy gradient and Q-learning approaches.} 

\item \emph{Efficient exploration.} \mainobj enables the first computationally efficient model-based and model-free algorithms for online (reward-free or reward-driven) reinforcement learning in MDPs with low coverability. \loose %

\end{enumerate}
\akedit{Empirically}, we find that \mainobj effectively drives off-the-shelf policy
optimization algorithms to explore the
state space.\loose

\end{abstract}

\section{Introduction}
\label{sec:intro}

Many applications of reinforcement learning and control
demand agents to maneuver through complex environments with
high-dimensional state spaces that necessitate function approximation \akedit{and sophisticated exploration}.
Toward addressing the high sample complexity of existing empirical paradigms \citep{mnih2015human,
  silver2016mastering,kober2013reinforcement,lillicrap2015continuous,li2016deep}, a recent body of theoretical
research provides structural conditions that facilitate
sample-efficient exploration, as well as an understanding of fundamental limits
\citep{russo2013eluder,jiang2017contextual,wang2020provably,du2021bilinear,jin2021bellman,foster2021statistical,foster2023tight}. Yet,
computational efficiency remains a barrier: outside of simple
settings \citep{azar2017minimax,jin2020provably}, our understanding of
algorithmic primitives for efficient exploration is limited.

In this paper, we propose \emph{exploration objectives} as a
conceptual framework to develop efficient algorithms for
exploration. Informally, an exploration objective is an optimization
objective that incentivizes a policy (or policy ensemble) to explore the
state space and gather data that can be used for downstream tasks
(e.g., policy optimization or evaluation).
To enable practical and efficient exploration, such an objective should satisfy
    three desiderata:
    \begin{enumerate}[leftmargin=*]
    \item \emph{Intrinsic complexity control.} 
      \akedit{Any policy (ensemble) that optimizes the objective
        should cover the state space to the best extent possible,}
      \akdelete{Any policy that
      optimizes the objective should guide the learner toward unseen regions of the state
      space as quickly as possible,} enabling sample complexity
      guarantees for \akedit{downstream} learning that reflect the
      intrinsic statistical difficulty of the underlying MDP. 
      \label{item:des1}
    \item \emph{Efficient planning.} When the MDP of interest is
      \emph{known}, 
      \akedit{it should be possible to optimize the objective in a computationally efficient manner,}
      \akdelete{the objective should support efficient computation and representation of
      exploratory policies,}ideally in a way that flexibly integrates
      with existing pipelines.
      \label{item:des2}
    \item \emph{Efficient exploration.} When the MDP is
      \emph{unknown}, 
      \akedit{it should be possible to optimize the objective in a computationally \emph{and} statistically efficient manner;}
      \akdelete{the objective should support computation- and sample-efficient
      discovery of exploratory policies;}
      the first two desiderata are necessary, but not
      sufficient here.\loose
      
      \label{item:des3}

    \end{enumerate}
Our development of exploration objectives, particularly our emphasis
on integrating with existing pipelines, is motivated by the large body
of empirical research equipping policy gradient methods and
value-based methods with
exploration bonuses
\citep{bellemare2016unifying,tang2017exploration,pathak2017curiosity,martin2017count,burda2018exploration,ash2021anti}.
Although a number of prior theoretical works either implicitly or
explicitly develop exploration objectives, they are either too general
to admit efficient planning and
exploration~\citep{jiang2017contextual,dann2018oracle,jin2021bellman,foster2021statistical,chen2022statistical,xie2023role,liu2023maximize},
or too narrow to apply to practical problems of interest without
losing fidelity to the theoretical
foundations~\citep[e.g.,][]{azar2017minimax,jin2018q,dani2008stochastic,li2010contextual,wagenmaker2022instance}.
The difficulty is that designing exploration objectives is intimately
tied to understanding what makes an MDP easy or hard to
explore, a deep statistical problem.  A useful objective must
succinctly distill such understanding into a usable, operational form,
but finding the right sweet spot between generality and tractability
is challenging.  We believe our approach and our results strike this
balance.\loose

\paragraph{Contributions}
We introduce a new exploration objective, \mainobj, which flexibly
supports computationally and statistically efficient exploration,
satisfying desiderata \savehyperref{item:des1}{(1)}, \savehyperref{item:des2}{(2)}, and \savehyperref{item:des3}{(3)}.
  \mainobj is associated with an intrinsic structural parameter, \akedit{$L_1$-}\emph{Coverability}
  \citep{xie2023role,amortila2023harnessing}, 
  \akedit{which controls the sample complexity of reinforcement learning in nonlinear function approximation settings, subsuming Block and Low-Rank MDPs.
  We prove that data gathered with a policy ensemble optimizing \mainobj supports downstream policy \arxiv{evaluation and }optimization with general function approximation (\cref{app:downstream}).\loose}
  \akdelete{which enables sample-efficient
  reinforcement learning in a general setting that subsumes Block and
  Low-Rank MDPs, yet supports nonlinear function
approximation; we prove that data gathered using these exploratory policies
  supports downstream policy evaluation and optimization with general
  function approximation (\cref{app:downstream}).}

  For planning in a known MDP, \mainobj \akedit{can be optimized efficiently}\akdelete{enables efficient computation of exploratory policies (``policy covers'') that}
  through reduction to (reward-driven)
  policy optimization, allowing for integration with off-the-shelf methods such as policy gradient (e.g., \texttt{PPO}) or Q-learning
  (e.g., \texttt{DQN}). 
  For online reinforcement learning, \mainobj yields
  the first computationally \akedit{and statistically} efficient model-based and model-free algorithms for
  (reward-free/driven) exploration in \akedit{$L_1$-coverable MDPs}. \icml{\cutedit{Technically, these results can be viewed as a successful
algorithmic application of the \emph{Decision-Estimation Coefficient (DEC)}
framework of
\citet{foster2021statistical,foster2023tight}.}}
\arxiv{Technically, these results can be viewed as a successful
algorithmic application of the \emph{Decision-Estimation Coefficient (DEC)}
framework of
\citet{foster2021statistical,foster2023tight}, highlighting
\emph{coverability} as a general setting in which the DEC framework
leads to provably efficient end-to-end algorithms.} \loose

We complement these theoretical results with an empirical validation,
where we find that the \mainobj objective effectively integrates with
off-the-shelf policy optimization algorithms, augmenting them with
the ability to explore the state space widely. \loose

\paragraph{Paper organization}
\icml{In what follows, we formalize exploration objectives
(\cref{sec:setting}), then introduce the \mainobj objective
(\cref{sec:overview}) and give algorithms for efficient planning
(\cref{sec:planning}) and model-based exploration
(\cref{sec:model_based}), along with experiments (\cref{sec:experiments}). \cut{Model-free extensions (\cref{sec:model_free}) and structural results (\cref{sec:structural}) are deferred to the Appendix.}}
\arxiv{In what follows, we formalize exploration objectives
(\cref{sec:setting}), then introduce the \mainobj objective
(\cref{sec:overview}) and give algorithms for efficient planning
(\cref{sec:planning}) and model-based exploration
(\cref{sec:model_based}). We then provide guarantees for model-free
exploration (\cref{sec:model_free}), structural results (\cref{sec:structural}), and experiments (\cref{sec:experiments}).}

\section{Online
  Reinforcement Learning and Exploration Objectives}

\label{sec:setting}

This paper focuses on \emph{reward-free reinforcement learning}
\citep{jin2020reward,wang2020reward,chen2022statistical}, in which the
aim is to compute an exploratory ensemble of policies that enables optimization of any downstream reward function; we consider planning (computing exploratory policies in a known MDP) and
online exploration (discovering exploratory policies in an unknown
MDP).

We work in an episodic finite-horizon setting. With $H$ denoting the horizon, a (reward-free) Markov decision process
  $M=\crl*{\cX, \cA, \crl{\Pm_h}_{h=0}^{H}}$, consists of a
  state space $\cX$, an action space $\cA$, a transition distribution
  $\Pm_h:\cX\times\cA\to\Delta(\cX)$, with the convention that
  $\Pm_0(\cdot\mid{}\emptyset)$ is the initial state distribution.\cut{\footnote{To simplify presentation, we assume that $\cX$ and $\cA$ are
countable; our results extend to handle continuous variables with an appropriate measure-theoretic
treatment.}}
  An episode in the MDP $M$ proceeds according to the following protocol.
  At the beginning of the episode, the learner selects a
  randomized, non-stationary \emph{policy}
  $\pi=(\pi_1,\ldots,\pi_H)$, where $\pi_h:\cX\to\Delta(\cA)$; we let
  $\PiRNS$ denote the set of all such policies, and $\PiNS$ denote the
  subset of deterministic policies.
  The episode evolves through the following process, beginning from
  $x_1\sim{}\Pm_0(\cdot\mid{}\emptyset)$. For $h=1,\ldots,H$:
  $a_h\sim\pi_h(x_h)$ and
  $x_{h+1}\sim{}P_h\sups{M}(\cdot\mid{}x_h,a_h)$. We let
  $\bbP^{\sM,\pi}$ denote the law under this process, and let
  $\En^{\sM,\pi}$ denote the corresponding expectation.

For \emph{planning}, where the underlying
MDP is known, we denote it by $M$. For \emph{online exploration},
where the MDP is unknown, we denote it by $\Mstar$; in this framework,
the learner must explore by interacting with $\Mstar$ in a sequence of
\emph{episodes} in which they execute a policy $\pi$ and observe the trajectory
$(x_1,a_1),\ldots,(x_H,a_H)$ that results.%

\paragraph{Additional notation} To simplify presentation, we assume that $\cX$ and $\cA$ are
countable; our results extend to handle continuous variables with an appropriate measure-theoretic
treatment. For an MDP $M$ and policy $\pi$, we define the induced
\emph{occupancy measure} for layer $h$ via
\icml{
$
d_h^{\sM,\pi}(x,a) = \bbP^{\sM,\pi}\brk{x_h=x,a_h=a}$
}
\arxiv{
\[
d_h^{\sM,\pi}(x,a) = \bbP^{\sM,\pi}\brk{x_h=x,a_h=a}
\]
}
and $d_h^{\sM,\pi}(x) = \bbP^{\sM,\pi}\brk{x_h=x}$. We use $\piunif$
to denote the uniform policy. We
define $\pi\circ_h\pi'$ as the policy that follows $\pi$ for layers
$h'<h$ and follows $\pi'$ for $h'\geq{}h$.\icml{
For a set $\cZ$, we let
        $\Delta(\cZ)$ denote the set of all probability distributions
        over $\cZ$. We write $f=\bigoht(g)$ to denote that $f =
        \bigoh(g\cdot{}\max\crl*{1,\mathrm{polylog}(g)})$.
        }
\arxiv{
For an integer $n\in\bbN$, we let $[n]$ denote the set
  $\{1,\dots,n\}$. For a set $\cZ$, we let
        $\Delta(\cZ)$ denote the set of all probability distributions
        over $\cZ$\padelete{and let $\cZ^{\c}$ denote the complement}. We adopt standard
        big-oh notation, and write $f=\bigoht(g)$ to denote that $f =
        \bigoh(g\cdot{}\max\crl*{1,\mathrm{polylog}(g)})$. \padelete{We use
        $\approxleq$ only in informal statements to emphasize the most
        notable elements of an inequality.}
        }

\subsection{Exploration Objectives}
\newcommand{\Obj}[1][M]{\Phi^{\sss{#1}}}%
\newcommand{\ObjOpt}[1][M]{\mathsf{Opt}^{\sss{#1}}}%
We introduce \emph{exploration objectives} as a conceptual framework to
study exploration in reinforcement learning. Exploration objectives
are policy optimization objectives; \akedit{they are defined over ensembles of policies (\emph{policy covers}), represented as distributions $p \in \Delta(\Pi)$ for a class $\Pi \subset \PiRNS$ of interest.} 
\akdelete{; we represent exploratory ensembles of policies (\emph{policy
  covers}) as mixture policies $p\in\Delta(\Pi)$ for a class
$\Pi\subseteq\PiRNS$ of interest. }
The defining
property of an exploration objective is to incentivize policy
ensembles $p\in\Delta(\Pi)$ to explore the state space and gather data that can be used for downstream policy
optimization or evaluation (i.e., offline RL).
\begin{definition}[Exploration objective]
  \label{def:exploration_objective}
For a reward-free MDP $M$ and policy class $\Pi$, a function $\Obj:\Delta(\Pi)\to\bbR_{+}$
is an \emph{exploration objective} if for any policy ensemble
$p\in\Delta(\Pi)$, one can optimize any downstream reward function
$R$ to precision $\vepspac$ in $M$ \dfedit{(i.e., produce $\pihat$ such that
  $\max_{\pi}J_R^{\sM}(\pi)-J_R^{\sM}(\pihat)\leq\eps$)\footnote{$J^{\sM}_R(\pi)$ denotes the expected reward of $\pi$ in $M$ under $R$.}} using
$\poly(\Obj(p),\vepspac^{-1})$
trajectories drawn from $\pi\sim{}p$ (under standard function
approximation assumptions).\loose
\end{definition}
Note that we allow the exploration objective to depend on the
underlying MDP $M$; if $M$ is unknown, then evaluating $\Obj(p)$
may be impossible without first exploring. 
\akedit{We also consider \emph{per-layer} objectives denoted as $\Obj_h$, optimized by a collection $\{p_h\}_{h=1}^H$ of policy ensembles.} 
We deliberately leave the details for reward
optimization vague in \cref{def:exploration_objective}, as this will
depend on the policy optimization and function approximation
techniques under consideration; see \cref{app:downstream} for
examples.

With \cref{def:exploration_objective} in mind, the desiderata in \cref{sec:intro} can
be restated formally as follows.
\begin{enumerate}[leftmargin=*]
\item \emph{Intrinsic complexity control.}
  The optimal objective value
  $\ObjOpt\ldef{}\inf_{p\in\Delta(\Pi)}\Obj(p)$ is bounded
  for a large class of MDPs $M$ of interest, ideally in a way that
  depends on intrinsic structural properties of the MDP. %
\item \emph{Efficient planning.} When the MDP $M$ is known, we can
  solve $\argmin_{p\in\Delta(\Pi)}\Obj(p)$ approximately (up to multiplicative or
  additive approximation factors) in a computationally
  efficient fashion, ideally in a way that reduces to standard
  computational primitives such as reward-driven policy optimization.
\item \emph{Efficient exploration.} When the MDP $M$ is unknown, we
  can approximately solve $\argmin_{p\in\Delta(\Pi)}\Obj(p)$ in a
  sample-efficient fashion, with sample complexity polynomial in the
  optimal objective value $\ObjOpt$, and with computational efficiency
   comparable to planning.
\end{enumerate}

As a basic example, for tabular MDPs where $\abs{\cX}$ is small, the work of \citet{jin2020reward} can be viewed as optimizing
the per-layer objective
\begin{equation}
  \Obj_h(p) =
\max_{x\in\cX}\frac{1}{\En_{\pi\sim{}p}\brk{d^{\pi}_h(x)}}\label{eq:tabular_obj}
\end{equation}
for each
layer $h\in\brk{H}$; the optimal value for this objective satisfies
$\ObjOpt_h = \bigoh(\abs{\cX}/\eta)$, where
$\eta\ldef{}\min_{x\in\cX}\max_{\pi\in\PiRNS}d^{\pi}_h(x)$ is a
\emph{reachability parameter}.  \akedit{This objective supports
  efficient planning and exploration, satisfying desiderata \destwo
  and \desthree, but is restrictive in terms of intrinsic complexity
  (desideratum \desone) because the optimal value, which scales with
  $|\cX|$, does not sharply control the sample complexity of
  reinforcement learning in the function approximation regime. Other (implicit or explicit)
  objectives studied in prior work similarly do not satisfy desideratum
  \desone~\citep{hazan2019provably,jin2020provably,agarwal2020flambe,modi2021model,uehara2022representation,zhang2022efficient,mhammedi2023representation},
  or are too general to admit computationally efficient planning and
  exploration, failing to satisfy desiderata \destwo and
  \desthree~\citep{jiang2017contextual,dann2018oracle,jin2021bellman,foster2021statistical,chen2022statistical,xie2023role,liu2023maximize}. See~\pref{app:additional} for further discussion.}

\begin{remark}
  Many existing algorithms---for example, those that use count-based
  bonuses \citep{azar2017minimax,jin2018q} or elliptic bonuses
  \citep{auer2002using,dani2008stochastic,li2010contextual,jin2020provably}---
  implicitly allude to specific exploration objectives, but do not
  appear to explicitly optimize them. We hope that by separating
  algorithms from objectives, our definition can bring clarity and
  better guide algorithm design going forward.
\end{remark}

\section{The \mainobj Objective}
\label{sec:overview}
This section introduces our main exploration objective,
\mainobj. Throughout the section, we work with a fixed (known) MDP
$M$, and an arbitrary set of policies $\Pi\subseteq\PiRNS$.

\paragraph{\mainobj objective}
\label{sec:teaser}

For a policy ensemble $p\in\Delta(\PiRNS)$ and  parameter
$\veps\in\brk{0,1}$ we define the \emph{\mainobj} objective by
\begin{align}
  \label{eq:mainobj}
    \CovM(p) =
  \sup_{\pi\in\Pi}\En^{\sM,\pi}\brk*{\frac{d^{\sM,\pi}_h(x_h,a_h)}{d_h^{\sM,p}(x_h,a_h)+\veps\cdot{}d_h^{\sM,\pi}(x_h,a_h)}},
  \end{align}
where we slightly overload notation
for occupancy measures and write $d_h^{\sM,p}(x,a) \ldef
\En_{\pi\sim{}p}\brk[\big]{d_h^{\sM,\pi}(x,a)}$.\footnote{Likewise, %
$\En^{\sM,p}$ and $\bbP^{\sM,p}$ denote the expectation and law for
the process where we sample policy $\pi\sim{}p$ and execute it in
$M$.} This objective\akdelete{, which is closely related to several standard notions of
coverage in \emph{offline} reinforcement learning \citep{farahmand2010error,xie2020q,zanette2021provable}}
encourages the ensemble $p\in\Delta(\Pi)$ to cover the state space at
least as well as any individual policy $\pi\in\Pi$, but in an
\emph{average-case} sense (with respect to the state distribution
induced by $\pi$ itself) that discounts hard-to-reach
states. Importantly, \mainobj only considers the \emph{relative probability} of
visiting states (that is, the ratio of occupancies), which is
\akedit{fundamentally different from ``tabular'' objectives such as
  \cref{eq:tabular_obj} from prior work~\citep{jin2020reward,hazan2019provably} and is}
essential to drive exploration in large state spaces. The approximation parameter $\veps>0$
allows one to
discard regions of the state space that have low relative probability for all
policies, \akedit{removing the reachability assumption required by \cref{eq:tabular_obj}}.

\akedit{\mainobj is closely related to previous \emph{optimal
  design}-based objectives in online RL, and to several standard notions
  of coverage in offline RL. Indeed, \mainobj can be viewed as a
  generalization of previously proposed optimal design-based
  objectives~\citep{wagenmaker2022beyond,wagenmaker2022instance,li2023reward,mhammedi2023efficient};
  see \cref{sec:structural,app:additional} for details.
  Regarding the latter, when
  $\veps=0$, $\CovM(p)$ coincides with
$L_1$-concentrability~\citep{farahmand2010error,xie2020q,zanette2021provable}, and is equivalent to the $\chi^2$-divergence
  between $d^{\pi}$ and $d^{p}$ up to a constant shift.}

Before turning to algorithmic development, we first show that \mainobj
is indeed a valid exploration objective in the sense of
\cref{def:exploration_objective}, then show that it satisfies
desideratum \desone, providing meaningful control over the intrinsic
complexity of exploration.

\akdelete{
\begin{remark}
  \mainobj can be viewed as a generalization of previous \emph{optimal
    design}-based objectives
  \citep{wagenmaker2022beyond,wagenmaker2022instance,li2023reward,mhammedi2023efficient};
  see \cref{app:additional,sec:structural} for details. When
  $\veps=0$, the value $\CovM(p)$ coincides with
$L_1$-concentrability, a widely used notion of coverage in offline
reinforcement learning
\citep{farahmand2010error,xie2020q,zanette2021provable}, and is equivalent to the $\chi^2$-divergence
  between $d^{\pi}$ and $d^{p}$ up to a constant shift.
\end{remark}
}

\paragraph{\mainobj enables downstream policy optimization}
\akedit{We prove that \mainobj enables downstream policy optimization}
through a
change-of-measure lemma, which shows that it is possible to transfer
the expected value for any function $g$ of interest (e.g., Bellman
error) under any policy $\pi$ to the expected value under $p$.
\begin{restatable}[Change of measure for \mainobj]{proposition}{com}
  \label{lem:com}
For any distribution $p\in\Delta(\PiRNS)$, we have that for all functions
$g:\cX\times\cA\to{}\brk{0,B}$, all $\pi\in\Pi$, and all $\veps>0$,%
\footnote{This result is
  meaningful in the parameter regime where
  $\CovM(p)<\nicefrac{1}{\veps}$. We refer to this regime as
  \emph{non-trivial}, as $\CovM(p)\leq{}\nicefrac{1}{\veps}$ holds
  vacuously for all $p$.}
  \icml{
\begin{align}
  \label{eq:com2}
  &\En^{\sM,\pi}\brk*{g(x_h,a_h)}\\
  &\leq{} 2\sqrt{\CovM(p)\cdot{}\En^{\sM,p}\brk*{g^2(x_h,a_h)}
  }
  + \CovM(p)\cdot{}(\veps{}B).\notag
\end{align}

}
\arxiv{
\begin{align}
  \label{eq:com2}
  \En^{\sM,\pi}\brk*{g(x_h,a_h)}\leq{} 2\sqrt{\CovM(p)\cdot{}\En^{\sM,p}\brk*{g^2(x_h,a_h)}
  }
  + \CovM(p)\cdot{}(\veps{}B).
\end{align}
}

\end{restatable}
\vspace{-.25em}
Using this result, one can prove that\akedit{, using data gathered from $p$,} standard offline reinforcement
learning algorithms such as Fitted Q-Iteration (\fqi) succeed 
\akedit{with sample complexity scaling with $\max_h \CovM(p)$.}
\akedit{One can similarly analyze} hybrid offline/online methods (online methods that require access to
exploratory data) such as \psdp \citep{bagnell2003policy} and \npg \citep{agarwal2021theory}; see \cref{app:downstream}
for details. Such methods will prove to be critical for
the reward-free reinforcement learning guarantees we present in the sequel.\loose

In light of \cref{lem:com}, which shows that \mainobj satisfies
\cref{def:exploration_objective}, we refer to any $p\in\Delta(\PiRNS)$ that
(approximately) optimizes the \mainobj objective as a \emph{policy
  cover} going forward.\footnote{This definition generalizes most notions
  of policy cover found in prior work
  \citep{du2019latent,misra2019kinematic,mhammedi2023representation,mhammedi2023efficient,huang2023reinforcement}.}

\paragraph{\mainopt provides intrinsic complexity control}
\label{sec:properties}
Of course, the guarantee in \cref{lem:com} is only useful if
desideratum \desone is satisfied, i.e. there
exist distributions $p\in\Delta(\Pi)$ for which the \mainobj
objective is bounded. To this
      end, we define the optimal value for the \mainobj objective,
      which we refer to as \emph{\mainopt}, as\arxiv{\footnote{When the MDP $M$ is clear from context, we drop dependence on $M$
  and write $\Cov(p)\equiv\CovM(p)$, $\CovOpt\equiv\CovOptM$, etc.}}
  \begin{align}
    \label{eq:mainobj_opt}
        \CovOptM = \inf_{p\in\Delta(\Pi)}\CovM(p),
  \end{align}
  and define $\CovOptMmax=\max_{h\in\brk{H}}\CovOptM$.

  We show that the \mainopt value $\CovOptM$ can be interpreted as an intrinsic
  structural parameter for the MDP $M$, and is bounded for standard
  MDP classes of interest.
  
  To do so, we draw a connection a structural parameter introduced by
  \citet{xie2023role} as a means to bridge online and offline
  RL, which we refer to as \emph{$L_{\infty}$-Coverability}:
  \begin{align}
    \label{eq:cinf}
    \CinfhM = \inf_{\mu\in\Delta(\cX\times\cA)}  \sup_{\pi\in\Pi}\sup_{(x,a)\in\cX\times\cA}\crl*{\frac{d_h^{\sM,\pi}(x,a)}{\mu(x,a)}},
  \end{align}
  with $\CinfM=\max_{h\in\brk{H}}\CinfhM$.

  $L_\infty$-Coverability measures the best
  possible \akedit{(worst-case) density ratio}\akdelete{concentrability value} that can be achieved if \akedit{one}\akdelete{an oracle}
  optimally designs the data distribution $\mu$ with knowledge of the
  underlying MDP. The main differences between the value $\CinfhM$ and
  the \mainopt value in \cref{eq:mainobj_opt} are that (i) \cref{eq:cinf} considers worst-case ($L_{\infty}$-type) rather than average-case coverage, and (ii)
\cref{eq:cinf} allows the distribution $\mu\in\Delta(\cX\times\cA)$ to
be arbitrary, while \cref{eq:mainobj_opt} requires the distribution to be
realized as a mixture of occupancies (in other words, \cref{eq:cinf}
allows for \emph{non-admissible} mixtures).\footnote{We adopt the notation
  $\mathsf{Cov}_{\cdot}^{\sM}$ for admissible
    variants of coverability and $C^{\sM}_{\cdot}$ for non-admissible
    variants of coverability.} Due to the latter
difference, it is unclear at first glance whether one can relate the
two objectives, since allowing for non-admissible mixtures could
potentially make the objective \eqref{eq:cinf} much smaller. Nonetheless, the following result, which uses a non-trivial
application of the minimax theorem\arxiv{ inspired by
\citet{dudik2011efficient,agarwal2014taming}}, shows that \mainopt
is
indeed bounded by $\CinfM$.
\begin{restatable}{proposition}{linf}
  \label{prop:linf}
  For all $\veps>0$, we have $      \CovOptM \leq{} \CinfhM$.
\end{restatable}
Examples
for which $\CinfM$---and consequently $\CovOptMmax$---is bounded by small problem-dependent
constants include Block MDPs \citep{xie2023role}
($\paedit{\CinfM}\leq{}\abs{\cS}\abs{\cA}$, where $\cS$ is the number of latent
states), linear/low-rank MDPs \citep{huang2023reinforcement}
($\paedit{\CinfM}\leq{}d\abs{\cA}$, where $d$ is the feature dimension), and analytically sparse low-rank MDPs
\citep{golowich2023exploring} ($\paedit{\CinfM}\leq{}k\abs{\cA}$, where $k$ is
the sparsity level). Importantly, these
examples (particularly Block MDPs and Low-Rank MDPs) \akedit{require}\akdelete{support} nonlinear function
approximation. See \arxiv{\cref{sec:linf_relaxation}}\icml{\cref{sec:planning_examples}} for
details; see also
\citet{xie2023role,amortila2023harnessing}. \pacomment{are there examples where $\Ccov$ is not bounded by $\CovOptMmax$ is?}

\akedit{We note that $L_1$- and $L_\infty$-Coverability are less
  general than structural parameters defined in terms of Bellman
  errors, such as Bellman rank~\citep{jiang2017contextual}, Bellman-Eluder dimension~\citep{jin2021bellman}, and Bilinear
  rank~\citep{du2021bilinear}, which are known to not admit computationally efficient
  learning algorithms~\citep{dann2018oracle}. In this sense, \mainobj strikes a balance
  between generality and tractability. We discuss this in more detail
  and discuss connections to other structural parameters in
  \cref{sec:structural}.}

\akdelete{In what follows, we turn to algorithms, and show that (i) \mainobj can
be optimized efficiently when the MDP $M$ is known
(\cref{sec:planning}), and (ii) can be optimized in a computation- and
sample-efficient fashion when $M$ is unknown
(\cref{sec:model_based}).
We refer to \cref{sec:structural} for connections to other structural
parameters, including feature coverage and
non-admissible variants of coverability.}

\section{Optimizing \mainobj: Efficient Planning}
\label{sec:planning}

Directly optimizing the \mainobj objective \eqref{eq:mainobj}
presents
challenges because the objective is quadratic in the occupancy
$d^{\sM,\pi}$. To address this issue, this section provides two
\emph{relaxations}---that is, relaxed objectives that upper bound
\mainobj---that are directly amenable to optimization (via reduction to standard reward-driven
 policy optimization), yet are still bounded for MDPs of interest. \cref{sec:linf_relaxation} presents a relaxation based on a
  connection to \emph{$L_\infty$-Coverability} (\cref{eq:cinf}), and
  \cref{sec:pushforward_relaxation} presents a relaxation based on a connection to \emph{Pushforward Coverability}.
The first relaxation is tighter, but requires stronger knowledge of the
underlying MDP. \arxiv{A more general recipe for deriving relaxations is
  presented in \cref{sec:relaxation_recipe}.

}To motivate our results, recall that given a $C$-approximate
minimizer $p\in\Delta(\Pi)$ for which
$\CovM(p)\leq{}C\cdot{}\CovOptM$, the sample complexity in \cref{lem:com}
degrades only by an $O(C)$ factor.

\subsection{The \linfcov Relaxation}
\label{sec:linf_relaxation}

  \newcommand{\muCov}{\obj_{\mu;h,\veps}}
  \newcommand{\muCovM}[1][M]{\obj_{\mu;h,\veps}^{\sss{#1}}}
  \newcommand{\muCovOpt}{\cov_{\mu;h,\veps}}
  \newcommand{\muCovOptM}[1][M]{\cov_{\mu;h,\veps}^{\sss{#1}}}
  \newcommand{\vepsapx}{\veps_{\mathrm{opt}}}

    \begin{algorithm}[tp]
    \setstretch{1.3}
     \begin{algorithmic}[1]
       \State \textbf{input}:
        Layer $h\in\brk{H}$, precision\arxiv{ parameter}
       $\veps\in\brk{0,1}$, 
       distribution $\mu\in\Delta(\cX\times\cA)$ w/
       $\Cinf\equiv\CinfhM(\mu)$, \arxiv{optimization
         tolerance}\icml{opt. tol.} $\vepsapx>0$.\loose
       \State Set $T=\frac{1}{\veps}$.
  \For{$t=1, 2, \cdots, T$} \label{line:linf_outer_loop}
  \State Compute $\pi\ind{t}\in\Pi$ such that 
      \arxiv{\begin{small}
          \begin{align}
      \En^{M,\pi\ind{t}}\brk*{\frac{\mu(x_h,a_h)}{\sum_{i<t}d_h^{\sM,\pi\ind{i}}(x_h,a_h)+\Cinf{}\mu(x_h,a_h)}}
      \geq \sup_{\pi\in\Pi}\En^{M,\pi}\brk*{\frac{\mu(x_h,a_h)}{\sum_{i<t}d_h^{\sM,\pi\ind{i}}(x_h,a_h)+\Cinf{}\mu(x_h,a_h)}}-\vepsapx.
           \end{align}\end{small}\label{line:linf_opt}
       }
       \icml{\begin{footnotesize}
          \begin{align*}
            &\En^{M,\pi\ind{t}}\brk*{\frac{\mu(x_h,a_h)}{\sum_{i<t}d_h^{\sM,\pi\ind{i}}(x_h,a_h)+\Cinf{}\mu(x_h,a_h)}}
            \geq\\
            &\sup_{\pi\in\Pi}\En^{M,\pi}\brk*{\frac{\mu(x_h,a_h)}{\sum_{i<t}d_h^{\sM,\pi\ind{i}}(x_h,a_h)+\Cinf{}\mu(x_h,a_h)}}-\vepsapx.
           \end{align*}\end{footnotesize}\label{line:linf_opt}
         }
  \EndFor
  \arxiv{\vspace{-10pt}}
    \icml{\vspace{-15pt}}
\State  Return $p=\unif(\pi\ind{1},\ldots,\pi\ind{T})$.
\end{algorithmic}
\caption{Approximate Policy Cover Computation via
  $L_{\infty}$-Coverability Relaxation}
\label{alg:linf_relaxation}
\end{algorithm}

Our first relaxation of the \mainobj objective assumes access to a distribution $\mu\in\Delta(\cX\times\cA)$ for which the
\emph{$L_{\infty}$-concentrability coefficient} $\CinfhM(\mu) \ldef
\sup_{\pi\in\Pi}\sup_{(x,a)\in\cX\times\cA}\crl*{\frac{d_h^{\sM,\pi}(x,a)}{\mu(x,a)}}$
\citep{chen2019information} is bounded. For such a distribution $\mu$, abbreviating
$\Cinf\equiv\CinfhM(\mu)$, we define
\icml{\begin{align}
\label{eq:linf_relaxation}
    \muCovM(p) \ldef
    \sup_{\pi\in\Pi}\En^{\sM,\pi}\brk*{\frac{\mu(x_h,a_h)}{d_h^{\sM,p}(x_h,a_h)+\veps\cdot{}\Cinf\mu(x_h,a_h)}}, 
      \end{align}
      and $\muCovOptM \ldef \inf_{p\in\Delta(\Pi)}\muCovM(p)$.%
       }%
\arxiv{\begin{align}
  \label{eq:linf_relaxation}
    \muCovM(p) =
    \sup_{\pi\in\Pi}\En^{\sM,\pi}\brk*{\frac{\mu(x_h,a_h)}{d_h^{\sM,p}(x_h,a_h)+\veps\cdot{}\Cinf\mu(x_h,a_h)}},
    \mathand \muCovOptM = \inf_{p\in\Delta(\Pi)}\muCovM(p).
       \end{align}
       }
       This objective upper bounds \mainobj,
and any $p\in\Delta(\PiRNS)$ that
optimizes it has $\CovM(p)\leq2\CinfhM(\mu)$.\loose
  \begin{restatable}{proposition}{linfrelaxation}
    \label{prop:linf_relaxation}
    For a distribution $\mu$ with $\Cinf\equiv\CinfhM(\mu)$, it holds
    that for all $p\in\Delta(\PiRNS)$,
    \begin{align}
      \CovM(p)\leq{}2\Cinf\cdot{}\muCovM(p).
    \end{align}
Furthermore, $\muCovOptM \leq{}1$ for all $\veps>0$.
\end{restatable}
Notably, given access to a distribution
$\mu\in\Delta(\cX\times\cA)$ that achieves the
$L_{\infty}$-Coverability value $\CinfhM$ in \cref{eq:cinf}, any
distribution $p\in\Delta(\PiRNS)$ that optimizes the relaxation in
\cref{eq:linf_relaxation} achieves \mainobj value $\CovM(p)\leq{}2\CinfhM$. However,
the relaxation supports arbitrary distributions
$\mu\in\Delta(\cX\times\cA)$, allowing one to trade off approximation
value and computation. Indeed, in some cases, it may be simpler to compute a
distribution $\mu\in\Delta(\cX\times\cA)$ that has suboptimal, yet
reasonable concentrability. Because $\mu$ is not required to be admissible, such a distribution can be
computed easily or in closed form for many MDP families of interest\akedit{.}
For example, in tabular MDPs we can simply take $\mu=\unif(\cX\times\cA)$\arxiv{; we outline more examples at the end of this section.}\icml{. See \cref{sec:planning_examples} for more examples.}

\paragraph{The algorithm}
\cref{alg:linf_relaxation} provides an iterative algorithm to compute
a distribution $p\in\Delta(\PiRNS)$ that optimizes the
$L_{\infty}$-relaxation in \cref{eq:linf_relaxation}. The algorithm
proceeds in $T$ steps. At each step $t\in\brk{T}$, given a sequence of
policies $\pi\ind{1},\ldots,\pi\ind{t-1}$ computed so far, the
algorithm computes a new policy $\pi\ind{t}$ by solving the policy
optimization problem
\icml{
  \begin{small}
    \begin{align}
      \label{eq:linf_opt}
      \pi\ind{t}=  \argmax_{\pi\in\Pi}\En^{\sM,\pi}\brk*{\frac{\mu(x_h,a_h)}{\sum_{i<t}d_h^{\sM,\pi\ind{i}}(x_h,a_h)+\Cinf{}\mu(x_h,a_h)}}
    \end{align}%
  \end{small}
 }%
\arxiv{
\begin{align}
  \label{eq:linf_opt}
\pi\ind{t}=  \argmax_{\pi\in\Pi}\En^{\sM,\pi}\brk*{\frac{\mu(x_h,a_h)}{\sum_{i<t}d_h^{\sM,\pi\ind{i}}(x_h,a_h)+\Cinf{}\mu(x_h,a_h)}}
\end{align}
}%
in \cref{line:linf_opt} (up to tolerance $\vepsapx>0$). After all $T$
rounds conclude, the algorithm returns the uniform mixture
$p=\unif(\pi\ind{1},\ldots,\pi\ind{T})$ as a policy cover.

The optimization problem \eqref{eq:linf_opt} aims to find a policy
$\pi\ind{t}$ that explores regions of the state space not already
covered by $\pi\ind{1},\ldots,\pi\ind{t-1}$. Critically, it is a standard \emph{reward-driven policy optimization} problem
with reward function\loose
\[
r\ind{t}_h(x,a) \ldef \frac{\mu(x,a)}{\sum_{i<t}d_h^{\sM,\pi\ind{i}}(x,a)+\Cinf{}\mu(x,a)}.
\]
In practice, one can solve \cref{eq:linf_opt} using standard policy
optimization algorithms \padelete{like PPO} (we take this approach in
\cref{sec:experiments}). However, since the MDP $M$ is known, we can
also take advantage of the vast
literature on algorithms for planning with a known model, as well as
algorithms like Policy Search by
Dynamic Programming (\citet{bagnell2003policy}) or Natural
Policy Gradient (\citet{agarwal2021theory}) that can use $\mu$ itself as a high-quality reset distribution.\loose

The following theorem shows that \cref{alg:linf_relaxation} converges
to a policy cover $p\in\Delta(\PiRNS)$ that optimizes
the relaxation in \cref{eq:linf_relaxation} (up to a small
$\log(\veps^{-1})$ multiplicative approximation factor) in a small number of iterations. 
\begin{restatable}{theorem}{linfoptimization}
  \label{thm:linf_optimization}
  For any $\veps\in\brk*{0,1}$ and $h\in\brk{H}$,
  given a distribution $\mu$ with $\Cinf\equiv\CinfhM(\mu)$, whenever
  $\vepsapx\leq{}\veps\log(2\veps^{-1})$, \cref{alg:linf_relaxation}
  with $T=\veps^{-1}$
  produces a distribution $p\in\Delta(\Pi)$ with
  $\abs*{\supp(p)}\leq{}\veps^{-1}$ such that
\begin{align}
  \muCovM(p) \leq{} 3\log(2\veps^{-1}),
\end{align}
and consequently $\CovM(p)\leq{}6\Cinf\log(2\veps^{-1})$.
\end{restatable}
\icml{See \cref{sec:planning_examples} for examples. \cutedit{Note that the algorithmic template in
    \cref{alg:linf_relaxation} can be applied in unknown MDPs as long as \cref{line:linf_opt} can be implemented sample-efficiently; this observation is the basis for the model-free online exploration
     algorithm given in \cref{sec:model_free}.}}
\arxiv{\cref{thm:linf_optimization} is proven using a
per-state/action elliptic potential argument inspired by the coverability-based regret bounds in \cite{xie2023role}. Note that the iteration complexity $T$ and the support size of the
resulting policy cover scale inversely with the approximation
parameter $\veps$ in the objective $\muCovM(p)$\icml{, leading to a
  computational-statistical tradeoff.}\arxiv{, leading to a
  computational-statistical tradeoff. \padelete{. This leads to a
computational-statistical tradeoff, where a tighter approximation
factor (which leads to tighter sample complexity guarantees) requires
more runtime and memory compared to a looser
approximation.} \pacomment{kept the more condensed icml version here}}}

\arxiv{
\begin{remark}
    Note that while the results in this section are presented
    for the setting in which the underlying MDP $M$ is ``known'' to
    the learner (planning), the algorithmic template in
    \cref{alg:linf_relaxation} can be applied even when $M$ is
    unknown, as long as the policy optimization step in
    \cref{line:linf_opt} can be implemented in a sample-efficient
    fashion. This observation is the basis for the model-free
    exploration algorithm given in \cref{sec:model_free}.
  \end{remark}
}

  \arxiv{
\paragraph{Examples}
As discussed above, the relaxation \eqref{eq:linf_relaxation} used by
\cref{alg:linf_relaxation} can be optimized efficiently whenever a
(non-admissible) state-action distribution
$\mu\in\Delta(\cX\times\cA)$ with low $L_\infty$-concentrability
$\CinfhM$ can be computed efficiently for the MDP $M$. Examples of MDP
classes that admit efficiently computable distributions with low
concentrability include:
\arxiv{\begin{itemize}}
  \icml{\begin{itemize}[leftmargin=*]}
\item When $M$ is a tabular MDP, the distribution $\mu(x,a) =
  \frac{1}{\abs*{\cX}\abs*{\cA}}$ (which clearly admits a closed form
  representation) achieves $\CinfhM(\mu)\leq\abs*{\cX}\abs*{\cA}$.
\item For a Block MDPs
  \citep{du2019latent,misra2019kinematic,zhang2022efficient,mhammedi2023representation}
  with latent state space $\cS$, emission distribution
  $q:\cS\to\Delta(\cX)$, and decoder $\phi^{\star}:\cX\to\cS$, the
  distribution
  $\mu(x,a)\ldef{}q(x\mid{}\phistar(x))\cdot{}\frac{1}{\abs*{\cS}\abs*{\cA}}$
  achieves $\CinfhM(\mu)\leq\abs*{\cS}\abs*{\cA}$
  \citep{xie2023role}. Again, this distribution admits a closed form
  representation when $M$ is explicitly specified.
\item For low-rank MDPs with the structure in \cref{eq:low_rank}, the
  distribution given by $\mu(x,a) =
  \frac{\nrm*{\psi_h(x)}_2}{\int\nrm*{\psi_h(x')}_2dx'}_\cdot\frac{1}{\abs{\cA}}$
  achieves $\CinfhM(\mu)\leq{}B\abs*{\cA}$ under the standard
  normalization assumption that
  $\int\nrm*{\psi_h(x')}_2dx'\leq{}B$ and $\nrm*{\phi(x,a)}_2\leq{}1$ for some (typically
  dimension-dependent) constant $B>0$
  \citep{golowich2023exploring}.
\end{itemize}
These examples highlight that for many settings of interest, computing a
covering distribution $\mu\in\Delta(\cX\times\cA)$ when the model is
known is significantly simpler than computing an explicit policy cover
$p\in\Delta(\PiRNS)$, showcasing the utility of
\cref{alg:linf_relaxation}.
}

\subsection{The \pushforwardcov Relaxation}
\label{sec:pushforward_relaxation}

\arxiv{
\arxiv{\begin{algorithm}[tp]}
  \icml{\begin{algorithm}[htp]}
    \setstretch{1.3}
     \begin{algorithmic}[1]
       \State \textbf{input}:
       Layer $h\in\brk{H}$, precision parameter
,       $\veps\in\brk{0,1}$, optimization tolerance $\vepsapx>0$.
       \State Set $T=\frac{1}{\veps}$.
  \For{$t=1, 2, \cdots, T$}
  \State Compute $\pi\ind{t}\in\Pi$ such that %
\begin{small}
    \begin{align*}
      \En^{M,\pi\ind{t}}\brk*{\frac{P_{h-1}^{\sM}(x_h\mid{}x_{h-1},a_{h-1})}{\sum_{i<t}d_h^{\sM,\pi\ind{i}}(x_h)+P_{h-1}^{\sM}(x_h\mid{}x_{h-1},a_{h-1})}}
      \geq \sup_{\pi\in\Pi}\En^{M,\pi}\brk*{\frac{P_{h-1}^{\sM}(x_h\mid{}x_{h-1},a_{h-1})}{\sum_{i<t}d_h^{\sM,\pi\ind{i}}(x_h)+P_{h-1}^{\sM}(x_h\mid{}x_{h-1},a_{h-1})}}-\vepsapx.
    \end{align*}
  \end{small}
  \label{line:pushforward_opt}
  \EndFor
  \vspace{-10pt}
\State  Return $p=\unif(\pi\ind{1},\ldots,\pi\ind{T})$.
\end{algorithmic}
\caption{Approximate Policy Cover Computation via
  Pushforward Coverability Relaxation}
\label{alg:pushforward_relaxation}
\end{algorithm}

}

The main drawback behind the $L_{\infty}$-Coverability relaxation in
the prequel is the assumption of access to a covering distribution
$\mu\in\Delta(\cX\times\cA)$.\footnote{Note that since we consider
  planning, this is a computational assumption, not a statistical assumption.} The next relaxation, which is inspired
by the notion of \emph{pushforward concentrability} in offline reinforcement learning
\citep{xie2021batch,foster2022offline}, removes this assumption at the
cost of giving a looser upper bound. This objective takes the form
\icml{
  {\begin{small}
    \begin{align}
      \label{eq:pushforward_relaxation}
      \pCovM(p) =
      \sup_{\pi\in\Pi}\En^{\sM,\pi}\brk*{\frac{P_{h-1}^{\sM}(x_h\mid{}x_{h-1},a_{h-1})}{d_h^{\sM,p}(x_h)+\veps\cdot{}P^{\sM}_{h-1}(x_h\mid{}x_{h-1}a_{h-1})}},
    \end{align}
  \end{small}}with $\pCovOptM \ldef \inf_{p\in\Delta(\Pi)}\pCovM(p)$. }%
\arxiv{
\begin{align}
  \label{eq:pushforward_relaxation}
    \pCovM(p) =
\sup_{\pi\in\Pi}\En^{\sM,\pi}\brk*{\frac{P_{h-1}^{\sM}(x_h\mid{}x_{h-1},a_{h-1})}{d_h^{\sM,p}(x_h)+\veps\cdot{}P^{\sM}_{h-1}(x_h\mid{}x_{h-1}a_{h-1})}},
    \mathand \pCovOptM = \inf_{p\in\Delta(\Pi)}\pCovM(p).
\end{align}
}%
    This objective replaces the covering distribution $\mu$ with the
    transition distribution $P^{\sM}_{h-1}$ for $M$, making it more
    practical to compute.
    Unlike the $L_{\infty}$-Coverability relaxation \eqref{eq:linf_relaxation}, the optimal value $\pCovOptM$ may not be bounded by
    \arxiv{the $L_{\infty}$-Coverability parameter }\icml{the
      parameter }$\CinfhM$. However, we show
    that the value can be controlled by a related \emph{pushforward
      coverability} parameter given by
      \icml{$\CpushhM = \inf_{\mu\in\Delta(\cX)}
        \sup_{(x,a,x')\in\cX\times\cA\times\cX}\crl*{\frac{P_{h-1}^{\sM}(x'\mid{}x,a)}{\mu(x')}}$,
      }
  \arxiv{\begin{align}
  \label{eq:cpush}
\CpushhM = \inf_{\mu\in\Delta(\cX)}  \sup_{(x,a,x')\in\cX\times\cA\times\cX}\crl*{\frac{P_{h-1}^{\sM}(x'\mid{}x,a)}{\mu(x')}},
\end{align}}
with $\CpushM=\max_{h\in\brk{H}}\CpushhM$.\footnote{This\arxiv{ definition} is
  inspired by\arxiv{ the notion of} \emph{pushforward
    concentrability}\icml{ \citep{xie2021batch}; $\CinfhM\leq\abs*{\cA}\cdot\CpushhM$ but the converse is not true.}\arxiv{,
  defined for a distribution $\mu$ by $\CpushhM(\mu) =
  \sup_{(x,a,x')\in\cX\times\cA\times\cX}\crl[\big]{\frac{P_{h-1}^{\sM}(x'\mid{}x,a)}{\mu(x')}}$
  \citep{xie2021batch}. Note that
$\CinfhM\leq\abs*{\cA}\cdot\CpushhM$ but the converse is not true.}} 
  \begin{restatable}{proposition}{pushforwardrelaxation}
    \label{prop:pushforward_relaxation}
    \akedit{Fix $h \in [H]$.} For any $p\in\Delta(\PiRNS)$, if we define $p'\in\Delta(\PiRNS)$
    as the distribution induced by sampling $\pi\sim{}p$ and executing
    $\pi\circ_{h}\piunif$, we have that for all $\veps>0$,
    \begin{align}
      \CovM(p')\leq{}\abs*{\cA}\cdot\pCovM(p).
    \end{align}
Furthermore, $\pCovOptM \leq{}\CpushhM$ for all $\veps>0$.
\end{restatable}
In particular, any
distribution $p\in\Delta(\PiRNS)$ that optimizes the relaxation in
\cref{eq:pushforward_relaxation} achieves \mainobj value
$\CovM(p')\leq\abs*{\cA}\cdot\CpushhM$.\footnote{The dependence on
  $\abs{\cA}$ in this result is natural, as pushforward coverability
  only \icml{considers state occupancies.}\arxiv{grants control over density ratios for state occupancies as
  opposed to state-action occupancies.}} Notable special cases where pushforward coverability is
bounded include:
\icml{(i) Tabular MDPs have $\CpushM \leq \abs*{\cX}$, (ii) Block MDPs
  \citep{du2019latent,misra2019kinematic,zhang2022efficient,mhammedi2023representation}
  with latent state space $\cS$ have $\CpushM\leq\abs*{\cS}$, (iii) Low-Rank MDPs \eqref{eq:low_rank} have
  $\CpushM\leq{}d$ \citep{xie2021batch}.
}
\arxiv{
\begin{itemize}
\item Tabular MDPs have $\CpushM \leq \abs*{\cX}$.
\item Block MDPs
  \citep{du2019latent,misra2019kinematic,zhang2022efficient,mhammedi2023representation}
  with latent state space $\cS$ have $\CpushM\leq\abs*{\cS}$.
\item Low-rank MDPs with the structure in \cref{eq:low_rank} have
  $\CpushM\leq{}d$ \citep{xie2021batch}.
\end{itemize}
}

\icml{
\paragraph{The algorithm}
An iterative algorithm to optimize the
pushforward coverability relaxation is given in
\cref{alg:pushforward_relaxation}\icml{ (deferred to
  \cref{sec:planning_omitted} for space)}. The algorithm follows the same
template as \cref{alg:linf_relaxation}, but at each step $t$, the policy $\pi\ind{t}$ is computed by solving
an objective based on \cref{eq:pushforward_relaxation}. The main
guarantee is as follows.
}

\arxiv{
\paragraph{The algorithm}
An iterative algorithm to compute
a distribution $p\in\Delta(\PiRNS)$ that optimizes the
pushforward coverability relaxation in
\cref{eq:pushforward_relaxation} is given in
\cref{alg:pushforward_relaxation}\icml{ (deferred to
  \cref{sec:planning_omitted} for space)}. The algorithm follows the same
template as \cref{eq:linf_relaxation}, with the only difference being that at each step $t$, given policies $\pi\ind{1},\ldots,\pi\ind{t-1}$ computed so far, the
algorithm computes the new policy $\pi\ind{t}$ by solving the alternative
policy optimization problem
\icml{$\pi\ind{t} = $
  \begin{small}
    \begin{align*}
      \label{eq:pushfoward_opt}
\argmax_{\pi\in\Pi}  \En^{\sM,\pi}\brk*{\frac{P_{h-1}^{\sM}(x_h\mid{}x_{h-1},a_{h-1})}{\sum_{i<t}d_h^{\sM,\pi\ind{i}}(x_h)+P_{h-1}^{\sM}(x_h\mid{}x_{h-1},a_{h-1})}}
    \end{align*}\end{small}}%
\arxiv{\begin{align}
  \label{eq:pushfoward_opt}
\pi\ind{t}=\argmax_{\pi\in\Pi}  \En^{\sM,\pi}\brk*{\frac{P_{h-1}^{\sM}(x_h\mid{}x_{h-1},a_{h-1})}{\sum_{i<t}d_h^{\sM,\pi\ind{i}}(x_h)+P_{h-1}^{\sM}(x_h\mid{}x_{h-1},a_{h-1})}}
\end{align}}in \cref{line:pushforward_opt} (up to tolerance $\vepsapx>0$). This
objective is a policy optimization problem with \emph{stochastic} rewards given by
\icml{
  \begin{small}
    \begin{equation}
      r\ind{t}_{h-1} \ldef \frac{P_{h-1}^{\sM}(x_h\mid{}x_{h-1},a_{h-1})}{\sum_{i<t}d_h^{\sM,\pi\ind{i}}(x_h)+P_{h-1}^{\sM}(x_h\mid{}x_{h-1},a_{h-1})}
      \label{eq:pushforward_reward}
    \end{equation}
  \end{small}%
}%
\arxiv{\begin{equation}
  r\ind{t}_{h-1} \ldef \frac{P_{h-1}^{\sM}(x_h\mid{}x_{h-1},a_{h-1})}{\sum_{i<t}d_h^{\sM,\pi\ind{i}}(x_h)+P_{h-1}^{\sM}(x_h\mid{}x_{h-1},a_{h-1})}
  \label{eq:pushforward_reward}
\end{equation}
}
for $x_h\sim{}P_{h-1}^{\sM}(\cdot\mid{}x_{h-1},a_{h-1})$. As before,
when the MDP $M$ is known, this is a standard
reward-driven planning problem. In addition, compared to the previous relaxation,
the reward \eqref{eq:pushforward_reward} involves only simple,
easy-to-compute quantities for the MDP $M$.

The following theorem shows that \cref{alg:pushforward_relaxation} converges
to a policy cover $p\in\Delta(\PiRNS)$ that optimizes
the relaxation in \cref{eq:linf_relaxation} up to a small
$\log(\veps^{-1})$ approximation factor.%
}

\begin{restatable}{theorem}{pushforwardoptimization}
  \label{thm:pushforward_optimization}
    For any $\veps\in\brk*{0,1}$ and $h\in\brk{H}$,
    whenever $\vepsapx\leq{}\CpushhM\cdot\veps\log(2\veps^{-1})$, \cref{alg:pushforward_relaxation}
  produces a distribution $p\in\Delta(\Pi)$ with
  $\abs*{\supp(p)}\leq{}\veps^{-1}$ such that
\begin{align}
  \pCovM(p) \leq{} 5\CpushhM\log(2\veps^{-1}).
\end{align}
Consequently, if we define $p'\in\Delta(\PiRNS)$
    as the distribution induced by sampling $\pi\sim{}p$ and executing
    $\pi\circ_{h}\piunif$, we have that $\CovM(p')\leq{}5\abs*{\cA}\CpushhM\log(2\veps^{-1})$.
\end{restatable}

\arxiv{\section{Efficient Model-Based Exploration via \mainobj}}
\icml{\section{\mbox{Efficient Online Exploration via \mainobj}}}
\label{sec:model_based}

In this section, we turn our attention to sample-efficient
\emph{online exploration} for the setting in which the underlying MDP $\Mstar$ is
unknown.
Throughout the section, we work with an arbitrary user-specified subset of policies
$\Pi\subseteq\PiRNS$.

    \paragraph{Model-based reinforcement learning setup}

We focus on \emph{model-based reinforcement learning},
  and assume access
  to a \emph{model class} $\cM$ that contains the true MDP $\Mstar$.
\begin{assumption}[Realizability]
  \label{ass:realizability}
  The learner has access to a class $\cM$ containing the true model $\Mstar$.
\end{assumption}
The class $\cM$ is user-specified, and can be parameterized by
deep neural networks or any other flexible function class, with the
best choice depending on the problem domain.%

For $M\in\cM$, we use $M(\pi)$ as shorthand for the law of the trajectory
$(x_1,a_1),\ldots,(x_H,a_H)$ for policy $\pi$ in $M$.
We define the squared Hellinger distance for probability measures $\bbP$ and $\bbQ$ \icml{by $\Dhels{\bbP}{\bbQ}=\int\prn[\big]{\sqrt{d\bbP}-\sqrt{d\bbQ}}^{2}$.}%
\arxiv{with a common
dominating measure $\omega$ by
\begin{equation}
  \label{eq:hellinger}
  \Dhels{\bbP}{\bbQ}=\int\prn[\bigg]{\sqrt{\frac{d\bbP}{d\omega}}-\sqrt{\frac{d\bbQ}{d\omega}}}^{2}d\omega.
\end{equation}
}

\paragraph{Estimation oracles}
Our algorithms and regret bounds use
the primitive of an \emph{estimation oracle}, denoted by $\AlgEst$, a user-specified
algorithm for estimation that is used to estimate the underlying model
$\Mstar$ from data
\arxiv{\citep{foster2020beyond,foster2023lecture,foster2021statistical,foster2023tight}}
\icml{\citep{foster2020beyond,foster2023lecture}} sequentially. At each episode $t$, given the data
$\hist\ind{t-1}=(\pi\ind{1},o\ind{1}),\ldots,(\pi\ind{t-1},o\ind{t-1})$
collected so far, \dfedit{where $o\ind{t}\ldef{}(x_1\ind{t},a_1\ind{t}),\ldots,(x_H\ind{t},a_H\ind{t})$,}
the estimation oracle constructs an estimate
\icml{$\Mhat\ind{t}=\AlgEst\prn[\big]{ \crl*{(\act\ind{i},
\obs\ind{i})}_{i=1}^{t-1} }$}
\arxiv{\[
\Mhat\ind{t}=\AlgEst\prn*{ \crl*{(\act\ind{i},
\obs\ind{i})}_{i=1}^{t-1} }
\]}
for the true MDP $\Mstar$.
\icml{We assume that $\AlgEst$ is an
  \emph{offline} estimation oracle, in the sense that each estimator $\Mhat\ind{t}$ has good out-of-sample
performance on the historical dataset $\hist\ind{t-1}$.\loose}

\arxiv{\padelete{We present sample complexity bounds based on two notions of
performance for the estimation oracle: \emph{online estimation} performance
and \emph{offline estimation} performance. For offline estimation, we
assume that each estimator $\Mhat\ind{t}$ has good out-of-sample
performance on the historical dataset $\hist\ind{t-1}$.}
We assume that $\AlgEst$ is an
  \emph{offline estimation oracle}, in the sense that each estimator $\Mhat\ind{t}$ has good out-of-sample
performance on the historical dataset $\hist\ind{t-1}$. \pacomment{deleted mentions of online estimation here}}
\begin{assumption}[Offline estimation oracle\arxiv{ for $\cM$}]
    \label{ass:offline_oracle}
	At each time $t\in[T]$, an offline estimation oracle
        $\AlgEst$ for $\cM$, given
        \arxiv{$$\hist\ind{t-1}=(\pi\ind{1},o\ind{1}),\ldots,(\pi\ind{t-1},o\ind{t-1})$$}%
        \icml{$\hist\ind{t-1}=(\pi\ind{1},o\ind{1}),\ldots,(\pi\ind{t-1},o\ind{t-1})$} with $o\ind{i}\sim\Mstar(\pi\ind{i})$ and
        $\pi\ind{i}\sim p\ind{i}$, returns an estimator
        $\Mhat\ind{t}\in\cM$ such that 
              \icml{$\EstOfft[t]\ldef$
        \begin{align*}
          \sum_{i<t}\En_{\act\ind{i}\sim{}p\ind{i}}\brk[\big]{\DhelsX{\big}{\Mhat\ind{t}(\pi\ind{i})}{\Mstar(\pi\ind{i})}}
          \leq \EstProbOff,
        \end{align*}}
        \arxiv{
        \begin{align*}
              \EstOfft[t] \ldef{}
          \sum_{i<t}\En_{\act\ind{i}\sim{}p\ind{i}}\brk[\big]{\DhelsX{\big}{\Mhat\ind{t}(\pi\ind{i})}{\Mstar(\pi\ind{i})}}
          \leq \EstProbOff,
        \end{align*}}
	for all $t\in\brk{T}$ with probability at least $1-\delta$ whenever $\Mstar\in\cM$, where $\EstProbOff$ is a
        known upper bound.\loose
      \end{assumption}
      As an example, the standard maximum likelihood estimator
      (\mle) satisfies \cref{ass:offline_oracle} with
      $\EstProbOff\leq\bigoh(\log(\abs{\cM}T/\delta))$ (e.g., \citet{foster2023lecture}).

 \padelete{\arxiv{
\paedit{
For \emph{online estimation}, we measure the oracle's estimation
performance in terms of cumulative Hellinger error, which we assume is
bounded as follows.
\begin{assumption}[Online estimation oracle for $\cM$]
    \label{ass:online_oracle}
	At each time $t\in[T]$, an online estimation oracle
        $\AlgEst$ for $\cM$ returns,
        given \arxiv{$$\hist\ind{t-1}=(\pi\ind{1},o\ind{1}),\ldots,(\pi\ind{t-1},o\ind{t-1})$$}\icml{$\hist\ind{t-1}=(\pi\ind{1},o\ind{1}),\ldots,(\pi\ind{t-1},o\ind{t-1})$}
        with $o\ind{i}\sim\Mstar(\pi\ind{i})$ and
        $\pi\ind{i}\sim p\ind{i}$, an estimator
        $\Mhat\ind{t}\in\cM$ such that whenever $\Mstar\in\cM$,
                \icml{
        \begin{align*}
              \EstOn &\ldef{}
          \sum_{t=1}^{T}\En_{\act\ind{t}\sim{}p\ind{t}}\brk[\big]{\DhelsX{\big}{\Mhat\ind{t}(\pi\ind{t})}{\Mstar(\pi\ind{t})}}\\
          &\leq \EstProbOn,
        \end{align*}
        }
        \arxiv{
        \begin{align*}
              \EstOn \ldef{}
          \sum_{t=1}^{T}\En_{\act\ind{t}\sim{}p\ind{t}}\brk[\big]{\DhelsX{\big}{\Mhat\ind{t}(\pi\ind{t})}{\Mstar(\pi\ind{t})}}
          \leq \EstProbOn,
        \end{align*}
        }
	with probability at least $1-\delta$, where $\EstProbOn$ is a
        known upper bound.%
      \end{assumption}
      See Section 4 of
      \citet{foster2021statistical} or \citet{foster2023lecture} for
      further background on online estimation.\\
      Our algorithms support both types of
      estimator, but are stated most directly stated in terms of online
      estimation error, and give tighter sample complexity bounds in
      this case. The requirement in \cref{ass:online_oracle} that
      the online estimator is \emph{proper} (i.e., has
      $\Mhat\ind{t}\in\cM$) is quite stringent, as generic online
      estimation algorithms (e.g., Vovk's aggregating algorithm) are
      improper, and proper algorithms are only known for specialized
      MDP classes such as tabular MDPs (see discussion in
      \citet{foster2021statistical}).\footnote{On the statistical
        side, it is straightforward to extend the results in this
        section to accommodate improper online estimators; we impose
        this restriction for \emph{computational} reasons, as this
        enables the application of the efficient planning results in
        \cref{sec:planning}.} This contrasts with offline estimation,
      where most standard algorithms such as \mle are proper. As such, our bounds based on online
      estimation are best thought of as secondary results, with our
      bounds based on offline estimation serving as the main results.\\
On the technical side, our interest in proper online estimation arises from
the following structural result, which shows that whenever the
\mainopt parameter is bounded, any algorithm with low offline estimation
error also enjoys low online estimation error (with polynomial loss in rate).
      \begin{restatable}[Offline-to-online]{lemma}{loneofflineonline}
        \label{lem:lone_offline_online}
          Any offline estimator $\Mhat\ind{t}$ that satisfies
  \cref{ass:offline_oracle} with estimation error bound $\EstProbOff$
  satisfies \cref{ass:online_oracle} with \arxiv{estimation error bound}
    \icml{
      \mbox{$\EstProbOn \leq \bigoht\prn[\big]{H\prn[\big]{\ConeM[\Mstar](1+\EstProbOff)}^{1/3}T^{2/3}}$}.}
  \arxiv{
  \begin{align}
    \EstProbOn \leq \bigoh\prn[\Big]{H\log H \prn*{\ConeM[\Mstar](1+\EstProbOff)}^{1/3}T^{2/3}
    }.
  \end{align}
  }
\end{restatable}
Note that $\ConeM[\Mstar]\leq\CovOptMmaxzero[\Mstar]$; we leave an
extension to $\CovOptMmax[\Mstar]$ for $\veps>0$ to future
work. \dfc{Can we handle the general $\veps$ extension?}\\
We also make
use of a tighter offline-to-online lemma based on the (larger)
$L_{\infty}$-Coverability parameter $\CinfM[\Mstar]$.
\begin{lemma}[\citet{xie2023role}]
  \label{lem:linf_offline_online}
  Any offline estimator $\Mhat\ind{t}$ that satisfies
  \cref{ass:offline_oracle} with estimation error bound $\EstProbOff$
  satisfies \cref{ass:online_oracle} with \arxiv{estimation error
    bound}
    \icml{
      \mbox{$\EstProbOn \leq \bigoht\prn[\big]{H\prn[\big]{\CinfM[\Mstar]
    T\cdot\EstProbOff}^{1/2} + H\CinfM[\Mstar] 
    }$}.}
  \arxiv{
  \begin{align}
    \EstProbOn \leq \bigoh\prn*{H\log H \sqrt{\CinfM[\Mstar]
    T\log{}T\cdot\EstProbOff} + H\log{}H\cdot\CinfM[\Mstar] 
    }.
  \end{align}}
\end{lemma}
Both lemmas lead to a degradation in rate with respect to
$T$, but lead to sublinear online estimation error whenever the
offline estimation error bound is sublinear\arxiv{; \citet{foster2023online}
show that some degradation in rate is unavoidable}.
}
\pacomment{move all of this}
}
}

\begin{algorithm}[tp]
  \arxiv{\setstretch{1.3}}
    \begin{algorithmic}[1]
      \icml{\State \textbf{input}: Estimation oracle $\AlgEst$, number of episodes $T\in\bbN$, approximation
       parameters $C\geq{}1$, $\veps\in\brk{0,1}$.}
       \arxiv{\State \textbf{input}:
       \Statex[1] Estimation oracle $\AlgEst$.
       \Statex[1] Number of episodes $T\in\bbN$, approximation
       parameters $C\geq{}1$, $\veps\in\brk{0,1}$.}
  \For{$t=1, 2, \cdots, T$}
  \State \mbox{\arxiv{Compute   estimated model}\icml{Estimate model:} $\Mhat\ind{t} = \AlgEst\ind{t}\prn[\big]{
    \crl*{(\act\ind{i},\obs\ind{i})}_{i=1}^{t-1} }$.}
\icml{
  \Statex[1]                   \algcommentlight{Plug-in
                    approximation to \mainobj objective.}
          \State \multiline{For each $h\in\brk{H}$, compute $(C,\veps)$-approx.
  policy cover $p_h\ind{t}$ for $\Mhat\ind{t}$: $      \CovM[\Mhat\ind{t}](p_h\ind{t})  =$
  \begin{small}
    \begin{align}
\sup_{\pi\in\Pi}\En^{\sMhat\ind{t},\pi}\brk*{\frac{\dm{\Mhat\ind{t}}{\pi}_{h}(x_h,a_h)}{\dm{\Mhat\ind{t}}{p_h\ind{t}}_{h}(x_h,a_h)+\veps\cdot{}d_h^{\sMhat\ind{t},\pi}(x_h,a_h)}}\leq{}C.\label{eq:cover_objective_model_based}
    \end{align}
  \end{small}}
                  }
\arxiv{
        \State For each $h\in\brk{H}$, compute $(C,\veps)$-approximate
  policy cover $p_h\ind{t}$ for $\Mhat\ind{t}$: %
        \begin{align}
\CovM[\Mhat\ind{t}](p_h\ind{t})  = \sup_{\pi\in\Pi}\En^{\sMhat\ind{t},\pi}\brk*{\frac{\dm{\Mhat\ind{t}}{\pi}_{h}(x_h,a_h)}{\dm{\Mhat\ind{t}}{p_h\ind{t}}_{h}(x_h,a_h)+\veps\cdot{}d_h^{\sMhat\ind{t},\pi}(x_h,a_h)}}\leq{}C.\label{eq:cover_objective_model_based}
        \end{align}
                  \hfill\algcommentlight{Plug-in
                    approximation to \mainobj objective.}
                  }
                  \icml{\State \multiline{Sample
                    $\pi\ind{t}\sim{}q\ind{t}=\unif(p_1\ind{t},\ldots,p_H\ind{t})$
                    and observe trajectory
                    $o\ind{t}=(x_1\ind{t},a_1\ind{t}),\ldots,(x_H\ind{t},a_H\ind{t})$.}
                  }
                  \arxiv{\State Let $q\ind{t}=\unif(p_1\ind{t},\ldots,p_H\ind{t})$
                  \State Sample $\pi\ind{t}\sim{}q\ind{t}$, \arxiv{observe trajectory}\icml{observe} $o\ind{t}=(x_1\ind{t},a_1\ind{t}),\ldots,(x_H\ind{t},a_H\ind{t})$.}
\EndFor
\State  \textbf{return} \arxiv{policy covers }$(p_1,\ldots,p_H)$, where $p_h\ldef\unif(p_h\ind{1},\ldots,p_h\ind{T})$.
\end{algorithmic}
\caption{Coverage-Driven Exploration (\mainalg)}
\label{alg:model_based_reward_free}
\end{algorithm}

\arxiv{\subsection{Algorithm}}
\icml{\fakepar{Algorithm}}
Our main algorithm for
reward-free reinforcement learning, Coverage-Driven Exploration
(\mainalg; \cref{alg:model_based_reward_free}), is based on a simple ``plug-in'' estimation-optimization paradigm: Repeatedly compute an estimate $\Mhat\ind{t}$
  for the true model $\Mstar$, then compute a policy cover
  $p\in\Delta(\Pi)$ that optimizes the \mainobj objective for $\Mhat\ind{t}$ (a
  \emph{plug-in} approximation to the true \mainobj objective) and use
  this to collect data; proceed until this process arrives at a
  high-quality cover for $\Mstar$.\loose

In more detail, \cref{alg:model_based_reward_free} proceeds
in $T$ episodes. At each episode $t$, the algorithm invokes the
user-specified estimation oracle $\AlgEst$ to produce an estimate
$\Mhat\ind{t}$ for the model $\Mstar$ based on the data collected so
far. Given this estimate, for each layer $h\in\brk{H}$, the algorithm
computes a $(C,\veps)$-approximate policy cover $p_h\ind{t}\in\Delta(\Pi)$ for
$\Mhat\ind{t}$:
\icml{$      \CovM[\Mhat\ind{t}](p_h\ind{t})  =$

  \vspace{-0.5cm}
  \begin{small}
    \begin{equation}
      \label{eq:cover_objective_inline}
\sup_{\pi\in\Pi}\En^{\sMhat\ind{t},\pi}\brk*{\frac{\dm{\Mhat\ind{t}}{\pi}_{h}(x_h,a_h)}{\dm{\Mhat\ind{t}}{p_h\ind{t}}_{h}(x_h,a_h)+\veps\cdot{}d_h^{\sMhat\ind{t},\pi}(x_h,a_h)}}\leq{}C.
    \end{equation}%
  \end{small}%
}%
\arxiv{\begin{equation}
  \label{eq:cover_objective_inline}
\CovM[\Mhat\ind{t}](p_h\ind{t})  = \sup_{\pi\in\Pi}\En^{\sMhat\ind{t},\pi}\brk*{\frac{\dm{\Mhat\ind{t}}{\pi}_{h}(x_h,a_h)}{\dm{\Mhat\ind{t}}{p_h\ind{t}}_{h}(x_h,a_h)+\veps\cdot{}d_h^{\sMhat\ind{t},\pi}(x_h,a_h)}}\leq{}C,
\end{equation}}
where $C>0$ is made sufficiently large to ensure
\cref{eq:cover_objective_inline} is feasible. This is a plug-in approximation to the true \mainobj
objective $\CovM[\Mstar](p)$. Given the approximate covers
$p_1\ind{t},\ldots,p_H\ind{t}$, the algorithm collects a new
trajectory $o\ind{t}$ by sampling
$\pi\ind{t}\sim{}q\ind{t}\ldef\unif(p_1\ind{t},\ldots,p_H\ind{t})$. This
trajectory is used to update the estimator $\Mhat\ind{t}$, and the
algorithm proceeds to the next episode. Once all episodes
conclude, the algorithm returns
$p_h\ldef{}\unif(p_h\ind{1},\ldots,p_h\ind{T})$ as the final cover for
each layer $h$.\loose

The plug-in \mainobj objective in \cref{eq:cover_objective_inline} can be solved efficiently using the relaxation-based
methods in \cref{sec:planning}, since the model
$\Mhat\ind{t}\in\cM$ is known, making this a pure (non-statistical) planning
problem. We leave the approximation parameter $C\geq{}1$\arxiv{ in
\cref{eq:cover_objective_inline}} as a
free parameter to accommodate the approximation factors these
relaxations incur (the sample complexity\arxiv{ for
\cref{alg:model_based_reward_free}} degrades linearly with $C$).

\arxiv{\subsection{Main Result}}
\label{sec:model_based_reward_free}
\arxiv{In its most general form, \cref{alg:model_based_reward_free} leads to
sample complexity guarantees based on \mainopt. However, we begin with a slightly tighter sample complexity bound at the cost of scaling with $L_{\infty}$-Coverability instead of \mainopt. To state the guarantees in the most compact form possible, we make the following assumption on the estimation oracle's error rate.
\begin{assumption}[Parametric estimation rate]
  \label{ass:parametric}
  The offline estimation oracle satisfies
  $\EstProbOff\leq\bigoh(\dest\log(\Best{}T/\delta))$ for parameters $\dest,\Best\in\bbN$.
\end{assumption}
    \begin{restatable}[Guarantee for
      \mainalg under $L_\infty$-Coverability]{theorem}{mbrfthree}
      \label{cor:mbrf3}
       \pacomment{changed the theorem caption}
      Let $\veps>0$ be given. Let $\Cinf\equiv\CinfM[\Mstar]$, and
      suppose that (i) we restrict $\cM$ such that all $M\in\cM$ have
      $\CinfM\leq\Cinf$, and (ii) we solve
      \cref{eq:cover_objective_model_based} with
      $C=\Cinf$ for all $t$.\footnote{We can take $\CinfM\leq\Cinf$
        \arxiv{without loss of generality}\icml{w.l.o.g.} when $\Cinf$ is known. In this
        case, solving \cref{eq:cover_objective_model_based} with
      $C=\Cinf$ is feasible by
        \cref{prop:linf}.} Then, given an offline estimation
      oracle satisfying \cref{ass:offline_oracle,ass:parametric}, using
      $T=\bigoht\prn*{\frac{H^8(\CinfM[\Mstar])^3\dest\log(\Best/\delta)}{\veps^{4}}}$
      episodes, \cref{alg:model_based_reward_free} produces policy
      covers $p_1,\ldots,p_H\in\Delta(\Pi)$ such that\loose
      \begin{align}
        \label{eq:mbrf3}
        \forall{}h\in\brk{H}:\quad\CovM[\Mstar](p_h)
        \leq 12H\cdot{}\CinfM[\Mstar]
      \end{align}
      \arxiv{with probability}\icml{w.p.} at least $1-\delta$. For a
      finite class $\cM$, if we use \mle as the estimator, we can take
      $T=\bigoht\prn*{\frac{H^8(\CinfM[\Mstar])^3\log(\abs*{\cM}/\delta)}{\veps^{4}}}$.
    \end{restatable}

    Our next result gives an analogous sample complexity  guarantee
    based on the \mainopt value itself.

\begin{restatable}[Guarantee for \mainalg under \mainopt]{theorem}{mbrftwo}
  \label{cor:mbrf2}
  Let $\veps>0$ be given.
    Suppose that (i) we restrict $\cM$ such that all $M\in\cM$ have
    $\CovOptMmax\leq\CovOptMmax[\Mstar]$, and (ii) we solve
    \cref{eq:cover_objective_model_based} with
    $C=\CovOptMmax[\Mstar]$ for all $t$\padelete{ (which is always
    feasible)}. Then, given access to an offline estimation oracle
    satisfying \cref{ass:offline_oracle,ass:parametric}, using
    $T=\bigoht\prn*{\frac{H^{12}(\CovOptMmaxzero[\Mstar])^4\dest\log(\Best/\delta)}{\veps^{6}}}$
    episodes, \cref{alg:model_based_reward_free} produces policy
    covers $p_1,\ldots,p_H\in\Delta(\Pi)$ such that
    \begin{align}
      \label{eq:mbrf2}
      \forall{}h\in\brk{H}:\quad\CovM[\Mstar](p_h)
      \leq 12H\cdot{}\CovOptMmax[\Mstar]
    \end{align}
    with probability at least $1-\delta$. In particular, for a finite
    class $\cM$, if we use \mle as the estimator, we can take
    $T=\bigoht\prn*{\frac{H^{12}(\CovOptMmaxzero[\Mstar])^4\log(\abs*{\cM}/\delta)}{\veps^{6}}}$.
  \end{restatable}
  \dfcomment{It seems hard to improve the result above so that 1) we have $\CovOptMmax[\Mstar]$ in the sample complexity instead of $\CovOptMmaxzero[\Mstar]$, and 2) so that we have $\CovOptM[\Mstar]$ on the rhs of \cref{eq:mbrf2} instead of $\CovOptMmax[\Mstar]$. do people find these issues offputting?}

\cref{cor:mbrf3} and \cref{cor:mbrf2} are derived as special cases of
a more general result (\cref{thm:model_based_reward_free}) which
allows for \emph{online estimation} oracles; we show in particular
that under \mainopt and $L_\infty$-coverability, any offline estimation
oracle is also an online estimation oracle
(cf. \cref{app:model_based}).

These results show for the
first time that it is possible to perform sample-efficient and
computationally-efficient reward-free exploration under
coverability. \icml{See \cref{app:model_based_proof} for the proof,
  which is based on a connection to the Decision-Estimation
  Coefficient of \citet{foster2021statistical,foster2023tight}.} Some key features are as follows.\loose
\icml{\begin{itemize}[leftmargin=*]}
\arxiv{\begin{itemize}}
\item \emph{Computational efficiency.} \pacomment{replaced this caption, was ``Simplicity and practicality''} Since maximum likelihood estimation (\mle) is
  a valid estimation oracle, \cref{alg:model_based_reward_free} is
  computationally efficient whenever 1) \mle can be performed
  efficiently, and 2) the plug-in \mainobj objective in \cref{eq:cover_objective_model_based} can be approximately optimized
  efficiently. As optimizing the objective involves only the
  estimated model $\Mhat\ind{t}$ (and hence is a computational
  problem), we can appeal to the relaxations in \cref{sec:planning} to
  accomplish this efficiently (via off-the-shelf planning methods). 
  
  Regarding the latter point, note that while \cref{cor:mbrf3} assumes for simplicity that
  \cref{eq:cover_objective_model_based} is solved with
$C=\CinfM[\Mstar]$, \arxiv{it should be clear that }if we solve the
objective for $C>\CinfM[\Mstar]$ the result continues to hold
with $\CinfM[\Mstar]$ replaced by $C$ in the sample complexity
bound and approximation guarantee. Consequently, if we solve
\cref{eq:cover_objective_model_based} using
\cref{alg:linf_relaxation}, the guarantees in \cref{cor:mbrf3} continue
to hold up to an $\bigoht(1)$ approximation factor.

\item \emph{Statistical efficiency.} \cref{cor:mbrf2} is the first
  result we are aware of that provides statistically efficient reward-free
  exploration under bounded \mainopt \paedit{(or even bounded $L_\infty$-Coverability)}. In particular, since the sample
  complexity scales as
  $\poly(\CovOptMmax[\Mstar],H,\log\abs{\cM},\veps^{-1})$, this shows
  that \mainopt is a sufficiently powerful structural parameter to
  enable sample-efficient learning with nonlinear function approximation. We expect that the precise sample
  complexity guarantees can be improved; \paedit{in particular, it would be
  interesting to remove the lossiness incurred by passing from offline
  to online estimation.} \pacomment{this should maybe be chopped now that we only mention offline-to-online briefly}
\end{itemize}
    }

    \icml{
      \icml{\fakepar{Main result}}
In its most general form, \cref{alg:model_based_reward_free} leads to
sample complexity guarantees based on \mainopt. Due to space
constraints, the main result we present here is a simpler guarantee
based on $L_{\infty}$-Coverability.
To state the guarantee in the most compact form possible, we make the following assumption on the estimation error rate.
\begin{assumption}[Parametric estimation rate]
  \label{ass:parametric}
  The offline estimation oracle satisfies
  $\EstProbOff\leq\bigoh(\dest\log(\Best{}T/\delta))$ for parameters $\dest,\Best\in\bbN$.
\end{assumption}
    \begin{restatable}[Main guarantee for
      \mainalg]{theorem}{mbrfthree}
      \label{cor:mbrf3}
      Let $\veps>0$ be given. Let $\Cinf\equiv\CinfM[\Mstar]$, and
      suppose that (i) we restrict $\cM$ such that all $M\in\cM$ have
      $\CinfM\leq\Cinf$, and (ii) we solve
      \cref{eq:cover_objective_model_based} with
      $C=\Cinf$ for all $t$.\footnote{We can take $\CinfM\leq\Cinf$
        \arxiv{without loss of generality}\icml{w.l.o.g.} when $\Cinf$ is known. In this
        case, solving \cref{eq:cover_objective_model_based} with
      $C=\Cinf$ is feasible by
        \cref{prop:linf}.} Then, given an offline estimation
      oracle satisfying \cref{ass:offline_oracle,ass:parametric}, using
      $T=\bigoht\prn*{\frac{H^8(\CinfM[\Mstar])^3\dest\log(\Best/\delta)}{\veps^{4}}}$
      episodes, \cref{alg:model_based_reward_free} produces policy
      covers $p_1,\ldots,p_H\in\Delta(\Pi)$ such that\loose
      \begin{align}
        \label{eq:mbrf3}
        \forall{}h\in\brk{H}:\quad\CovM[\Mstar](p_h)
        \leq 12H\cdot{}\CinfM[\Mstar]
      \end{align}
      \arxiv{with probability}\icml{w.p.} at least $1-\delta$. For a
      finite class $\cM$, if we use \mle as the estimator, we can take
      $T=\bigoht\prn*{\frac{H^8(\CinfM[\Mstar])^3\log(\abs*{\cM}/\delta)}{\veps^{4}}}$.
    \end{restatable}
    
\cref{cor:mbrf3} shows for the
first time that it is possible to perform sample-efficient and
computationally efficient reward-free exploration under
coverability. We refer to \cref{app:model_based_proof} for the proof,
  which is based on a connection to the Decision-Estimation
  Coefficient of \citet{foster2021statistical,foster2023tight}. Some key features of the result are as follows.

\noindent
 \emph{Computational efficiency.} As maximum likelihood estimation (\mle) is
  a valid estimation oracle, \cref{alg:model_based_reward_free} is
  computationally efficient whenever 1) \mle can be performed
  efficiently, and 2) the plug-in \mainobj objective in \cref{eq:cover_objective_model_based} can be approximately optimized
  efficiently. As the objective involves only the
  estimated model $\Mhat\ind{t}$ (and hence is a computational
  problem), we can use the relaxations in \cref{sec:planning} to
  \arxiv{accomplish}\icml{do} this efficiently (via off-the-shelf \akdelete{planning }methods).\loose

  Regarding the latter point, note that while \cref{cor:mbrf3} assumes for simplicity that
  \cref{eq:cover_objective_model_based} is solved with
$C=\CinfM[\Mstar]$, \arxiv{it should be clear that }if we solve the
objective for $C>\CinfM[\Mstar]$ the result continues to hold
with $\CinfM[\Mstar]$ replaced by $C$ in the sample complexity
bound and approximation guarantee. Consequently, if we solve
\cref{eq:cover_objective_model_based} using
\cref{alg:linf_relaxation}, the guarantees in \cref{cor:mbrf3} continue
to hold up to an $\bigoht(1)$ approximation factor.\loose

  \noindent
  \emph{Statistical efficiency.} \mainalg achieves sample complexity
  guarantees based on $L_{\infty}$-Coverability in a computationally
  eficient fashion for the first time. Compared to previous
  inefficient algorithms based on coverability
  \citep{xie2023role,liu2023provable,amortila2023harnessing} our
  theorem has somewhat looser sample complexity, and requires model-based function
  approximation. See \cref{sec:model_free} for a model-free
  counterpart based on weight function learning. \pacomment{winner of the ``biggest result with the least text in main body'' award}

  \mainalg is not limited to $L_{\infty}$-Coverability. \cref{sec:model_based_general} gives more
  general results (\cref{thm:model_based_reward_free}, \cref{cor:mbrf2}) which
  achieve sample complexity
  $\poly(\CovOptMmax[\Mstar],H,\log\abs{\cM},\veps^{-1})$, showing that \mainopt is itself a sufficiently powerful structural parameter to
  enable sample-efficient learning with nonlinear function approximation.\loose
  
  \pacomment{swap this out for $L_1$ result}

}

  \icml{
      \fakepar{Application to downstream policy optimization}
By \cref{lem:com}, the policy covers
$p_{1:H}$ returned by \cref{alg:model_based_reward_free} can
be used to optimize any downstream reward function using \arxiv{standard
offline RL algorithms}\icml{offline RL}, giving end-to-end guarantees for reward-driven
online RL. We sketch an example using \arxiv{maximum likelihood}\icml{\mle} for policy optimization in
\cref{app:rf_example} (\arxiv{see also}\icml{cf.} \cref{app:downstream}).\loose
}
  
\arxiv{
  \paragraph{Application to downstream policy optimization}
\dfc{Remind how reward-driven RL is defined; can give pointer to
  definition in the appendix. Specifically want to explain that this
  gives a bound on online sample complexity for reward-driven}  
By \cref{lem:com} (see also \cref{app:downstream}), the policy covers
$p_1,\ldots,p_H$ returned by \cref{alg:model_based_reward_free} can
be used to optimize any downstream reward function using standard
offline RL algorithms. For concreteness, we sketch an example which
uses maximum likelihood (\mle) for offline policy optimization; see
\cref{app:downstream} for further examples and details.

We now sketch some basic examples in which our main results can be
applied to give end-to-end guarantees
\begin{example}[Tabular MDPs]
  For tabular MDPs with $\abs*{\cX}\leq{}S$ and $\abs*{\cA}\leq{}A$,
  we can construct online estimators for which
  $\EstProbOn=\bigoht(HS^2A)$, so that
  \cref{thm:model_based_reward_free} gives sample complexity
$T=\frac{\poly(H,S,A)}{\veps^2}$ to compute policy covers such that $\CovM[\Mstar](p_h)
        \leq 12H\cdot{}\CinfM[\Mstar]$. \akc{$\EstProbOn$ not defined, maybe it's ok}
      \end{example}
      \begin{example}[Low-Rank MDPs]
        Consider the Low-Rank MDP model in \cref{eq:low_rank} with
        dimension $d$ and
        suppose, following
        \citet{agarwal2020flambe,uehara2022representation}, that we
        have access to classes $\Phi$ and $\Psi$ such that
        $\phi_{h}\in\Phi$ and $\psi_h\in\Psi$. Then \mle achieves
        $\EstProbOff=\bigoht(\log(\abs{\Phi}\abs{\Psi}))$, and we can
        take $\CovOptM[\Mstar]\leq{}d\abs*{\cA}$, so \cref{cor:mbrf2}
        gives sample complexity $T=\frac{\poly(H,d,\abs*{\cA},\log(\abs{\Phi}\abs{\Psi}))}{\veps^6}$ to compute policy covers such that $\CovM[\Mstar](p_h)
        \leq 12H\cdot{}\CinfM[\Mstar]$. \akc{Change $\CinfM[\Mstar]$ to $\CovOptM[\Mstar]$?}
      \end{example}
    }

    \arxiv{As an extension, in \cref{app:model_based_reward_based} we give a reward-driven counterpart to
\cref{alg:model_based_reward_free}, which directly optimizes a given
reward function online. This does not improve upon \paedit{the approach above}, but the analysis is slightly more direct.
}\loose

\arxiv{  \paragraph{Overview of proof}
\cref{alg:model_based_reward_free} can be viewed as an
application of (a reward-free variant of) the Estimation-to-Decisions framework of
\citet{foster2021statistical,foster2023tight}\icml{.}\arxiv{; we make this connection
more explicit through the reward-driven results in
\cref{app:model_based_reward_based}.} %
\icml{
Leveraging this perspective, the
crux of the analysis is to show that each episode, one of two good
possibilities occurs: Either 1) each plug-in maximizer $p_h\ind{t}$ has
good coverage for the true MDP $\Mstar$, or 2) the resulting
trajectory $o\ind{t}$ has high information content, allowing us to
improve the estimate at the next episode. As long as the estimation
oracle $\Mhat\ind{t}$ is consistent, the latter situation cannot occur
too many times before we arrive at a near-optimal policy cover.  
}
\arxiv{Briefly, the key step in the proof for \cref{cor:mbrf3,cor:mbrf2} is to
show that for each round $t$, we have that
\begin{align*}
  \CovOptM[\Mstar](p_h\ind{t})
  \approxleq 
  \max_{h}\CovOptM[\Mhat\ind{t}](p_h\ind{t})
  +
  \frac{1}{\veps}\sqrt{H^3C\cdot{}\En_{\pi\sim{}q\ind{t}}\brk[\big]{\DhelsX{\big}{\Mhat\ind{t}(\pi)}{\Mstar(\pi)}}}
  + \frac{H}{\veps^2}\cdot\En_{\pi\sim{}q\ind{t}}\brk[\big]{\DhelsX{\big}{\Mhat\ind{t}(\pi)}{\Mstar(\pi)}}
\end{align*}
for $C>0$ defined as in \cref{alg:model_based_reward_free}. This equation can be thought of as a reward-free analogue of a bound
on the Decision-Estimation Coefficient (DEC) of
\citet{foster2021statistical,foster2023tight}, and makes precise the
reasoning that by optimizing the plug-in approximation to the \mainobj
objective, we either 1) cover the true MDP $\Mstar$ well, or 2)
achieve large information gain (as quantified by the instantaneous
estimation error
$\En_{\pi\sim{}q\ind{t}}\brk[\big]{\DhelsX{\big}{\Mhat\ind{t}(\pi)}{\Mstar(\pi)}}$).}
}

\arxiv{
\section{Efficient Model-Free Exploration via \mainobj}
\label{sec:model_free}

Our algorithms in the previous section show that the \mainobj
objective and \mainopt parameter enable sample-efficient online
reinforcement learning, but one potential drawback is that they
require model-based realizability, a strong form of function
approximation that may not always be realistic. In this section, we
give \emph{model-free} algorithms to perform reward-free exploration
and optimize the \mainobj objective that do not require model
realizability, and instead require a weaker form of \emph{density
  ratio} or \emph{weight
  function} realizability, a modeling
approach that has been widely used in offline reinforcement learning \citep{liu2018breaking,uehara2020minimax,yang2020off,uehara2021finite,jiang2020minimax,xie2020q,zhan2022offline,chen2022offline,rashidinejad2022optimal,ozdaglar2023revisiting}
and recently adapted to the online setting
\citep{amortila2023harnessing}. The main algorithm we present computes
a policy cover that achieves a bound on the \mainopt objective that
scales with the pushforward coverability parameter (\cref{sec:pushforward_relaxation}), 
but the weaker modeling assumptions make it
applicable in a broader range of settings.

Throughout this section, we take $\Pi=\PiNS$ as the set of all
non-stationary policies.

\subsection{Algorithm}

\paedit{Our model-free algorithm}, \mfalg, is presented in \cref{alg:model_free}. The
algorithm builds a collection of policy covers
$p_1,\ldots,p_H\in\Delta(\PiRNS)$ layer-by-layer in an inductive
fashion. For each layer $h\in\brk{H}$, given policy
covers $p_1,\ldots,p_{h-1}$ for the preceding layers, the algorithm
computes $p_h$ by (approximately) implementing the iterative algorithm
for policy cover construction given in
\cref{alg:pushforward_relaxation} (\cref{sec:pushforward_relaxation}),
in a data-driven fashion.\icml{\footnote{Note that while the results in
  \cref{sec:planning} are presented
    for the setting in which the underlying MDP $M$ is ``known'' to
    the learner (planning), the algorithmic template in
    \cref{alg:linf_relaxation} can be applied even when $M$ is
    unknown, as long as the policy optimization step in
    \cref{line:linf_opt} can be implemented in a sample-efficient
    fashion. This observation is the basis for the algorithms we
    present in this section.
    }}

In more detail, recall the \emph{pushforward coverability} relaxation
\begin{align*}
    \pCovM[\Mstar](p) =
\sup_{\pi\in\Pi}\En^{\sMstar,\pi}\brk*{\frac{P_{h-1}^{\sMstar}(x_h\mid{}x_{h-1},a_{h-1})}{d_h^{\sMstar,p}(x_h)+\veps\cdot{}P^{\sMstar}_{h-1}(x_h\mid{}x_{h-1}a_{h-1})}}
\end{align*}
for the \mainobj objective given in \cref{sec:pushforward_relaxation}. For layer $h$, \cref{alg:model_free}
approximately minimizes this objective by computing a sequence of policies $\pi\ind{h,1},\ldots,\pi\ind{h,T}$,
where each policy 
\begin{align}
\pi\ind{h,t}\approx\argmax_{\pi\in\Pi}\En^{\sMstar,\pi}\brk*{
  \frac{\Pmstar_{h-1}(x_h\mid{}x_{h-1},a_{h-1})}{\sum_{i<t}d_h\ind{\pi\ind{h,i}}(x_h)+\Pmstar_{h-1}(x_h\mid{}x_{h-1},a_{h-1})}}
  \label{eq:mf_argmax}
\end{align}
is computed in a data-driven, online fashion that makes use of the
preceding policy covers $p_1,\ldots,p_{h-1}$. The algorithm then computes the
cover $p_h$ via
$p_h=\unif(\pi\ind{h,1}\circ_h\piunif,\ldots,\pi\ind{h,T}\circ_h\piunif)$.

Our
planning analysis in \cref{sec:pushforward_relaxation} shows that as
long as the approximation error in \cref{eq:mf_argmax} is small, $p_h$
will indeed be an approximate policy cover that minimizes
$\pCovM[\Mstar](p_h)$. To achieve this, \cref{alg:model_free} makes
use of two subroutines. The first subroutine, \estimateweight
(\cref{alg:estimate_weight}), invoked in \cref{line:estimate_weight},
uses function approximation to estimate a weight function
$\what\ind{t}_h$ such that
\begin{align}
  \label{eq:weight_est}
  \what_h\ind{t}(x_h\mid{}x_{h-1},a_{h-1})
  \approx w_h\ind{t}(x_h\mid{}x_{h-1},a_{h-1})
  \ldef{} \frac{\Pmstar_{h-1}(x_h\mid{}x_{h-1},a_{h-1})}{\sum_{i<t}d_h\ind{\pi\ind{h,i}}(x_h)+\Pmstar_{h-1}(x_h\mid{}x_{h-1},a_{h-1})}.
\end{align}
The second subroutine, \policyopt, is a hyperparameter to the algorithm,
and approximately solves the policy optimization problem
\begin{align*}
\pi\ind{h,t}\approx\argmax_{\pi\in\Pi}\En^{\sMstar,\pi}\brk*{
\what_h\ind{t}(x_h\mid{}x_{h-1},a_{h-1})},
\end{align*}
treating the estimated weight function $\what_h\ind{t}$ as a
reward. The \policyopt subroutine makes use of exploratory data
collected using preceding policy
covers $p_1,\ldots,p_{h-1}$, and hence does not have to perform
systematic exploration. Indeed, we show that any hybrid offline/online
method (that is, any online method that requires access to an
exploratory policy) that satisfies a certain ``local'' policy
optimization guarantee is sufficient, with \psdp{} (\cref{alg:psdp}; 
\citep{bagnell2003policy}) and Natural Policy Gradient
\citep{agarwal2021theory} being natural choices; for our analysis, we make use of \psdp.

\begin{algorithm}[tp]
    \setstretch{1.3}
     \begin{algorithmic}[1]
       \State \textbf{parameters}:
       \Statex[1] Weight function class $\cW=\cW_{1:H}$.
       \Statex[1] \mbox{Policy opt. subroutine
       $\policyopt_h(r_{1:h},p_{1:h},\eps,\delta)$.\hfill\algcommentlight{Optimizes reward 
         $r_h$ using policy covers $p_{1:h}$.}}
       \Statex[1] Approximation
       parameter $\veps\in\prn{0,1/2}$, failure probability $\delta\in(0,e^{-1})$.
       \State Set $T=\frac{1}{\veps}$ and $p\ind{1}=\piunif$.
       \State Set $\epsweight=c\cdot(\CpushM[\Mstar]/\abs{\cA})^{1/2}\veps^{1/2}$, $\epsopt=c'\cdot{}\veps^2$, 
       and $\deltaweight=\deltaopt=\delta/(2HT)$, for suff.
       small constants $c,c'>0$.
   \For{$h=2, \cdots, H$}
   \For{$t=1,\ldots,T$}
   \State
   $\what\ind{t}_h\gets\estimateweight_{h,t}(p_{h-1},\crl*{\pi\ind{h,i}}_{i<t};\epsweight,\deltaweight,\cW)$. \hfill
   \algcommentlight{\cref{alg:estimate_weight}.} \label{line:estimate_weight}
   \Statex\hfill\algcommentlight{Estimate for
     $w_h\ind{t}(x_h\mid{}x_{h-1},a_{h-1}) = \frac{\Pmstar_{h-1}(x_h\mid{}x_{h-1},a_{h-1})}{\sum_{i<t}d_h\ind{\pi\ind{h,i}}(x_h)+\Pmstar_{h-1}(x_h\mid{}x_{h-1},a_{h-1})}$.}
   \State \mbox{Define reward function $r\ind{h,t}$ via
   $r_{h-1}\ind{h,t}(x_{h}\mid{}x_{h-1},a_{h-1})=\what\ind{t}_h(x_h\mid{}x_{h-1},a_{h-1})$
   and $r_{h'}\ind{h,t}=0\;\;\forall{}h'\neq{}h-1$.}\label{line:mf_reward}
      \Statex\hfill\algcommentlight{$r_{h-1}\ind{h,t}$ can be interpreted as a
     stochastic reward for layer $h-1$.}
   \State $\pi\ind{h,t}\gets \policyopt_{h-1}(p_{1:h-1}, r\ind{h,t};\epsopt,\deltaopt)$.\label{line:policy_opt}
   \Statex\hfill \algcommentlight{Approximately solve $\argmax_{\pi\in\Pi}\En^{\sMstar,\pi}\brk*{ \frac{\Pmstar_{h-1}(x_h\mid{}x_{h-1},a_{h-1})}{\sum_{i<t}d_h\ind{\pi\ind{h,i}}(x_h)+\Pmstar_{h-1}(x_h\mid{}x_{h-1},a_{h-1})}}$.}
      \EndFor
      \State Set
      $p_h=\unif(\pi\ind{h,1}\circ_h\piunif,\ldots,\pi\ind{h,T}\circ_h\piunif)$.
\EndFor
\State  \textbf{return} Policy covers $(p_1,\ldots,p_H)$.
\end{algorithmic}
\caption{Coverage-Driven Exploration via Weight Function Estimation (\mfalg)}
\label{alg:model_free}
\end{algorithm}

In what follows, we describe the \estimateweight and \policyopt
subroutines and the corresponding statistical assumptions in more detail.

\paragraph{Weight function estimation and realizability}
To perform weight function estimation, we assume access to a weight function class $\cW=\cW_{1:H}$, with
$\cW_h\subseteq(\cX\times\cA\times\cX\to\bbR_{+})$ that is capable of
representing the weight function $w_h\ind{t}$ in
\cref{eq:weight_est}. While we can directly assume that the weight function in \cref{eq:weight_est} is realized by $\cW$ (cf. \cref{ass:weight_realizability_strong}), it turns out (cf. \cref{prop:weight_realizability}) that the following weaker form of weight function realizability is sufficient.
\begin{assumption}[Weight function realizability]
  \label{ass:weight_realizability_weak}
  \pacomment{removed ``weak version'' from caption} For all $h\geq{}2$ and all $\pi\in\PiNS$
  \begin{align*}
w_h^{\pi}(x'\mid{}x,a)\ldef{}\frac{\Pmstar_{h-1}(x'\mid{}x,a)}{d_h^{\sMstar,\pi}(x')}
    \in \cW_h.
  \end{align*}
\end{assumption}

\cref{ass:weight_realizability_weak}
is new to the best of our knowledge, and is naturally suited to the
pushforward coverability objective. 
While this assumption involves the
forward transition probability, it is weaker than
model-based realizability because it only requires modeling the
\emph{relative} transition probability, as the following example shows.\loose

\begin{example}\label{ex:block-mdp-example}
  For a Block MDP  \citep{du2019latent,misra2019kinematic,zhang2022efficient,mhammedi2023representation}
  with latent state space $\cS$ and decoder class
  $\Phi\subseteq(\cX\to\cS)$, we can satisfy
  \cref{ass:weight_realizability_weak} with
  $\log\abs*{\cW}\leq\bigoht\prn{\abs*{\cS}^2\abs*{\cA}+\log\abs*{\Phi}}$,\footnote{Formally
    this requires a standard covering argument; we omit the details.}
  yet any realizable model class must have
  $\log\abs*{\cM}=\bigom(\abs*{\cX})\gg\abs*{\cS}$ in general.\loose
\end{example}
See \citet{amortila2023harnessing} for further discussion around
weight function realizability in online RL.

\begin{algorithm}[tp]
    \setstretch{1.3}
     \begin{algorithmic}[1]
       \State \textbf{parameters}:
       \Statex[1] Layer $h\geq{}2$, iteration $t\in\bbN$.
       \Statex[1] Distribution $p_{h-1}\in\Delta(\PiRNS)$ , policies $\pi\ind{1},\ldots,\pi\ind{t-1}\in\PiNS$.
       \Statex[1] Error tolerance $\eps\in(0,1)$, failure probability
       $\delta\in(0,1)$.
       \Statex[1] Weight function class $\cW=\cW_{1:H}$ with $\cW_h\subseteq(\cX\times\cA\times\cX\to\brk{0,1})$.
       \State Let $n=\nweight(\eps,\delta)\ldef{}\frac{40\log(\abs{\cW}\delta^{-1})}{\eps^2}$.
       \State Let
       $q\ldef{}\frac{1}{2}p_{h-1}+\frac{1}{2(t-1)}\sum_{i<t}\pi\ind{i}\circ_{h-1}\piunif$
       if $t\geq{}1$ and $q\ldef{}p_{h-1}$ otherwise.
       \State Let $\cD_1=\cD_2=\emptyset$.
       \State For each $j\in\brk{n}$, draw $\pi\sim{}q$ and sample
       $(x\ind{j}_{h-1},a\ind{j}_{h-1},x\ind{j}_h)\sim{}\pi$. Add $(x\ind{j}_{h-1},a\ind{j}_{h-1},x\ind{j}_h)$
       to both $\cD_1$ and $\cD_2$.
       \For{$i<t$}
       \State Draw $n$ samples
       $\crl*{(x_{h-1}\ind{j},a_{h-1}\ind{j},x_h\ind{j})}_{j\in\brk{n}}$
       independently by drawing $\pi\sim{}q$ and
       $(x_{h-1}\ind{j},a_{h-1}\ind{j},x_h\ind{j})\sim\pi$.
       \State Draw $n$ samples
       $\crl*{(\tilde{x}_{h-1}\ind{j},\tilde{a}_{h-1}\ind{j},\tilde{x}_h\ind{j})}_{j\in\brk{n}}$
       by sampling
       $(\tilde{x}_{h-1}\ind{j},\tilde{a}_{h-1}\ind{j},\tilde{x}_h\ind{j})\sim\pi\ind{i}$.
       \State Add
       $\crl*{(x_{h-1}\ind{j},a_{h-1}\ind{j},x_h\ind{j})}_{j\in\brk{n}}$
       to $\cD_1$ and add
       $\crl*{(x_{h-1}\ind{j},a_{h-1}\ind{j},\tilde{x}_h\ind{j})}_{j\in\brk{n}}$
       to $\cD_2$.
       \EndFor
       \State Set $\what \ldef \argmax_{w\in\cW_h}
       \Eh_{\cD_1}\brk*{\log(w(x_h\mid{}x_{h-1},a_{h-1}))} -
       t\cdot\Eh_{\cD_2}\brk*{w(x_h\mid{}x_{h-1},a_{h-1})}$.\label{line:weight_alg}\hfill\algcommentlight{See \cref{eq:weight_estimator_log}.}
       \State  \textbf{return} $\what$.
\end{algorithmic}
\caption{Weight Function Estimation ($\estimateweight_{h,t}(p_{h-1},\crl*{\pi\ind{i}}_{i<t};\eps,\delta,\cW)$)}
\label{alg:estimate_weight}
\end{algorithm}

Our weight function estimation subroutine, \estimateweight, is given
in \cref{alg:estimate_weight}. To motivate the algorithm, consider the
following abstract setting. Let $\cZ$ be a set. We receive samples
$\cD_1=\crl{z_\mu\ind{1},\ldots,z_\mu\ind{n}}\in\cZ$
and $\cD_2=\crl{z_\nu\ind{1},\ldots,z_\nu\ind{n}}\in\cZ$, where $z_\mu\ind{t}\sim\mu\ind{t}\in\Delta(\cZ)$ and
$z_\nu\ind{t}\sim\nu\ind{t}\in\Delta(\cZ)$. The distributions $\mu\ind{t}$ and $\nu\ind{t}$ can be chosen
in an adaptive fashion based on $z_\mu\ind{1},z_\nu\ind{1},\ldots,z_\mu\ind{t-1},z_\nu\ind{t-1}$. We define
$\mu=\frac{1}{n}\sum_{t=1}^{n}\mu\ind{t}$ and
$\nu=\frac{1}{n}\sum_{t=1}^{n}\nu\ind{t}$, and our goal is to estimate
the density ratio
$\wstar(z) \ldef \frac{\mu(z)}{\nu(z)}$.
Given \emph{realizable} weight function class $\cW$ with
$w^{\star}\in\cW$, \citet{nguyen2010estimating} (see also
\citet{katdare2023marginalized}) propose the estimator
\begin{align}
  \label{eq:weight_estimator_body}
\what \ldef \argmax_{w\in\cW} \Eh_{\cD_1}\brk*{\log(w)} - \Eh_{\cD_2}\brk*{w},
\end{align}
where $\Eh_\cD\brk{\cdot}$ denotes the empirical expectation with
respect to a dataset $\cD$. We show
(\cref{thm:weight_estimator_log} in \cref{app:model_free}) that this
estimator ensures that
\begin{align}\label{eq:estimateweight-guarantee}
  \En_{z\sim\nu}\brk[\big]{\prn[\big]{\sqrt{\what(z)}-\sqrt{w^{\star}(z)}}^2}
  \approxleq \bigoh(\nrm*{w^{\star}}_{\infty})\cdot\frac{\log(\abs{\cW}\delta^{-1})}{n}
\end{align}
with high probability. \cref{alg:estimate_weight} simply applies this
technique to estimate the weight function in \cref{eq:weight_est};
\paedit{to ensure realizability of \cref{eq:weight_est} under
  \cref{ass:weight_realizability_weak}, we apply the method with an
  expanded weight function class (defined in
  \cref{prop:weight_realizability}); see \cref{sec:mf_general} for details.} \padelete{the
main guarantee for this method is given in \cref{lem:weight_function_mf} in \cref{app:model_free}}

\begin{remark}[Comparison to contrastive learning \citep{misra2019kinematic}]

We remark our weight function estimation subroutine (\cref{alg:estimate_weight}) can be viewed as a form of contrastive learning, with the target (population) weight function in Eq. \eqref{eq:weight_est} bearing strong similarity to the target (Bayes-optimal) function from the regression problem used in the \texttt{HOMER} algorithm (cf. Line 11 of Algorithm 3 and Lemma 9 of \citeauthor{misra2019kinematic}). We also note that both \mfalg and \texttt{HOMER} find policy covers inductively via policy optimization on exploratory reward functions, and thus we can view \mfalg as a natural generalization of the \texttt{HOMER} algorithm to settings beyond the Block MDP. Additionally, our analysis improves dependencies present in \texttt{HOMER}'s sample complexity (notably, the dependence on the minimal visitation probability). 
\end{remark}

\paragraph{Policy optimization subroutine}
The policy optimization subroutine, \policyopt, is a hyperparameter to the algorithm,
and can be any subroutine that approximately solves the policy optimization problem
\begin{align*}
\pi\ind{h,t}\approx\argmax_{\pi\in\Pi}\En^{\sMstar,\pi}\brk*{
\what_h\ind{t}(x_h\mid{}x_{h-1},a_{h-1})},
\end{align*}
which treats the estimated weight function $\what_h\ind{t}$ as a
reward for layer $h-1$. \policyopt does not need to perform global
policy optimization---instead, we only require a certain form of
``local'' guarantee with respect to the approximate policy covers
$p_{1},\ldots,p_{h-1}$; see \cref{ass:policy_opt} for details.

For concreteness, we make use \psdp \citep{bagnell2003policy},
described in \cref{app:policy_opt}, as the \policyopt subroutine. \psdp
optimizes a given reward function $r_{h-1}$ by collecting exploratory
data with $p_1,\ldots,p_{h-1}$ and applying approximate dynamic
programming with a user-specified value function class $\cQ$. We make the following
value-based realizability assumption, which asserts that $\cQ$ is
expressive enough to allow \psdp to optimize any weight function in
the class $\cW$.
\begin{assumption}
  \label{ass:qpi_weight}
  We have access to a value function class $\cQ=\cQ_{1:H}$ with
  $\cQ_h\subseteq(\cX\times\cA\to\brk{0,1})$ such that for all
  $h\in\brk{H}$, $w\in\cW_h$, and $\pi\in\PiNS$, we have
  $Q_\ell^{\sMstar,\pi}(\cdot,\cdot;w)\in\cQ_\ell$ for all $\ell\leq{}h-1$, where 
  \begin{align*}
  Q_{\ell}^{\sMstar,\pi}(x,a;w)=\En^{\sMstar,\pi}\brk*{
  w_h(x_h\mid{}x_{h-1},a_{h-1})\mid{}x_\ell=x,a_\ell=a
  }
  \end{align*}
  is the Q-function for $\pi$ under the (stochastic) reward function defined via
$r_{h-1}=w(x_h\mid{}x_{h-1},a_{h-1})$ and $r_{h'} = 0$ for $h' \neq h-1$. 
\end{assumption}

\subsection{Main Result}

The main guarantee for \cref{alg:model_free} applied with \psdp is as follows.
\begin{theorem}[Main result for \cref{alg:model_free} with \psdp]\label{thm:model_free_psdp}
  \label{cor:model_free_psdp}
    Let $\veps\in(0,1/2)$ and $\delta\in(0,e^{-1})$ be given, and
    suppose that we have a weight function class $\cW$ and value
    function class $\cQ$ such that \cref{ass:weight_realizability_weak,ass:qpi_weight} are satisfied.     Then with
  \psdp (\cref{alg:psdp}) as a policy optimization subroutine, \paedit{and using an expanded weight function class defined in \cref{prop:weight_realizability}}, \cref{alg:model_free}
  produces policy covers $p_1,\ldots,p_H\in\Delta(\PiRNS)$ such that with probability at least $1-\delta$, for all $h\in\brk{H}$,
\[
  \pCovM[\Mstar](p_h)   \leq{} 170H\log(\veps^{-1})\cdot{}\CpushM[\Mstar],
\]
and does so using at most 
\begin{align*}
  N \leq \bigoht\prn*{\frac{H\abs*{\cA}\log(\abs{\cW}\delta^{-1})}{\veps^4}
        +\frac{H^4\abs*{\cA}\log(\abs{\cQ}\delta^{-1})}{\veps^5}}
\end{align*}
episodes.
\end{theorem}

\cref{thm:model_free_psdp} is a special case of a more general result
(\cref{thm:model_free}) which allows for general policy optimization
algorithms that satisfy a certain ``local optimality'' guarantee
(\cref{ass:policy_opt}); we obtain the result by verifying that \psdp
satisfies this condition.  Let us discuss some key features.
\begin{itemize}
\item \emph{Sample efficiency.} The sample complexity in \cref{cor:model_free_psdp} scales as
  $\poly(H,\abs*{\cA},\log\abs{\cW},\log\abs*{\cQ},\veps^{-1})$, and
  hence is efficient for large state spaces; we
  expect that the precise polynomial factors can be tightened, as can
  the approximation ratio in the objective value, though we leave this
  for future work. As with our results in preceding sections, the
  resulting policy covers $p_1,\ldots,p_H$ can be directly used for
  downstream policy optimization.
\item \emph{Computational efficiency.} With \psdp as the \policyopt subroutine, \cref{alg:model_free}
  is computationally efficient whenever i) the weight function
  estimation objective in \cref{line:weight_alg} of
  \cref{alg:estimate_weight} can be solved efficiently over $\cW$, and
  ii) square loss regression over the class $\cQ$ can be solved
  efficiently. We consider these to be fairly mild assumptions.
\item \emph{Practicality.} In practice, we expect that the subroutine
  \policyopt can be implemented using off-the-shelf deep RL methods
  (e.g., PPO or SAC), which are known to perform well given access to
  exploratory data. In this sense, \cref{alg:model_free} can be viewed
  as a new approach to equipping existing deep RL methods with
  exploration, with the weight function-based rewards in
  \cref{line:mf_reward} acting as exploration bonuses.\loose
\end{itemize}

As a concrete example, we can instantiate \cref{cor:model_free_psdp}
for Block MDPs to derive end-to-end guarantees under standard assumptions.

\begin{example}[Sample complexity for Block MDP]
    For a Block MDP  \citep{du2019latent,misra2019kinematic,zhang2022efficient,mhammedi2023representation}
  with latent state space $\cS$ and decoder class
  $\Phi\subseteq(\cX\to\cS)$, we can satisfy
  \cref{ass:weight_realizability_weak} with
  $\log\abs*{\cW}\leq\bigoht\prn{\abs*{\cS}^2\abs*{\cA}+\log\abs*{\Phi}}$
  and $\log\abs*{\cQ}\leq\bigoht\prn{\abs*{\cS}\abs*{\cA}+\log\abs*{\Phi}}$,
  so \cref{cor:model_free_psdp} gives sample complexity
  $N \leq \bigoht\prn*{\frac{H^4\abs*{\cS}^2\abs*{\cA}^2\log(\abs{\Phi}\delta^{-1})}{\veps^5}}$.
\end{example}

We view the results in this section as a proof of concept, showing
that the planning methods derived in \cref{sec:planning} for iteratively optimizing \mainobj 
can be implemented in a data-driven
fashion. We leave sample-efficient counterparts for other relaxations
for future work.

}

\arxiv{
\section{\mainobj: Structural Properties}
\label{sec:structural}
In this section we draw connections between \mainopt, the structural parameter induced by \mainobj, and other structural parameters and objectives, focusing on (i) a non-admissible variant of \mainopt (\cref{sec:non_admissible}),  and (ii) feature coverage (\cref{sec:feature_coverage}). We also show that alternative exploration objectives \emph{do not} induce meaningful structural parameters in the same fashion as \mainobj (\cref{app:insufficient}).

\subsection{Connection to Non-Admissible \mainobj}
\label{sec:non_admissible}

Inspired by \cref{eq:cinf}, we can also define a non-admissible counterpart
to \mainopt as follows.
  \begin{align}
    \label{eq:cone}
    \ConehM= \inf_{\mu\in\Delta(\cX\times\cA)}  \sup_{\pi\in\Pi}\En^{\sM,\pi}\brk*{\frac{d_h^{\sM,\pi}(x_h,a_h)}{\mu(x_h,a_h)}},
  \end{align}
  and $\ConeM=\max_{h\in\brk{H}}\ConehM$. This quantity was used to
  provide sample complexity bounds for online reinforcement learning by
  \citet{liu2023provable} (generalizing the results of
  \citet{xie2023role}), and the associated concentrability coefficient $\ConehM(\mu)= \sup_{\pi\in\Pi}\En^{\sM,\pi}\brk*{\frac{d_h^{\sM,\pi}(x_h,a_h)}{\mu(x_h,a_h)}}$
  is widely used in offline reinforcement
  learning \citep{farahmand2010error,xie2020q}. The following result, which
  uses the minimax theorem in a similar fashion to \cref{prop:linf}, shows
  that it is possible to bound \mainopt by this quantity in spite of non-admissibility, albeit with
  some loss in rate.
\begin{restatable}{proposition}{lone}
  \label{prop:lone}
    For all $\veps>0$, it holds that $        \CovOptM \leq{} 1 + 2\sqrt{\frac{\ConehM}{\veps}}$.
\end{restatable}
Note that while the bound in \cref{prop:lone} grows as
$\sqrt{\frac{1}{\veps}}$, it still leads to non-trivial sample
complexity bounds through our main results (which give meaningful guarantees whenever $\CovOptM$ grows
sublinearly with $\veps^{-1}$), though the resulting rates
are worse than in the case where $\CovOptM$ is bounded by an absolute
constant.\loose

\subsection{Connection to Feature Coverage}
\label{sec:feature_coverage}
Another well-studied notion of coverage from offline reinforcement
learning is \emph{feature coverage} \citep{jin2021pessimism,zanette2021provable,wagenmaker2023leveraging}. Consider the Low-Rank MDP
framework
\citep{RendleFS10,YaoSPZ14,agarwal2020flambe,modi2021model,zhang2022efficient,mhammedi2023efficient},
in which the transition distribution is assumed to factorize as
\begin{align}
  \label{eq:low_rank}
    P^{\sM}_{h-1}(x_{h}\mid{}x_{h-1},a_{h-1}) = \tri*{\phi_{h-1}(x_{h-1},a_{h-1}),\psi_h(x_h)},
\end{align}
where $\phi_{h-1}(x,a),\psi_h(x)\in\bbR^{d}$ are (potentially unknown)
feature maps. For offline reinforcement learning in Low-Rank MDPs \citep{jin2021pessimism,zanette2021provable,wagenmaker2023leveraging}, feature coverage for a distribution
$\mu\in\Delta(\cX\times\cA)$ is given by $\CfeathM(\mu) =
\sup_{\pi\in\Pi}\nrm*{\En^{\sM,\pi}\brk*{\phi_{h}(x_h,a_h)}}^2_{\Sigma_\mu^{-1}}$. We
define the associated feature coverability coefficient by
\begin{align}
  \label{eq:feature_coverability}
\CfeathM = \inf_{\mu\in\Delta(\cX\times\cA)}  \sup_{\pi\in\Pi}\nrm*{\En^{\sM,\pi}\brk*{\phi_{h}(x_h,a_h)}}^2_{\Sigma_\mu^{-1}},
\end{align}
where
$\Sigma_\mu\ldef{}\En_{(x,a)\sim\mu}\brk*{\phi_h(x,a)\phi_h(x,a)^{\trn}}$,
and define $\CfeatM=\max_{h\in\brk{H}}\CfeathM$. Note that one always
has $\CfeatM\leq{}d$, as a consequence of the existence of G-optimal
designs \citep{kiefer1960equivalence,huang2023reinforcement}. The
following result shows that \mainopt is always bounded by feature coverability.
\begin{restatable}{proposition}{featurecoverage}
  \label{prop:feature_coverage}
  Suppose the MDP $M$ obeys the low-rank structure in
  \cref{eq:low_rank}. Then for all $h\in\brk{H}$, we have $\ConehM\leq\abs{\cA}\cdot\CfeathoM$,
and consequently $    \CovOptM \leq{} 1 + 2\sqrt{\frac{\abs{\cA}\cdot\CfeathoM}{\veps}}$.
\end{restatable}
We remark that if we restrict the distribution
$\mu\in\Delta(\cX\times\cA)$ in \cref{eq:feature_coverability} to be
realized as a mixture of occupancies, \cref{prop:feature_coverage} can
be strengthened to give $\CovOptM \leq{}
\bigoh\prn[\big]{\abs{\cA}\cdot\CfeathoM}$. Note that through our main
results, \cref{prop:feature_coverage}, gives guarantees that scale
with $\abs{\cA}$. This is necessary in the setting where the feature
map $\phi_h$ is unknown, which is covered by our results, but is suboptimal in the setting where
$\phi$ is known.

\subsection{Insufficiency of Alternative Notions of Coverage}
\label{app:insufficient}

\mainobj measures coverage of the mixture policy $p\in\Delta(\PiRNS)$
with respect to the $L_1(d_h^{\sM,\pi})$-norm. It is also reasonable to
consider variants of the objective based on $L_q$-norms for $q>1$,
which provide stronger coverage, but may have larger optimal value. In
particular, a natural $L_{\infty}$-type analogue of \cref{eq:mainobj}
is given by
\begin{align}
  \label{eq:linf_realizable}
  \ICovM(p) =
  \sup_{\pi\in\Pi}\sup_{(x,a)\in\cX\times\cA}\crl*{\frac{d^{\sM,\pi}_h(x,a)}{d_h^{\sM,p}(x,a)+\veps\cdot{}d_h^{\sM,\pi}(x,a)}}.
\end{align}
is identical to the $L_{\infty}$-coverability
coefficient $\CinfM$ \eqref{eq:cinf} studied in
\citet{xie2023role}, except that we restrict the data distribution
$\mu\in\Delta(\cX\times\cA)$ to be admissible, i.e. realized by a mixture policy
$p\in\Delta(\PiRNS)$ (we also incorporate the term
$\veps\cdot{}d_h^{\sM,\pi}(x,a)$ in the denominator to ensure the
ratio is well-defined). The following lemma shows, perhaps
surprisingly, that in stark contrast to \mainobj, it is not possible
to bound the optimal value of the admissible $L_{\infty}$-Coverage objective
in \cref{eq:linf_realizable} in terms of the non-admissible
coverability coefficient $\CinfM$.
\begin{restatable}{proposition}{linfimpossibility}
  \label{prop:linf_impossibility}
  There exists an MDP $M$ and policy class
  $\Pi\subset\PiRNS$ with horizon $H=1$ such that $\CinfhM\leq{}2$ (and hence
  $\CovOptM\leq{}2$ as well), yet for all
  $\veps>0$,
  \begin{align}
    \label{eq:linf_impossibility}
    \inf_{p\in\Delta(\Pi)}\ICovM(p) \geq \frac{1}{\veps},
  \end{align}
  and in particular $\inf_{p\in\Delta(\Pi)}\obj_{\infty;h,0}^{\sM}(p)=\infty$.
\end{restatable}
Note that one trivially has $\inf_{p\in\Delta(\Pi)}\ICovM(p)\leq\frac{1}{\veps}$, and hence
\cref{eq:linf_impossibility} shows that the optimal value is vacuously large for the MDP in this example. In contrast, we
have $\CovOptM\leq{}2$ even when $\veps=0$. \padelete{This highlights that
average-case notions of coverage are critical when our goal is to
build a policy cover (i.e., a concrete ensemble of policies that
covers the state space).} More generally, one can show that similar
failure modes hold for admissible $L_q$-Coverage for any $q > 1$.

}

\section{Experimental Evaluation}
\label{sec:experiments}
We present proof-of-concept experiments to
\paedit{validate our theoretical results}.\footnote{Code \arxiv{is }available at \arxiv{\url{https://github.com/philip-amortila/l1-coverability}}\icml{\href{https://github.com/philip-amortila/l1-coverability}{github.com/philip-amortila/l1-coverability}}.} We focus on the \emph{planning} problem
(\cref{sec:planning}), and consider the classical \mountaincar
environment~\citep{brockman2016openai}. We optimize the \mainobj objective via
\cref{alg:linf_relaxation} and compare this to the \maxent
algorithm \citep{hazan2019provably} and uniform exploration. We
find that \mainobj %
explores the
state space more quickly and effectively than the baselines.

\icml{\fakepar{Experimental setup}}\arxiv{\paragraph{Experimental setup}}
\mountaincar is a continuous domain with two-dimensional states and actions\arxiv{ \akedit{in which an agent must use momentum to navigate a car out of a valley}}. We consider a deterministic starting
state near the bottom of the ``valley'' in the environment, so
that deliberate exploration is required.

We optimize \mainobj using \cref{alg:linf_relaxation}, approximating
the occupancies $d^{\pi\ind{i}}$ via count-based estimates on a
discretization of the state-action space. We set $\mu$ to the
uniform distribution and $\Cinf$ to the number of state-action pairs
(in the discretization).
For \maxent, we use the implementation
from \citet{hazan2019provably}. \cref{alg:linf_relaxation} and \maxent
both require access to a subroutine for reward-driven planning
(%
to solve the %
problem in \cref{eq:linf_opt} or %
an analogous
subproblem). For both, we use a policy gradient
method (\reinforce; \citet{sutton1999policy}) as the reward-driven
planner, following \citet{hazan2019provably}; our policy class is parameterized as a two layer neural network.
For details on the environment, architectures, and hyperparameters, see \cref{app:exp-details}.

\icml{\fakepar{Results}}\arxiv{\paragraph{Results}}
Results are visualized in \cref{fig:exp-results}. We measure
the number of unique \paedit{discretized} states discovered by each algorithm's policy covers, and
find that \mainobj (\cref{alg:linf_relaxation}) outperforms the \maxent and uniform exploration
baselines by a
factor of two.\arxiv{\footnote{\paedit{Interestingly, while it reaches fewer states, the \maxent baseline finds a policy cover with
    higher entropy than \mainobj, indicating that entropy may
    not be the best proxy for exploration.}}} We also perform a
qualitative comparison by visualizing the occupancies induced by the learned policy
covers, and find that the cover obtained with
\mainobj explores a much larger portion of the state space than
the baselines; notably, \mainobj explores both hills in the
\mountaincar environment, while the baselines fail to do so. We find
these results promising, and plan to perform a
large-scale evaluation in future work; see
\cref{app:exp-details} for further experimental results.

\arxiv{
\begin{figure}[tp]
\centering
\begin{subfigure}{.37\textwidth}
\includegraphics[width=\linewidth]{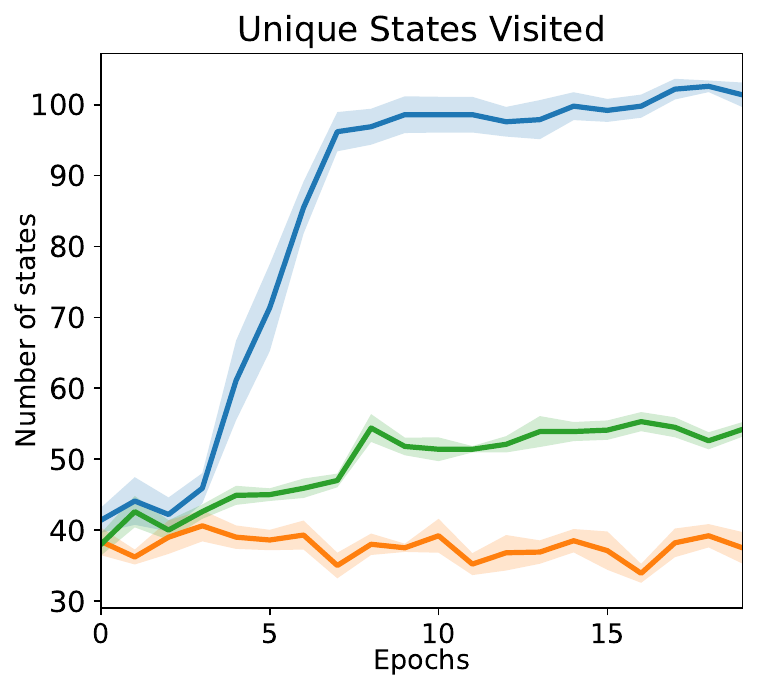}
  \label{fig:states-plot}
\end{subfigure}
\hspace{1.55cm}
\begin{subfigure}{.33\textwidth}
    \raisebox{0cm}{\includegraphics[width=\linewidth]{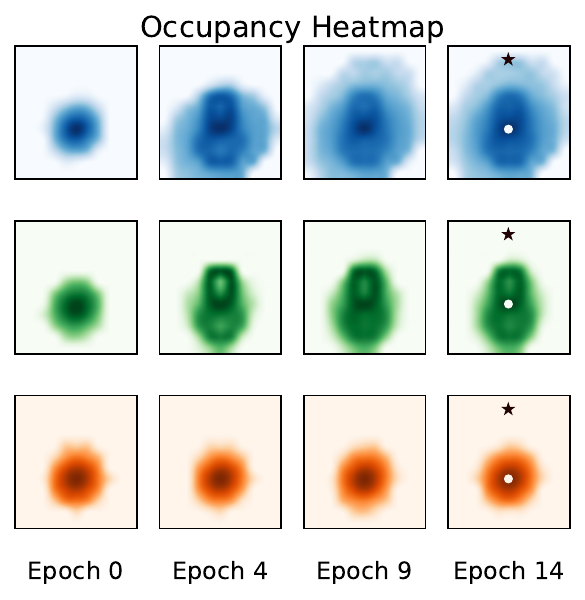}}
  \label{fig:heatmaps}
\end{subfigure}\\
\vspace{-.4cm}
\begin{subfigure}{\linewidth}
\centering
\hspace{0.75cm}
\includegraphics[width=.35\linewidth]{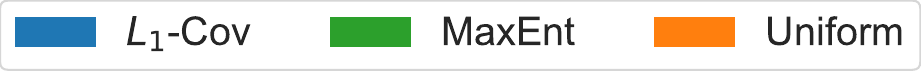}
\label{fig:sub3}
\end{subfigure}
\caption{
 Number of unique discrete states visited (mean/standard error over 10 runs) and
 occupancy heatmaps
for each policy cover obtained by \mainobj
 (\cref{alg:linf_relaxation}), \maxent, and uniform exploration. 
Each epoch
 comprises a single policy update in \cref{alg:linf_relaxation}
 and \maxent, obtained through 1000 steps of \reinforce with rollouts of length 400. Heatmap legend: velocity (x-axis), position (y-axis), start state ($\bullet$), goal state with $0$ velocity ($\star$).}
\label{fig:exp-results}
\end{figure}
}
\icml{
\begin{figure}[htp]
\centering
\begin{subfigure}{.27\textwidth}
\includegraphics[width=\linewidth]{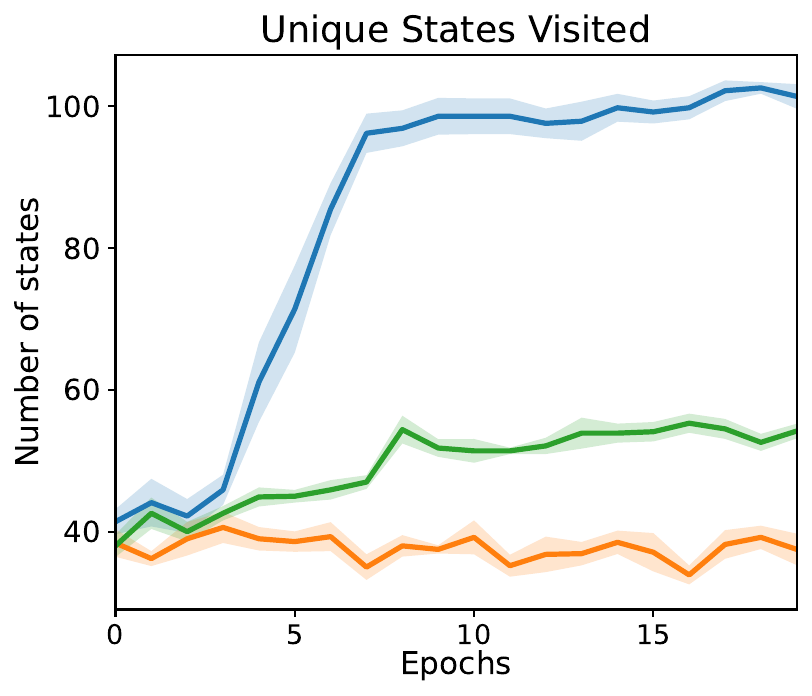}
  \label{fig:states-plot}

\end{subfigure}
\begin{subfigure}{.2\textwidth}
    \raisebox{0.144cm}{\includegraphics[width=\linewidth]{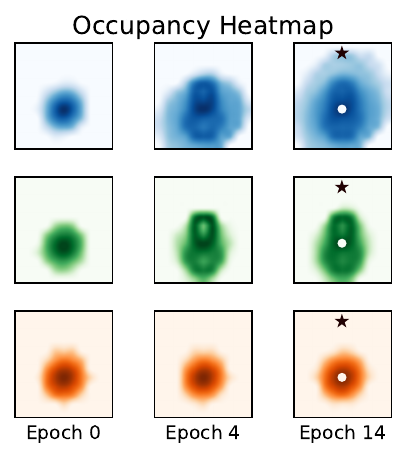}}
  \label{fig:heatmaps}
\end{subfigure}\\
\vspace{-.4cm}
\begin{subfigure}{\linewidth}
\centering
\includegraphics[width=.6\linewidth]{legend_all.pdf}
\label{fig:sub3}
\end{subfigure}
\caption{
Unique discrete states visited (mean/standard error over 10 runs) and
 occupancy heatmaps
for each policy cover obtained by \mainobj
 (\cref{alg:linf_relaxation}), \maxent, and uniform actions. 
Each epoch comprises one policy cover update\icml{.}\arxiv{in \cref{alg:linf_relaxation} and \maxent,  obtained through 1000 steps of \reinforce with rollouts of length 400.}
 Heatmap legend: velocity (x-axis), position (y-axis), start state ($\bullet)$, goal state \arxiv{(with $0$ velocity)} 
 ($\star$)\icml{.} 	%
 	}
\label{fig:exp-results}
\end{figure}
}

\section{Discussion and Open Problems}
\label{sec:discussion}
Our results show that the \mainobj objective serves as a scalable
primitive for exploration in reinforcement learning, providing
sample efficiency, computational efficiency, and flexibility. On the
theoretical side, our results raise a number of interesting questions
for future work: \icml{(i) What are the weakest representation conditions under which coverability
  leads to computationally efficient online reinforcement learning
  guarantees? (ii) Can we generalize the \mainobj objective further while still
  allowing for practical/computationally efficient optimization?}
\arxiv{
    \begin{itemize}
\item What are the weakest representation conditions under which coverability
  leads to computationally efficient online reinforcement learning
  guarantees? \arxiv{Our results in \cref{sec:model_free} show that
  weight function realizability suffices, but it would be interesting
  to see how far this assumption can be relaxed.}
\item Can we generalize the \mainobj objective further while still
  allowing for practical/computationally efficient optimization? \arxiv{For
  example, it is natural to
  consider notions of coverage that reflect the structure of a value
  function class \citep{xie2021bellman} for reward-driven RL.}
\end{itemize}
}
\arxiv{
On the empirical side, we plan to evaluate whether out-of-the box deep
reinforcement learning methods (e.g., PPO) equipped with the \mainobj
objective can be competitive
with state-of-the-art online exploration techniques
\citep{ecoffet2019go,ecoffet2021first,badia2020never,badia2020agent57,guo2022byol}.}

\icml{
\fakepar{Additional results}
Several additional results are deferred to the appendix due to space
constraints. \cref{sec:model_free} builds on \cref{sec:planning,sec:model_based} to
    provide efficient algorithms for reward-free exploration under a weaker form of \emph{model-free} function
    approximation based on density ratio realizability, and
    \cref{sec:structural} provides connections between \mainopt and
    other intrinsic structural properties.\loose
}

\clearpage
\icml{
\section*{Impact Statement}
This paper presents work whose goal is to advance the theoretical
foundations of reinforcement learning. There are many potential societal consequences of our work, none which we feel must be specifically highlighted here.
}
\bibliography{refs} 
\icml{
\bibliographystyle{icml2024}
}

\clearpage

\appendix  

\icml{
\onecolumn
}

\arxiv{
\renewcommand{\contentsname}{Contents of Appendix}
\addtocontents{toc}{\protect\setcounter{tocdepth}{2}}
{\hypersetup{hidelinks}
\tableofcontents
}
}

\icml{
\renewcommand{\contentsname}{Contents of Appendix}
\addtocontents{toc}{\protect\setcounter{tocdepth}{2}}
{\hypersetup{hidelinks}
\tableofcontents
}
}

\clearpage

\section{Additional Related Work}
\label{app:additional}

This section discusses additional related work not already covered in detail.

\paragraph{Theoretical objectives for exploration}
On the computational side, our work can be viewed as building on a recent line of research that tries to
make exploration computationally efficient by inductively building policy covers
and using them to guide exploration. Up to this point, Block MDPs \citep{du2019latent,misra2019kinematic,zhang2022efficient,mhammedi2023representation} and Low-Rank
MDPs
\citep{agarwal2020flambe,modi2021model,uehara2022representation,zhang2022efficient,mhammedi2023representation}
are the most general classes that have been addressed by this line of
research. Our work expands the scope of problems for which efficient
exploration is possible beyond these classes to include the more
general coverability framework. This is meaningful generalization, as it allows for nonlinear transition
dynamics for the first time.

The \mainobj objective can be viewed as a generalization of
\emph{optimal design}, an exploration objective which has previously
been studied in the context of tabular and Low-Rank MDPs
\citep{wagenmaker2022beyond,wagenmaker2022instance,li2023reward,mhammedi2023efficient};
\cref{sec:feature_coverage} makes this connection explicit. In particular, \citet{li2023reward} consider an optimal design objective for tabular MDPs which is equivalent to \cref{eq:linf_relaxation} with $\mu=\unif(\cX)$. Outside of
this paper, the only work we are aware of that considers optimal
design-like objectives for nonlinear settings beyond Low-Rank MDPs is
\citet{foster2021statistical}, but their algorithms are not efficient
in terms of offline estimation oracles.

Lastly, a theoretical exploration objective worth mentioning is maximum
entropy exploration \citep{hazan2019provably,jin2020reward}. Existing
theoretical guarantees for this objective are limited to tabular MDPs,
and we suspect the objective is not sufficiently powerful to allow for
downstream policy optimization in more general classes of MDPs.

\paragraph{Coverability}
Compared to previous guarantees based on coverability
\citep{xie2023role,liu2023provable,amortila2023harnessing} our guarantees have
somewhat looser sample complexity and require stronger function
approximation. However, these works only consider reward-driven
exploration and are computationally inefficient. A second
contribution of our work is to show that \mainopt, which can be
significantly weaker than $L_{\infty}$-Coverability, is sufficient for
sample-efficient online reinforcement learning on its
own.\footnote{\citet{liu2023provable} also provide guarantees based on
  \mainopt, but their regret bounds contain a lower-order term that
  can be as large as the number of states in general.}

\paragraph{Instance-dependent algorithms and complexity}
\citet{wagenmaker2022beyond,wagenmaker2022instance} provide
instance-dependent regret bounds for tabular MDPs and linear MDPs that
scale with problem-dependent quantities closely related to
\mainopt. These results are tailored to the linear setting, and their
bounds contain lower-order terms that scale explicitly with the
dimension and/or number of states.

\paragraph{General-purpose complexity measures}
A long line of research provides structural complexity measures that
enable sample-efficient exploration in reinforcement learning
\citep{russo2013eluder,jiang2017contextual,sun2019model,wang2020provably,du2021bilinear,jin2021bellman,foster2021statistical,xie2023role,foster2023tight}. 
Coverability is incomparable to most prior complexity measures
\citep{xie2023role}, but is subsumed by the Decision-Estimation
Coefficient \citep{foster2021statistical}, as well as the (less general) Sequential Extrapolation
Coefficient and related complexity measures
\citep{liu2023maximize}. The main contrast between these works and our
own is that they do not provide computationally efficient algorithms
in general.

A handful of works extend the approaches above to accommodate
reward-free reinforcement learning, but are still computationally
inefficient, and do not explicitly suggest exploration objectives \citep{chen2022statistical,xie2023role,chen2022unified}.

\paragraph{Further related work}
Coverability is closely related to a notion of \emph{smoothness} used
in a line of recent work on smoothed online learning and related
problems
\citep{haghtalab2020smoothed,haghtalab2022oracle,haghtalab2022smoothed,block2022smoothed,block2023sample,daskalakis2023smooth}. We
are not aware of explicit technical connections between our techniques
and these works, but it would be interesting to explore this in more detail.

\section{Experimental Evaluation: Details and Additional Results}\label{app:exp-details}

\dfc{Explain how we evaluate unique states visited (incl number of samples we draw to evaluate) and why it's not monotonic.}

\subsection{Experimental Details}\label{sec:exp-details}
\paragraph{\maxent baseline}  We implement our algorithm by building
on top of the codebase for \maxent from \citet{hazan2019provably}. The
\maxent algorithm follows a similar template to ours, in that it 
iteratively defines a reward function based on past occupancies and
optimizes it to find a new exploratory policy. In \maxent, the reward function is defined as 
\[
	r_t = \frac{d\Dkl{\unif}{X}}{d X} \Big|_{X = \wh{d}^{p\ind{t}}},
\]
where $\wh{d}^{p\ind{t}}$ is the estimated occupancy for the policy
cover $p\ind{t}$ at time $t$ \akedit{ and $\unif$ denotes the uniform distribution over $\cX\times\cA$}. We use the same neural
network architecture (described in detail below) to represent policies, the same algorithm for
approximate planning, the same discretization scheme to approximate
occupancies, and the exact same hyperparameters for discretization
resolution, number of training steps, length of rollouts, learning
rates, and so on. We found that our method worked ``out-of-the-box''
without any modifications to their architecture or hyperparameters. 

Starting from their codebase, we obtain an implementation of \cref{alg:linf_relaxation} simply by modifying the reward function given to the planner from the \maxent reward function to the \mainobj reward function (\cref{line:linf_opt} in \cref{alg:linf_relaxation}). \akdelete{The only substantial difference between our implementation and theirs is that we modify the \mountaincar environment to have a deterministic start state, which requires more deliberate exploration to find a high-quality cover.}
 
\paragraph{Environment}

We evaluate on the \texttt{MountainCarContinuous-v0} environment
(henceforth simply \mountaincar) from the OpenAI Gym \citep{brockman2016openai}. The state space of the environment is two-dimensional, with a position value (denoted by $\xi$) that is in the interval $[-1.2,0.6]$ and a velocity value (denoted by $\rho$) that is in the interval $[-0.07, 0.07]$. The dynamics are deterministic and defined by the physics of a sinusoidal valley, we refer the reader to the documentation for the precise equations.\footnote{\url{https://www.gymlibrary.dev/environments/classic_control/mountain_car_continuous/}} The goal state (the ``flag'') is at the top of the right hill, with a position $\xi=+0.45$. The bottom of the valley corresponds to the coordinate $\xi=-\pi/6 \approx -0.52$. We modify the environment \akdelete{so that the starting distribution is }\akedit{to have }a deterministic starting state with a position of $\xi = -0.5$ and a velocity of $\rho = 0$\akedit{, so that more deliberate exploration is required to find a high-quality cover}. This means that occupancies and covers are evaluated when rolling out from this deterministic starting state. The action space is continuous in the interval $[-1,1]$, with negative values corresponding to forces applied to the left and positive values corresponding to forces applied to the right. To simplify, we will consider that the environment only has 3 actions, namely we only allow actions in $\{-1, 0, 1\}$. Finally, we take a horizon of $H=200$, meaning that we terminate rollouts after $200$ steps (if the goal state has not been reached yet).

\paragraph{Implementation of \cref{alg:linf_relaxation}} We make a few changes to \cref{alg:linf_relaxation}. Firstly, while \mountaincar has a time horizon and is thus
  non-stationary, we approximate the dynamics as stationary. Thus, we replace $d^{M,\pi}_h$ in \cref{line:linf_opt} of \cref{alg:linf_relaxation} by a stationary analogue $d^{M,\pi}_\mathtt{stat}$, and only invoke \cref{alg:linf_relaxation} once rather than at each layer $h$. Namely, we define 
  \[
  	d^{M,\pi}_\mathtt{stat} \coloneqq \frac{1}{H}\sum_h d^{M,\pi}_h,
  \]    
  where $H=200$ is the horizon which we define for \mountaincar. We will henceforth simply write $d^\pi \coloneqq d^{M,\pi}_\mathtt{stat}$ for compactness. Similarly, we choose a stationary distribution $\mu$ (to be described shortly) to pass as input to the algorithm. Secondly, rather than solving the policy optimization problem for the reward function 
  \[
  	r(x,a)  = \frac{\mu(x,a)}{\sum_{i<t}d^{\sM,\pi\ind{i}}(x,a)+\Cinf{}\mu(x,a)}
  \] 
  as written \cref{line:linf_opt}, we instead apply some regularization and solve the policy optimization problem for the more general \mainobj-based reward function defined in \cref{eq:linf_relaxation}, namely
  \begin{equation}\label{eq:eps-reward}
  	r(x,a) = \frac{\mu(x,a)}{\sum_{i<t}d^{\sM,\pi\ind{i}}(x,a)+ \varepsilon \Cinf{}\mu(x,a)},
  \end{equation}
  for a small parameter $\varepsilon>0$. This has the effect of
  magnifying the reward difference between visited and unvisited
  states. We also (linearly) renormalize so that this reward function
  lies in the interval $[0,1]$.  We found that regularizing with
  $\varepsilon$ helps our approximate planning subroutine (discussed shortly) solve the policy optimization and recover a better policy for the \mainobj-based reward function.
  
\paragraph{Approximating $d^\pi$ and $\mu$} To approximate $d^\pi$ and choose $\mu$, we define a discretized state space for \mountaincar. Namely, we discretize the position space $[-1.2, 0.6]$ in $12$ evenly-spaced intervals and the velocity space $[-0.07,0.07]$ in $11$ evenly-spaced intervals. This gives a tabular discretization with $12 \times 11 = 132$ states. We then approximate $d^\pi(x,a)$ via a simple count-based estimate on the discretized space. That is, letting the discretized space be denoted by $\cB = \{ b_1, \ldots, b_{132}\}$ where each $b_i$ is a bin, and given trajectories $(x\ind{n}_h, a\ind{n}_h)_{n \in [N], h \in [H]}$, we estimate 
	\[
		d^\pi_{\texttt{disc}}(b_i,a) \coloneqq \frac{1}{H} \sum_{h=1}^H 		d^\pi_{\texttt{disc}, h}(b_i,a), \quad \text{ where } \quad d^\pi_{\texttt{disc}, h}(b_i,a) = \frac{1}{N} \sum_{n=1}^N \indic\{x\ind{n}_h \in b_i, a\ind{n}_h = a \}
	\]
	and then for any $x,a$ we assign
	\[
		d^\pi(x,a) = d^\pi_{\texttt{disc}}(b_{i_x}, a),
	\]	
	where $i_x$ is the index of the bin $b_i$ in which state $x$ lies. For the distribution $\mu$ and the $L_\infty$ coverability parameter $\Cinf$, we take $\mu$ to be uniform over the discretized state-action space, that is we take 
	\[
		\mu(x,a) = \frac{1}{|\cB||\cA|} = \frac{1}{396},
	\]	
	and 
	\[
		\Cinf = |\cB||\cA| = 396.
	\]
	
\paragraph{Approximate planner} As an approximate planner, we take a simple implementation of the \reinforce algorithm \citep{sutton1999policy} given in the PyTorch \citep{paszke2019pytorch} GitHub repository\footnote{\url{https://github.com/pytorch/examples/blob/main/reinforcement_learning/reinforce.py}}, which only differs from the classical \reinforce in that it applies variance-smoothing (with some parameter $\sigma$) to the returns. When solving the policy optimization problem, we allow \reinforce to use the original stochastic reset distribution for \mountaincar, that is the reset distribution which samples $\xi \in [-0.6,-0.4]$ uniformly and has $\rho = 0$. 
\paragraph{Architecture, optimizer, hyperparameters} The policy class we use, and which \reinforce optimizes over, is obtained by a set of fully-connected feedforward neural nets with ReLU activation, 1 hidden layer, and 128 hidden units. The input is 2-dimensional (corresponding to the 2-dimensional state space of \mountaincar) and the output is a 3-dimensional vector; we obtain a distribution over the action set $\{-1,0,1\}$ by taking a softmax over the output. We use Xavier initialization \citep{glorot2010understanding}.

For \reinforce, we take a discount factor of $0.99$, and a variance smoothing parameter of $\sigma = 0.05$. We train \reinforce with horizons of length $400$. We take $\pi\ind{t}$, the policy which approximates \cref{line:linf_opt} of \cref{alg:linf_relaxation}, to be the policy returned after $1000$ \reinforce updates, with one update after each rollout. The update in \reinforce use the Adam optimizer \citep{kingma2014adam} with a learning rate of $10^{-3}$.

We estimate all occupancies with $N=100$ rollouts of length $H=200$. We calculate the mixture occupancies $d^p$, for $p = \unif(\pi\ind{1}, \ldots,\pi\ind{t})$, by estimating the occupancy for each $d^{\pi_i}$ separately and averaging via
	\[
		d^p(x,a) = \frac{1}{t} \sum_{i=1}^t d^{\pi\ind{i}}(x,a).
	\]	
\akdelete{For the discretization resolution, we divide the position interval into 12 bins and the velocity interval into 11 bins.} 
We train for 20 epochs, corresponding to $T=20$ in the loop of \cref{line:linf_outer_loop} of \cref{alg:linf_relaxation}. For the regularized reward of \cref{eq:eps-reward}, we take $\varepsilon = 10^{-4}$.

\subsection{Additional Experimental Results} 

\subsubsection{MountainCar}

In addition to the results of \cref{sec:experiments}, in \cref{fig:more-exp-results} we report the
entropy and \mainopt of the state distributions for the policy
covers found by the three algorithms
throughout training. Namely, for each cover $p$, we estimate its occupancy $d^p$ via the procedure defined above, and measure the entropy of our estimate for $d^p$. In the second plot, we take the estimate of $d^p$ and measure its objective value $\muCovM(p)$, where $\muCovM(p)$ is the $L_\infty$-Coverability relaxation of \mainopt which we recall is defined via
\begin{align}
  \label{eq:linf_relaxation_appendix}
  \muCovM(p) =
    \sup_{\pi\in\Pi}\En^{\sM,\pi}\brk*{\frac{\mu(x_h,a_h)}{d_h^{\sM,p}(x_h,a_h)+\veps\cdot{}\Cinf\mu(x_h,a_h)}}.
 \end{align}
We measure this quantity with $\veps = 10^{-4}$, and in the same way that we approximate \cref{line:linf_opt} in \cref{alg:linf_relaxation}, namely by calling \reinforce with a stochastic starting distribution as an approximate planner, with the same parameters and architecture. We note that there are two sources of approximation error in our measurements of $\muCovM(p)$, namely the estimation error of $d^p$ (via a count-based estimate on the discretized space) and the optimization error of \reinforce, and thus values we report should be taken as approximations of the true \mainopt values.\footnote{For instance, notice the variability over epochs of the uniform random baseline, which has a constant $\muCovM(p)$ value.}

 \begin{figure}[tp]
\centering
\begin{subfigure}{.37\textwidth}
\includegraphics[width=\linewidth]{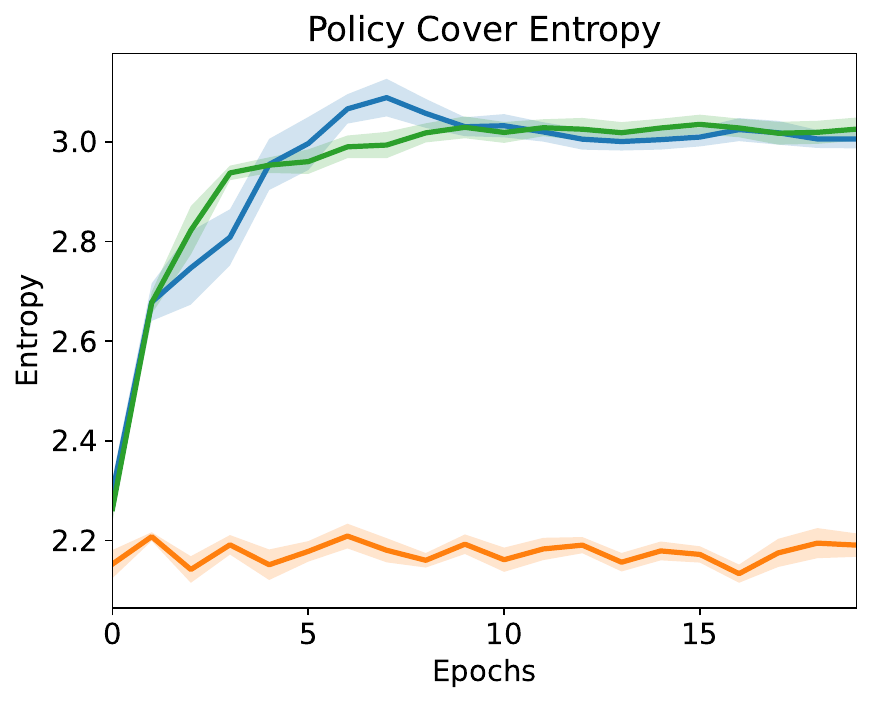}
  \label{fig:ent-plot}
\end{subfigure}
\hspace{1cm}
\begin{subfigure}{.37\textwidth}
\includegraphics[width=\linewidth]{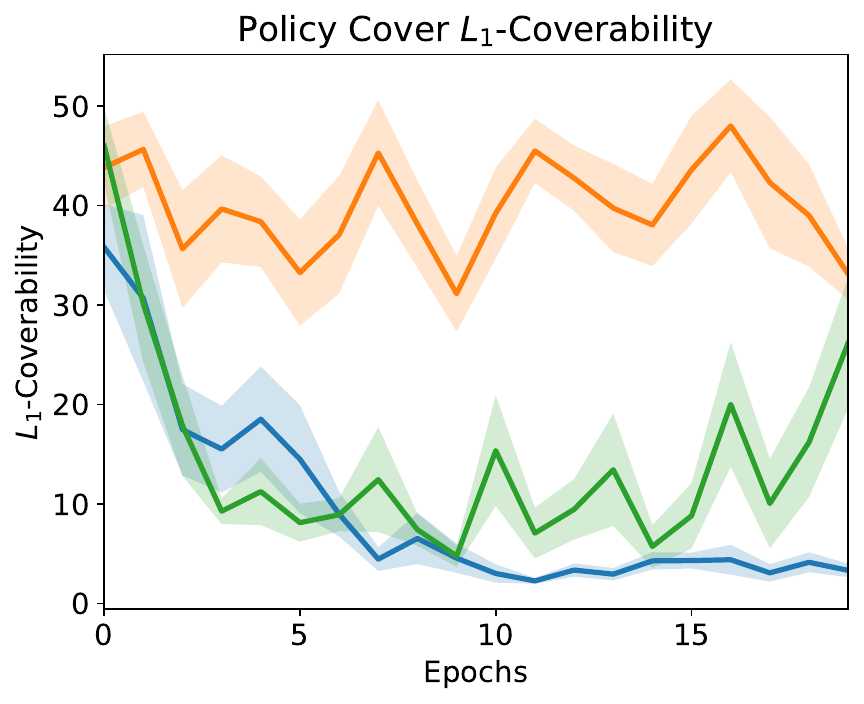}
  \label{fig:ent-plot}
\end{subfigure}
\begin{subfigure}{\linewidth}
\vspace{-0.4cm}
\centering
\hspace{0.75cm}
\includegraphics[width=.35\linewidth]{legend_all.pdf}
\label{fig:sub3}
\end{subfigure}
\caption{Entropy and \mainopt measured on each policy cover obtained from \mainobj (\cref{alg:linf_relaxation}), \maxent, and uniform exploration on the \mountaincar environment. %
	We plot the mean and standard error across 10 runs. Each epoch corresponds to a single policy update in \cref{alg:linf_relaxation}
 and \maxent, obtained through 1000 steps of \reinforce with rollouts of length 400.}
\label{fig:more-exp-results}
\end{figure}

\paragraph{Results}
We find that \mainobj and \maxent recover policy covers with similar entropy values, with the entropy of the \maxent cover being slightly larger (despite the \maxent cover visiting fewer states, as seen in our results in \cref{sec:experiments}). \icml{Interestingly, the \maxent baseline has higher entropy while visiting fewer states (as seen in our results in \cref{sec:experiments}), indicating
that entropy may not be the best proxy for exploration.} For \mainobj,
we find that \cref{alg:linf_relaxation} attains the smallest \mainobj
values, indicating a better cover\akdelete{, although these values are only an
approximation the true \mainobj objective due to the use of \reinforce
as an approximate planner for solving the maximization problem in \cref{eq:linf_relaxation_appendix}}. 

\subsubsection{Pendulum}

\begin{figure}[tp]
\centering
\begin{subfigure}{.37\textwidth}
\includegraphics[width=\linewidth]{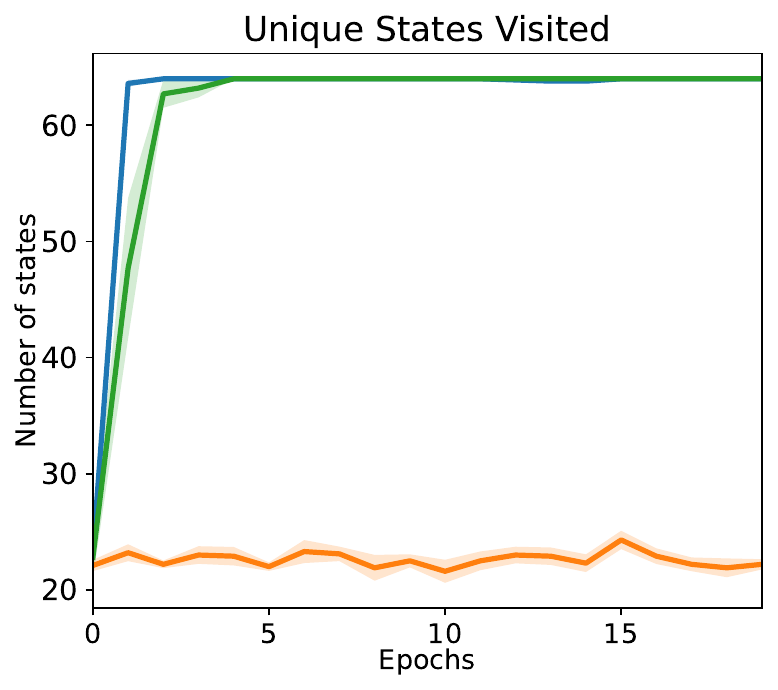}
  \label{fig:ent-plot}
\end{subfigure}
\hspace{1.55cm}
\begin{subfigure}{.33\textwidth}
 \raisebox{0.03cm}{\includegraphics[width=\linewidth]{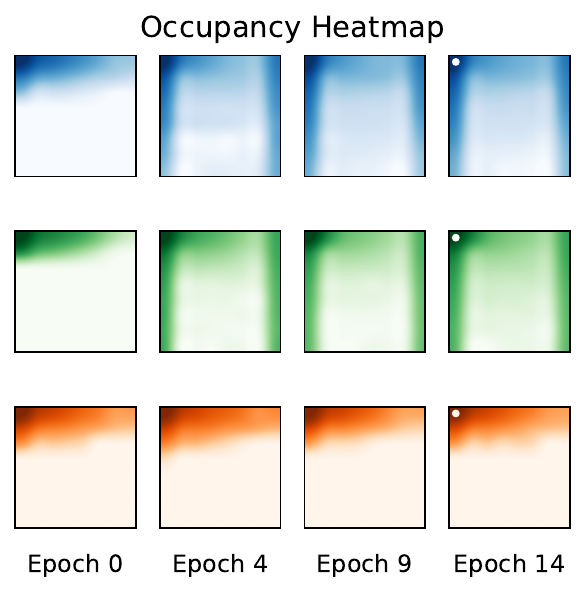}}
  \label{fig:ent-plot}
\end{subfigure}

\begin{subfigure}{\linewidth}
\vspace{-0.4cm}
\centering
\hspace{0.75cm}
\includegraphics[width=.35\linewidth]{legend_all.pdf}
\label{fig:sub3}
\end{subfigure}
\caption{Number of discrete states visited (mean and standard error over 10 runs) and
 occupancy heatmaps for each policy cover obtained from \mainobj
 (\cref{alg:linf_relaxation}), \maxent, and uniform exploration  in the \pendulum environment.
Each epoch
 comprises a single policy update in \cref{alg:linf_relaxation}
 and \maxent, obtained through 1000 steps of \reinforce with rollouts of length 400. Heatmap axes: torque (x-axis) and angle (y-axis). Start state indicated by $\bullet$. }
\label{fig:more-exp-results-pendulum-states}
\end{figure}

\begin{figure}[tp]
\centering
\begin{subfigure}{.37\textwidth}
\includegraphics[width=\linewidth]{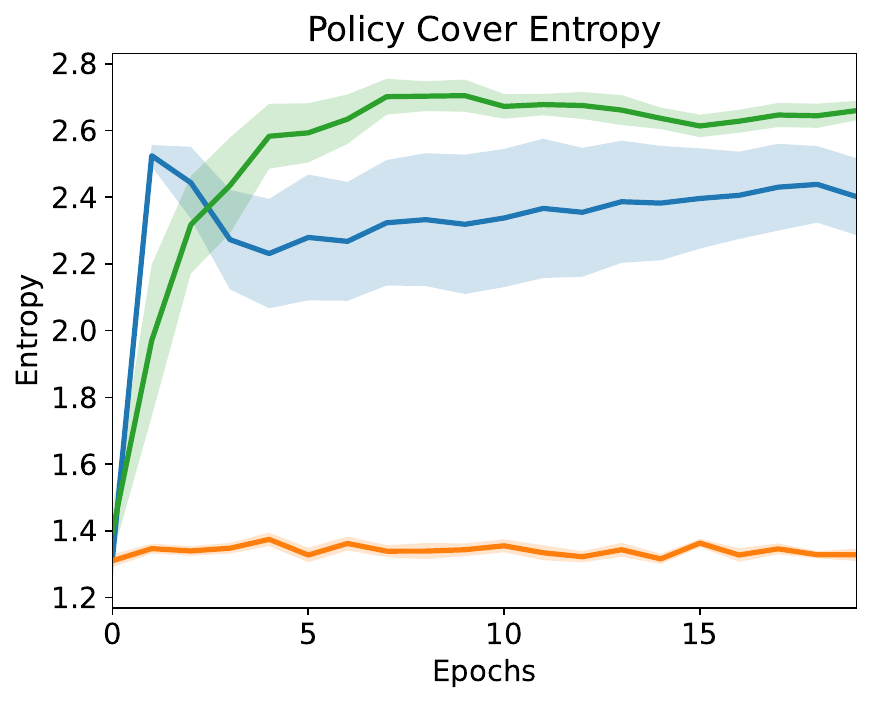}
  \label{fig:ent-plot}
\end{subfigure}
\hspace{1cm}
\begin{subfigure}{.37\textwidth}
\includegraphics[width=\linewidth]{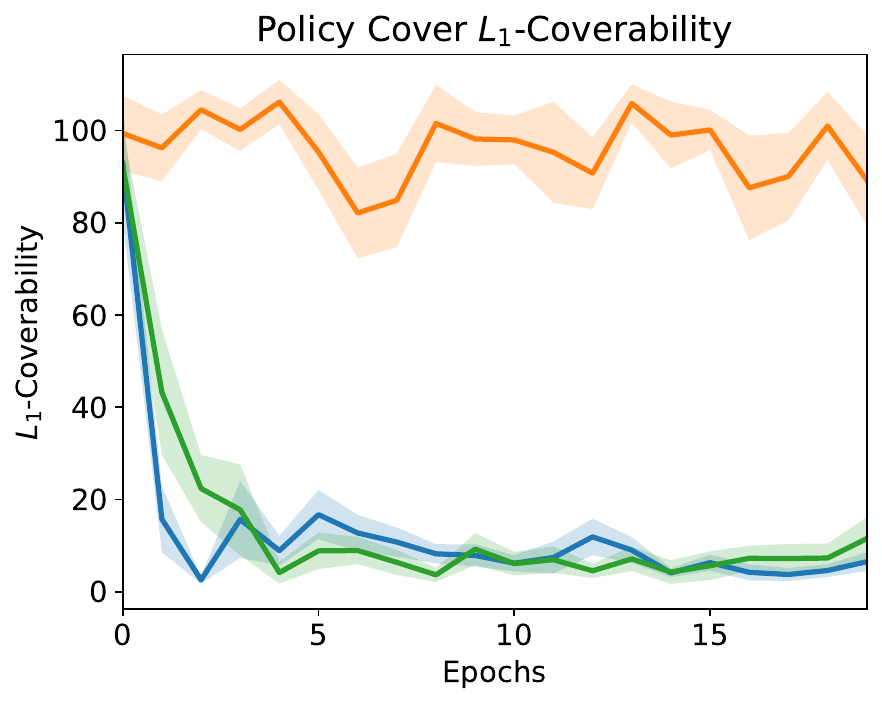}
  \label{fig:ent-plot}
\end{subfigure}
\begin{subfigure}{\linewidth}
\vspace{-0.4cm}
\centering
\hspace{0.75cm}
\includegraphics[width=.35\linewidth]{legend_all.pdf}
\label{fig:sub3}
\end{subfigure}
\caption{Entropy and \mainopt measured on each policy cover obtained from \mainobj (\cref{alg:linf_relaxation}), \maxent, and uniform exploration on the \pendulum environment. We plot the mean and standard error across 10 runs. Each epoch
 corresponds to a single policy update in \cref{alg:linf_relaxation}
 and \maxent, obtained through 1000 steps of \reinforce with rollouts of length 400.}
\label{fig:more-exp-results-pendulum-ent}
\end{figure}

\paragraph{Environment}

We evaluate on the \texttt{Pendulum-v0} environment (henceforth simply
\pendulum) from the OpenAI Gym \citep{brockman2016openai}. The state space of the environment is two-dimensional, with an angle value (denoted by $\theta$) that is in the interval $[-\pi,\pi]$ and a velocity value (denoted by $\rho$) that is in the interval $[-8, 8]$. The dynamics are deterministic and defined by the physics of an inverted pendulum. %
    We modify the starting distribution to be a deterministic starting
    state with a position of $\theta = \pi$ and a velocity of $\rho =
    0$. This means that occupancies and covers are evaluated when
    rolling out from this deterministic starting state. The action
    space is continuous in the interval $[-2,2]$, with negative values
    corresponding to torque applied to the left and positive values
    corresponding to torque applied to the right. To simplify, we only
    allow actions in $\{-1, 0, 1\}$, so that $\abs{\cA}=3$. Finally, we take a horizon of $H=200$, meaning that we terminate rollouts after $200$ steps.

\paragraph{Discretization and other hyperparameters}

We apply all the same implementation details as in \cref{sec:exp-details}. We take a discretization resolution for \pendulum of $8 \times 8$. We take a \reinforce planner with a uniform starting distribution with $\theta \in [-\pi,\pi]$ and $\rho \in [-1,1]$. The architecture, optimizer, and hyperparameters are the same as for \mountaincar. 

\paragraph{Results}
As with \mountaincar, we measure the number of unique states visited
and the occupancy heatmaps
(\cref{fig:more-exp-results-pendulum-states}) as well as the entropy
and \mainobj values
(\cref{fig:more-exp-results-pendulum-ent}). Overall, we find that
\mainobj and \maxent obtain similar values and are able to explore
every state in the (discretized) \pendulum environment, indicating
that exploration is somewhat easier than for the \mountaincar environment.

\subsection{Additional Discussion}

While the \mountaincar and \pendulum environments are fairly simple, we note that our
algorithm exhibits robustness in several different ways, indicating
that it may scale favorably to more challenging domains. Firstly, we
do not expect that \reinforce is finding a near-optimal policy for the
planning problem in \cref{line:linf_opt}, which indicates that our
method is robust to optimization errors. We also note that the choice
of $\mu$ used in our implementation of \cref{alg:linf_relaxation},
which is defined as uniform over the discretized space, is a
heuristic and is not guaranteed to have good coverage with respect to the
true continuous MDP. This indicates that our method is robust to the
choice of $\mu$ and $\Cinf$. Lastly, our method worked ``out-of-the-box'' with the same hyperparameters used in the \maxent implementation
in \citet{hazan2019provably}, and was found to behave similarly with different hyperparameters, which indicates that our method is robust to the choice of hyperparameters.

\clearpage

\section{Technical Tools}
\label{app:technical}

\subsection{Minimax Theorem}

\begin{lemma}[Sion's Minimax Theorem \citep{sion1958minimax}]
  \label{lem:sion}
  Let $\cX$ and $\cY$ be convex sets in linear topological
  spaces, and assume $\cX$ is compact. Let
  $F:\cX\times\cY\to\bbR$ be such that (i) $f(x, \cdot)$ is
  concave and upper semicontinuous over $\cY$ for all
  $x\in\cX$ and (ii) $f(\cdot,y)$ is convex and lower
  semicontinuous over $\cX$ for all $y\in\cY$. Then
  \begin{equation}
    \label{eq:minimax_theorem}
    \inf_{x\in\cX}\sup_{y\in\cY}f(x,y) = \sup_{y\in\cY}\inf_{x\in\cX}f(x,y).
  \end{equation}
\end{lemma}

\subsection{Concentration}

      \begin{lemma}
    \label{lem:martingale_chernoff}
    For any sequence of real-valued random variables $\prn{X_t}_{t\leq{}T}$ adapted to a
    filtration $\prn{\filt_t}_{t\leq{}T}$, it holds that with probability at least
    $1-\delta$, for all $T'\leq{}T$,
    \begin{equation}
      \label{eq:martingale_chernoff}
\sum_{t=1}^{T'}-\log\prn*{\En_{t-1}\brk*{e^{-X_t}}}      \leq \sum_{t=1}^{T'}X_t
      + \log(\delta^{-1}).
    \end{equation}
  \end{lemma}

  \begin{lemma}[Freedman's inequality (e.g., \citet{agarwal2014taming})]
  \label{lem:freedman}
  Let $(X_t)_{t\leq{T}}$ be a real-valued martingale difference
  sequence adapted to a filtration $\prn{\filt_t}_{t\leq{}T}$. If
  $\abs*{X_t}\leq{}R$ almost surely, then for any $\eta\in(0,1/R)$, with probability at least $1-\delta$,
    \[
      \sum_{t=1}^{T}X_t \leq{} \eta\sum_{t=1}^{T}\En_{t-1}\brk*{X_t^{2}} + \frac{\log(\delta^{-1})}{\eta}.
    \]
  \end{lemma}

  \begin{lemma}[Corollary of \cref{lem:freedman}]
      \label{lem:multiplicative_freedman}
            Let $(X_t)_{t\leq{T}}$ be a sequence of random
      variables adapted to a filtration $\prn{\filt_{t}}_{t\leq{}T}$. If
  $0\leq{}X_t\leq{}R$ almost surely, then with probability at least
  $1-\delta$,
  \begin{align*}
    &\sum_{t=1}^{T}X_t \leq{}
                        \frac{3}{2}\sum_{t=1}^{T}\En_{t-1}\brk*{X_t} +
                        4R\log(2\delta^{-1}),
    \intertext{and}
      &\sum_{t=1}^{T}\En_{t-1}\brk*{X_t} \leq{} 2\sum_{t=1}^{T}X_t + 8R\log(2\delta^{-1}).
  \end{align*}
    \end{lemma}

\subsection{Information Theory}

\begin{lemma}[e.g., \citet{foster2021statistical}]
  \label{lem:hellinger_pair}
    For any pair of random variables $(X,Y)$,
    \[
      \En_{X\sim{}\bbP_X}\brk*{\Dhels{\bbP_{Y\mid{}X}}{\bbQ_{Y\mid{}X}}} \leq{} 
      4\Dhels{\bbP_{X,Y}}{\bbQ_{X,Y}}.      
    \]
  \end{lemma}

      \begin{restatable}[\citet{foster2021statistical}]{lemma}{mpmin}
  \label{lem:mp_min}
  Let $\bbP$ and $\bbQ$ be probability measures on $(\cX,\filt)$. For all
  $h:\cX\to\bbR$ with $0\leq{}h(X)\leq{}B$ almost surely under $\bbP$
  and $\bbQ$, we have
  \begin{align}
    \label{eq:mp_min}
\En_{\bbP}\brk*{h(X)}
    &\leq{} 3\En_{\bbQ}\brk*{h(X)} + 4B\cdot\Dhels{\bbP}{\bbQ}.
  \end{align}
\end{restatable}

\subsection{Reinforcement Learning}

Proofs for the following lemmas can be found in \citet{foster2021statistical}.

\begin{lemma}[Global simulation lemma]
  \label{lem:simulation_global}
  Let $M$ and $M'$ be MDPs with $\sum_{h=1}^{H}r_h\in\brk*{0,1}$ almost surely, and let $\act\in\PiGen$.
Then we have
\begin{align}
  \label{eq:simulation_global}
  \abs*{\fm(\pi)-\fmp(\pi)}
  &\leq \Dhel{M(\act)}{M'(\act)}
  \leq{} \frac{1}{2\eta} + \frac{\eta}{2}\Dhels{M(\act)}{M'(\act)}\quad\forall{}\eta>0.
\end{align}
\end{lemma}

  \begin{lemma}[Local simulation lemma]
    \label{lem:simulation}
      For any pair of MDPs $M=(\Pm,\Rm)$ and $\Mbar=(\Pmbar,\Rmbar)$ with
  the same initial state distribution and $\sum_{h=1}^{H}r_h\in\brk*{0,1}$,
    \begin{align}
      \label{eq:simulation}
          \abs*{\fm(\pi)- \fmbar(\pi)}
      &\leq{}
              \sum_{h=1}^{H}\En^{\sMbar,\pi}\brk*{\Dhel{\Pm_h(x_h,a_h)}{\Pmbar_h(x_h,a_h)}
              + \Dhel{\Rm_h(x_h,a_h)}{\Rmbar_h(x_h,a_h)}
        }.
    \end{align}
  \end{lemma}

\subsection{Helper Lemmas}

\begin{lemma}[E.g., \citet{xie2023role}]
\label{lem:elliptic_potential}
Let $d\ind{1}, d\ind{2}, \ldots, d\ind{T}$ be an arbitrary sequence of distributions over a set $\cZ$, and let $\mu\in\Delta(\cZ)$ be a distribution such that $d\ind{t}(z) / \mu(z) \leq C$ for all $(z,t) \in \cZ \times [T]$. Then for all $z \in \cZ$, we have
\begin{align*}
\sum_{t = 1}^{T} \frac{d\ind{t}(z)}{\sum_{i < t} d\ind{i}(z) + C \cdot \mu(z)} \leq 2\log(2T).
\end{align*}
\end{lemma}

\begin{lemma}
  \label{lem:pi_to_mu}
  For any distribution $p\in\Delta(\PiRNS)$, $\pi\in\PiRNS$, $\mu\in\bbR^{\cX\times\cA}_{+}$, and $\veps,\delta>0$, we have that
  \begin{align*}
    \En^{\sM,\pi}\brk*{\frac{d^{\sM,\pi}_h(x_h,a_h)}{d_h^{\sM,p}(x_h,a_h)+\veps\cdot{}d_h^{\sM,\pi}(x_h,a_h)}}
    &\leq
      \En^{\sM,\pi}\brk*{\frac{d^{\sM,\pi}_h(x_h,a_h)}{d_h^{\sM,p}(x_h,a_h)+\delta\cdot\mu(x_h,a_h)}} \\
      &\qquad + \frac{\delta}{\veps}\cdot\En^{\sM,\pi}\brk*{\frac{\mu(x_h,a_h)}{d_h^{\sM,p}(x_h,a_h)+\delta\cdot{}\mu(x_h,a_h)}}.
  \end{align*}
\end{lemma}
\begin{proof}[\pfref{lem:pi_to_mu}]
  The result follows by observing that we can bound
  \begin{align*}
    &\En^{\sM,\pi}\brk*{\frac{d^{\sM,\pi}_h(x_h,a_h)}{d_h^{\sM,p}(x_h,a_h)+\veps\cdot{}d_h^{\sM,\pi}(x_h,a_h)}}
    -
      \En^{\sM,\pi}\brk*{\frac{d^{\sM,\pi}_h(x_h,a_h)}{d_h^{\sM,p}(x_h,a_h)+\delta\cdot\mu(x_h,a_h)}}\\
    &=
      \sum_{x\in\cX,a\in\cA}\frac{(d_h^{\sM,\pi}(x,a))^2\prn*{\delta\cdot\mu(x,a)-\veps\cdot{}d_h^{\sM,\pi}(x,a)}}{
      (d_h^{\sM,p}(x,a)+\veps\cdot{}d_h^{\sM,\pi}(x,a))(d_h^{\sM,p}(x,a)+\delta\cdot\mu(x,a))}\\
    &\leq
      \sum_{x\in\cX,a\in\cA}\frac{(d_h^{\sM,\pi}(x,a))^2\prn*{\delta\cdot\mu(x,a)}}{
      (d_h^{\sM,p}(x,a)+\veps\cdot{}d_h^{\sM,\pi}(x,a))(d_h^{\sM,p}(x,a)+\delta\cdot\mu(x,a))}\\
        &\leq
          \frac{\delta}{\veps}\sum_{x\in\cX,a\in\cA}\frac{d_h^{\sM,\pi}(x,a)\mu(x,a)}{
      d_h^{\sM,p}(x,a)+\delta\cdot\mu(x,a)}.
  \end{align*}
\end{proof}

\begin{lemma}
  \label{lem:pi_to_mu_pushforward}
  For any $\pi\in\PiRNS$, $d\in\bbR_{+}^{\cX}$, $\mu\in\bbR^{\cX}_{+}$, and $\veps,\delta>0$, we have that
  \begin{align*}
\En^{\sM,\pi}\brk*{\frac{P_{h-1}^{\sM}(x_h\mid{}x_{h-1},a_{h-1})}{d(x_h)+\veps\cdot{}P_{h-1}^{\sM}(x_h\mid{}x_{h-1},a_{h-1}) }}
    \leq
      \En^{\sM,\pi}\brk*{\frac{P_{h-1}^{\sM}(x_h\mid{}x_{h-1},a_{h-1})}{d(x_h)+\delta\cdot\mu(x_h)}}
      + \frac{\delta}{\veps}\cdot\En^{\sM,\pi}\brk*{\frac{\mu(x_h)}{d(x_h)+\delta\cdot{}\mu(x_h)}}.
  \end{align*}
\end{lemma}
\begin{proof}[\pfref{lem:pi_to_mu_pushforward}]
  The result follows using similar reasoning to \cref{lem:pi_to_mu}:
  \begin{align*}
    &\En^{\sM,\pi}\brk*{\frac{P_{h-1}^{\sM}(x_h\mid{}x_{h-1},a_{h-1})}{d(x_h)+\veps\cdot{}P_{h-1}^{\sM}(x_h\mid{}x_{h-1},a_{h-1})}}
    -
      \En^{\sM,\pi}\brk*{\frac{P_{h-1}^{\sM}(x_h\mid{}x_{h-1},a_{h-1})}{d(x_h)+\delta\cdot\mu(x_h)}}\\
    &=
      \sum_{x\in\cX,a\in\cA,x'\in\cX}d_{h-1}^{\sM,\pi}(x,a)\frac{(P_{h-1}^{\sM}(x'\mid{}x,a))^2\prn*{\delta\cdot\mu(x')-\veps\cdot{}P_{h-1}^{\sM}(x'\mid{}x,a)}}{
      (d(x')+\veps\cdot{}P_{h-1}^{\sM,}(x'\mid{}x,a))(d(x')+\delta\cdot\mu(x'))}\\
        &\leq
\sum_{x\in\cX,a\in\cA,x'\in\cX}d_{h-1}^{\sM,\pi}(x,a)\frac{(P_{h-1}^{\sM}(x'\mid{}x,a))^2\prn*{\delta\cdot\mu(x')}}{
          (d(x')+\veps\cdot{}P_{h-1}^{\sM,}(x'\mid{}x,a))(d(x')+\delta\cdot\mu(x'))}\\
            &\leq
\frac{\delta}{\veps}\sum_{x\in\cX,a\in\cA,x'\in\cX}d_{h-1}^{\sM,\pi}(x,a)\frac{P_{h-1}^{\sM}(x'\mid{}x,a)\mu(x')}{
              d(x')+\delta\cdot\mu(x')}\\
    &=\frac{\delta}{\veps}\cdot\En^{\sM,\pi}\brk*{\frac{\mu(x_h)}{d(x_h)+\delta\cdot{}\mu(x_h)}}.
  \end{align*}
\end{proof}

\icml{\clearpage}

\part{Additional Results and Discussion}

\section{\mainobj: Application to Downstream Policy Optimization}
\label{app:downstream}
In this section, we show how to use data gathered using policy covers
with bounded \mainobj to perform offline policy optimization. Here,
there is an underlying (reward-free) MDP $\Mstar=\crl*{\cX, \cA,
  \crl{\Pmstar_h}_{h=0}^{H}}$ and a given policy class $\Pi$, and we
have access to policy covers $p_1,\ldots,p_H$ such that the \mainobj
objective $\CovM[\Mstar](p_h)$ is small for all $h$. For an arbitrary user-specified reward distribution $R=\crl*{R_h}_{h=1}^{H}$ with $\sum_{h=1}^{H}r_h\in\brk{0,1}$ almost surely, we define
\[
  J^{\sMstar}_R(\pi)\ldef{}\En^{\sMstar,\pi}\brk*{\sum_{h=1}^{H}r_h}
\]
as the value under $r_h\sim{}R(x_h,a_h)$. Our goal is to use trajectories drawn from $p_1,\ldots,p_H$ to compute a policy $\pihat$ such 
\[
  \max_{\pi\in\Pi}\Jmstar_R(\pi) - \Jmstar_R(\pihat) \leq \vepspac
\]
using as few samples as possible.

We present guarantees for two standard offline policy optimization
methods: Maximum Likelihood Estimation (\mle) and Fitted Q-Iteration
(FQI). While both of these algorithms are fully offline, combining
them with the online algorithms for reward-free exploration in
\cref{sec:model_based} leads to end-to-end algorithms for online
reward-driven exploration. We expect that similar guarantees can be proven for other
standard offline policy optimization methods
\citep{munos2003error,antos2008learning,chen2019information,xie2020q,xie2021batch,
  jin2021pessimism,rashidinejad2021bridging,foster2022offline,zhan2022offline},
as well as offline policy evaluation  methods
\citep{liu2018breaking,uehara2020minimax,yang2020off,uehara2021finite}
and hybrid offline/online methods \citep{bagnell2003policy,agarwal2021theory,song2022hybrid}.

\paragraph{Maximum likelihood}
Let $\cM$ be a realizable model class for which $\Mstar\in\cM$. We define the Maximum Likelihood algorithm as follows:
\begin{itemize}
\item For each $h\in\brk{H}$, gather $n$ trajectories $\crl*{o\ind{h,t}}_{t\in\brk{n}}$, where $o\ind{h,t}=(x_1\ind{h,t},a_1\ind{h,t},r_1\ind{h,t}),\ldots, (x_H\ind{h,t},a_H\ind{h,t},r_H\ind{h,t})$ by executing $\pi\ind{h,t}\sim{}p_h$ in $\Mstar$ with $R$ as the reward distribution.
\item Set $\Mhat=\argmax_{M\in\cM}\sum_{h=1}^{H}\sum_{t=1}^{n}\log(M(o\ind{h,t}\mid{}\pi\ind{h,t}))$, where $M(o\mid{}\pi)$ denotes the likelihood of trajectory $o$ under $\pi$ in $M$.
\item Let $\pihat=\argmax_{\pi\in\Pi}\Jmstar_R(\pi)$.
\end{itemize}

\begin{proposition}
  \label{prop:mle_offline}
  Assume that $\Mstar\in\cM$ and let $R=\crl*{R_h}_{h=1}^{H}$ be a reward distribution with $\sum_{h=1}^{H}r_h\in\brk{0,1}$ almost surely for all $M\in\cM$. For any $n\in\bbN$, given policy covers $p_1,\ldots,p_H$, the Maximum Likelihood algorithm ensures that with probability at least $1-\delta$,
  \begin{align}
    \label{eq:mle_offline}
    \Jmstar_R(\pistar) - \Jmstar_R(\pihat)
    \leq
    8H\prn*{\sqrt{\max_h\CovM[\Mstar](p_h)\cdot{}\frac{\log(\abs{\cM}/\delta)}{n}}
          + \max_h\CovM[\Mstar](p_h)\cdot{}\veps{}},
  \end{align}
  and uses $H\cdot{}n$ trajectories total.
\end{proposition}

\paragraph{Fitted Q-iteration}
Let a value function class
$\cQ=\cQ_1,\ldots,\cQ_H\subset(\cX\times\cA)\to\brk{0,1}$ be
given. For a value function $Q$, define the Bellman operator for
$\Mstar$ with reward distribution $R$ by
$\brk[\big]{\cT^{\sMstar}_{R,h}Q}(x,a)\ldef\En^{\sMstar}\brk{r_h+\max_{a'\in\cA}Q(x_{h+1},a')\mid{}x_h=x,a_h=a}$
under $r_h\sim{}R(x_h,a_h)$. We make the Bellman completeness
assumption that for all $Q\in\cQ_{h+1}$,
$\brk[\big]{\cT^{\sMstar}_{R,h}Q}\in\cQ_h$. We define the \mainobj
objective with respect to the policy class
$\Pi=\crl*{\pi_Q}_{Q\in\cQ}$, where $\pi_{Q,h}(x)\ldef{}\argmax_{a\in\cA}Q_h(x,a)$.

The Fitted Q-Iteration algorithm is defined as follows:
\begin{itemize}
\item For each $h=H,\ldots,1$:
  \begin{itemize}
  \item Gather $n$ trajectories $\crl*{o\ind{h,t}}_{t\in\brk{n}}$,
    where
    $o\ind{h,t}=(x_1\ind{h,t},a_1\ind{h,t},r_1\ind{h,t}),\ldots,
    (x_H\ind{h,t},a_H\ind{h,t},r_H\ind{h,t})$ by executing
    $\pi\ind{h,t}\sim{}p_h$ in $\Mstar$ with $R$ as the reward
    distribution.
  \item Set $\Qhat_h=\argmin_{Q\in\cQ_h}\sum_{t=1}^{n}\prn[\big]{
      Q(x_h\ind{h,t},a_h\ind{h,t})
      -r_h\ind{h,t}
      -\max_{a'\in\cA}\Qhat_{h+1}(x_{h+1}\ind{h,t},a')
      }^2
    $.
\end{itemize}
\item Let $\pihat_h(x)\ldef{}\argmax_{a\in\cA}\Qhat_h(x,a)$.
\end{itemize}
\begin{proposition}
  \label{prop:fqi_offline}
  Let $R=\crl*{R_h}_{h=1}^{H}$ be a reward distribution with $r_h\in\brk{0,1}$ and $\sum_{h=1}^{H}r_h\in\brk{0,1}$ almost surely, and assume that for all $Q\in\cQ_{h+1}$, $\brk[\big]{\cT^{\sMstar}_{R,h}Q}\in\cQ_h$. For any $n\in\bbN$, given policy covers $p_1,\ldots,p_H$, the Fitted Q-Iteration algorithm ensures that with probability at least $1-\delta$,
  \begin{align}
    \label{eq:fqi_offline}
    \Jmstar_R(\pistar) - \Jmstar_R(\pihat)
    \leq
    \bigoh(H)\cdot\prn*{\sqrt{\max_h\CovM[\Mstar](p_h)\cdot{}\frac{\log(\abs{\cQ}H/\delta)}{n}}
          + \max_h\CovM[\Mstar](p_h)\cdot{}\veps{}},
  \end{align}
  and uses $H\cdot{}n$ trajectories total.
  
\end{proposition}

\subsection{Proofs}

\begin{proof}[\pfref{prop:mle_offline}]
  By the standard generalization bound for \mle (e.g., \citet{foster2021efficient,foster2023lecture}), we are guaranteed that with probability at least $1-\delta$,
  \begin{align*}
    \sum_{h=1}^{H}\En^{\sMstar,p_h}\brk*{\Dhels{\Pmhat_h(x_h,a_h)}{\Pmstar_h(x_h,a_h)}}
    \leq \sum_{h=1}^{H}\En_{\pi\sim{}p_h}\brk*{\Dhels{\Mhat(\pi)}{\Mstar(\pi)}} \leq 
    2\frac{\log(\abs{\cM}/\delta)}{n}
    .
  \end{align*}
  Let $\pistar\ldef{}\argmax_{\pi\in\Pi}\Jmstar_R(\pi)$.
  Using \cref{lem:simulation}, we have that
  \begin{align*}
    \Jmstar_R(\pistar) - \Jmstar_R(\pihat)
    \leq{}     \Jmstar_R(\pistar) - \Jmhat_R(\pistar) + \Jmhat_R(\pihat) - \Jmstar_R(\pihat) 
    \leq 2\max_{\pi\in\Pi}
    \sum_{h=1}^{H}\En^{\Mstar,\pi}\brk*{\Dhel{\Pmhat_h(x_h,a_h)}{\Pmstar_h(x_h,a_h)}
        }.
  \end{align*}
   Since
      $\Dhel{\Pmhat_h(x_h,a_h)}{\Pmstar_h(x_h,a_h)}\in\brk{0,\sqrt{2}}$, 
      we can use \cref{lem:com} to bound
        \begin{align*}
          \max_{\pi\in\Pi}\En^{\sMstar,\pi}\brk*{\Dhel{\Pmhat_h(x_h,a_h)}{\Pmstar_h(x_h,a_h)}} 
        &\leq 
          2\sqrt{\CovM[\Mstar](p_h)\cdot{}\En^{\sMstar,p_h}\brk*{\Dhels{\Pmhat_h(x_h,a_h)}{\Pmstar_h(x_h,a_h)}}
          }\\
          &\qquad + \sqrt{2}\CovM[\Mstar](p_h)\veps{}\\
                  &\leq 
          2\sqrt{2\CovM[\Mstar](p_h)\cdot{}\frac{\log(\abs{\cM}/\delta)}{n}}
          + \sqrt{2}\cdot\CovM[\Mstar](p_h)\cdot{}\veps{}.
        \end{align*}
        Combining this with the preceding inequalities yields the result.
  
\end{proof}

\begin{proof}[\pfref{prop:fqi_offline}]
  By a standard generalization bound for FQI (e.g., \citet{xie2020q,xie2023role}), under the Bellman completeness assumption, it holds that with probability at least $1-\delta$,
  \begin{align*}
    \sum_{h=1}^{H}\En^{\sMstar,p_h}\brk*{\prn*{\Qhat_h(x_h,a_h)-\brk[\big]{\cT^{\sMstar}_{R,h}\Qhat_{h+1}}(x_h,a_h)}^2}
\leq    \bigoh\prn*{\frac{\log(\abs{\cQ}H/\delta)}{n}}.
  \end{align*}
  At the same time, using a finite-horizon adaptation of \citet[Lemma 4]{xie2020q}, we have that
  \begin{align*}
    \Jmstar_R(\pistar) - \Jmstar_R(\pihat)
    &\leq 2\max_{\pi\in\Pi}
    \sum_{h=1}^{H}\En^{\Mstar,\pi}\brk*{\abs*{\Qhat_h(x_h,a_h)-\brk[\big]{\cT^{\sMstar}_{R,h}\Qhat_{h+1}}(x_h,a_h)}}.
  \end{align*}
   Since $\abs[\big]{\Qhat_h(x_h,a_h)-\brk[\big]{\cT^{\sMstar}_{R,h}\Qhat_{h+1}}(x_h,a_h)}\in\brk{0,1}$,
      we can use \cref{lem:com} to bound
        \begin{align*}
&\max_{\pi\in\Pi}\En^{\sMstar,\pi}\brk*{\abs[\big]{\Qhat_h(x_h,a_h)-\brk[\big]{\cT^{\sMstar}_{R,h}\Qhat_{h+1}}(x_h,a_h)}} \\
        &\leq 
2\sqrt{\CovM[\Mstar](p_h)\cdot{}\En^{\sMstar,p_h}\brk*{\prn[\big]{\Qhat_h(x_h,a_h)-\brk[\big]{\cT^{\sMstar}_{R,h}\Qhat_{h+1}}(x_h,a_h)}^2}
          }
          + \CovM[\Mstar](p_h)\cdot{}\veps{}\\
                  &\leq 
                   \bigoh\prn*{\sqrt{\CovM[\Mstar](p_h)\cdot{}\frac{\log(\abs{\cQ}H/\delta)}{n}}
          + \CovM[\Mstar](p_h)\cdot{}\veps{}}.
        \end{align*}
\arxiv{        Combining this with the preceding inequalities yields the result.}
      \end{proof}

\icml{\clearpage}

\icml{
\section{Efficient Model-Free Exploration via \mainobj}
\label{sec:model_free}

\clearpage
\section{\mainobj: Structural Properties}
\label{sec:structural}
  
}

\clearpage 

\arxiv{
\section{Efficient Policy Cover Computation: Recipe for
  Relaxations}
\label{sec:relaxation_recipe}

The algorithms for efficient policy cover computation presented in
\cref{sec:planning} (\cref{alg:linf_relaxation} and \cref{alg:pushforward_relaxation})
and their corresponding objectives can
be viewed as a special case of a general recipe for deriving
approximate optimization objectives for \mainobj, which we expect to
find broader use. We sketch the approach here.

Let $\tau=(x_1,a_1),\ldots,(x_H,a_H)$ denote a trajectory, and let
$\cT\ldef{}(\cX\times\cA)^{H}$ denote the space of trajectories. Suppose we have access to a \emph{weight function}
$w_{\veps}:\Delta(\Pi)\times\cT\to\bbR_{+}$ defined for $\veps>0$ with the following
  properties:
  \begin{enumerate}
  \item \emph{Relaxation property.} There is a (potentially problem-dependent) constant
    $C_1>0$ such that for all $\pi\in\Pi$ and $p\in\Delta(\Pi)$, 
    \begin{align}
      \label{eq:relax1}
      \En^{\sM,\pi}\brk*{\frac{d^{\sM,\pi}_h(x_h,a_h)}{d_h^{\sM,p}(x_h,a_h)+\veps\cdot{}d_h^{\sM,\pi}(x_h,a_h)}}
      \leq{} C_1\cdot\En^{\sM,\pi}\brk*{w_{\veps}(p;\tau)}.\tag{R1}
    \end{align}
  \item \emph{Potential property.} For all sequences $\pi\ind{1},\ldots,\pi\ind{T}$, if we define
    $p\ind{t}=\unif(\pi\ind{1},\ldots,\pi\ind{t-1})$ and $\veps_t=\frac{1}{t-1}$, then
    \begin{align}
      \label{eq:relax2}
      \frac{1}{T}\sum_{t=1}^{T}\En^{\sM,\pi\ind{t}}\brk*{w_{\veps_t}(p\ind{t};\tau)}
      \leq{} C_2\cdot\log(T).\tag{R2}
    \end{align}
  \end{enumerate}

  \begin{restatable}{proposition}{relaxation}
    \label{prop:relaxation}
    Consider a weight function $w_{\veps}:\Delta(\Pi)\times\cT\to\bbR_{+}$
    satisfying properties \eqref{eq:relax1} and \eqref{eq:relax2}
    above. Consider an algorithm that repeatedly solves the
    optimization problem
    \begin{align*}
      \pi\ind{t}=\argmax_{\pi\in\Pi}\En^{\sM,\pi}\brk*{w_{\veps_t}(p\ind{t};\tau)}
    \end{align*}
    for $p\ind{t}=\unif(\pi\ind{1},\ldots,\pi\ind{t-1})$ and
    $\veps_t=\frac{1}{t-1}$. This algorithm guarantees that for all $T\in\bbN$,
  \begin{align}
    \CovMveps{\veps_T}(p\ind{T})
    \leq{} C_1C_2\log(T).
  \end{align}    
\end{restatable}
\cref{alg:linf_relaxation} is a special case of this framework with
$w_{\veps}(p;\tau)\ldef{}\frac{\mu(x_h,a_h)}{d^{p}_h(x_h,a_h)+\veps\cdot\CinfhM(\mu)\mu(x_h,a_h)}$,
and \cref{alg:pushforward_relaxation} is a special case with
$w_{\veps}(p;\tau)\ldef{}\frac{P_{h-1}(x_h\mid{}x_{h-1},a_{h-1})}{d^{p}_h(x_h)+\veps\cdot
  P_{h-1}(x_h\mid{}x_{h-1},a_{h-1})}$.

\begin{proof}[\pfref{prop:relaxation}]
  Observe that by definition, the value
  \[
    (T-1)\cdot\CovMveps{\veps_T}(p\ind{T})
    = \max_{\pi\in\Pi}\En^{\sM,\pi}
    \brk*{\frac{d_h^{\sM,\pi}(x_h,a_h)}{\sum_{i<t}d_h^{\sM,\pi\ind{i}}(x_h,a_h)
        + d_h^{\sM,\pi}(x_h,a_h)}}
  \]
  is decreasing in $T$. As a result, we have that
  \begin{align*}
    (T-1)\cdot\CovMveps{\veps_T}(p\ind{T})
    \leq{}     \frac{1}{T}\sum_{t=1}^{T}(t-1)\cdot\CovMveps{\veps_t}(p\ind{t}),
  \end{align*}
  which implies that
  \begin{align*}
    \CovMveps{\veps_T}(p\ind{T})
    \leq{}
      \frac{1}{T}\sum_{t=1}^{T}\CovMveps{\veps_t}(p\ind{t})
    \leq{}
      \frac{C_1}{T}\sum_{t=1}^{T}\max_{\pi\in\Pi}\En^{\sM,\pi}\brk*{w_{\veps_t}(p\ind{t};\tau)}
    =
      \frac{C_1}{T}\sum_{t=1}^{T}\En^{\sM,\pi\ind{t}}\brk*{w_{\veps_t}(p\ind{t};\tau)}
    \leq     C_1C_2\log(T).
  \end{align*}
    
  \end{proof}

\clearpage
}

\arxiv{
\section{Model-Based Algorithms for Reward-Driven RL}
\label{app:model_based_reward_based}
  As an extension, this section gives a model-based algorithm (\cref{alg:model_based_reward_based})
  that directly performs reward-driven exploration under the same set of
  assumptions as in \cref{sec:model_based}, with sample complexity
  determined by the \mainopt parameter. Our results here serve as reward-driven
  counterparts to the results in \cref{sec:model_based}. Note that
  while the
  reward-free results themselves lead to guarantees for reward-driven RL
  via \cref{app:downstream} (e.g., \cref{cor:rf_teaser}), the approach
  we give here is slightly more direct, and may be of independent interest.

\paragraph{Reward-driven reinforcement learning}
In reward-driven reinforcement learning, the underlying MDP $\Mstar=\crl*{\cX, \cA,
  \crl{\Pmstar_h}_{h=0}^{H},\crl{\Rmstar_h}_{h=1}^{H}}$ in the online
reinforcement learning protocol is 
equipped with a reward distribution $\Rmstar_h:\cX\times\cA\to\Delta(\bbR)$ is
  the reward distribution for layer $h$, and each episode in the MDP
  generates a trajectory $(x_1,a_1,r_1),\ldots,(x_H,a_H,r_H)$ via
  $a_h\sim\pi_h(x_h)$,   $r_h\sim\Rmstar(x_h,a_h)$, and
  $x_{h+1}\sim{}P_h\sups{\Mstar}(\cdot\mid{}x_h,a_h)$. Define
  $\Jmstar(\pi)\ldef{}\En^{\sMstar,\pi}\brk*{\sum_{h=1}^{H}r_h}$, and let
  $\pimstar\in\argmax_{\pi\in\PiRNS}J(\pi)$ denote an optimal policy
  that satisfies the Bellman equation. The aim for this setting is
  to learn a $\veps$-optimal policy $\pihat$ such that
  \begin{align}
    \label{eq:eps_opt}
\Jmstar(\pimstar) - \Jmstar(\pihat) \leq \veps
  \end{align}
with high probability, using as few episodes of online interaction as
possible.

For the results in this section, we assume access to a model class $\cM$ containing the true MDP $\Mstar=\crl*{\cX, \cA,
  \crl{\Pmstar_h}_{h=0}^{H},\crl{\Rmstar_h}_{h=1}^{H}}$. For $M\in\cM$, we use $M(\pi)$ as shorthand for the law of the
trajectory $(x_1,a_1,r_1),\ldots,(x_H,a_H,r_H)$. We take $\Pi$ to be an arbitrary set of policies
such that $\pimstar\in\Pi$. In particular, it suffices to set
$\Pi=\crl*{\pim\mid{}M\in\cM}$. As in \cref{sec:model_based}, we
assume access to an estimation algorithm $\AlgEst$ satisfying either
\cref{ass:online_oracle} or \cref{ass:offline_oracle}.

\begin{algorithm}[h]
    \setstretch{1.3}
     \begin{algorithmic}[1]
       \State \textbf{parameters}:
       \Statex[1] Estimation oracle $\AlgEst$.
       \Statex[1] Number of episodes $T\in\bbN$, approximation parameters $C\geq{}1$, $\veps\in\brk{0,1}$.%
  \For{$t=1, 2, \cdots, T$}
  \State Compute estimated model $\Mhat\ind{t} = \AlgEst\ind{t}\prn[\Big]{
    \crl*{(\act\ind{i},\obs\ind{i})}_{i=1}^{t-1} }$.
  \State For each $h\in\brk{H}$, compute $(C,\veps)$-approximate
  policy cover $p_h\ind{t}$ for $\Mhat\ind{t}$: %
        \begin{align}
\CovM[\Mhat\ind{t}](p_h\ind{t})  = \sup_{\pi\in\Pi}\En^{\sMhat\ind{t},\pi}\brk*{\frac{\dm{\Mhat\ind{t}}{\pi}_{h}(x_h,a_h)}{\dm{\Mhat\ind{t}}{p_h\ind{t}}_{h}(x_h,a_h)+\veps\cdot{}d_h^{\sMhat\ind{t},\pi}(x_h,a_h)}}\leq{}C.\label{eq:cover_objective_model_based_reward_based}
        \end{align}
          \hfill\algcommentlight{Plug-in
    approximation to \mainobj objective.}
      \State Let $q\ind{t}=\frac{1}{2}\cdot\pi_{\sMhat\ind{t}} + \frac{1}{2}\cdot\unif(p_1\ind{t},\ldots,p_H\ind{t})$
      \State Sample $\pi\ind{t}\sim{}q\ind{t}$ and observe trajectory $o\ind{t}=(x_1\ind{t},a_1\ind{t},r_1\ind{t}),\ldots,(x_H\ind{t},a_H\ind{t},r_H\ind{t})$.
\EndFor
\State  \textbf{return}
$\pihat\ldef\unif(\pi_{\sMhat\ind{1}},\ldots,\pi_{\sMhat\ind{T}})$.
\end{algorithmic}
\caption{Exploration via Estimation and Plug-In Coverability
  Optimization (reward-driven; \mainalgr)}
\label{alg:model_based_reward_based}
\end{algorithm}

\paragraph{Algorithm}
\mainalgr, \cref{alg:model_based_reward_based}, is a reward-driven
counterpart to \mainalg (\cref{alg:model_based_reward_free}). Like
\cref{alg:model_based_reward_free}, the algorithm repeatedly invokes
the estimation algorithm $\AlgEst$ to obtain an estimator
$\Mhat\ind{t}$, then optimizes the \mainobj objective for the
estimator and uses this to explore. Compared the reward-free case, the
only difference is that the algorithm mixes the greedy policy
$\pi_{\Mhat\ind{t}}$ into the exploration distribution. After all $T$
rounds of exploration conclude, the algorithm returns
$\pihat\ldef\unif(\pi_{\sMhat\ind{1}},\ldots,\pi_{\sMhat\ind{T}})$ as
the final policy.

\paragraph{Main result}
The main sample complexity guarantee for
\cref{alg:model_based_reward_based} is as follows.

\begin{restatable}[Main result for \cref{alg:model_based_reward_based}]{theorem}{modelbasedrewardbased}
  \label{thm:model_based_reward_based}
With parameters $T\in\bbN$, $C\geq{}1$, and $\veps>0$ and an online
estimation oracle satisfying \cref{ass:online_oracle}, whenever the
optimization problem in \cref{eq:cover_objective_model_based} is
feasible at every round, 
\cref{alg:model_based_reward_based} produces a policy $\pihat$ such
that
\begin{align}
  \label{eq:mbrb1}
  \Jmstar(\pimstar) - \Jmstar(\pihat)
  \leq 17\sqrt{\frac{H^3C\cdot{}\EstProbOn}{T}}
      + 6HC\cdot{}\veps{}.
\end{align}
with probability at least $1-\delta$.
\end{restatable}
  This leads to bounds for the following special
  cases. %
  \begin{restatable}[Main guarantee for \mainopt]{corollary}{mbrbtwo}
    \label{cor:mbrb2}
  Let $\vepspac>0$ be given and set $\veps=0$.
  Suppose that (i) we restrict $\cM$ such that all $M\in\cM$ have
  $\CovOptMmax\leq\CovOptMmax[\Mstar]$, and (ii) we solve
  \cref{eq:cover_objective_model_based_reward_based} with
  $C\leq\CovOptMmax[\Mstar]$ for all $t$ (which is always
  feasible). Then, given access to an offline estimation oracle
  satisfying \cref{ass:offline_oracle,ass:parametric}, using
  $T=\bigoht\prn*{\frac{H^{12}(\CovOptMmaxzero[\Mstar])^4\dest\log(\Best/\delta)}{\vepspac^{6}}}$
  episodes, \cref{alg:model_based_reward_based} produces a policy
  $\pihat$ such that
  \begin{align}
    \label{eq:mbrb2}
    \Jmstar(\pimstar) - \Jmstar(\pihat)
    \leq \vepspac
  \end{align}
  with probability at least $1-\delta$. In particular, for a finite
  class $\cM$, if we use \mle as the estimator, we can take
  $T=\bigoht\prn*{\frac{H^{12}(\CovOptMmaxzero[\Mstar])^4\log(\abs*{\cM}/\delta)}{\vepspac^{6}}}$.
\end{restatable}

\begin{restatable}[Main guarantee for
  $L_{\infty}$-Coverability]{corollary}{mbrbthree}
  \label{cor:mbrb3}
  Let $\vepspac>0$ be given, and set $\veps=0$. Let
  $\Cinf\equiv\CinfM[\Mstar]$, and suppose that (i) we restrict $\cM$
  such that all $M\in\cM$ have $\CinfM\leq\Cinf$, and (ii) we solve
  \cref{eq:cover_objective_model_based_reward_based} with $C\leq\Cinf$
  for all $t$.\footnote{This is always feasible by \cref{prop:linf}.}
  Then, given access to an offline estimation oracle satisfying
  \cref{ass:offline_oracle}, using
  $T=\bigoht\prn*{\frac{H^8(\CinfM[\Mstar])^3\dest\log(\Best/\delta)}{\vepspac^{4}}}$
  episodes, \cref{alg:model_based_reward_based} produces a policy
  $\pihat$ such that
  \begin{align}
    \label{eq:mbrb3}
    \Jmstar(\pimstar) - \Jmstar(\pihat)
    \leq \vepspac
  \end{align}
  with probability at least $1-\delta$. In particular, for a finite
  class $\cM$, if we use \mle as the estimator, we can take
  $T=\bigoht\prn*{\frac{H^8(\CinfM[\Mstar])^3\log(\abs*{\cM}/\delta)}{\vepspac^{4}}}$.
\end{restatable}

As with \cref{alg:model_based_reward_free},
\cref{alg:model_based_reward_based} is statistically efficient, with
sample complexity determined by \mainopt, and is computationally efficient
whenever 1) \mle can be performed efficiently, and 2) the \mainobj
objective can be (approximately) optimized for the estimated models $\Mhat\ind{1},\ldots,\Mhat\ind{T}$.

\paragraph{Proof sketch}
To prove \cref{thm:model_based_reward_based}, we draw a connection to
the Decision-Estimation Coefficient (DEC) of
\citet{foster2021statistical,foster2023tight}. Consider the (offset) PAC DEC defined as follows \citep{foster2023tight}:
\begin{align}
  \label{eq:dec_bound}
  \decopac(\cM,\Mhat)
  \ldef \inf_{p,q\in\Delta(\Pi)}\sup_{M\in\cM}\crl*{
  \En_{\pi\sim{}p}\brk*{
  \Jm(\pim)-\Jm(\pi)}
  -\gamma\cdot{}\En_{\pi\sim{}q}\brk[\big]{\DhelsX{\big}{\Mhat(\pi)}{M(\pi)}}
  }.
\end{align}

The following result shows that the exploration strategy in
\cref{alg:model_based_reward_based} certifies a bound on the PAC DEC.
\begin{restatable}{lemma}{decbound}
  \label{lem:dec_bound}
  Consider the reward-driven setting. Let an MDP $\Mhat=\crl*{\cX, \cA,
  \crl{\Pmhat_h}_{h=0}^{H},\crl{\Rmhat_h}_{h=1}^{H}}$ be given, and let $p_1,\ldots,p_H\in\Delta(\Pi)$ be
  $(C,\veps)$-policy covers for $\Mhat$, i.e.
  \begin{align}
    \CovM[\Mhat](p_h) \leq C\quad\forall{}h\in\brk{H}.
  \end{align}
  Then the distribution $q\ldef \frac{1}{2}\pimhat + \frac{1}{2}\unif(p_1,\ldots,p_H)$ ensures that for all MDPs $M=\crl*{\cX, \cA,
  \crl{\Pm_h}_{h=0}^{H},\crl{\Rm_h}_{h=1}^{H}}$,
\begin{align*}
      \Jm(\pim) - \Jm(\pimhat)
        &\leq{}
      3\sqrt{32H^3C\cdot{}\En_{\pi\sim{}q}\brk[\big]{\DhelsX{\big}{\Mhat(\pi)}{M(\pi)}}}
      + 4\sqrt{2}HC\cdot{}\veps{}.
      \end{align*}
      Equivalently, the pair of distributions $(\indic_{\pimhat},q)$
      certifies that for all $\gamma>0$,
      $\decopac(\cM,\Mhat)\leq\bigoh\prn*{
        \frac{H^3C}{\gamma} + HC\cdot\veps
        }$, where $\cM$ is the set of all MDPs with the same state
        space, action space, and initial state distribution as $\cM$.
\end{restatable}
In light of this observation, \cref{alg:model_based_reward_based} can
be viewed as an application of the Estimation-to-Decisions (\etd)
meta-algorithm \citet{foster2021statistical,foster2023tight}, and the
proof of \cref{thm:model_based_reward_based} follows by combining the
DEC bound in \cref{lem:dec_bound} with the generic regret analysis for \etd.

\subsection{Proofs}

\begin{proof}[\pfref{thm:model_based_reward_based}]
  Using \cref{lem:dec_bound}, we have that
    \begin{align*}
\sum_{t=1}^{T}      \Jmstar(\pimstar) - \Jmstar(\pi_{\sMhat\ind{t}})
        &\leq{}
          3\sum_{t=1}^{T}\sqrt{32H^3C\cdot{}\En_{\pi\sim{}q\ind{t}}\brk[\big]{\DhelsX{\big}{\Mhat\ind{t}(\pi)}{\Mstar(\pi)}}}
          + 4\sqrt{2}HCT\cdot{}\veps{}\\
              &\leq{}
      3\sqrt{32H^3CT\cdot{}\sum_{t=1}^{T}\En_{\pi\sim{}q\ind{t}}\brk[\big]{\DhelsX{\big}{\Mhat\ind{t}(\pi)}{\Mstar(\pi)}}}
                + 4\sqrt{2}HCT\cdot{}\veps{}\\
                    &\leq{}
      3\sqrt{32H^3CT\cdot{}\EstProbOn}
      + 4\sqrt{2}HCT\cdot{}\veps{},
    \end{align*}
where the last line holds with probability at least $1-\delta$ by
\cref{ass:online_oracle}. Rearranging, it follows that
\begin{align}
  \label{eq:intermediate}
  \En\brk*{\Jmstar(\pimstar) - \Jmstar(\pihat)}
  \leq 3\sqrt{\frac{32H^3C\cdot{}\EstProbOn}{T}}
      + 4\sqrt{2}HC\cdot{}\veps{}.
\end{align}

\end{proof}

\begin{proof}[\pfref{cor:mbrb2}]
  We can take $C\leq\CovOptMmax[\Mstar]$, and using
  \cref{lem:lone_offline_online}, we have
  \begin{align*}
    \EstProbOn \leq \bigoht\prn[\Big]{H\prn*{\ConeM[\Mstar](1\vee\EstProbOff)}^{1/3}T^{2/3}
    } \leq \bigoht\prn[\Big]{H \prn*{\CovOptMmax[\Mstar](1\vee\EstProbOff)}^{1/3}T^{2/3}
    },
  \end{align*}
  so that \cref{eq:intermediate} gives
  \begin{align*}
    \En\brk*{\Jmstar(\pimstar) - \Jmstar(\pihat)}
    \bigoht\prn*{
    \prn*{\frac{H^{12}(\CovOptMmax[\Mstar])^4(1\vee\EstProbOff)}{T}}^{1/6} 
    } \leq \eps,
  \end{align*}
  where the final inequality uses the choice for $T$ in the corollary
  statement.
\end{proof}

\begin{proof}[\pfref{cor:mbrb3}]
  We can take $C\leq\CinfM[\Mstar]$, and using
  \cref{lem:linf_offline_online}, we have
  \begin{align*}
    \EstProbOn \leq \bigoht\prn*{H\sqrt{\CinfM[\Mstar]
    T\cdot\EstProbOff} + H\cdot\CinfM[\Mstar] 
    }
    \leq \bigoht\prn*{H\sqrt{\CinfM[\Mstar]
    T\cdot(1\vee\EstProbOff)}    },
  \end{align*}
  so that \cref{eq:intermediate} gives
  \begin{align*}
    \En\brk*{\Jmstar(\pimstar) - \Jmstar(\pihat)}
    \leq \bigoht\prn*{
    \prn*{\frac{H^8(\CinfM[\Mstar])^3(1\vee\EstProbOff)}{T}}^{1/4} 
    } \leq \eps,
  \end{align*}
  where the final inequality uses the choice for $T$ in the corollary
  statement.
\end{proof}

\begin{proof}[\pfref{lem:dec_bound}]
 Let an arbitrary MDP $M=\crl*{\cX, \cA,
  \crl{\Pm_h}_{h=0}^{H},\crl{\Rm_h}_{h=1}^{H}}$ be fixed. To begin,
using a variant of the simulation lemma (\cref{lem:simulation_global}), we can bound
  \begin{align*}
    \Jm(\pim) - \Jm(\pimhat) 
    &\leq\Jm(\pim) - \Jmhat(\pimhat) +\DhelX{\big}{\Mhat(\pimhat)}{M(\pimhat)}\\
    &\leq\Jm(\pim) - \Jmhat(\pimhat) +2\En_{\pi\sim{}q}\brk[\big]{\DhelX{\big}{\Mhat(\pi)}{M(\pi)}}.
  \end{align*}
  Next, using another simulation lemma (\cref{lem:simulation}), we have
  \begin{align*}
    \Jm(\pim) - \Jmhat(\pimhat)
    \leq     \Jm(\pim) - \Jmhat(\pim) 
    \leq     \sum_{h=1}^{H}\Emhat[\pim]\brk[\bigg]{\underbrace{\Dhel{\Pmhat_h(x_h,a_h)}{\Pm_h(x_h,a_h)}+\Dhel{\Rmhat_h(x_h,a_h)}{\Rm_h(x_h,a_h)}}_{\rdef\errm_h(x_h,a_h)}}.
  \end{align*}
  Let $h\in\brk{H}$ be fixed. Since $\errm_h(x,a)\in\brk{0,2\sqrt{2}}$,
    we can use \cref{lem:com} to bound
  \begin{align*}
        \Emhat[\pim]\brk*{\errm_h(x_h,a_h)} 
    &\leq 
      2\sqrt{\CovM[\Mhat](p_h)\cdot{}\En^{\sMhat,p_h}\brk*{(\errm(x_h,a_h))^2}
  }
      + 2\sqrt{2}\cdot\CovM[\Mhat](p_h)\cdot{}\veps{}\\
        &\leq 
      2\sqrt{2H\CovM[\Mhat](p_h)\cdot{}\En^{\sMhat,q}\brk*{(\errm(x_h,a_h))^2}
  }
          + 4\sqrt{2}\cdot\CovM[\Mhat](p_h)\cdot{}\veps{}\\
            &\leq 
      2\sqrt{2HC\cdot{}\En^{\sMhat,q}\brk*{(\errm(x_h,a_h))^2}
  }
              + 4\sqrt{2}C\cdot{}\veps{}
        \leq 
              2\sqrt{32HC\cdot{}\En_{\pi\sim{}q}\brk[\big]{\DhelsX{\big}{\Mhat(\pi)}{M(\pi)}}}
      + 4\sqrt{2}C\cdot{}\veps{},
  \end{align*}
where the last inequality follows form
\cref{lem:hellinger_pair}. Summing across all layers, we conclude that
  \begin{align*}
    \Jm(\pim) - \Jm(\pimhat)
    &\leq{}
      2\sqrt{32H^3C\cdot{}\En_{\pi\sim{}q}\brk[\big]{\DhelsX{\big}{\Mhat(\pi)}{M(\pi)}}}
      + 4\sqrt{2}HC\cdot{}\veps{}
      + \En_{\pi\sim{}q}\brk[\big]{\DhelX{\big}{\Mhat(\pi)}{M(\pi)}}\\
        &\leq{}
      3\sqrt{32H^3C\cdot{}\En_{\pi\sim{}q}\brk[\big]{\DhelsX{\big}{\Mhat(\pi)}{M(\pi)}}}
      + 4\sqrt{2}HC\cdot{}\veps{}.
  \end{align*}

\end{proof}

\clearpage

}

\icml{\clearpage}

\part{Proofs}

\section{Proofs from \creftitle{sec:overview}}
\label{app:overview}

\com*

  \begin{proof}[\pfref{lem:com}]
    Let $p\in\Delta(\PiRNS)$ and $g:\cX\times\cA\to\bbR$ be given. We
    first prove the following, more general inequality:
    \begin{align}
  \label{eq:com1}
  \En^{\sM,\pi}\brk*{g(x_h,a_h)}
  \leq{} \sqrt{\CovM(p)\cdot{}\prn*{\En^{\sM,p}\brk*{g^2(x_h,a_h)}+ \veps\cdot{}\En^{\sM,\pi}\brk*{g^2(x_h,a_h)}}}.
\end{align}
    Using
Cauchy-Schwarz, we have
\begin{align*}
  \En^{\sM,\pi}\brk*{g(x_h,a_h)}
  &= \sum_{x\in\cX,a\in\cA}d_h^{\sM,\pi}(x,a)g(x,a)\\
  &=
    \sum_{x\in\cX,a\in\cA}d_h^{\sM,\pi}(x,a)\cdot\frac{(d_h^{p\sM,p}(x,a)+\veps\cdot{}d_h^{\sM,\pi}(x,a))^{1/2}}{(d_h^{\sM,p}(x,a)+\veps\cdot{}d_h^{\sM,\pi}(x,a))^{1/2}}\cdot{}g(x,a)\\
  &\leq
\prn*{\sum_{x\in\cX,a\in\cA}\frac{(d_h^{\sM,\pi}(x,a))^{2}}{d_h^{\sM,p}(x,a)+\veps\cdot{}d_h^{\sM,\pi}(x,a)}}^{1/2}
    \prn*{\sum_{x\in\cX,a\in\cA}(d_h^{\sM,p}(x,a)+\veps\cdot{}d_h^{\sM,\pi}(x,a))g^2(x,a)}^{1/2}\\
  &=\sqrt{\CovM(p)\cdot{}\prn*{\En^{\sM,p}\brk*{g^2(x_h,a_h)}+ \veps\cdot{}\En^{\sM,\pi}\brk*{g^2(x_h,a_h)}}}.
\end{align*}
This establishes \cref{eq:com1}. To prove \cref{eq:com2}, we first
bound
\begin{align*}
  \sqrt{\CovM(p)\cdot{}\prn*{\En^{\sM,p}\brk*{g^2(x_h,a_h)}+
  \veps\cdot{}\En^{\sM,\pi}\brk*{g^2(x_h,a_h)}}}
  &\leq \sqrt{\CovM(p)\cdot{}\En^{\sM,p}\brk*{g^2(x_h,a_h)}} \\
	&\qquad + \sqrt{\CovM(p)\cdot\veps\cdot{}\En^{\sM,\pi}\brk*{g^2(x_h,a_h)}}.
\end{align*}
Next, we note that if $g\in\brk*{0,B}$, we can use AM-GM to bound
\begin{align*}
  \sqrt{\CovM(p)\cdot\veps\cdot{}\En^{\sM,\pi}\brk*{g^2(x_h,a_h)}}
  \leq{}
  \sqrt{\CovM(p)\cdot(\veps{}B)\cdot{}\En^{\sM,\pi}\brk*{g(x_h,a_h)}}
  \leq{} \frac{\CovM(p)\cdot(\veps{}B)}{2} +\frac{1}{2}\En^{\sM,\pi}\brk*{g(x_h,a_h)}.
\end{align*}
The result now follows by rearranging.  
\end{proof}

\linf*

\begin{proof}[\pfref{prop:linf}]
  Let $\delta>0$ be given. Using \cref{lem:pi_to_mu} and the definition of $\CinfhM$, there exists $\mu\in\Delta(\cX\times\cA)$ such that  
\begin{align*}  \CovOptM&\leq
\prn*{1+\frac{\delta}{\veps}}\CinfhM\cdot \inf_{p\in\Delta(\Pi)}\sup_{\pi\in\Pi}\En^{\sM,\pi}\brk*{\frac{\mu(x_h,a_h)}{d_h^{\sM,p}(x_h,a_h)+\delta\cdot{}\CinfhM\mu(x_h,a_h)}}
\\
&=\prn*{1+\frac{\delta}{\veps}}\CinfhM\cdot\inf_{p\in\Delta(\Pi)}\sup_{q\in\Delta(\Pi)}\En_{\pi\sim{}q}\En^{\sM,\pi}\brk*{\frac{\mu(x_h,a_h)}{d_h^{\sM,p}(x_h,a_h)+\delta\cdot{}\CinfhM{}\mu(x_h,a_h)}}.
  \end{align*}
  Observe that the function
  \begin{align*}
(p,q)\mapsto{}\En_{\pi\sim{}q}\En^{\sM,\pi}\brk*{\frac{\mu(x_h,a_h)}{d_h^{\sM,p}(x_h,a_h)+\delta\cdot{}\CinfhM\mu(x_h,a_h)}}
  \end{align*}
  is convex-concave. In addition, it is straightforward to see that
  the function is jointly Lipschitz with respect to total variation
  distance whenever $\veps,\delta>0$. Hence, using the minimax theorem
  (\cref{lem:sion}), we have that
  \begin{align*}
&\inf_{p\in\Delta(\Pi)}\sup_{q\in\Delta(\Pi)}\En_{\pi\sim{}q}\En^{\sM,\pi}\brk*{\frac{\mu(x_h,a_h)}{d_h^{\sM,p}(x_h,a_h)+\delta\cdot{}\CinfhM{}\mu(x_h,a_h)}}\\
    &=\sup_{q\in\Delta(\Pi)}\inf_{p\in\Delta(\Pi)}\En_{\pi\sim{}q}\En^{\sM,\pi}\brk*{\frac{\mu(x_h,a_h)}{d_h^{\sM,p}(x_h,a_h)+\delta\cdot{}\CinfhM{}\mu(x_h,a_h)}}\\
    &\leq\sup_{q\in\Delta(\Pi)}\En_{\pi\sim{}q}\En^{\sM,\pi}\brk*{\frac{\mu(x_h,a_h)}{d_h^{\sM,q}(x_h,a_h)+\delta\cdot{}\CinfhM{}\mu(x_h,a_h)}}\\
    &= \sum_{x\in\cX,a\in\cA}\frac{d_h^{\sM,q}(x,a)\mu(x,a)}{d_h^{\sM,q}(x,a)+\delta\cdot{}\CinfhM{}\mu(x,a)}\leq{}1.
  \end{align*}
  To conclude, we take $\delta\to{}0$.
\end{proof}

\arxiv{\section{Proofs from \creftitle{sec:planning}}}
\icml{\section{Proofs and Additional Details from \creftitle{sec:planning}}}
\label{app:planning}
\icml{
  \subsection{Omitted Algorithms}
  \label{sec:planning_omitted}

\subsection{Examples for \creftitle{alg:linf_relaxation}}
\label{sec:planning_examples}

As discussed in \cref{sec:linf_relaxation}, the
$L_\infty$-coverability relaxation \eqref{eq:linf_relaxation} used by
\cref{alg:linf_relaxation} can be optimized efficiently whenever a
(non-admissible) state-action distribution
$\mu\in\Delta(\cX\times\cA)$ with low $L_\infty$-concentrability
$\CinfhM$ can be computed efficiently for the MDP $M$. Examples of MDP
classes the admit efficiently computable distributions with low
concentrability include:
\begin{itemize}[leftmargin=*]
\item When $M$ is a tabular MDP, the distribution $\mu(x,a) =
  \frac{1}{\abs*{\cX}\abs*{\cA}}$ (which clearly admits a closed form
  representation) achieves $\CinfhM(\mu)\leq\abs*{\cX}\abs*{\cA}$.
\item For a Block MDPs
  \citep{du2019latent,misra2019kinematic,zhang2022efficient,mhammedi2023representation}
  with latent state space $\cS$, emission distribution
  $q:\cS\to\Delta(\cX)$, and decoder $\phi^{\star}:\cX\to\cS$, the
  distribution
  $\mu(x,a)\ldef{}q(x\mid{}\phistar(x))\cdot{}\frac{1}{\abs*{\cS}\abs*{\cA}}$
  achieves $\CinfhM(\mu)\leq\abs*{\cS}\abs*{\cA}$
  \citep{xie2023role}. Again, this distribution admits a closed form
  representation when $M$ is explicitly specified.
\item For low-rank MDPs with the structure in \cref{eq:low_rank}, the
  distribution given by $\mu(x,a) =
  \frac{\nrm*{\psi_h(x)}_2}{\int\nrm*{\psi_h(x')}_2dx'}_\cdot\frac{1}{\abs{\cA}}$
  achieves $\CinfhM(\mu)\leq{}B\abs*{\cA}$ under the standard
  normalization assumption that
  $\int\nrm*{\psi_h(x')}_2dx'\leq{}B$ and $\nrm*{\phi(x,a)}_2\leq{}1$ for some (typically
  dimension-dependent) constant $B>0$
  \citep{mhammedi2023efficient,golowich2023exploring}. Alternatively
  we can compute a set of policies $\pi\ind{1},\ldots,\pi\ind{d}$ that
  form a barycentric spanner for the set
  $\crl*{\En^{\sM,\pi}\brk*{\phi(x_h,a_h)}}_{\pi\in\Pi}$ and choose
  $\mu(x,a)=\frac{1}{d\abs*{\cA}}\sum_{i=1}^{d}d_h^{\sM,\pi\ind{i}}(x)$,
  which achieves $\CinfhM(\mu)\leq{}d\abs*{\cA}$
  \citep{huang2023reinforcement}.\dfcomment{barycentric spanner is
    kinda nontrivial algorithmically---maybe we shouldn't include here?}
\end{itemize}
These examples highlight that for many settings of interest, computing a
covering distribution $\mu\in\Delta(\cX\times\cA)$ when the model is
known is significantly simpler than computing an explicit policy cover
$p\in\Delta(\PiRNS)$, showcasing the utility of
\cref{alg:linf_relaxation}.

}
\subsection{Proofs from \creftitle{sec:linf_relaxation}}
\linfrelaxation*
\begin{proof}[\pfref{prop:linf_relaxation}]
  Fix $\mu$ and abbreviate $\Cinf\equiv\CinfhM(\mu)$.
  Observe that for any $\pi\in\Pi$ and $p\in\Delta(\PiRNS)$,
  \cref{lem:pi_to_mu} implies that we can bound
  \begin{align*}
    &\En^{\sM,\pi}\brk*{\frac{d^{\sM,\pi}_h(x_h,a_h)}{d_h^{\sM,p}(x_h,a_h)+\veps\cdot{}d_h^{\sM,\pi}(x_h,a_h)}}\\
    &\leq
    \En^{\sM,\pi}\brk*{\frac{d^{\sM,\pi}_h(x_h,a_h)}{d_h^{\sM,p}(x_h,a_h)+\veps\cdot{}\Cinf\mu(x_h,a_h)}}
    + \Cinf
      \En^{\sM,\pi}\brk*{\frac{\mu(x_h,a_h)}{d_h^{\sM,p}(x_h,a_h)+\veps\cdot{}\Cinf\mu(x_h,a_h)}}\\
        &\leq
          2\Cinf \En^{\sM,\pi}\brk*{\frac{\mu(x_h,a_h)}{d_h^{\sM,p}(x_h,a_h)+\veps\cdot{}\Cinf\mu(x_h,a_h)}}.
  \end{align*}
For the claim that $\muCovOptM \leq{}1$, see the proof of \cref{prop:linf}.
  
\end{proof}
\linfoptimization*
\begin{proof}[\pfref{thm:linf_optimization}]
  Let us abbreviate $\wt{d}^{t}=\sum_{i<t}d^{\sM,\pi\ind{i}}$. Observe
  that for $T=\frac{1}{\veps}$, we have
  \begin{align*}
\muCovM(p) = \sup_{\pi\in\Pi}\En^{\sss{M},\pi}\brk*{\frac{\mu(x_h,a_h)}{\En_{\pi'\sim{}p}\brk[\big]{\dm{M}{\pi'}_h(x_h,a_h)}+\frac{\Cinf}{T}\mu(x_h,a_h)}}
    = 
  T\cdot\sup_{\pi\in\Pi}\En^{\sss{M},\pi}\brk*{\frac{\mu(x_h,a_h)}{\dtil\ind{T+1}_h(x_h,a_h)+\Cinf\mu(x_h,a_h)}},
  \end{align*}
  and hence it suffices to bound the quantity on the right-hand
  side. Observe that for all $t\in\brk{T}$, we have that
  \begin{align*}
    \sup_{\pi\in\Pi}\En^{\sss{M},\pi}\brk*{\frac{\mu(x_h,a_h)}{\dtil\ind{t}_h(x_h,a_h)+\Cinf\mu(x_h,a_h)}}
    \leq   
    \sup_{\pi\in\Pi}\En^{\sss{M},\pi}\brk*{\frac{\mu(x_h,a_h)}{\dtil\ind{t-1}_h(x_h,a_h)+\Cinf\mu(x_h,a_h)}},
  \end{align*}
and consequently
\begin{align*}
  T\cdot\sup_{\pi\in\Pi}\En^{\sss{M},\pi}\brk*{\frac{\mu(x_h,a_h)}{\dtil\ind{T+1}_h(x_h,a_h)+\Cinf\mu(x_h,a_h)}}
  &\leq
    \sum_{t=1}^{T}\sup_{\pi\in\Pi}\En^{\sss{M},\pi}\brk*{\frac{\mu(x_h,a_h)}{\dtil\ind{t}_h(x_h,a_h)+\Cinf\mu(x_h,a_h)}}\\
  &\leq
    \sum_{t=1}^{T}\En^{\sss{M},\pi\ind{t}}\brk*{\frac{\mu(x_h,a_h)}{\dtil\ind{t}_h(x_h,a_h)+\Cinf\mu(x_h,a_h)}}
    +\vepsapx{}T.
\end{align*}
Finally, we note that
\begin{align*}
  \sum_{t=1}^{T}\En^{\sss{M},\pi\ind{t}}\brk*{\frac{\mu(x_h,a_h)}{\dtil\ind{t}_h(x_h,a_h)+\Cinf\mu(x_h,a_h)}}
  & = \sum_{x\in\cX,a\in\cA}\sum_{t=1}^{T}\mu(x,a)\frac{d^{\sM,\pi\ind{t}}(x,a)}{\dtil\ind{t}_h(x,a)+\Cinf\mu(x,a)}.
\end{align*}
Since $\sup_{\pi\in\Pi}d_h^{\sM,\pi}(x,a)\leq\Cinf\mu(x,a)$ for all
$(x,a)\in\cX\times\cA$, \cref{lem:elliptic_potential} implies that
\begin{align*}
  \sum_{x\in\cX,a\in\cA}\sum_{t=1}^{T}\mu(x,a)\frac{d^{\sM,\pi\ind{t}}(x,a)}{\dtil\ind{t}_h(x,a)+\Cinf\mu(x,a)}
  \leq{}2\log(2T),
\end{align*}
allowing us to conclude that
\begin{align*}
  \muCovM(p) 
  \leq{} 2\log(2T) + \vepsapx{}T.
\end{align*}
  
\end{proof}

\subsection{Proofs from \creftitle{sec:pushforward_relaxation}}

\pushforwardrelaxation*
\begin{proof}[\pfref{prop:pushforward_relaxation}]
  We first note that
  \begin{align*}
    \CovM(p')
    \leq{} \abs*{\cA}\cdot{}
    \sup_{\pi\in\Pi}\En^{\sM,\pi}\brk*{\frac{d^{\sM,\pi}_h(x_h)}{d_h^{\sM,p}(x_h)+\veps\cdot{}d_h^{\sM,\pi}(x_h)}}.
  \end{align*}
Next, we write
  \begin{align*}
    \En^{\sM,\pi}\brk*{\frac{d^{\sM,\pi}_h(x_h)}{d_h^{\sM,p}(x_h)+\veps\cdot{}d_h^{\sM,\pi}(x_h)}}
      =
    \sum_{x\in\cX}\frac{(d^{\sM,\pi}_h(x))^2}{d_h^{\sM,p}(x)+\veps\cdot{}d_{h}^{\sM,\pi}(x)}.
  \end{align*}
  We now state and prove a basic technical lemma.
\begin{lemma}
  \label{lem:convex_f}
  For all $\veps,\delta>0$, the function
  $f(x) = \frac{x^2}{\delta+\veps{}x}$
  is convex over $\bbR_{+}$.
\end{lemma}
\begin{proof}[\pfref{lem:convex_f}]
  This follows by verifying through direct calculation that
  \begin{align*}
    f'(x) = \veps\cdot\frac{x^2}{(\delta+\veps{}x)^2},\mathand
    f''(x) = 4\veps{}\delta{}\cdot\frac{x}{(\delta+\veps{}x)^3}\geq{}0.
  \end{align*}
\end{proof}
  By \cref{lem:convex_f}, the function
  \[
    d\mapsto{}
    \frac{(d)^2}{d_h^{\sM,p}(x)+\veps\cdot{}d}
  \]
  is convex for all $x$. Hence, writing
  $d^{\sM,\pi}_h(x)=\En^{\sM,\pi}\brk*{P_{h-1}^{\sM}(x\mid{}x_{h-1},a_{h-1})}$,
  Jensen's inequality implies that for all $x$,
  \begin{align*}
    \frac{(d^{\sM,\pi}_h(x))^2}{d_h^{\sM,p}(x)+\veps\cdot{}d_{h}^{\sM,\pi}(x)}
    \leq{}\En^{\sM,\pi}\brk*{
    \frac{(P_{h-1}^{\sM}(x\mid{}x_{h-1},a_{h-1}))^2}{d_h^{\sM,p}(x)+\veps\cdot{}P_{h-1}^{\sM}(x\mid{}x_{h-1},a_{h-1}) }
    }.
  \end{align*}
  We conclude that
  \begin{align*}
    \En^{\sM,\pi}\brk*{\frac{d^{\sM,\pi}_h(x_h)}{d_h^{\sM,p}(x_h)+\veps\cdot{}d_h^{\sM,\pi}(x_h)}}
    &\leq{} \En^{\sM,\pi}\brk*{\sum_{x\in\cX}
    \frac{(P_{h-1}^{\sM}(x\mid{}x_{h-1},a_{h-1}))^2}{d_h^{\sM,p}(x)+\veps\cdot{}P_{h-1}^{\sM}(x\mid{}x_{h-1},a_{h-1}) }
    }\\
    &= \En^{\sM,\pi}\brk*{
    \frac{P_{h-1}^{\sM}(x_h\mid{}x_{h-1},a_{h-1})}{d_h^{\sM,p}(x)+\veps\cdot{}P_{h-1}^{\sM}(x_h\mid{}x_{h-1},a_{h-1}) }
    } \leq \pCovM(p).
  \end{align*}

We now prove the bound on $\pCovOptM$.    Let $\delta>0$ be
given. Using the definition of $\CpushhM$ and the same argument as \cref{lem:pi_to_mu}, there exists $\mu\in\Delta(\cX)$ such that  
\begin{align*}  \pCovOptM&\leq
\prn*{1+\frac{\delta}{\veps}}\CpushhM\cdot \inf_{p\in\Delta(\Pi)}\sup_{\pi\in\Pi}\En^{\sM,\pi}\brk*{\frac{\mu(x_h)}{d_h^{\sM,p}(x_h)+\delta\cdot{}\CpushhM\mu(x_h)}}
\\
&=\prn*{1+\frac{\delta}{\veps}}\CpushhM\cdot\inf_{p\in\Delta(\Pi)}\sup_{q\in\Delta(\Pi)}\En_{\pi\sim{}q}\En^{\sM,\pi}\brk*{\frac{\mu(x_h)}{d_h^{\sM,p}(x_h)+\delta\cdot{}\CpushhM{}\mu(x_h)}}.
  \end{align*}
  Observe that the function
  \begin{align*}
(p,q)\mapsto{}\En_{\pi\sim{}q}\En^{\sM,\pi}\brk*{\frac{\mu(x_h)}{d_h^{\sM,p}(x_h)+\delta\cdot{}\CpushhM\mu(x_h)}}
  \end{align*}
  is convex-concave. In addition, it is straightforward to see that
  the function is jointly Lipschitz with respect to total variation
  distance whenever $\veps,\delta>0$. Hence, using the minimax theorem
  (\cref{lem:sion}), we have that
  \begin{align*}
&\inf_{p\in\Delta(\Pi)}\sup_{q\in\Delta(\Pi)}\En_{\pi\sim{}q}\En^{\sM,\pi}\brk*{\frac{\mu(x_h)}{d_h^{\sM,p}(x_h)+\delta\cdot{}\CpushhM{}\mu(x_h)}}\\
    &=\sup_{q\in\Delta(\Pi)}\inf_{p\in\Delta(\Pi)}\En_{\pi\sim{}q}\En^{\sM,\pi}\brk*{\frac{\mu(x_h)}{d_h^{\sM,p}(x_h)+\delta\cdot{}\CpushhM{}\mu(x_h)}}\\
    &\leq\sup_{q\in\Delta(\Pi)}\En_{\pi\sim{}q}\En^{\sM,\pi}\brk*{\frac{\mu(x_h)}{d_h^{\sM,q}(x_h)+\delta\cdot{}\CpushhM{}\mu(x_h)}}\\
    &= \sum_{x\in\cX}\frac{d_h^{\sM,q}(x)\mu(x)}{d_h^{\sM,q}(x)+\delta\cdot{}\CpushhM{}\mu(x)}\leq{}1.
  \end{align*}
  To conclude, we take $\delta\to{}0$.

\end{proof}

\pushforwardoptimization*
\begin{proof}[\pfref{thm:pushforward_optimization}]
  Let us abbreviate $\wt{d}\ind{t}_h=\sum_{i<t}d_h^{\sM,\pi\ind{i}}$. Observe
  that for $T=\frac{1}{\veps}$, we have
  \begin{align*}
\pCovM(p) &= \sup_{\pi\in\Pi}\En^{\sss{M},\pi}\brk*{\frac{P_{h-1}^{\sM}(x_h\mid{}x_{h-1},a_{h-1})}{\En_{\pi'\sim{}p}\brk[\big]{\dm{M}{\pi'}_h(x_h)}+\frac{1}{T}P_{h-1}^{\sM}(x_h\mid{}x_{h-1},a_{h-1})}}\\
    &= 
T\cdot\sup_{\pi\in\Pi}\En^{\sss{M},\pi}\brk*{\frac{P_{h-1}^{\sM}(x_h\mid{}x_{h-1},a_{h-1})}{\dtil\ind{T+1}_h(x_h)+P_{h-1}^{\sM}(x_h\mid{}x_{h-1},a_{h-1})}},
  \end{align*}
  and hence it suffices to bound the quantity on the right-hand
  side. Observe that for all $t\in\brk{T}$, we have that
  \begin{align*}
    \sup_{\pi\in\Pi}\En^{\sss{M},\pi}\brk*{\frac{P_{h-1}^{\sM}(x_h\mid{}x_{h-1},a_{h-1})}{\dtil\ind{t}_h(x_h)+P_{h-1}^{\sM}(x_h\mid{}x_{h-1},a_{h-1})}}
    \leq   
    \sup_{\pi\in\Pi}\En^{\sss{M},\pi}\brk*{\frac{P_{h-1}^{\sM}(x_h\mid{}x_{h-1},a_{h-1})}{\dtil\ind{t-1}_h(x_h)+P_{h-1}^{\sM}(x_h\mid{}x_{h-1},a_{h-1})}},
  \end{align*}
and consequently
\begin{align*}
T\cdot\sup_{\pi\in\Pi}\En^{\sss{M},\pi}\brk*{\frac{P_{h-1}^{\sM}(x_h\mid{}x_{h-1},a_{h-1})}{\dtil\ind{T+1}_h(x_h)+P_{h-1}^{\sM}(x_h\mid{}x_{h-1},a_{h-1})}}
  &\leq
    \sum_{t=1}^{T}\sup_{\pi\in\Pi}\En^{\sss{M},\pi}\brk*{\frac{P_{h-1}^{\sM}(x_h\mid{}x_{h-1},a_{h-1})}{\dtil\ind{t}_h(x_h)+P_{h-1}^{\sM}(x_h\mid{}x_{h-1},a_{h-1})}}\\
  &\leq
\sum_{t=1}^{T}\En^{\sss{M},\pi\ind{t}}\brk*{\frac{P_{h-1}^{\sM}(x_h\mid{}x_{h-1},a_{h-1})}{\dtil\ind{t}_h(x_h)+P_{h-1}^{\sM}(x_h\mid{}x_{h-1},a_{h-1})}}
    +\vepsapx{}T.
\end{align*}
  Now, let $\mu\in\Delta(\cX)$ attain the value of $\CpushhM$. Using
  \cref{lem:pi_to_mu_pushforward} with $\veps=1$ and $\delta=\CpushhM$,
  we have that for all $\pi\in\Pi$, 
  \begin{align*} 
    \En^{\sM,\pi}\brk*{\frac{P_{h-1}^{\sM}(x_h\mid{}x_{h-1},a_{h-1})}{\dtil_h\ind{t}(x_h)+P_{h-1}^{\sM}(x_h\mid{}x_{h-1},a_{h-1})}}
    &\leq
      \En^{\sM,\pi}\brk*{\frac{P_{h-1}^{\sM}(x_h\mid{}x_{h-1},a_{h-1})}{\dtil_h\ind{t}(x_h) +\CpushhM\mu(x_h)}}
      +
    \CpushhM\cdot\En^{\sM,\pi}\brk*{\frac{\mu(x_h)}{\dtil_h\ind{t}(x_h)+\CpushhM\mu(x_h)}}\\
    &\leq
      2\CpushhM\cdot\En^{\sM,\pi}\brk*{\frac{\mu(x_h)}{\dtil_h\ind{t}(x_h)+\CpushhM\mu(x_h)}}.
  \end{align*}
Hence, we can bound
\begin{align*}
  \sum_{t=1}^{T}\En^{\sss{M},\pi\ind{t}}\brk*{\frac{P_{h-1}^{\sM}(x_h\mid{}x_{h-1},a_{h-1})}{\dtil\ind{t}_h(x_h)+P_{h-1}^{\sM}(x_h\mid{}x_{h-1},a_{h-1})}}
  &\leq 2\CpushhM\sum_{t=1}^{T}\En^{\sss{M},\pi\ind{t}}\brk*{\frac{\mu(x_h)}{\dtil\ind{t}_h(x_h)+\CpushhM\mu(x_h)}}\\
  & = 2\CpushhM\sum_{x\in\cX}\sum_{t=1}^{T}\mu(x)\frac{d^{\sM,\pi\ind{t}}(x)}{\dtil\ind{t}_h(x)+\CpushhM\mu(x)}.
\end{align*}
Since $\sup_{\pi\in\Pi}d_h^{\sM,\pi}(x)
\leq{}\sup_{x'\in\cX,a\in\cA}P^{\sM}_{h-1}(x\mid{}x',a)\leq\CpushhM\mu(x)$ for all
$x\in\cX$, \cref{lem:elliptic_potential} implies that
\begin{align*}
  \sum_{x\in\cX,}\sum_{t=1}^{T}\mu(x)\frac{d^{\sM,\pi\ind{t}}(x)}{\dtil\ind{t}_h(x)+\Cinf\mu(x)}
  \leq{}2\log(2T).
\end{align*}
We conclude that
\begin{align}
  \label{eq:pushforward_potential}
\sum_{t=1}^{T}\En^{\sss{M},\pi\ind{t}}\brk*{\frac{P_{h-1}^{\sM}(x_h\mid{}x_{h-1},a_{h-1})}{\dtil\ind{t}_h(x_h)+P_{h-1}^{\sM}(x_h\mid{}x_{h-1},a_{h-1})}}
  \leq 4\CpushhM\log(2T)
\end{align}
and
$\pCovM(p) 
  \leq{} 4\CpushhM\log(2T) + \vepsapx{}T$.
\end{proof}

\arxiv{\section{Proofs and Additional Details from \creftitle{sec:model_based}}}
\icml{\section{Proofs and Additional Details from \creftitle{sec:model_based}}}
\label{app:model_based}
\arxiv{
  This section is organized as follows:
  \begin{itemize}
  \item \cref{sec:model_based_general} presents our most general
    guarantee for \cref{alg:model_based_reward_free},
    \cref{thm:model_based_reward_free}, and derives sample complexity
    bounds based on \mainopt as a consequence.
  \item \cref{app:rf_example} presents applications of
    these results to downstream policy optimization.
  \item \cref{sec:model_based_technical} presents preliminary
    technical lemmas.
  \item \cref{app:model_based_proof} proves
    \cref{thm:model_based_reward_free}, proving \cref{cor:mbrf3,cor:mbrf2} as a corollary.
  \end{itemize}
}
\icml{
  This section is organized as follows:
  \begin{itemize}
  \item \cref{sec:model_based_general} presents our most general
    guarantee for \cref{alg:model_based_reward_free},
    \cref{thm:model_based_reward_free}, and derives sample complexity
    bounds based on \mainopt as a consequence.
  \item \cref{app:rf_example} presents applications of
    these results to downstream policy optimization.
  \item \cref{sec:model_based_technical} presents preliminary
    technical lemmas.
  \item \cref{app:model_based_proof} proves
    \cref{thm:model_based_reward_free}, proving \cref{cor:mbrf3} as a corollary.
  \end{itemize}
  }

\arxiv{
\subsection{General Guarantees for
    \creftitle{alg:model_based_reward_free}}
    \label{sec:model_based_general}
In this section, we present general guarantees for
\mainalg (\creftitle{alg:model_based_reward_free}) that (i) make us of online (as
opposed to offline) estimation oracles, allowing for faster rates, and
(ii) enjoy sample complexity scaling with \mainopt, improving upon the
$L_\infty$-Coverability-based guarantees in \cref{sec:model_based}.

\subsubsection{Online Estimation Oracles}
For \emph{online estimation}, we measure the oracle's estimation
performance in terms of cumulative Hellinger error, which we assume is
bounded as follows.
\begin{assumption}[Online estimation oracle for $\cM$]
    \label{ass:online_oracle}
	At each time $t\in[T]$, an online estimation oracle
        $\AlgEst$ for $\cM$ returns,
        given \arxiv{$$\hist\ind{t-1}=(\pi\ind{1},o\ind{1}),\ldots,(\pi\ind{t-1},o\ind{t-1})$$}\icml{$\hist\ind{t-1}=(\pi\ind{1},o\ind{1}),\ldots,(\pi\ind{t-1},o\ind{t-1})$}
        with $o\ind{i}\sim\Mstar(\pi\ind{i})$ and
        $\pi\ind{i}\sim p\ind{i}$, an estimator
        $\Mhat\ind{t}\in\cM$ such that whenever $\Mstar\in\cM$,
        \begin{align*}
              \EstOn \ldef{}
          \sum_{t=1}^{T}\En_{\act\ind{t}\sim{}p\ind{t}}\brk[\big]{\DhelsX{\big}{\Mhat\ind{t}(\pi\ind{t})}{\Mstar(\pi\ind{t})}}
          \leq \EstProbOn,
        \end{align*}
	with probability at least $1-\delta$, where $\EstProbOn$ is a
        known upper bound.%
      \end{assumption}
      See Section 4 of
      \citet{foster2021statistical} or \citet{foster2023lecture} for
      further background on online estimation.
      \cref{alg:model_based_reward_free} supports offline and online
      estimators, but is most straightforward to analyze for online
      estimators, and gives tighter sample complexity bounds in
      this case. The requirement in \cref{ass:online_oracle} that
      the online estimator is \emph{proper} (i.e., has
      $\Mhat\ind{t}\in\cM$) is quite stringent, as generic online
      estimation algorithms (e.g., Vovk's aggregating algorithm) are
      improper, and proper algorithms are only known for specialized
      MDP classes such as tabular MDPs (see discussion in
      \citet{foster2021statistical}).\footnote{On the statistical
        side, it is straightforward to extend the results in this
        section to accommodate improper online estimators; we impose
        this restriction for \emph{computational} reasons, as this
        enables the application of the efficient planning results in
        \cref{sec:planning}.} This contrasts with offline estimation,
      where most standard algorithms such as \mle are proper. As such,
      we present bounds based on online estimators as secondary results, with our
      bounds based on offline estimation serving as the main results.

\paragraph{Offline-to-online conversion}      
On the technical side, our interest in proper online estimation arises from
the following structural result, which shows that whenever the
\mainopt parameter is bounded, any algorithm with low offline estimation
error also enjoys low online estimation error (with polynomial loss in rate).
      \begin{restatable}[Offline-to-online]{lemma}{loneofflineonline}
        \label{lem:lone_offline_online}
          Any offline estimator $\Mhat\ind{t}$ that satisfies
  \cref{ass:offline_oracle} with estimation error bound $\EstProbOff$
  satisfies \cref{ass:online_oracle} with \arxiv{estimation error bound}
    \icml{
      \mbox{$\EstProbOn \leq \bigoht\prn[\big]{H\prn[\big]{\ConeM[\Mstar](1+\EstProbOff)}^{1/3}T^{2/3}}$}.}
  \arxiv{
  \begin{align}
    \EstProbOn \leq \bigoh\prn[\Big]{H\log H \prn*{\ConeM[\Mstar](1+\EstProbOff)}^{1/3}T^{2/3}
    }.
  \end{align}
  }
\end{restatable}
Note that $\ConeM[\Mstar]\leq\CovOptMmaxzero[\Mstar]$; we leave an
extension to $\CovOptMmax[\Mstar]$ for $\veps>0$ to future
work. \dfc{Can we handle the general $\veps$ extension?}\\
We also make
use of a tighter offline-to-online lemma based on the (larger)
$L_{\infty}$-Coverability parameter $\CinfM[\Mstar]$.
\begin{restatable}[\citet{xie2023role,foster2023online}]{lemma}{linfofflineonline}
  \label{lem:linf_offline_online}
  Any offline estimator $\Mhat\ind{t}$ that satisfies
  \cref{ass:offline_oracle} with estimation error bound $\EstProbOff$
  satisfies \cref{ass:online_oracle} with \arxiv{estimation error
    bound}
    \icml{
      \mbox{$\EstProbOn \leq \bigoht\prn[\big]{H\prn[\big]{\CinfM[\Mstar]
    T\cdot\EstProbOff}^{1/2} + H\CinfM[\Mstar] 
    }$}.}
  \arxiv{
  \begin{align}
    \EstProbOn \leq \bigoh\prn*{H\log H \sqrt{\CinfM[\Mstar]
    T\log{}T\cdot\EstProbOff} + H\log{}H\cdot\CinfM[\Mstar] 
    }.
  \end{align}
  }
\end{restatable}
Both lemmas lead to a degradation in rate with respect to
$T$, but lead to sublinear online estimation error whenever the
offline estimation error bound is sublinear\arxiv{; \citet{foster2023online}
  show that some degradation in rate is unavoidable}.

\subsubsection{General Guarantees for \creftitle{alg:model_based_reward_free}}\label{app:general-guarantees-model-based-reward-free}
Our most general guarantee for \cref{alg:model_based_reward_free},
which assumes access to an online estimation oracle, is as follows.
\begin{restatable}[General guarantee for \cref{alg:model_based_reward_free}]{theorem}{modelbasedrewardfree}
  \label{thm:model_based_reward_free}
  With parameters $T\in\bbN$, $C\geq{}1$, and $\veps>0$ and an online
estimation oracle satisfying \cref{ass:online_oracle}, whenever the
optimization problem in \cref{eq:cover_objective_model_based} is
feasible at every round, 
\cref{alg:model_based_reward_free} produces a policy covers $p_1,\ldots,p_H\in\Delta(\Pi)$ such
that with probability at least $1-\delta$,
\icml{$  \forall{}h\in\brk{H}$: $\CovM[\Mstar](p_h) \leq$
  \begin{small}
    \begin{align}
                                                 11HC
      +\frac{12}{\veps}\sqrt{\frac{H^3C\cdot{}\EstProbOn}{T}}
        +\frac{8H}{\veps^2}\frac{\EstProbOn}{T}.   \label{eq:mbrf1}
    \end{align}%
  \end{small}%
}%
\arxiv{\begin{align}
  \label{eq:mbrf1}
  \forall{}h\in\brk{H}:\quad\CovM[\Mstar](p_h)
\leq                                                  11HC
+\frac{12}{\veps}\sqrt{\frac{H^3C\cdot{}\EstProbOn}{T}}
                                                +\frac{8H}{\veps^2}\frac{\EstProbOn}{T}.
\end{align}}
\end{restatable}
\arxiv{
\cref{cor:mbrf3} is derived by combining this result with
\cref{lem:linf_offline_online}. \cref{cor:mbrf2} is derived by combining this result with \cref{lem:lone_offline_online}, 
allowing us to give sample complexity
guarantees based on \mainopt that support offline
estimation oracles. 
This result shows that \mainopt is itself a sufficiently powerful structural parameter to
  enable sample-efficient learning with nonlinear function
  approximation. Note that while \cref{cor:mbrf2} assumes for simplicity that
\cref{eq:cover_objective_model_based} is solved with
$C=\CovOptMmax[\Mstar]$, it should be clear that if we solve the
objective for $C>\CovOptMmax[\Mstar]$ the result continues to hold
with $\CovOptMmax[\Mstar]$ replaced by $C$ in the sample complexity
bound and approximation guarantee.
}

\icml{
\cref{cor:mbrf3} is derived by combining this result with
\cref{lem:linf_offline_online}. The next result instantiates
\cref{thm:model_based_reward_free} with
\cref{lem:lone_offline_online}, allowing us to give sample complexity
guarantees based on \mainopt that support offline
estimation oracles. 
\begin{restatable}[Main guarantee for
  \cref{alg:model_based_reward_free} under \mainopt]{corollary}{mbrftwo}
  \label{cor:mbrf2}
  Let $\veps>0$ be given.
    Suppose that (i) we restrict $\cM$ such that all $M\in\cM$ have
    $\CovOptMmax\leq\CovOptMmax[\Mstar]$, and (ii) we solve
    \cref{eq:cover_objective_model_based} with
    $C=\CovOptMmax[\Mstar]$ for all $t$ (which is always
    feasible). Then, given access to an offline estimation oracle
    satisfying \cref{ass:offline_oracle,ass:parametric}, using
    $T=\bigoht\prn*{\frac{H^{12}(\CovOptMmaxzero[\Mstar])^4\dest\log(\Best/\delta)}{\veps^{6}}}$
    episodes, \cref{alg:model_based_reward_free} produces policy
    covers $p_1,\ldots,p_H\in\Delta(\Pi)$ such that
    \begin{align}
      \label{eq:mbrf2}
      \forall{}h\in\brk{H}:\quad\CovM[\Mstar](p_h)
      \leq 12H\cdot{}\CovOptMmax[\Mstar]
    \end{align}
    with probability at least $1-\delta$. In particular, for a finite
    class $\cM$, if we use \mle as the estimator, we can take
    $T=\bigoht\prn*{\frac{H^{12}(\CovOptMmaxzero[\Mstar])^4\log(\abs*{\cM}/\delta)}{\veps^{6}}}$.
  \end{restatable}
  \dfcomment{It seems hard to improve the result above so that 1) we have $\CovOptMmax[\Mstar]$ in the sample complexity instead of $\CovOptMmaxzero[\Mstar]$, and 2) so that we have $\CovOptM[\Mstar]$ on the rhs of \cref{eq:mbrf2} instead of $\CovOptMmax[\Mstar]$. do people find these issues offputting?}
This result shows that \mainopt is itself a sufficiently powerful structural parameter to
  enable sample-efficient learning with nonlinear function
  approximation.
  
}

}

\icml{
\subsection{General Guarantees for
    \creftitle{alg:model_based_reward_free}}
    \label{sec:model_based_general}
In this section, we present general guarantees for
\mainalg (\creftitle{alg:model_based_reward_free}) that (i) make us of online (as
opposed to offline) estimation oracles, allowing for faster rates, and
(ii) enjoy sample complexity scaling with \mainopt, improving upon the
$L_\infty$-Coverability-based guarantees in \cref{sec:model_based}.

\subsubsection{Online Estimation Oracles}
For \emph{online estimation}, we measure the oracle's estimation
performance in terms of cumulative Hellinger error, which we assume is
bounded as follows.
\begin{assumption}[Online estimation oracle for $\cM$]
    \label{ass:online_oracle}
	At each time $t\in[T]$, an online estimation oracle
        $\AlgEst$ for $\cM$ returns,
        given \arxiv{$$\hist\ind{t-1}=(\pi\ind{1},o\ind{1}),\ldots,(\pi\ind{t-1},o\ind{t-1})$$}\icml{$\hist\ind{t-1}=(\pi\ind{1},o\ind{1}),\ldots,(\pi\ind{t-1},o\ind{t-1})$}
        with $o\ind{i}\sim\Mstar(\pi\ind{i})$ and
        $\pi\ind{i}\sim p\ind{i}$, an estimator
        $\Mhat\ind{t}\in\cM$ such that whenever $\Mstar\in\cM$,
        \begin{align*}
              \EstOn \ldef{}
          \sum_{t=1}^{T}\En_{\act\ind{t}\sim{}p\ind{t}}\brk[\big]{\DhelsX{\big}{\Mhat\ind{t}(\pi\ind{t})}{\Mstar(\pi\ind{t})}}
          \leq \EstProbOn,
        \end{align*}
	with probability at least $1-\delta$, where $\EstProbOn$ is a
        known upper bound.%
      \end{assumption}
      See Section 4 of
      \citet{foster2021statistical} or \citet{foster2023lecture} for
      further background on online estimation.
      \cref{alg:model_based_reward_free} supports offline and online
      estimators, but is most straightforward to analyze for online
      estimators, and gives tighter sample complexity bounds in
      this case. The requirement in \cref{ass:online_oracle} that
      the online estimator is \emph{proper} (i.e., has
      $\Mhat\ind{t}\in\cM$) is quite stringent, as generic online
      estimation algorithms (e.g., Vovk's aggregating algorithm) are
      improper, and proper algorithms are only known for specialized
      MDP classes such as tabular MDPs (see discussion in
      \citet{foster2021statistical}).\footnote{On the statistical
        side, it is straightforward to extend the results in this
        section to accommodate improper online estimators; we impose
        this restriction for \emph{computational} reasons, as this
        enables the application of the efficient planning results in
        \cref{sec:planning}.} This contrasts with offline estimation,
      where most standard algorithms such as \mle are proper. As such,
      we present bounds based on online estimators as secondary results, with our
      bounds based on offline estimation serving as the main results.

\paragraph{Offline-to-online conversion}      
On the technical side, our interest in proper online estimation arises from
the following structural result, which shows that whenever the
\mainopt parameter is bounded, any algorithm with low offline estimation
error also enjoys low online estimation error (with polynomial loss in rate).
      \begin{restatable}[Offline-to-online]{lemma}{loneofflineonline}
        \label{lem:lone_offline_online}
          Any offline estimator $\Mhat\ind{t}$ that satisfies
  \cref{ass:offline_oracle} with estimation error bound $\EstProbOff$
  satisfies \cref{ass:online_oracle} with \arxiv{estimation error bound}
    \icml{
      \mbox{$\EstProbOn \leq \bigoht\prn[\big]{H\prn[\big]{\ConeM[\Mstar](1+\EstProbOff)}^{1/3}T^{2/3}}$}.}
  \arxiv{
  \begin{align}
    \EstProbOn \leq \bigoh\prn[\Big]{H\log H \prn*{\ConeM[\Mstar](1+\EstProbOff)}^{1/3}T^{2/3}
    }.
  \end{align}
  }
\end{restatable}
Note that $\ConeM[\Mstar]\leq\CovOptMmaxzero[\Mstar]$; we leave an
extension to $\CovOptMmax[\Mstar]$ for $\veps>0$ to future
work. \dfc{Can we handle the general $\veps$ extension?}\\
We also make
use of a tighter offline-to-online lemma based on the (larger)
$L_{\infty}$-Coverability parameter $\CinfM[\Mstar]$.
\begin{restatable}[\citet{xie2023role}]{lemma}{linfofflineonline}
  \label{lem:linf_offline_online}
  Any offline estimator $\Mhat\ind{t}$ that satisfies
  \cref{ass:offline_oracle} with estimation error bound $\EstProbOff$
  satisfies \cref{ass:online_oracle} with \arxiv{estimation error
    bound}
    \icml{
      \mbox{$\EstProbOn \leq \bigoht\prn[\big]{H\prn[\big]{\CinfM[\Mstar]
    T\cdot\EstProbOff}^{1/2} + H\CinfM[\Mstar] 
    }$}.}
  \arxiv{
  \begin{align}
    \EstProbOn \leq \bigoh\prn*{H\log H \sqrt{\CinfM[\Mstar]
    T\log{}T\cdot\EstProbOff} + H\log{}H\cdot\CinfM[\Mstar] 
    }.
  \end{align}
  }
\end{restatable}
Both lemmas lead to a degradation in rate with respect to
$T$, but lead to sublinear online estimation error whenever the
offline estimation error bound is sublinear\arxiv{; \citet{foster2023online}
  show that some degradation in rate is unavoidable}.

\subsubsection{General Guarantees for \creftitle{alg:model_based_reward_free}}
Our most general guarantee for \cref{alg:model_based_reward_free},
which assumes access to an online estimation oracle, is as follows.
\begin{restatable}[General guarantee for \cref{alg:model_based_reward_free}]{theorem}{modelbasedrewardfree}
  \label{thm:model_based_reward_free}
  With parameters $T\in\bbN$, $C\geq{}1$, and $\veps>0$ and an online
estimation oracle satisfying \cref{ass:online_oracle}, whenever the
optimization problem in \cref{eq:cover_objective_model_based} is
feasible at every round, 
\cref{alg:model_based_reward_free} produces a policy covers $p_1,\ldots,p_H\in\Delta(\Pi)$ such
that with probability at least $1-\delta$,
\icml{$  \forall{}h\in\brk{H}$: $\CovM[\Mstar](p_h) \leq$
  \begin{small}
    \begin{align}
                                                 11HC
      +\frac{12}{\veps}\sqrt{\frac{H^3C\cdot{}\EstProbOn}{T}}
        +\frac{8H}{\veps^2}\frac{\EstProbOn}{T}.   \label{eq:mbrf1}
    \end{align}%
  \end{small}%
}%
\arxiv{\begin{align}
  \label{eq:mbrf1}
  \forall{}h\in\brk{H}:\quad\CovM[\Mstar](p_h)
\leq                                                  11HC
+\frac{12}{\veps}\sqrt{\frac{H^3C\cdot{}\EstProbOn}{T}}
                                                +\frac{8H}{\veps^2}\frac{\EstProbOn}{T}.
\end{align}}
\end{restatable}
\cref{cor:mbrf3} is derived by combining this result with
\cref{lem:linf_offline_online}. The next result instantiates
\cref{thm:model_based_reward_free} with
\cref{lem:lone_offline_online}, allowing us to give sample complexity
guarantees based on \mainopt that support offline
estimation oracles. 
\begin{restatable}[Main guarantee for
  \cref{alg:model_based_reward_free} under \mainopt]{corollary}{mbrftwo}
  \label{cor:mbrf2}
  Let $\veps>0$ be given.
    Suppose that (i) we restrict $\cM$ such that all $M\in\cM$ have
    $\CovOptMmax\leq\CovOptMmax[\Mstar]$, and (ii) we solve
    \cref{eq:cover_objective_model_based} with
    $C=\CovOptMmax[\Mstar]$ for all $t$ (which is always
    feasible). Then, given access to an offline estimation oracle
    satisfying \cref{ass:offline_oracle,ass:parametric}, using
    $T=\bigoht\prn*{\frac{H^{12}(\CovOptMmaxzero[\Mstar])^4\dest\log(\Best/\delta)}{\veps^{6}}}$
    episodes, \cref{alg:model_based_reward_free} produces policy
    covers $p_1,\ldots,p_H\in\Delta(\Pi)$ such that
    \begin{align}
      \label{eq:mbrf2}
      \forall{}h\in\brk{H}:\quad\CovM[\Mstar](p_h)
      \leq 12H\cdot{}\CovOptMmax[\Mstar]
    \end{align}
    with probability at least $1-\delta$. In particular, for a finite
    class $\cM$, if we use \mle as the estimator, we can take
    $T=\bigoht\prn*{\frac{H^{12}(\CovOptMmaxzero[\Mstar])^4\log(\abs*{\cM}/\delta)}{\veps^{6}}}$.
  \end{restatable}
  \dfcomment{It seems hard to improve the result above so that 1) we have $\CovOptMmax[\Mstar]$ in the sample complexity instead of $\CovOptMmaxzero[\Mstar]$, and 2) so that we have $\CovOptM[\Mstar]$ on the rhs of \cref{eq:mbrf2} instead of $\CovOptMmax[\Mstar]$. do people find these issues offputting?}
This result shows that \mainopt is itself a sufficiently powerful structural parameter to
  enable sample-efficient learning with nonlinear function
  approximation. Note that while \cref{cor:mbrf2} assumes for simplicity that
\cref{eq:cover_objective_model_based} is solved with
$C=\CovOptMmax[\Mstar]$, it should be clear that if we solve the
objective for $C>\CovOptMmax[\Mstar]$ the result continues to hold
with $\CovOptMmax[\Mstar]$ replaced by $C$ in the sample complexity
bound and approximation guarantee.

}

\arxiv{\subsection{Applying \creftitle{alg:model_based_reward_free} to Downstream Policy Optimization}
  \label{app:rf_example}
By \cref{lem:com} (see also \cref{app:downstream}), the policy covers
$p_1,\ldots,p_H$ returned by \cref{alg:model_based_reward_free} can
be used to optimize any downstream reward function using standard
offline RL algorithms. This leads to end-to-end guarantees for
reward-driven PAC RL. For concreteness, we sketch an example which
uses maximum likelihood (\mle) for offline policy optimization; see
\cref{app:downstream} for further examples and details.
\begin{corollary}[Application to reward-free reinforcement learning]
  \label{cor:rf_teaser}
Given access
  to $H\cdot{}n$ trajectories from the policy covers $p_1,\ldots,p_H$ produced by
  \cref{alg:model_based_reward_free} (configured as in
  \cref{cor:mbrf3}) and a realizable model class with $\Mstar\in\cM$,
  for any reward distribution $R=\crl*{R_h}_{h=1}^{H}$ with $\sum_{h=1}^{H}r_h\in\brk{0,1}$, the Maximum Likelihood Estimation algorithm
  (described in \cref{app:downstream}) produces a policy $\pihat$ such
  that with probability at least $1-\delta$,
  \icml{$            \Jmstar_R(\pistar) - \Jmstar_R(\pihat)
    \leq
    \bigoh(H)\cdot\prn[\Big]{\sqrt{H\CinfM[\Mstar]\cdot\frac{\log(\abs{\cM}/\delta)}{n}}
          + H\CinfM[\Mstar]\cdot{}\veps{}}$,
    }
  \arxiv{\begin{align}
    \label{eq:mle_teaser}
    \Jmstar_R(\pistar) - \Jmstar_R(\pihat)
    \leq
    \bigoh(H)\cdot\prn*{\sqrt{H\CovOptMmax[\Mstar]\cdot\frac{\log(\abs{\cM}/\delta)}{n}}
          + H\CovOptMmax[\Mstar]\cdot{}\veps{}},
    \end{align}}
    where
    $\Jmstar_R(\pi)\ldef{}\En^{\sMstar,\pi}\brk*{\sum_{h=1}^{H}r_h}$
    denotes the expected reward in $\Mstar$ when $R$ is the reward distribution.
\end{corollary}
\arxiv{As an extension, in \cref{app:model_based_reward_based} we give a reward-driven counterpart to
\cref{alg:model_based_reward_free}, which directly optimizes a given
reward function online. This approach does not improve upon
\cref{cor:rf_teaser}, but the analysis is slightly more direct.
}
We now sketch some basic examples in which \cref{cor:rf_teaser} can be
applied.%
\begin{example}[Tabular MDPs]
  For tabular MDPs with $\abs*{\cX}\leq{}S$ and $\abs*{\cA}\leq{}A$,
  we can construct online estimators for which
  $\EstProbOn=\bigoht(HS^2A)$, so that
  \cref{cor:mbrf3} gives sample complexity
$T=\frac{\poly(H,S,A)}{\veps^2}$ to compute policy covers such that $\CovM[\Mstar](p_h)
        \leq 12H\cdot{}\CinfM[\Mstar]$.
      \end{example}
      \begin{example}[Low-Rank MDPs]
        Consider the Low-Rank MDP model in \cref{eq:low_rank} with
        dimension $d$ and
        suppose, following
        \citet{agarwal2020flambe,uehara2022representation}, that we
        have access to classes $\Phi$ and $\Psi$ such that
        $\phi_{h}\in\Phi$ and $\psi_h\in\Psi$. Then \mle achieves
        $\EstProbOff=\bigoht(\log(\abs{\Phi}\abs{\Psi}))$, and we can
        take $\CinfM[\Mstar]\leq{}d\abs*{\cA}$, so \cref{cor:mbrf3}
        gives sample complexity $T=\frac{\poly(H,d,\abs*{\cA},\log(\abs{\Phi}\abs{\Psi}))}{\veps^4}$ to compute policy covers such that $\CovM[\Mstar](p_h)
        \leq 12H\cdot{}\CinfM[\Mstar]$.
      \end{example}
      
}

\icml{\subsection{Applying \creftitle{alg:model_based_reward_free} to Downstream Policy Optimization}
  \label{app:rf_example}
By \cref{lem:com} (see also \cref{app:downstream}), the policy covers
$p_1,\ldots,p_H$ returned by \cref{alg:model_based_reward_free} can
be used to optimize any downstream reward function using standard
offline RL algorithms. This leads to end-to-end guarantees for
reward-driven PAC RL. For concreteness, we sketch an example which
uses maximum likelihood (\mle) for offline policy optimization; see
\cref{app:downstream} for further examples and details.
\begin{corollary}[Application to reward-free reinforcement learning]
  \label{cor:rf_teaser}
Given access
  to $H\cdot{}n$ trajectories from the policy covers $p_1,\ldots,p_H$ produced by
  \cref{alg:model_based_reward_free} (configured as in
  \cref{cor:mbrf3}) and a realizable model class with $\Mstar\in\cM$,
  for any reward distribution $R=\crl*{R_h}_{h=1}^{H}$ with $\sum_{h=1}^{H}r_h\in\brk{0,1}$, the Maximum Likelihood Estimation algorithm
  (described in \cref{app:downstream}) produces a policy $\pihat$ such
  that with probability at least $1-\delta$,
  \icml{$            \Jmstar_R(\pistar) - \Jmstar_R(\pihat)
    \leq
    \bigoh(H)\cdot\prn[\Big]{\sqrt{H\CinfM[\Mstar]\cdot\frac{\log(\abs{\cM}/\delta)}{n}}
          + H\CinfM[\Mstar]\cdot{}\veps{}}$,
    }
  \arxiv{\begin{align}
    \label{eq:mle_teaser}
    \Jmstar_R(\pistar) - \Jmstar_R(\pihat)
    \leq
    \bigoh(H)\cdot\prn*{\sqrt{H\CovOptMmax[\Mstar]\cdot\frac{\log(\abs{\cM}/\delta)}{n}}
          + H\CovOptMmax[\Mstar]\cdot{}\veps{}},
    \end{align}}
    where
    $\Jmstar_R(\pi)\ldef{}\En^{\sMstar,\pi}\brk*{\sum_{h=1}^{H}r_h}$
    denotes the expected reward in $\Mstar$ when $R$ is the reward distribution.
\end{corollary}
\arxiv{As an extension, in \cref{app:model_based_reward_based} we give a reward-driven counterpart to
\cref{alg:model_based_reward_free}, which directly optimizes a given
reward function online. This approach does not improve upon
\cref{cor:rf_teaser}, but the analysis is slightly more direct.
}
We now sketch some basic examples in which \cref{cor:rf_teaser} can be
applied.%
\begin{example}[Tabular MDPs]
  For tabular MDPs with $\abs*{\cX}\leq{}S$ and $\abs*{\cA}\leq{}A$,
  we can construct online estimators for which
  $\EstProbOn=\bigoht(HS^2A)$, so that
  \cref{cor:mbrf3} gives sample complexity
$T=\frac{\poly(H,S,A)}{\veps^2}$ to compute policy covers such that $\CovM[\Mstar](p_h)
        \leq 12H\cdot{}\CinfM[\Mstar]$.
      \end{example}
      \begin{example}[Low-Rank MDPs]
        Consider the Low-Rank MDP model in \cref{eq:low_rank} with
        dimension $d$ and
        suppose, following
        \citet{agarwal2020flambe,uehara2022representation}, that we
        have access to classes $\Phi$ and $\Psi$ such that
        $\phi_{h}\in\Phi$ and $\psi_h\in\Psi$. Then \mle achieves
        $\EstProbOff=\bigoht(\log(\abs{\Phi}\abs{\Psi}))$, and we can
        take $\CinfM[\Mstar]\leq{}d\abs*{\cA}$, so \cref{cor:mbrf3}
        gives sample complexity $T=\frac{\poly(H,d,\abs*{\cA},\log(\abs{\Phi}\abs{\Psi}))}{\veps^4}$ to compute policy covers such that $\CovM[\Mstar](p_h)
        \leq 12H\cdot{}\CinfM[\Mstar]$.
      \end{example}
      
}
    
\arxiv{\subsection{Technical Lemmas}}
\icml{\subsection{Technical Lemmas}}
\label{sec:model_based_technical}

  \loneofflineonline*

  \begin{proof}[\pfref{lem:lone_offline_online}]
    Let us abbreviate $d_h\ind{t}=d_h^{\sMstar,\pi\ind{t}}$.
    By Lemma A.11 of \citet{foster2021statistical}, we have that
    \begin{align*}
      \EstOn
      &=
      \sum_{t=1}^{T}\En_{\act\ind{t}\sim{}p\ind{t}}\brk*{\Dhels{\Mstar(\pi\ind{t})}{\Mhat\ind{t}(\pi\ind{t})}}\\
      &\leq \bigoh(\log(H))
        \cdot\sum_{h=1}^{H} \sum_{t=1}^{T} 
        \En_{\pi\ind{t}\sim{}p\ind{t}}\En^{\sMstar,\pi\ind{t}}\brk*{\Dhels{\Pmstar_h(x_h,a_h)}{\Pmstar[\Mhat\ind{t}]_h(x_h,a_h)}
        + \Dhels{\Rmstar_h(x_h,a_h)}{\Rmstar[\Mhat\ind{t}]_h(x_h,a_h)}}.
    \end{align*}
    At the same time, for all $t$, we have by
    \cref{lem:hellinger_pair} that
    \begin{align*}
      &\sum_{i<t}
\En_{\pi\ind{i}\sim{}p\ind{i}}\En^{\sMstar,\pi\ind{i}}\brk*{\Dhels{\Pmstar_h(x_h,a_h)}{\Pmstar[\Mhat\ind{t}]_h(x_h,a_h)}
        +
      \Dhels{\Rmstar_h(x_h,a_h)}{\Rmstar[\Mhat\ind{t}]_h(x_h,a_h)}}\\
      &\leq{} 4
      \sum_{i<t}\En_{\act\ind{i}\sim{}p\ind{i}}\brk*{\Dhels{\Mstar(\pi\ind{i})}{\Mhat\ind{t}(\pi\ind{i})}}
      \leq{} 4\EstOfft[t].
    \end{align*}
    The result now follows by applying
    \cref{lem:lone_potential}---stated and proven below---to
    the function $g(x,a) = \Dhels{\Pmstar_h(x,a)}{\Pmstar[\Mhat\ind{t}]_h(x,a)}
        +
      \Dhels{\Rmstar_h(x,a)}{\Rmstar[\Mhat\ind{t}]_h(x,a)}$ for each
      layer $h\in\brk{H}$, which has $B\leq{}4$. \dfcomment{minor gap
        in proof: Lemma considers a fixed sequence of policies, not a
        randomized sequence.}
    \end{proof}

      \linfofflineonline*

  \begin{proof}[\pfref{lem:linf_offline_online}]
    Let us abbreviate $d_h\ind{t}=d_h^{\sMstar,\pi\ind{t}}$.
    By Lemma A.11 of \citet{foster2021statistical}, we have that
    \begin{align*}
      \EstOn
      &=
      \sum_{t=1}^{T}\En_{\act\ind{t}\sim{}p\ind{t}}\brk*{\Dhels{\Mstar(\pi\ind{t})}{\Mhat\ind{t}(\pi\ind{t})}}\\
      &\leq \bigoh(\log(H))
        \cdot\sum_{h=1}^{H} \sum_{t=1}^{T} 
        \En_{\pi\ind{t}\sim{}p\ind{t}}\En^{\sMstar,\pi\ind{t}}\brk*{\Dhels{\Pmstar_h(x_h,a_h)}{\Pmstar[\Mhat\ind{t}]_h(x_h,a_h)}
        + \Dhels{\Rmstar_h(x_h,a_h)}{\Rmstar[\Mhat\ind{t}]_h(x_h,a_h)}}.
    \end{align*}
    At the same time, for all $t$, we have by
    \cref{lem:hellinger_pair} that
    \begin{align*}
      &\sum_{i<t}
\En_{\pi\ind{i}\sim{}p\ind{i}}\En^{\sMstar,\pi\ind{i}}\brk*{\Dhels{\Pmstar_h(x_h,a_h)}{\Pmstar[\Mhat\ind{t}]_h(x_h,a_h)}
        +
      \Dhels{\Rmstar_h(x_h,a_h)}{\Rmstar[\Mhat\ind{t}]_h(x_h,a_h)}}\\
      &\leq{} 4
      \sum_{i<t}\En_{\act\ind{i}\sim{}p\ind{i}}\brk*{\Dhels{\Mstar(\pi\ind{i})}{\Mhat\ind{t}(\pi\ind{i})}}
      \leq{} 4\EstOfft[t].
    \end{align*}
    The result now follows by applying
    \cref{lem:linf_potential_var}---stated and proven below---to
    the function $g(x,a) = \Dhels{\Pmstar_h(x,a)}{\Pmstar[\Mhat\ind{t}]_h(x,a)}
        +
      \Dhels{\Rmstar_h(x,a)}{\Rmstar[\Mhat\ind{t}]_h(x,a)}$ for each
      layer $h\in\brk{H}$, which has $B\leq{}4$.
  \end{proof}

  \begin{lemma}
    \label{lem:lone_potential}
    	Fix an MDP $M$ and layer $h\in\brk{H}$. Suppose we have a sequence of functions $g\ind{1},\ldots,g\ind{T}\in\brk{0,B}$
	and policies $\pi\ind{1},\ldots,\pi\ind{T}$ such that 
	\begin{align}
          \label{eq:potential_condition}
		\forall t \in[T], \quad  \sum_{i<t}\En^{\sM,\pi\ind{i}}\brk*{(g\ind{t}(x_h,a_h))^2}
		\leq \beta^2.
	\end{align}
	Then it holds that
	\begin{align}
          \sum_{t=1}^{T}\En^{\sM,\pi\ind{t}}\brk*{g\ind{t}(x_h,a_h)}
          =2(\ConehM)^{1/3}(\beta^2B+B^3)^{1/3}T^{2/3}.
	\end{align}
    
      \end{lemma}

  \begin{proof}[\pfref{lem:lone_potential}]%
    \newcommand{\sA}{\mathbf{A}}%
    \newcommand{\sB}{\mathbf{B}}%
    \newcommand{\sC}{\mathbf{C}}%
    Let $\mu\in\Delta(\cX\times\cA)$ denote a distribution that
    achieves the value $\ConehM$. Let us abbreviate
    $d\ind{t}_h(x,a)=d^{\sM,\pi\ind{t}}_h(x,a)$ and 
    \begin{align}
      \dtil_h\ind{t}(x,a)=\sum_{i<t}d_h^{\sM,\pi\ind{i}}(x,a).
    \end{align}
    Observe that by \cref{eq:potential_condition} and the assumption
    that $g\ind{t}\in\brk{0,B}$, we have that
    \begin{align}
      \label{eq:dtone_bound}
      \sum_{x\in\cX,a\in\cA}\dtil^{t+1}_h(x,a)(g\ind{t}(x,a))^2\leq\beta^2+B^2\rdef\alpha^2.
    \end{align}
    To begin, fix a parameter $\lambda>0$ and define
    \begin{align}
      \tau(x,a)\ldef\min\crl*{t\mid{}\dtil_h^{t+1}(x,a)\geq \lambda\mu(x,a)}.
    \end{align}
    We can bound
    \begin{align}
      \sum_{t=1}^{T}\En^{\sM,\pi\ind{t}}\brk*{g\ind{t}(x_h,a_h)}
      \leq \sum_{t=1}^{T}\En^{\sM,\pi\ind{t}}\brk*{g\ind{t}(x_h,a_h)\indic\crl*{t\geq\tau(x_h,a_h)}}
      + B\sum_{t=1}^{T}\En^{\sM,\pi\ind{t}}\brk*{\indic\crl*{t<\tau(x_h,a_h)}}.
    \end{align}
    For the second term above, we can write
    \begin{align*}
      \sum_{t=1}^{T}\En^{\sM,\pi\ind{t}}\brk*{\indic\crl*{t<\tau(x_h,a_h)}}
      &=\sum_{x,a}\sum_{t=1}^{T}d_h\ind{t}(x,a)
        \indic\crl*{t<\tau(x,a)}\\
      &=\sum_{x,a}\dtil_h\ind{\tau(x,a)}(x,a)
        < \lambda\sum_{x,a}\mu(x,a)=\lambda,
    \end{align*}
    where the final inequality uses the definition of $\tau(x,a)$.

    For the first, term, using Cauchy-Schwarz, we can bound
    \begin{align*}
      &\sum_{t=1}^{T}\En^{\sM,\pi\ind{t}}\brk*{g\ind{t}(x_h,a_h) \indic\crl*{t\geq\tau(x_h,a_h)}}\\
      &=
        \sum_{t=1}^{T}\sum_{x\in\cX,a\in\cA}d_h\ind{t}(x,a)g\ind{t}(x,a) \indic\crl*{t\geq\tau(x,a)}\\
      &= \sum_{t=1}^{T}\sum_{x\in\cX,a\in\cA}
        \frac{d_h\ind{t}(x,a)}{(\dtil_h\ind{t+1}(x,a))^{1/2}}\indic\crl*{t\geq\tau(x,a)}
        \cdot(\dtil_h\ind{t+1}(x,a))^{1/2}g\ind{t}(x,a)\\
      &\leq{} \sA^{1/2}\sB^{1/2},
    \end{align*}
    where
    \begin{align}
      &\sA\ldef{} \sum_{t=1}^{T}\sum_{x\in\cX,a\in\cA}
      \frac{(d_h\ind{t}(x,a))^2}{\dtil_h\ind{t+1}(x,a)}\indic\crl*{t\geq\tau(x,a)},\mathand
        \sB\ldef{}
      \sum_{t=1}^{T}\sum_{x\in\cX,a\in\cA}
      \dtil_h\ind{t+1}(x,a)(g\ind{t}(x,a))^2.
    \end{align}
    \cref{eq:dtone_bound} implies that
    \begin{align*}
      \sB=
      \sum_{t=1}^{T}\sum_{x\in\cX,a\in\cA}
      \dtil_h\ind{t+1}(x,a)(g\ind{t}(x,a))^2
      \leq{} \alpha^2{}T.
    \end{align*}

      It remains to bound term $\sA$. From the definition of
      $\tau(x,a)$, we can bound
      \begin{align*}
        \sA= \sum_{t=1}^{T}\sum_{x\in\cX,a\in\cA}
             \frac{(d_h\ind{t}(x,a))^2}{\dtil_h\ind{t+1}(x,a)}\indic\crl*{t\geq\tau(x,a)}
           \leq \frac{1}{\lambda}\sum_{t=1}^{T}\sum_{x\in\cX,a\in\cA}
        \frac{(d_h\ind{t}(x,a))^2}{\mu(x,a)}
        \leq \frac{\ConehM{}T}{\lambda},
      \end{align*}
      where the last inequality uses that $\mu$ achieves the value of $\ConehM$.

    Combining the results so far, we have that
    \begin{align*}
      \sum_{t=1}^{T}\En^{\sM,\pi\ind{t}}\brk*{g\ind{t}(x_h,a_h)}
      \leq{} (\lambda^{-1}\ConehM T)^{1/2}(\alpha^2T)^{1/2}
      + \lambda{}B
      = (\ConehM\alpha^2)^{1/2}T/\lambda^{1/2}
      + \lambda{}B.
    \end{align*}
    We choose $\lambda=(\ConehM\alpha^2)^{1/3}T^{2/3}/B^{2/3}$ to
    balance the terms, which gives a bound of the form
    \[
      2 (\ConehM\alpha^2B)^{1/3}T^{2/3}
      =2(\ConehM)^{1/3}(\beta^2B+B^3)^{1/3}T^{2/3}.
    \]
  \end{proof}

    \begin{lemma}
    \label{lem:linf_potential_var}
    	Fix an MDP $M$ and layer $h\in\brk{H}$. Suppose we have a sequence of functions $g\ind{1},\ldots,g\ind{T}\in\brk{0,B}$
	and policies $\pi\ind{1},\ldots,\pi\ind{T}$ such that 
	\begin{align}
          \label{eq:linf_potential_condition}
		\forall t \in[T], \quad  \sum_{i<t}\En^{\sM,\pi\ind{i}}\brk*{(g\ind{t}(x_h,a_h))^2}
		\leq \beta^2.
	\end{align}
	Then it holds that
	\begin{align}
          \sum_{t=1}^{T}\En^{\sM,\pi\ind{t}}\brk*{g\ind{t}(x_h,a_h)}
          =\bigoh\prn*{\sqrt{\CinfhM T \log(T)\cdot\beta^2} + \CinfhM{}B}.
	\end{align}
      \end{lemma}
      \begin{proof}[\pfref{lem:linf_potential_var}]
        See proof of Theorem 1 in \citet{xie2023role}.
      \end{proof}

\arxiv{\subsection{Proofs from
    \creftitle{sec:model_based_reward_free}}}
\icml{\subsection{Proofs from \creftitle{sec:model_based} and
    \creftitle{sec:model_based_general}}}
\label{app:model_based_proof}

\modelbasedrewardfree*

\icml{
  \paragraph{Overview of proof}
  Before diving into the proof of \cref{thm:model_based_reward_free},
  we sketch the high-level idea. The crux of the analysis is to
show that for each round $t$, we have that
\begin{align}
  \label{eq:rf_key}
  \CovOptM[\Mstar](p_h\ind{t})
  \approxleq 
  \max_{h}\CovOptM[\Mhat\ind{t}](p_h\ind{t})
  +
  \frac{1}{\veps}\sqrt{H^3C\cdot{}\En_{\pi\sim{}q\ind{t}}\brk[\big]{\DhelsX{\big}{\Mhat\ind{t}(\pi)}{\Mstar(\pi)}}}
  + \frac{H}{\veps^2}\cdot\En_{\pi\sim{}q\ind{t}}\brk[\big]{\DhelsX{\big}{\Mhat\ind{t}(\pi)}{\Mstar(\pi)}}.
\end{align}
Averaging across all rounds $t$, this allows us to conclude that
\begin{align*}
  \frac{1}{T}\sum_{t=1}^{T}  \CovOptM[\Mstar](p_h\ind{t})
  \approxleq 
    HC +
\frac{1}{\veps}\sqrt{H^3C\cdot{}\frac{1}{T}\sum_{t=1}^{T}\En_{\pi\sim{}q\ind{t}}\brk[\big]{\DhelsX{\big}{\Mhat\ind{t}(\pi)}{\Mstar(\pi)}}}
  +
\frac{H}{\veps^2}\cdot\frac{1}{T}\sum_{t=1}^{T}\En_{\pi\sim{}q\ind{t}}\brk[\big]{\DhelsX{\big}{\Mhat\ind{t}(\pi)}{\Mstar(\pi)}}.
\end{align*}
Since $ \CovOptM[\Mstar](p_h) \leq \frac{1}{T}\sum_{t=1}^{T}  \CovOptM[\Mstar](p_h\ind{t})$ and 
$\sum_{t=1}^{T}\En_{\pi\sim{}q\ind{t}}\brk[\big]{\DhelsX{\big}{\Mhat\ind{t}(\pi)}{\Mstar(\pi)}}\leq\EstProbOn$
  with probability at least $1-\delta$, this proves the result.

\cref{eq:rf_key} can be thought as a reward-free analogue of a bound
on the Decision-Estimation Coefficient (DEC) of
\citet{foster2021statistical,foster2023tight} and makes precise the
reasoning that by optimizing the plug-in approximation to the \mainobj
objective, we either 1) cover the true MDP $\Mstar$ well, or 2)
achieve large information gain (as quantified by the instantaneous
estimation error
$\En_{\pi\sim{}q\ind{t}}\brk[\big]{\DhelsX{\big}{\Mhat\ind{t}(\pi)}{\Mstar(\pi)}}$).
}

\begin{proof}[\pfref{thm:model_based_reward_free}]
  Observe that by Jensen's inequality, we have that for
  all $h\in\brk{H}$,
  \begin{align*}
    T\cdot{}\CovM[\Mstar](p_h)
    \leq \sum_{t=1}^{T}\CovM[\Mstar](p_h\ind{t}).
  \end{align*}
We will show how to bound the \rhs above. To begin, we state two technical lemmas, both proven in the sequel.
  \begin{restatable}{lemma}{valsmoothing}
    \label{lem:val_smoothing}
    Fix $h\in\brk{H}$. For all mixture policies $p\in\Delta(\Pi)$ and MDPs $\Mhat=\crl*{\cX, \cA,
  \crl{\Pmhat_h}_{h=0}^{H}}$, it holds that
    \begin{align*}
      \CovM[\Mstar](p)\leq
      \max_{\pi}\En^{\sMstar,\pi}\brk*{\frac{d^{\sMstar,\pi}_h(x_h,a_h)
      + d^{\sMhat,\pi}_h(x_h,a_h)}{d_h^{\sMstar,p}(x_h,a_h)+\veps{}\cdot{}(d^{\sMstar,\pi}_h(x_h,a_h)+d^{\sMhat,\pi}_h(x_h,a_h))}}.
    \end{align*}
  \end{restatable}

For the next result, for a given reward-free MDP $M=\crl*{\cX, \cA,
  \crl{\Pm_h}_{h=0}^{H}}$ and reward distribution
$R=\crl*{R_h}_{h=1}^{H}$, we define
\[
  J^{\sM}_R(\pi)\ldef{}\En^{\sM,\pi}\brk*{\sum_{h=1}^{H}r_h}
\]
as the value under $r_h\sim{}R(x_h,a_h)$.
  
\begin{restatable}{lemma}{decboundrewardfree}
  \label{lem:dec_bound_reward_free}
  Consider the reward-free setting. Let an MDP $\Mhat=\crl*{\cX, \cA,
  \crl{\Pmhat_h}_{h=0}^{H}}$ be given, and let $p_1,\ldots,p_H\in\Delta(\Pi)$ be
  $(C,\veps)$-policy covers for $\Mhat$, i.e.
  \begin{align}
    \CovM[\Mhat](p_h) \leq C\quad\forall{}h\in\brk{H}.
  \end{align}
  Then the distribution $q\ldef\unif(p_1,\ldots,p_H)$ ensures that for all MDPs $M=\crl*{\cX, \cA,
  \crl{\Pm_h}_{h=0}^{H}}$, all reward distributions
$R=\crl*{R_h}_{h=1}^{H}$ with $\sum_{h=1}^{H}r_h\in\brk{0,B}$ almost
surely, and all policies $\pi\in\Pi$,
\begin{align*}
      \Jm_R(\pi) - \Jmhat_R(\pi)
    &\leq{}
      4B\sqrt{H^3C\cdot{}\En_{\pi\sim{}q}\brk[\big]{\DhelsX{\big}{\Mhat(\pi)}{M(\pi)}}}
      + \sqrt{2}BHC\cdot{}\veps{}.
\end{align*}
  
\end{restatable}

For the remainder of the proof, we abbreviate $d_h^{\sM,\pi}\equiv
d_h^{\sM,\pi}(x_h,a_h)$ whenever the argument is clear from
context. Let $h\in\brk{H}$ be fixed. We observe that by \cref{lem:val_smoothing}, 
  \begin{align*}
    \sum_{t=1}^{T}\CovM[\Mstar](p_h\ind{t})
&    \leq{}
                \sum_{t=1}^{T}                   \max_{\pi\in\Pi}\En^{\sMstar,\pi}\brk*{
\frac{d^{\sMstar,\pi}_h}{d_h^{\sMstar,p\ind{t}_h}+\veps{}\cdot{}\prn{d^{\sMstar,\pi}_h+d^{\sMhat\ind{t},\pi}_h}}
    }
    +
                                                                      \max_{\pi\in\Pi}\En^{\sMstar,\pi}\brk*{
\frac{d^{\sMhat\ind{t},\pi}_h}{d_h^{\sMstar,p\ind{t}_h}+\veps{}\cdot{}\prn{d^{\sMstar,\pi}_h+d^{\sMhat\ind{t},\pi}_h}}
    }.
\end{align*}
For any $\pi\in\Pi$, by applying \cref{lem:dec_bound_reward_free} for each
step $t$ with the deterministic rewards
  \begin{align*}
    r_h\ind{t}(x,a) =
    \frac{d^{\sMstar,\pi}_h(x,a)}{d^{\sMstar,p\ind{t}}_h(x,a) + \veps\cdot{}\prn{d^{\sMstar,\pi}_h(x,a)+d^{\sMhat\ind{t},\pi}_h(x,a)}}\indic\crl*{h'=h},
  \end{align*}
  which satisfy $\sum_{h=1}^{H}r_h\ind{t}\in\brk*{0,\veps^{-1}}$
  almost surely, we can bound
  \begin{align*}
&\En^{\sMstar,\pi}\brk*{\frac{d^{\sMstar,\pi}_h}{d_h^{\sMstar,p_h\ind{t}}+\veps{}\cdot{}\prn{d^{\sMstar,\pi}_h+d^{\sMhat\ind{t},\pi}_h}}}\\
    &\leq
      \En^{\sMhat\ind{t},\pi}\brk*{
\frac{d^{\sMstar,\pi}_h}{d_h^{\sMstar,p_h\ind{t}}+\veps{}\cdot{}\prn{d^{\sMstar,\pi}_h+d^{\sMhat\ind{t},\pi}_h}}}
      +
\frac{4}{\veps}\sqrt{H^3C\cdot{}\En_{\pi\sim{}q\ind{t}}\brk[\big]{\DhelsX{\big}{\Mhat\ind{t}(\pi)}{\Mstar(\pi)}}}
      + \sqrt{2}HC,\\
    &=
      \En^{\sMstar,\pi}\brk*{
\frac{d^{\sMhat\ind{t},\pi}_h}{d_h^{\sMstar,p_h\ind{t}}+\veps{}\cdot{}\prn{d^{\sMstar,\pi}_h+d^{\sMhat\ind{t},\pi}_h}}}
      +
\frac{4}{\veps}\sqrt{H^3C\cdot{}\En_{\pi\sim{}q\ind{t}}\brk[\big]{\DhelsX{\big}{\Mhat\ind{t}(\pi)}{\Mstar(\pi)}}}
      + \sqrt{2}HC,
  \end{align*}
where we have simplified the second term by noting that
$\sqrt{2}BHC\cdot\veps=\sqrt{2}\veps^{-1}HC\cdot\veps=\sqrt{2}HC$. \dfcomment{this
  will make us get an extra $H^2$ multiplicative factor on the final
  objective value. Can we remove this?} It follows that
\begin{align*}
      \sum_{t=1}^{T}\CovM[\Mstar](p_h\ind{t})
&    \leq{}
                                                2\sum_{t=1}^{T}
                                                \max_{\pi\in\Pi}\En^{\sMstar,\pi}\brk*{
\frac{d^{\sMhat\ind{t},\pi}_h}{d_h^{\sMstar,p\ind{t}_h}+\veps{}\cdot{}\prn{d^{\sMstar,\pi}_h+d^{\sMhat\ind{t},\pi}_h}}
                                                }
+\frac{4}{\veps}\sum_{t=1}^{T}\sqrt{H^3C\cdot{}\En_{\pi\sim{}q\ind{t}}\brk[\big]{\DhelsX{\big}{\Mhat\ind{t}(\pi)}{\Mstar(\pi)}}} \\
      &\qquad + \sqrt{2}HCT.
\end{align*}
  
We now appeal to \cref{lem:dec_bound_reward_free} once more. For any $\pi$, by applying \cref{lem:dec_bound_reward_free} again at each step $t$ with the rewards
  \begin{align*}
    r_h\ind{t}(x,a) =
    \frac{d^{\sMhat\ind{t},\pi}_h(x,a)}{d^{\sMstar,p\ind{t}}_h(x,a) + \veps\cdot{}\prn{d^{\sMstar,\pi}_h(x,a)+d^{\sMhat\ind{t},\pi}_h(x,a)}}\indic\crl*{h'=h},
  \end{align*}
  which satisfy $\sum_{h=1}^{H}r_h\in\brk{0,\veps^{-1}}$ almost
  surely, allows us to bound
      \begin{align*}
&\En^{\sMstar,\pi}\brk*{\frac{d^{\sMhat\ind{t},\pi}_h}{d_h^{\sMstar,p_h\ind{t}}+\veps{}\cdot{}\prn{d^{\sMstar,\pi}_h+d^{\sMhat\ind{t},\pi}_h}}}\\
    &\leq
      \En^{\sMhat\ind{t},\pi}\brk*{
\frac{d^{\sMhat\ind{t},\pi}_h}{d_h^{\sMstar,p_h\ind{t}}+\veps{}\cdot{}\prn{d^{\sMstar,\pi}_h+d^{\sMhat\ind{t},\pi}_h}}}
      +
\frac{4}{\veps}\sqrt{H^3C\cdot{}\En_{\pi\sim{}q\ind{t}}\brk[\big]{\DhelsX{\big}{\Mhat\ind{t}(\pi)}{\Mstar(\pi)}}}
      + \sqrt{2}HC,\\
    &\leq
      \En^{\sMhat\ind{t},\pi}\brk*{
\frac{d^{\sMhat\ind{t},\pi}_h}{d_h^{\sMstar,p_h\ind{t}}+\veps{}\cdot{}d^{\sMhat\ind{t},\pi}_h}}
      +
\frac{4}{\veps}\sqrt{H^3C\cdot{}\En_{\pi\sim{}q\ind{t}}\brk[\big]{\DhelsX{\big}{\Mhat\ind{t}(\pi)}{\Mstar(\pi)}}}
      + \sqrt{2}HC.
      \end{align*}
      We conclude that
      \begin{align}
      \sum_{t=1}^{T}\CovM[\Mstar](p_h\ind{t})
&    \leq{}
                                                2\sum_{t=1}^{T}
                                                \max_{\pi\in\Pi}\En^{\sMhat\ind{t},\pi}\brk*{
\frac{d^{\sMhat\ind{t},\pi}_h}{d_h^{\sMstar,p\ind{t}_h}+\veps{}\cdot{}d^{\sMhat\ind{t},\pi}_h}
                                                }
+\frac{12}{\veps}\sum_{t=1}^{T}\sqrt{H^3C\cdot{}\En_{\pi\sim{}q\ind{t}}\brk[\big]{\DhelsX{\big}{\Mhat\ind{t}(\pi)}{\Mstar(\pi)}}}
                                                + 3\sqrt{2}HCT.\notag\\
                                                &    \leq{}
                                                2\sum_{t=1}^{T}
                                                \max_{\pi\in\Pi}\En^{\sMhat\ind{t},\pi}\brk*{
\frac{d^{\sMhat\ind{t},\pi}_h}{d_h^{\sMstar,p\ind{t}_h}+\veps{}\cdot{}d^{\sMhat\ind{t},\pi}_h}
                                                }
+\frac{12}{\veps}\sqrt{H^3CT\cdot{}\EstProbOn}
      + 3\sqrt{2}HCT.\label{eq:rf0}
      \end{align}

  We now appeal to the following lemma.
  \begin{restatable}{lemma}{dpcom}
    \label{lem:dp_com}
    Consider the reward-free setting. For any pair of MDP $\Mhat=\crl*{\cX, \cA,
      \crl{\Pmhat_h}_{h=0}^{H}}$ and $M=\crl*{\cX, \cA,
  \crl{\Pm_h}_{h=0}^{H}}$, any policy $\pi\in\PiRNS$, and any distribution $p\in\Delta(\PiRNS)$,
it holds that
    \begin{align*}
      \En^{\sMhat,\pi}\brk*{
      \frac{d^{\sMhat,\pi}_h}{d_h^{\sMstar,p}+\veps{}\cdot{}d^{\sMhat,\pi}_h}}
      \leq 3\En^{\sMhat,\pi}\brk*{
      \frac{d^{\sMhat,\pi}_h}{d_h^{\sMhat,p}+\veps{}\cdot{}d^{\sMhat,\pi}_h}}
      +
      \frac{4}{\veps^2}      \En_{\pi\sim{}p}\brk*{
      \Dhels{\Mhat(\pi)}{\Mstar(\pi)}
      }.
    \end{align*}
  \end{restatable}
  Combining \cref{eq:rf0} with \cref{lem:dp_com}, we have that
      \begin{align}
      \sum_{t=1}^{T}\CovM[\Mstar](p_h\ind{t})
&    \leq{}
                                                6\sum_{t=1}^{T}\CovM[\Mhat\ind{t}](p_h\ind{t})
+\frac{12}{\veps}\sqrt{H^3CT\cdot{}\EstProbOn}
                                                + 3\sqrt{2}HCT
                                                +\frac{8}{\veps^2}\sum_{t=1}^{T}\En_{\pi\sim{}p_h\ind{t}}\brk*{\Dhels{\Mhat\ind{t}(\pi)}{\Mstar(\pi)}}
                                                            \notag\\
        &    \leq{}
                                                6\sum_{t=1}^{T}\CovM[\Mhat\ind{t}](p_h\ind{t})
+\frac{12}{\veps}\sqrt{H^3CT\cdot{}\EstProbOn}
                                                + 3\sqrt{2}HCT
                                                +\frac{8H}{\veps^2}\sum_{t=1}^{T}\En_{\pi\sim{}q\ind{t}}\brk*{\Dhels{\Mhat\ind{t}(\pi)}{\Mstar(\pi)}}
                                                            \notag\\
                &    \leq{}
                                                6\sum_{t=1}^{T}\CovM[\Mhat\ind{t}](p_h\ind{t})
+\frac{12}{\veps}\sqrt{H^3CT\cdot{}\EstProbOn}
                                                + 3\sqrt{2}HCT
                                                +\frac{8H}{\veps^2}\EstProbOn.
                                                \label{eq:rf1}
      \end{align}
      To conclude the proof of \cref{eq:mbrf1}, we note that it follows from the definition of
    $p_h\ind{t}$ that for all $h\in\brk{H}$ and $t\in\brk{T}$,
    \begin{align*}
      \CovM[\Mhat\ind{t}](p_h\ind{t})
      \leq{} C.
    \end{align*}
    Hence, we have that
    \begin{align*}
      \sum_{t=1}^{T}\CovM[\Mstar](p_h\ind{t})
      \leq{}                                                 6CT
+\frac{12}{\veps}\sqrt{H^3CT\cdot{}\EstProbOn}
                                                + 3\sqrt{2}HCT
      +\frac{8H}{\veps^2}\EstProbOn\\
      \leq{}                                                 11HCT
+\frac{12}{\veps}\sqrt{H^3CT\cdot{}\EstProbOn}
                                                +\frac{8H}{\veps^2}\EstProbOn.
    \end{align*}
    This implies that
        \begin{align*}
          \CovM[\Mstar](p_h)
          \leq{}                                                 11HC
+\frac{12}{\veps}\sqrt{\frac{H^3C\cdot{}\EstProbOn}{T}}
                                                +\frac{8H}{\veps^2}\frac{\EstProbOn}{T}.
        \end{align*}
        as desired.

      \end{proof}

      \mbrftwo*
      \begin{proof}[\pfref{cor:mbrf2}]
        We prove \cref{eq:mbrf2} as a consequence of
        \cref{thm:model_based_reward_free}. We can take
        $C\leq\CovOptMmax[\Mstar]$, and using
        \cref{lem:lone_offline_online}, we have
        \begin{align*}
          \EstProbOn \leq \bigoht\prn[\Big]{H\prn*{\ConeM[\Mstar](1\vee\EstProbOff)}^{1/3}T^{2/3}
          } \leq \bigoht\prn[\Big]{H \prn*{\CovOptMmaxzero[\Mstar](1\vee\EstProbOff)}^{1/3}T^{2/3}
          },
        \end{align*}
        so that \cref{eq:mbrf1} gives
        \begin{align*}
          \forall{}h\in\brk{H}:\quad\CovM[\Mstar](p_h)
          \leq                                                  11HC
          +
          \bigoht\prn*{
          \frac{1}{\veps}\prn*{\frac{H^{12}(\CovOptMmaxzero[\Mstar])^4(1\vee\EstProbOff)}{T}}^{1/6} 
          } \leq 12HC,
        \end{align*}
        where the final inequality uses the choice for $T$ in the
        corollary statement.
      \end{proof}

\mbrfthree*
      
\begin{proof}[\pfref{cor:mbrf3}]
  We prove \cref{eq:mbrf3} as a consequence of
  \cref{thm:model_based_reward_free}. We can take
  $C\leq\CinfM[\Mstar]$, and using \cref{lem:linf_offline_online}, we
  have
  \begin{align*}
    \EstProbOn \leq \bigoht\prn*{H\sqrt{\CinfM[\Mstar]
    T\cdot\EstProbOff} + H\cdot\CinfM[\Mstar] 
    }
    \leq \bigoht\prn*{H\sqrt{\CinfM[\Mstar]
    T\cdot(1\vee\EstProbOff)}    },
  \end{align*}
  so that \cref{eq:mbrf1} gives
  \begin{align*}
    \forall{}h\in\brk{H}:\quad\CovM[\Mstar](p_h)
    \leq                                                  11HC
    +\bigoht\prn*{
    \frac{1}{\veps}\prn*{\frac{H^8(\CinfM[\Mstar])^3(1\vee\EstProbOff)}{T}}^{1/4} 
    } \leq 12HC,
  \end{align*}
  where the final inequality uses the choice for $T$ in the \arxiv{corollary}\icml{theorem}
  statement.
\end{proof}

  \subsubsection{Supporting Lemmas}
      
      \valsmoothing*  
  \begin{proof}[\pfref{lem:val_smoothing}]
    To keep notation compact, let us suppress the dependence on $x_h$ and $a_h$.
    For all $\pi\in\PiRNS$ and $p\in\Delta(\PiRNS)$, we have that
    \begin{align*}
&      \En^{\sMstar,\pi}\brk*{\frac{d^{\sMstar,\pi}_h
      }{d_h^{\sMstar,p}+\veps{}\cdot{}d^{\sMstar,\pi}_h}}
      -       \En^{\sMstar,\pi}\brk*{\frac{d^{\sMstar,\pi}_h
      }{d_h^{\sMstar,p}+\veps{}\cdot{}(d^{\sMstar,\pi}_h+d^{\sMhat,\pi}_h)}}
      \\
      &=      \En^{\sMstar,\pi}\brk*{\frac{d^{\sMstar,\pi}_h\cdot\veps{}\cdot{}d^{\sMhat,\pi}_h
      }{(d_h^{\sMstar,p}+\veps{}\cdot{}d^{\sMstar,\pi}_h)
        (d_h^{\sMstar,p}+\veps{}\cdot{}(d^{\sMstar,\pi}_h+d^{\sMhat,\pi}_h))}}\\
      &\leq      \En^{\sMstar,\pi}\brk*{\frac{d^{\sMhat,\pi}_h
      }{d_h^{\sMstar,p}+\veps{}\cdot{}(d^{\sMstar,\pi}_h+d^{\sMhat,\pi}_h)}}.
    \end{align*}
    This proves the result.
    
  \end{proof}

    \decboundrewardfree*

    \begin{proof}[\pfref{lem:dec_bound_reward_free}]
      \arxiv{This proof proceeds in a similar fashion to
      \cref{lem:dec_bound}.  }Let an arbitrary MDP
      $M=\crl*{\cX, \cA, \crl{\Pm_h}_{h=0}^{H}}$, reward distribution
      $R=\crl*{R_h}_{h=1}^{H}$, and policy $\pi\in\Pi$ be fixed.
      To begin, using the simulation lemma (\cref{lem:simulation}), we
      have
      \begin{align*}
        \Jm_R(\pi) - \Jmhat_R(\pi)
        \leq     B\cdot\sum_{h=1}^{H}\Emhat[\pi]\brk[\Big]{\Dhel{\Pmhat_h(x_h,a_h)}{\Pm_h(x_h,a_h)}},
      \end{align*}
      where we have used that both MDPs have the same reward
      distribution.  Let $h\in\brk{H}$ be fixed. Since
      $\Dhel{\Pmhat_h(x_h,a_h)}{\Pm_h(x_h,a_h)}\in\brk{0,\sqrt{2}}$,
      we can use \cref{lem:com} to bound
      \begin{align*}
        \Emhat[\pi]\brk*{\Dhel{\Pmhat_h(x_h,a_h)}{\Pm_h(x_h,a_h)}} 
        &\leq 
          2\sqrt{\CovM[\Mhat](p_h)\cdot{}\En^{\sMhat,p_h}\brk*{\Dhels{\Pmhat_h(x_h,a_h)}{\Pm_h(x_h,a_h)}}
          }
          + \sqrt{2}\cdot\CovM[\Mhat](p_h)\cdot{}\veps{}\\
        &\leq 
          2\sqrt{H\CovM[\Mhat](p_h)\cdot{}\En^{\sMhat,q}\brk*{\Dhels{\Pmhat_h(x_h,a_h)}{\Pm_h(x_h,a_h)}}
          }
          + \sqrt{2}\cdot\CovM[\Mhat](p_h)\cdot{}\veps{}\\
        &\leq 
          2\sqrt{HC\cdot{}\En^{\sMhat,q}\brk*{\Dhels{\Pmhat_h(x_h,a_h)}{\Pm_h(x_h,a_h)}}
          }
          + \sqrt{2}C\cdot{}\veps{} \\
        &\leq 
          4\sqrt{HC\cdot{}\En_{\pi\sim{}q}\brk[\big]{\DhelsX{\big}{\Mhat(\pi)}{M(\pi)}}}
          + \sqrt{2}C\cdot{}\veps{},
      \end{align*}
      where the last inequality follows form
      \cref{lem:hellinger_pair}. Summing across all layers, we
      conclude that
      \begin{align*}
        \Jm_R(\pi) - \Jmhat_R(\pi)
        &\leq{}
          4B\sqrt{H^3C\cdot{}\En_{\pi\sim{}q}\brk[\big]{\DhelsX{\big}{\Mhat(\pi)}{M(\pi)}}}
          + \sqrt{2}BHC\cdot{}\veps{}.
      \end{align*}

    \end{proof}

\dpcom*
  
        \begin{proof}[\pfref{lem:dp_com}]
    Consider any $c\geq{}1$. For any $\pi\in\PiRNS$ and $p\in\Delta(\PiRNS)$, we can write
    \begin{align*}
      &\En^{\sMhat,\pi}\brk*{
      \frac{d^{\sMhat,\pi}_h}{d_h^{\sMstar,p}+\veps{}\cdot{}d^{\sMhat,\pi}_h}}
      -c\cdot\En^{\sMhat,\pi}\brk*{
        \frac{d^{\sMhat,\pi}_h}{d_h^{\sMhat,p}+\veps{}\cdot{}d^{\sMhat,\pi}_h}}\\
      &=
        \En^{\sMhat,\pi}\brk*{
        \frac{d^{\sMhat,\pi}_h\prn*{d_h^{\sMhat,p}+\veps{}\cdot{}d^{\sMhat,\pi}_h-c\cdot\prn*{
        d_h^{\sMstar,p}+\veps{}\cdot{}d^{\sMhat,\pi}_h
        }}}{\prn*{d_h^{\sMstar,p}+\veps{}\cdot{}d^{\sMhat,\pi}_h}\prn*{d_h^{\sMhat,p}+\veps{}\cdot{}d^{\sMhat,\pi}_h}}}\\
            &\leq
        \En^{\sMhat,\pi}\brk*{
        \frac{d^{\sMhat,\pi}_h\prn*{d_h^{\sMhat,p}-c\cdot
        d_h^{\sMstar,p}
              }}{\prn*{d_h^{\sMstar,p}+\veps{}\cdot{}d^{\sMhat,\pi}_h}\prn*{d_h^{\sMhat,p}+\veps{}\cdot{}d^{\sMhat,\pi}_h}}}\\
                  &=
        \En^{\sMhat,p}\brk*{
        \frac{(d^{\sMhat,\pi}_h)^2
                    }{\prn*{d_h^{\sMstar,p}+\veps{}\cdot{}d^{\sMhat,\pi}_h}\prn*{d_h^{\sMhat,p}+\veps{}\cdot{}d^{\sMhat,\pi}_h}}}
                    - c\cdot         \En^{\sMstar,p}\brk*{
        \frac{(d^{\sMhat,\pi}_h)^2
                    }{\prn*{d_h^{\sMstar,p}+\veps{}\cdot{}d^{\sMhat,\pi}_h}\prn*{d_h^{\sMhat,p}+\veps{}\cdot{}d^{\sMhat,\pi}_h}}
                    }.
    \end{align*}
    Observe that
    \begin{align*}
              \frac{(d^{\sMhat,\pi}_h)^2
                    }{\prn*{d_h^{\sMstar,p}+\veps{}\cdot{}d^{\sMhat,\pi}_h}\prn*{d_h^{\sMhat,p}+\veps{}\cdot{}d^{\sMhat,\pi}_h}}
      \leq \frac{1}{\veps^2}
  \end{align*}
  almost surely. Consequently, \cref{lem:mp_min}
  implies that for $c=3$,
  \begin{align*}
    &        \En^{\sMhat,p}\brk*{
        \frac{(d^{\sMhat,\pi}_h)^2
                    }{\prn*{d_h^{\sMstar,p}+\veps{}\cdot{}d^{\sMhat,\pi}_h}\prn*{d_h^{\sMhat,p}+\veps{}\cdot{}d^{\sMhat,\pi}_h}}}
                    - 3\cdot         \En^{\sMstar,p}\brk*{
        \frac{(d^{\sMhat,\pi}_h)^2
                    }{\prn*{d_h^{\sMstar,p}+\veps{}\cdot{}d^{\sMhat,\pi}_h}\prn*{d_h^{\sMhat,p}+\veps{}\cdot{}d^{\sMhat,\pi}_h}}
                    }.
    \\
    &\leq \frac{4}{\veps^2}\Dhels{d^{\sMhat,p}_h}{d^{\sMstar,p}_h}.
  \end{align*}
  Finally, note that by joint convexity of Hellinger distance and the
  data processing inequality, we have that
  \begin{align*}
    \Dhels{d^{\sMhat,p}_h}{d^{\sMstar,p}_h}
    \leq
      \En_{\pi\sim{}p}\brk*{
      \Dhels{d^{\sMhat,\pi}_h}{d^{\sMstar,\pi}_h}
      }
    \leq
      \En_{\pi\sim{}p}\brk*{
      \Dhels{\Mhat(\pi)}{\Mstar(\pi)}
    }.
  \end{align*}
    
  \end{proof}

\clearpage

\icml{
\section{Proofs from \creftitle{sec:structural}}
\label{app:structural}
\lone*

\begin{proof}[\pfref{prop:lone}]
Let $\mu\in\Delta(\cX\times\cA)$ be the
  distribution that attains the value of $\ConehM$.
  Using \cref{lem:pi_to_mu}, we have that for all $\pi\in\Pi$ and $p\in\Delta(\Pi)$,
    \begin{align*}
    &\En^{\sM,\pi}\brk*{\frac{d^{\sM,\pi}_h(x_h,a_h)}{d_h^{\sM,p}(x_h,a_h)+\veps\cdot{}d_h^{\sM,\pi}(x_h,a_h)}}\\
    &\leq
      2\En^{\sM,\pi}\brk*{\frac{d^{\sM,\pi}_h(x_h,a_h)}{d_h^{\sM,p}(x_h,a_h)+\veps\cdot(d_h^{\sM,\pi}(x_h,a_h)+\mu(x_h,a_h))}}
      +
      \En^{\sM,\pi}\brk*{\frac{\mu(x_h,a_h)}{d_h^{\sM,p}(x_h,a_h)+\veps\cdot(d_h^{\sM,\pi}(x_h,a_h)+\mu(x_h,a_h))}}\\
          &\leq
      2\En^{\sM,\pi}\brk*{\frac{d^{\sM,\pi}_h(x_h,a_h)}{d_h^{\sM,p}(x_h,a_h)+\veps\cdot(d_h^{\sM,\pi}(x_h,a_h)+\mu(x_h,a_h))}}
      + \En^{\sM,\pi}\brk*{\frac{\mu(x_h,a_h)}{d_h^{\sM,p}(x_h,a_h)+\veps\cdot\mu(x_h,a_h))}}.
    \end{align*}
    Observe that we can bound
    \begin{align*}
      &\En^{\sM,\pi}\brk*{\frac{d^{\sM,\pi}_h(x_h,a_h)}{d_h^{\sM,p}(x_h,a_h)+\veps\cdot(d_h^{\sM,\pi}(x_h,a_h)+\mu(x_h,a_h))}}\\
      &=
        \sum_{x\in\cX,a\in\cA}\frac{(d^{\sM,\pi}_h(x,a))^2}{d_h^{\sM,p}(x,a)+\veps\cdot(d_h^{\sM,\pi}(x,a)+\mu(x,a))}\\
      &=
        \sum_{x\in\cX,a\in\cA}\frac{d^{\sM,\pi}_h(x,a)\mu^{1/2}(x,a)}{d_h^{\sM,p}(x,a)+\veps\cdot(d_h^{\sM,\pi}(x,a)+\mu(x,a))}\cdot\frac{d^{\sM,\pi}_h(x,a)}{\mu^{1/2}(x,a)}\\
            &\leq
              \prn*{\sum_{x\in\cX,a\in\cA}\mu(x,a)\frac{(d^{\sM,\pi}_h(x,a))^2}{(d_h^{\sM,p}(x,a)+\veps\cdot(d_h^{\sM,\pi}(x,a)+\mu(x,a)))^2}}^{1/2}
              \cdot\prn*{\sum_{x\in\cX,a\in\cA}\frac{(d^{\sM,\pi}_h(x,a))^2}{\mu(x,a)}}^{1/2}\\
                  &\leq
\prn*{\frac{1}{\veps}\sum_{x\in\cX,a\in\cA}\mu(x,a)\frac{d^{\sM,\pi}_h(x,a)}{d_h^{\sM,p}(x,a)+\veps\cdot(d_h^{\sM,\pi}(x,a)+\mu(x,a))}}^{1/2}
                    \cdot\prn*{\ConehM}^{1/2}\\
                        &\leq
\prn*{\frac{1}{\veps}\En^{\sM,\pi}\brk*{\frac{\mu(x_h,a_h)}{d_h^{\sM,p}(x_h,a_h)+\veps\cdot{}\mu(x_h,a_h)}}}^{1/2}
              \cdot\prn*{\ConehM}^{1/2}.
    \end{align*}
Hence, if we define
    \begin{align*}
      V \ldef{}\inf_{p\in\Delta(\Pi)}\sup_{\pi\in\Pi}\En^{\sM,\pi}\brk*{\frac{\mu(x_h,a_h)}{d_h^{\sM,p}(x_h,a_h)+\veps\cdot{}\mu(x_h,a_h)}},
    \end{align*}
    this argument establishes that
    \begin{align*}
      \CovOptM \leq{} V + 2\sqrt{\frac{\ConehM}{\veps}\cdot{}V}.
    \end{align*}
We now claim that $V\leq{}1$. To see this, observe that the function
  \begin{align*}
(p,q)\mapsto{}\En_{\pi\sim{}q}\En^{\sM,\pi}\brk*{\frac{\mu(x_h,a_h)}{d_h^{\sM,p}(x_h,a_h)+\veps\cdot\mu(x_h,a_h)}}
  \end{align*}
  is convex-concave. In addition, it is straightforward to see that
  the function is jointly Lipschitz with respect to total variation
  distance whenever $\veps>0$. Hence, using the minimax theorem
  (\cref{lem:sion}), we have that
  \begin{align*}
V&=\inf_{p\in\Delta(\Pi)}\sup_{q\in\Delta(\Pi)}\En_{\pi\sim{}q}\En^{\sM,\pi}\brk*{\frac{\mu(x_h,a_h)}{d_h^{\sM,p}(x_h,a_h)+\veps\cdot{}\mu(x_h,a_h)}}\\
    &=\sup_{q\in\Delta(\Pi)}\inf_{p\in\Delta(\Pi)}\En_{\pi\sim{}q}\En^{\sM,\pi}\brk*{\frac{\mu(x_h,a_h)}{d_h^{\sM,p}(x_h,a_h)+\veps\cdot{}\mu(x_h,a_h)}}\\
    &\leq\sup_{q\in\Delta(\Pi)}\En_{\pi\sim{}q}\En^{\sM,\pi}\brk*{\frac{\mu(x_h,a_h)}{d_h^{\sM,q}(x_h,a_h)+\veps\cdot{}\mu(x_h,a_h)}}\\
    &= \sum_{x\in\cX,a\in\cA}\frac{d_h^{\sM,q}(x,a)\mu(x,a)}{d_h^{\sM,q}(x,a)+\veps\cdot{}\mu(x,a)}\leq{}1.
  \end{align*}
  
\end{proof}

\featurecoverage*

\begin{proof}[\pfref{prop:feature_coverage}]
  We first note that
  \[
    \ConehM\leq \abs{\cA}\cdot{} \inf_{\mu\in\Delta(\cA)}\sup_{\pi\in\Pi}\En^{\sM,\pi}\brk*{\frac{d_h^{\sM,\pi}(x_h)}{\mu(x_h)}}.
  \]
Let $\nu\in\Delta(\cX\times\cA)$ be arbitrary, and let
$\mu\ldef{}\nu\circ_{h-1}P^{\sM}_{h-1}$ be
  the distribution induced by sampling $(x_{h-1},a_{h-1})\sim\nu$ and
  $x_h\sim{}P_{h-1}^{\sM}(\cdot\mid{}x_{h-1},a_{h-1})$. Let $\pi\in\Pi$ be
  arbitrary. We can write
  \begin{align*}
    \En^{\sM,\pi}\brk*{\frac{d^{\sM,\pi}_h(x_{h})}{\mu(x_h)}}
    = \tri*{\En^{\sM,\pi}\brk*{\phi_{h-1}(x_{h-1},a_{h-1})}, \underbrace{\sum_{x\in\cX}\psi_h(x)\frac{d^{\sM,\pi}_h(x)}{\mu(x)}}_{\rdef{}w}}.
  \end{align*}
  Using Cauchy-Schwarz and defining
  $\Sigma_\nu\ldef\En_{(x_{h-1},a_{h-1})\sim\nu}\brk*{\phi_{h-1}(x_{h-1},a_{h-1})\phi_{h-1}(x_{h-1},a_{h-1})^{\trn}}$,
  we can bound
  \begin{align*}
    \tri*{\En^{\sM,\pi}\brk*{\phi_{h-1}(x_{h-1},a_{h-1})}, w}
    &=\tri*{\Sigma_\nu^{-1/2}\En^{\sM,\pi}\brk*{\phi_{h-1}(x_{h-1},a_{h-1})},
      \Sigma_\nu^{1/2}w}\\
    &\leq
      \frac{1}{2}\nrm*{\En^{\sM,\pi}\brk*{\phi_{h-1}(x_{h-1},a_{h-1})}}_{\Sigma_\nu^{-1}}^2
      + \frac{1}{2}\nrm*{w}_{\Sigma_\nu}^2.
  \end{align*}
  We can write
  \begin{align*}
    \nrm*{w}_{\Sigma_\nu}^2
    &= \En_{(x_{h-1},a_{h-1})\sim\nu}\brk*{\tri*{\phi_{h-1}(x_{h-1},a_{h-1}),w}^2}\\
    &= \En_{(x_{h-1},a_{h-1})\sim\nu}\brk*{\tri*{\phi_{h-1}(x_{h-1},a_{h-1}),
      \sum_{x\in\cX}\psi_h(x)\frac{d^{\sM,\pi}_h(x)}{\mu(x)}}^2} \\
    &=
      \En_{(x_{h-1},a_{h-1})\sim\nu}\brk*{\prn*{\En^{\sM}\brk*{\frac{d^{\sM,\pi}_h(x_h)}{\mu(x_h)}\mid{}x_{h-1},a_{h-1}}}^2}\\
    &\leq
      \En_{x_h\sim{}\nu\circ{}P}\brk*{\prn*{\frac{d^{\sM,\pi}_h(x_h)}{\mu(x_h)}}^2}
    = \En_{x_h\sim{}\mu}\brk*{\prn*{\frac{d^{\sM,\pi}_h(x_h)}{\mu(x_h)}}^2}
     = \En^{\sM,\pi}\brk*{\frac{d^{\sM,\pi}_h(x_h)}{\mu(x_h)}}.
  \end{align*}
  Hence, we have shown that
  \begin{align*}
    \En^{\sM,\pi}\brk*{\frac{d^{\sM,\pi}_h(x_h)}{\mu(x_h)}}
    \leq
    \frac{1}{2}\nrm*{\En^{\sM,\pi}\brk*{\phi_{h-1}(x_{h-1},a_{h-1})}}_{\Sigma_\nu^{-1}}^2
    + \frac{1}{2}\En^{\sM,\pi}\brk*{\frac{d^{\sM,\pi}_h(x_h)}{\mu(x_h)}}.
  \end{align*}
  To conclude, we rearrange and recall that 1) $\pi$ is
  arbitrary, and 2) we are free to choose $\nu$ to minimize the right-hand
  side. From here, the claim follows from \cref{prop:lone}.
\end{proof}

\linfimpossibility*

\begin{proof}[\pfref{prop:linf_impossibility}]
  Consider an MDP $M$ with horizon $H=1$, a singleton state space
  $\cS=\crl*{\sfrak}$, and action space $\cA=\bbN\cup\crl{\perp}$. For
  each $i\in\bbN$, we define a randomized policy $\pi\ind{i}$ via
  \begin{align*}
    \pi\ind{i}(a\mid{}\sfrak)
    =       \left\{
    \begin{array}{ll}
      1-\frac{1}{2i^2},&a=\perp,\\
      \frac{1}{2i^2},&a=i,\\
      0,&\text{o.w.}
                     \end{array}
              \right.,
  \end{align*}
  so that
  \[
  d_1^{\sM,\pi\ind{i}}(\sfrak,a)=
  \left\{
    \begin{array}{ll}
      1-\frac{1}{2i^2},&z=\perp,\\
      \frac{1}{2i^2},&z=i,\\
      0,&\text{o.w.}
    \end{array}
    \right.
  \]
  We set $\Pi=\crl*{\pi\ind{i}}_{i\in\bbN}$, and abbreviate
  $d\ind{i}(\sfrak,a) = d_1^{\sM,\pi\ind{i}}(\sfrak,a)$ going forward.

  We first bound $\CinfM$. We choose $\mu$ by setting $\mu(\sfrak,\perp)=\frac{1}{2}$ and $\mu(\sfrak,i) =
\frac{3}{\pi^2}\cdot\frac{1}{i^2}$, which has
$\sum_{a\in\cA}\mu(\sfrak,a)=1$. It is fairly immediate to see that
for all $i$, we have $\frac{d\ind{i}(\sfrak,\perp)}{\mu(\sfrak,\perp)}\leq{}2$ and 
\begin{align*}
  \frac{d\ind{i}(\sfrak,i)}{\mu(\sfrak,i)} = \frac{\pi^2}{6}\leq{}2.
\end{align*}
This shows that $\CinfM\leq{}2$.  On the other hand, for any $p\in\Delta(\bbN)$, we have
\begin{align*}
\ICovM(p) \geq \sup_{i\in\bbN}\sup_{j\in\bbN}\crl*{\frac{d\ind{i}(\sfrak,j)}{\En_{k\sim{}p}\brk*{d\ind{k}(\sfrak,j)}+\veps\cdot{}d\ind{i}(\sfrak,j)}}
  &\geq \sup_{i\in\bbN}\frac{d\ind{i}(\sfrak,i)}{\En_{k\sim{}p}\brk*{d\ind{k}(\sfrak,i)+\veps\cdot{}d\ind{i}(\sfrak,i)}}\\
  &= \sup_{i\in\bbN}\frac{1/2i^2}{(p(i)+\veps)\cdot(1/2i^2)}\\
  &= \sup_{i\in\bbN}\frac{1}{p(i)+\veps}=\frac{1}{\veps},
\end{align*}
where the conclusion holds because $\sum_{i\in\bbN}p(i)\leq{}1$, which
means for all $\delta>0$, there exists $i$ such that $p(i)\leq\delta$.
  
\end{proof}

\section{Proofs from \creftitle{sec:model_free}}
\label{app:model_free}

\arxiv{
  This section is organized as follows.
  \begin{itemize}
  \item \cref{sec:mf_general} presents our most general
    guarantee for \cref{alg:model_free},
    \cref{thm:model_free}. The sample complexity bound in
    \cref{thm:model_free_psdp} is derived as a consequence.
  \item \cref{sec:mf_prelim} presents technical preliminaries for the
    proof of \cref{thm:model_free}.
  \item \cref{sec:mf_weight} presents self-contained guarantees for
    the weight function estimation technique used within \cref{alg:model_free}.
  \item \cref{app:policy_opt} presents self-contained guarantees for
    the \psdp subroutine for policy optimization used within \cref{alg:model_free}.
  \item Finally, \cref{sec:mf_proof} combines the preceding
    development to prove \cref{thm:model_free}.
  \end{itemize}
}

\subsection{General Guarantees for \creftitle{alg:model_free}}
\label{sec:mf_general}

We first present general assumptions on the weight function estimation
and policy optimization subroutines under which \cref{alg:model_free}
can be analyzed, then present our most general result, \cref{thm:model_free}.

\subsubsection{Weight Function Realizability}

\cref{thm:model_free_psdp} is analyzed under the weight function
realizability assumption in
\cref{ass:weight_realizability_weak}. However, \cref{alg:model_free}
is most directly analyzed in terms of the following, slightly stronger
weight function assumption, which we show is implied by
\cref{ass:weight_realizability_weak}. To motivate the assumption, recall that we seek to estimate a weight function $\what\ind{t}_h$ approximating \cref{eq:weight_est}.
\begin{assumption}[Weight function realizability---strong version]
  \label{ass:weight_realizability_strong}
  For a parameter $T\in\bbN$, we assume that for all $h\geq{}2$, all $t\in\brk{T}$, and all policies
  $\pi\ind{1},\ldots,\pi\ind{t-1}\in\PiNS$, we have that
  \begin{align*}
w_h^{\pi\ind{1},\ldots,\pi\ind{t}}(x'\mid{}x,a)\ldef{}\frac{\Pmstar_{h-1}(x'\mid{}x,a)}{\sum_{i<t}d_h^{\sMstar,\pi\ind{i}}(x')+\Pmstar_{h-1}(x'\mid{}x,a)}
    \in \cW_h.
  \end{align*}
We assume without loss of generality that $\nrm*{w}_{\infty}\leq{}1$
for all $w\in\cW_h$.  
\end{assumption}
To compute a policy cover that approximately solves
$p_h=\argmin_{p\in\Delta(\PiRNS)}\pCovM[\Mstar](p_h)$ for parameter
$\veps>0$, we require that \cref{ass:weight_realizability_strong}
holds for $T=\frac{1}{\veps}$.

\cref{alg:model_free} enjoys tighter guarantees when
\cref{ass:weight_realizability_strong} is satisfied, but the following result shows that \cref{ass:weight_realizability_weak}
implies \cref{ass:weight_realizability_strong} at the cost of a small
degradation in rate.
\begin{proposition}
  \label{prop:weight_realizability}
  For any $T\in\bbN$, given a weight function class $\cW$ satisfying
  \cref{ass:weight_realizability_weak}, the induced class $\cW'$ given by
  \begin{align*}
    \cW'_h\ldef
    \crl*{
    (x,a,x')\mapsto{}
    \frac{1}{1+\sum_{i<t}\frac{1}{w_h\ind{i}(x'\mid{}x,a)} 
    } \mid{} w_h\ind{1},\ldots,w_{h}\ind{t-1}\in\cW_h, t\in\brk{T}}
  \end{align*}
  satisfies \cref{ass:weight_realizability_strong}, and has $\log\abs*{\cW'_h}\leq\bigoh\prn*{T\cdot\log\abs*{\cW_h}}$.
\end{proposition}
For $T=\frac{1}{\veps}$ this
increases the weight function class size from $\log\abs*{\cW}$ to
$\bigoh\prn*{\frac{1}{\veps}\cdot\log\abs*{\cW}}$, leading to an extra
$\frac{1}{\veps}$ factor in the final sample complexity bound for our
main result (\cref{thm:model_free_psdp}). \padelete{Consequently, we restrict our attention to
\cref{ass:weight_realizability_strong} going forward.}

\begin{remark}[Sufficiency of \cref{ass:weight_realizability_strong}]
  When invoked with layer $h\geq{}2$ and iteration $t\geq{}2$ within
  \cref{alg:model_free}, \estimateweight (\cref{alg:estimate_weight})
  collects datasets $\cD_1\sim\mu$ and $\cD_2\sim\nu$ such that
  \begin{align*}
    \mu(x'\mid{}x,a)=\Pmstar_{h-1}(x'\mid{}x,a),\quad\nu(x'\mid{}x,a)=\frac{1}{t}\prn*{\sum_{i<t}d_h^{\sMstar,\pi\ind{h,i}}(x')
    + \Pmstar_{h-1}(x'\mid{}x,a)},
  \end{align*}
  and
  \begin{align}
    \label{eq:roll_in_dist}
    \mu(x,a) = \nu(x,a) = \frac{1}{2}\prn*{d_{h-1}^{\sMstar,p_{h-1}}(x,a) + \frac{1}{t-1}\sum_{i<t}d_{h-1}^{\sMstar,\pi\ind{h,i}\circ_{h-1}\piunif}(x,a)}.
  \end{align}
  Then, in \cref{line:weight_alg}, the algorithm computes the
  estimator \cref{eq:weight_estimator_body} with respect to the class
  $t\cdot{}\cW$, which is guaranteed to have
  $\frac{\mu(x,a,x')}{\nu(x,a,x')}
    = t\cdot{}w_h\ind{t}(x'\mid{}x,a) \in t\cdot{}\cW$
  under \cref{ass:weight_realizability_strong}.
\end{remark}

\subsubsection{Policy Optimization Subroutine}

This section presents general conditions for the subroutine \policyopt under which \mfalg obtains the same guarantees as in \cref{thm:model_free_psdp}, and establishes that \psdp satisfies this assumption. 

To formalize the requirement of \policyopt, recall that for each layer $h\geq{}2$, iteration $t\in\brk{T}$, and each
$\ell\leq{}h-1$, we define
\begin{align*}
  Q_{\ell}^{\sMstar,\pi}(x,a;\what_h\ind{t})=\En^{\sMstar,\pi}\brk*{
  \what_h\ind{t}(x_h\mid{}x_{h-1},a_{h-1})\mid{}x_\ell=x,a_\ell=a
  }
\end{align*}
as the Q-function for a policy $\pi\in\PiNS$ under the (stochastic) reward
$r_{h-1}\ind{t}=\what_h\ind{t}(x_h\mid{}x_{h-1},a_{h-1})$ in
\cref{alg:model_free}. We assume that the policy
$\pi\ind{t,h}=\policyopt_{h-1}(r\ind{h,t}, p_{1:h-1},\eps,\delta)$
approximately maximizes this Q-function under $p_1,\ldots,p_{h-1}$.
\begin{assumption}[Local optimality for policy optimization]
  \label{ass:policy_opt}
  For any fixed iteration $h\geq{}2$ and $t\in\brk{T}$, the subroutine
  $\policyopt_{h-1}(r\ind{h,t}, p_{1:h-1},\eps,\delta)$ produces a
  policy $\pi\ind{h,t}$ such that with probability at least $1-\delta$,
  \begin{align}
    \label{eq:policy_opt}
    \sum_{\ell=1}^{h-1}\En^{\sMstar,p_\ell}
  \brk*{\max_{a\in\cA}Q_{\ell}^{\sMstar,\pi\ind{h,t}}(x_\ell,a;\what_h\ind{t})
  -Q_{\ell}^{\sMstar,\pi\ind{h,t}}(x_\ell,\pi\ind{h,t}(x_\ell);\what_h\ind{t})}
    \leq \eps,
  \end{align}
  and does so using $\Nopt(\eps,\delta)$ episodes.
\end{assumption}
This assumption asserts that $\pi\ind{t,h}$ cannot be substantially
improved, but only with respect to the state distribution induced by
$p_1,\ldots,p_{h-1}$.\footnote{Note that if $p_1,\ldots,p_{h-1}$
  uniformly cover all policies, then \cref{ass:policy_opt} implies
  that $\pi\ind{h,t}$ is globally optimal by the performance
  difference lemma. However, \cref{eq:policy_opt} can still lead to
  useful guarantees in the presence of partial coverage, which our
  analysis critically exploits.}
This is a weak guarantee that can be achieved
using only data collected from $p_1,\ldots,p_{h-1}$ (e.g., via offline
RL methods or hybrid offline/online methods), and does not require
systematic exploration. 

The next result shows that \psdp satisfies \cref{ass:policy_opt} under
the value function realizability assumption in \cref{ass:qpi_weight}. 
      \begin{lemma}[Local optimality for \psdp]
        \label{lem:psdp_simple}
        Suppose \cref{ass:qpi_weight} holds. Then for any $\eps,\delta\in(0,1)$, the subroutine
        $\pi\ind{h,t}=\psdp_{h-1}(r\ind{h,t}, p_{1:h-1},\eps,\delta)$
        satisfies \cref{ass:policy_opt}, and does so using at most
        $\Npsdp(\eps,\delta)=\bigoh\prn*{\frac{H^3\abs*{\cA}\log(\abs{\cQ}H\delta^{-1})}{\eps^2}}$ episodes.
      \end{lemma}
      See \cref{app:policy_opt} for details.
We expect that similar guarantees can be proven for Natural Policy
Gradient, Conservative Policy Iteration, and other standard local
search methods. Different subroutines may allow one to make use of
weaker function approximation requirements.

\subsubsection{General Guarantee for \mfalg (\creftitle{alg:model_free})}

Our most general guarantee for \mfalg is given below.

\begin{restatable}[General guarantee for \cref{alg:model_free}]{theorem}{modelfree}
  \label{thm:model_free}
  Let $\veps\in(0,1/2)$ and $\delta\in(0,e^{-1})$ be given, and suppose that \cref{ass:weight_realizability_strong} and
  \cref{ass:policy_opt} are satisfied. Then \cref{alg:model_free}
  produces policy covers $p_1,\ldots,p_H\in\Delta(\PiRNS)$ such that with probability at least $1-\delta$, for all $h\in\brk{H}$,
\[
  \pCovM[\Mstar](p_h)   \leq{} 170H\log(\veps^{-1})\cdot{}\CpushM[\Mstar],
\]
and does so using at most 
\begin{align*}
  N \leq \bigoht\prn*{\frac{H\abs*{\cA}\log(\abs{\cW}\delta^{-1})}{\veps^3}
  +\frac{H}{\veps}\Nopt(c\veps^2,\delta/2HT)}
\end{align*}
episodes, where $c>0$ is a sufficiently small absolute constant.
\end{restatable}
In particular, this result shows that we can optimize the pushforward
coverability objective (and consequently the \mainobj objective, via
\cref{prop:pushforward_relaxation}), up to small
$\bigoh(H\log(\veps^{-1}))$ approximation factor. The sample
complexity is polynomial in all relevant problem parameters whenever
the subroutine \policyopt has polynomial sample complexity. Note that the sample complexity for the first term is of order $\nicefrac{1}{\veps^3}$ (as opposed to the slower $\nicefrac{1}{\veps^4}$ in \cref{thm:model_free_psdp}) since we are stating this result under the stronger weight function realizability assumption (\cref{ass:weight_realizability_strong}). 

Combining \cref{thm:model_free} with \cref{prop:weight_realizability}
and the guarantee for \psdp (\cref{lem:psdp_simple}), we obtain \cref{thm:model_free_psdp}.

\subsection{Technical Preliminaries}
\label{sec:mf_prelim}

\begin{lemma}
  \label{lem:mp_weight}
  For any distribution $\omega\in\Delta(\cZ)$ and any pair of functions $w,w':\cZ\to\bbR_{+}$,
  \begin{align*}
        \En_{\omega}\brk*{w} \leq 3\En_{\omega}\brk*{w'} + 2\En_{\omega}\brk[\big]{(\sqrt{w}-\sqrt{w'})^2}.
  \end{align*}
\end{lemma}
\begin{proof}[\pfref{lem:mp_weight}]
  By AM-GM, we have
  \begin{align*}
    \abs*{\En_{\omega}\brk*{w}-\En_{\omega}\brk*{w'}}
    &\leq     \En_{\omega}\brk*{\abs{\sqrt{w}-\sqrt{w'}}(\sqrt{w}+\sqrt{w'})}
    \\
    &\leq
      \sqrt{\En_{\omega}\brk[\big]{(\sqrt{w}+\sqrt{w'})^2}\cdot\En_{\omega}\brk[\big]{(\sqrt{w}-\sqrt{w'})^2}}\\
    &\leq     \frac{1}{2}(\En_{\omega}\brk{w}+\En_{\omega}\brk*{w'}) + \frac{1}{2}\En_{\omega}\brk[\big]{(\sqrt{w}-\sqrt{w'})^2}.
  \end{align*}
  Rearranging, we conclude that
  \begin{align*}
    \En_{\omega}\brk*{w} \leq 3\En_{\omega}\brk*{w'} + 2\En_{\omega}\brk[\big]{(\sqrt{w}-\sqrt{w'})^2}.
  \end{align*}
\end{proof}

\begin{lemma}[e.g., \citet{xie2023role,mhammedi2023efficient}]
  \label{lem:linf_potential}
  Consider a set $\cZ$ and a sequence of distributions
  $d\ind{1},\ldots,d\ind{T}\in\Delta(\cZ)$ for which there exists a
  distribution $\mu\in\Delta(\cZ)$ such that
  $\sup_{z\in\cZ}\crl*{\frac{d\ind{t}(z)}{\mu(z)}}\leq{}C$ for all
  $t\in\brk{T}$. For any sequence of functions
  $g\ind{1},\ldots,g\ind{T}\subset(\cZ\to\brk{-B,B})$, it holds that 
    \begin{align}
    \sum_{t=1}^{T}\En_{z\sim{}d\ind{t}}\brk*{g(z)}
    \leq
    \sqrt{2C\log(2T)\sum_{t=1}^{T}\sum_{i<t}\En_{z\sim{}d\ind{i}}\brk*{(g\ind{t}(z))^2}}
    + 2CB.
    \end{align}
\end{lemma}

\subsection{Weight Function Estimation}
\label{sec:mf_weight}
In this section, we give self-contained guarantees for the statistical
problem of estimating the density ratio (or, ``weight function'') for
a pair of distributions.

Consider the following setting. Let $\cZ$ be a set. We receive samples
$z_\mu\ind{1},\ldots,z_\mu\ind{n}\in\cZ$
and $z_\nu\ind{1},\ldots,z_\nu\ind{n}\in\cZ$, where $z_\mu\ind{t}\sim\mu\ind{t}\in\Delta(\cZ)$ and
$z_\nu\ind{t}\sim\nu\ind{t}\in\Delta(\cZ)$. The distributions $\mu\ind{t}$ and $\nu\ind{t}$ can be chosen
in an adaptive fashion based on $z_\mu\ind{1},z_\nu\ind{1},\ldots,z_\mu\ind{t-1},z_\nu\ind{t-1}$. We define
$\mu=\frac{1}{n}\sum_{t=1}^{n}\mu\ind{t}$ and
$\nu=\frac{1}{n}\sum_{t=1}^{n}\nu\ind{t}$, and our goal is to estimate
the density ratio
\begin{align*}
  \wstar(z) \ldef \frac{\mu(z)}{\nu(z)}.
\end{align*}
We assume that $\nrm*{w_{\star}}_{\infty}\leq{}B$, and assume access
to a \emph{realizable} weight function class $\cW$ with $w^{\star}\in\cW$.
Following \citet{nguyen2010estimating} (see also
\citet{katdare2023marginalized}), we consider the estimator
\begin{align}
  \label{eq:weight_estimator_log}
\what \ldef \argmax_{w\in\cW} \Eh_{\mu}\brk*{\log(w)} - \Eh_{\nu}\brk*{w},
\end{align}
where $\Eh_\mu\brk{\cdot}$ denotes the empirical expectation with
respect to $z_\mu\ind{1},\ldots,z_\mu\ind{n}$ and $\Eh_\nu\brk{\cdot}$
denotes the empirical expectation with respect to
$z_\nu\ind{1},\ldots,z_\nu\ind{n}$. The following theorem gives a finite-sample bound for this
estimator, which may be of independent interest.
\begin{theorem}
  \label{thm:weight_estimator_log}
  Suppose that $w^{\star}\in\cW$ and $\sup_{w\in\cW}\nrm*{w}_{\infty}\leq{}B$.
  The estimator in \cref{eq:weight_estimator_log} ensures that with
  probability at least $1-\delta$,
  \begin{align*}
        \Dhelsnu{\what}{\wstar}
      \leq \frac{20B\log(\abs{\cW}\delta^{-1})}{n},
  \end{align*}
  where $\Dhelsnu{w}{w'}\ldef\En_{\nu}\brk[\big]{\prn[\big]{\sqrt{w}-\sqrt{w'}}^2}$.
\end{theorem}

\begin{remark}[Extension to contextual weight function estimation]
  \label{rem:contextual_weight}
An immediate corollary for \cref{thm:weight_estimator_log} concerns
the following ``contextual'' setting. Suppose that $\cZ=\cX\times\cY$,
and that for all $t$, $z_\mu\ind{t}=(x\ind{t},y_\mu\ind{t})$ and
$z_\nu=(x\ind{t}, y_\nu\ind{t})$ have the same marginal distribution
for $x\ind{t}$, i.e. $\mu\ind{t}(x,y)=\mu\ind{t}(y\mid{}x)\omega(x)$ and
$\nu\ind{t}(x,y)=\nu\ind{t}(y\mid{}x)\omega(x)$ for some
$\omega\in\Delta(\cX)$. Define
$\mu(y\mid{}x)=\frac{1}{n}\sum_{t=1}^{n}\mu\ind{t}(y\mid{}x)$ and
$\nu(y\mid{}x)=\frac{1}{n}\sum_{t=1}^{n}\nu\ind{t}(y\mid{}x)$, and let
$\wstar(y\mid{}x)=\frac{\mu(y\mid{}x)}{\nu(y\mid{}x)}$. Then, given a
class of weight functions $\cW$ with $\wstar\in\cW$, where each
$w\in\cW$ has the form $w(y\mid{}x)$ and $\nrm*{w}_{\infty}\leq{}B$,
the estimator in \cref{eq:weight_estimator_log} ensures that
\begin{align*}
      \En_{x\sim\omega,y\sim\nu(\cdot\mid{}x)}\brk*{\abs*{\what(y\mid{}x)-\wstar(y\mid{}x)}}
      \leq 10B\sqrt{\frac{\log(\abs{\cW}\delta^{-1})}{n}}.
  \end{align*}
\end{remark}

\begin{proof}[\pfref{thm:weight_estimator_log}]
  Define $V(w) = \En_{\mu}\brk*{\log(w)} - \En_{\nu}\brk*{w}$ and
  $\Vhat(w) = \Eh_{\mu}\brk*{\log(w)} - \Eh_{\nu}\brk*{w}$, and note
  that $\wstar=\frac{\mu}{\nu}=\argmax_{w}V(w)$. We begin by
  performing concentration on the log-loss terms. Define
  $X_t(w)=\frac{1}{2}\prn*{\log(\wstar(z_\mu\ind{t}))-\log(w(z_\mu\ind{t}))}$. By
  \cref{lem:martingale_chernoff} and a union bound, we have that with
  probability at least $1-\delta$, for all $w\in\cW$,
  \begin{align*}
    \frac{1}{n}\sum_{t=1}^{n}-\log\prn*{\En_{\mu\ind{t}}\brk*{\exp\prn*{\frac{1}{2}\log(w/\wstar)}}}
    \leq{} \frac{1}{2}\prn*{\Eh_{\mu}\brk*{\log(\wstar)}-\Eh_{\mu}\brk*{\log(w)}}
    + \frac{\log(\abs{\cW}\delta^{-1})}{n}.
  \end{align*}
  Note that
  $\En_{\mu\ind{t}}\brk*{\exp\prn*{\frac{1}{2}\log(w/\wstar)}}=\En_{\mu\ind{t}}\brk*{\sqrt{w/\wstar}}$. We
  now state and prove a basic technical lemma. \begin{lemma}
  \label{lem:log_lower}
  For all $x>0$,
  $-\log(x) \geq{} 1-x$.
\end{lemma}
\begin{proof}[\pfref{lem:log_lower}]
  Let $f(x)=-\log(x)$. Since $f$ is convex, we have that for all $x>0$,
  \begin{align*}
    f(x)\geq{}f(1) + f'(1)(x-1).
  \end{align*}
  Noting that $f(1)=0$ and $f'(1)=-1$, the result is established.
\end{proof}

Using \cref{lem:log_lower}, we have
  \begin{align*}
    \frac{1}{n}\sum_{t=1}^{n}-\log\prn*{\En_{\mu\ind{t}}\brk*{\exp\prn*{\frac{1}{2}\log(w/\wstar)}}}
    &\geq 1-
      \frac{1}{n}\sum_{t=1}^{n}\En_{\mu\ind{t}}\brk*{\sqrt{w/\wstar}}
    \\
    &= 1- \En_{\mu}\brk*{\sqrt{w/\wstar}}
    = 1- \En_{\nu}\brk*{\sqrt{w\cdot\wstar}},
  \end{align*}
  where the last line uses that
  $\En_{\mu}\brk*{\sqrt{w/\wstar}}=\En_{\nu}\brk*{\wstar\sqrt{w/\wstar}}=\En_{\nu}\brk*{\sqrt{w\cdot\wstar}}$. By
  direct calculation, we have that
  \begin{align*}
    \Dhelsnu{w}{\wstar}=\En_{\nu}\brk*{\prn*{\sqrt{w}-\sqrt{\wstar}}^2}
    = \En_{\nu}\brk*{\wstar} +
      \En_{\nu}\brk*{w}-2\En_{\nu}\brk*{\sqrt{w\cdot\wstar}}
    =1 + \En_{\nu}\brk*{w}-2\En_{\nu}\brk*{\sqrt{w\cdot\wstar}},
  \end{align*}
  so that
  \begin{align*}
    1- \En_{\nu}\brk*{\sqrt{w\cdot\wstar}}
    = \frac{1}{2}\Dhelsnu{w}{\wstar} + \frac{1}{2}\prn*{1 - \En_{\nu}\brk*{w}}.
  \end{align*}
  Specializing to $\what$, we have
  \begin{align*}
    \frac{1}{2}\Dhelsnu{\what}{\wstar} + \frac{1}{2}\prn*{1 -
    \En_{\nu}\brk*{\what}}
    &\leq{} \frac{1}{2}\prn*{\Eh_{\mu}\brk*{\log(\wstar)}-\Eh_{\mu}\brk*{\log(\what)}}
    + \log(\abs{\cW}\delta^{-1})\\
    &= \frac{1}{2}\prn*{\Vhat(\wstar)-\Vhat(\what)}
      + \frac{1}{2}\prn*{\Eh_{\nu}\brk*{\wstar}-\Eh_{\nu}\brk*{\what}}
    + \frac{\log(\abs{\cW}\delta^{-1})}{n}.
  \end{align*}
  Since $\Vhat(\wstar)-\Vhat(\what)\leq0$, rearranging gives
  \begin{align}
    \label{eq:w_log0}
    \Dhelsnu{\what}{\wstar}
    &\leq{} \prn*{\Eh_{\nu}\brk*{\wstar}-\Eh_{\nu}\brk*{\what}} -
    \prn*{1-\Eh_{\nu}\brk*{\what}}
      + \frac{2 \log(\abs{\cW}\delta^{-1})}{n}\\
    &= \prn*{\Eh_{\nu}\brk*{\wstar}-\Eh_{\nu}\brk*{\what}} -
          \prn*{\En_{\nu}\brk*{\wstar}-\En_{\nu}\brk*{\what}}
    + \frac{2\log(\abs{\cW}\delta^{-1})}{n}.
  \end{align}
  Using \cref{lem:freedman}, we have that for all $\eta\leq{}1/2B$, with probability at least
  $1-\delta$, for all $w\in\cW$
  \begin{align*}
    \prn*{\Eh_{\nu}\brk*{\wstar}-\Eh_{\nu}\brk*{w}} -
    \prn*{\En_{\nu}\brk*{\wstar}-\En_{\nu}\brk*{w}}
    &\leq \frac{\eta}{n}\sum_{t=1}^{n}\En_{\nu_t}\brk*{(w-\wstar)^2} +
    \frac{\log(\abs{\cW}\delta^{-1})}{\eta{}n}\\
    &= \eta\En_{\nu}\brk*{(w-\wstar)^2} + \frac{\log(\abs{\cW}\delta^{-1})}{\eta{}n}.
  \end{align*}
  We further observe that
  \begin{align*}
    \En_{\nu}\brk*{(w-\wstar)^2}
    = \En_{\nu}\brk*{(\sqrt{w}-\sqrt{\wstar})^2
    (\sqrt{w}+\sqrt{\wstar})^2}
    \leq 4B\En_{\nu}\brk*{(\sqrt{w}-\sqrt{\wstar})^2}
    = 4B\cdot\Dhelsnu{w}{\wstar},
  \end{align*}
  so choosing $\eta = \frac{1}{8B}$ gives
  \begin{align}
    \label{eq:w_log1}
    \prn*{\Eh_{\nu}\brk*{\wstar}-\Eh_{\nu}\brk*{w}} -
    \prn*{\En_{\nu}\brk*{\wstar}-\En_{\nu}\brk*{w}}
    \leq \frac{1}{2}\Dhelsnu{w}{\wstar}
    + \frac{8B\log(\abs{\cW}\delta^{-1})}{n}.
  \end{align}
  Combining this with \cref{eq:w_log0}, we conclude that
  \begin{align*}
        \Dhelsnu{\what}{\wstar}
    &\leq \frac{1}{2}\Dhelsnu{\what}{\wstar}
    + \frac{8B\log(\abs{\cW}\delta^{-1})}{n}
    + 2 \frac{\log(\abs{\cW}\delta^{-1})}{n},
  \end{align*}
  which implies that
    \begin{align*}
        \Dhelsnu{\what}{\wstar}
      \leq \frac{20B\log(\abs{\cW}\delta^{-1})}{n},
    \end{align*}
    after simplifying.
\end{proof}

\subsection{Policy Optimization Subroutines}
\label{app:policy_opt}

\subsubsection{Policy Search by Dynamic Programming (\psdp)}

\begin{algorithm}[htp]
	\caption{$\texttt{PSDP}_h(p_{1:h}, r_{1:h};\eps,\delta, \cQ_{1:h})$: Policy Search by Dynamic Programming (cf. \citet{bagnell2003policy})}
	\label{alg:psdp}
	\begin{algorithmic}[1]\onehalfspacing
          \State\textbf{input:}
		\begin{itemize}
			\item Target layer $h\in[H]$, policy covers
                          $p_{1:h}$, reward functions
                          $r_{1:h}$.
                          \item Accuracy
                            parameters $\eps,\delta\in\brk{0,1}$.
                          \item Function classes $\cQ_{1:h}$.
                          \end{itemize}
                          \State Let $n = \npsdp(\eps,\delta)\ldef{}
                          c\cdot{}\frac{H^2\abs*{\cA}\log(\abs{\cQ}H\delta^{-1})}{\eps^2}$
                          for a sufficiently large numerical constant $c>0$.
		\For{$\ell=h, \dots, 1$} 
		\State $\cD_\ell \gets\emptyset$. 
		\For{$\npsdp$ times}
		\State Sample $\pi \sim p_\ell$.
		\State Sample $(x_\ell, a_\ell, \sum_{\ell'=\ell}^h r_{\ell'}(x_{\ell'},a_{\ell'}))\sim
		\pi\circ_{\ell} \piunif \circ_{\ell+1} \pihat$.
		\State Update dataset: $\cD_{\ell} \gets \cD_{\ell} \cup \crl[\big]{\prn[\big]{x_\ell, a_\ell,  \sum_{\ell'=\ell}^h {r}_{\ell'}(x_{\ell'},a_{\ell'})}}$.
		\EndFor
		\State Solve regression: 
		\[\Qhat_\ell\gets\argmin_{Q\in \cQ_\ell}  \sum_{(x, a, R)\in\cD_{\ell}} (Q(x,a)  - R)^2.\] \label{eq:mistake}
		\State Define
                $\pihat_\ell(x) = \argmax_{a\in \cA}  \Qhat_\ell(x,a)$.
		\EndFor
		\State \textbf{return:} Policy $\pihat$.
	\end{algorithmic}
\end{algorithm}

This section presents self-contained guarantees for Policy Search by
Dynamic Programming (\psdp{}, \cref{alg:psdp})
\citep{bagnell2003policy}, which performs local policy optimization given access to
exploratory distributions $p_{1:h}\in\Delta(\PiRNS)$. \psdp takes as input an arbitrary reward functions
$r_{1:h}:\cX \times \cA \rightarrow \brk{0,1}$ and a function
class $\cQ=\cQ_{1:h}$, where $\cQ_\ell\subseteq \{Q: \cX\times
\cA\rightarrow \brk{0,1}\}$, that can represent certain $Q$-functions for these rewards.

We prove that with high probability, the output $\pihat = \psdp_h(p_{1:h}, r_{1:h};\eps,\delta, \cQ)$ is an approximate maximizer of the objective 
\begin{align}
	\max_{\pi \in \PiNS} \En^{\sMstar,\pi}\left[\sum_{\ell=1}^{h} r_\ell(x_\ell,a_\ell)\right],
\end{align}
in a ``local'' sense with respect to $p_{1:h}$ (cf. \cref{eq:vepsopt}).

To prove guarantees for \psdp, we make use of the following
realizability assumption for the class $\cQ = \cQ_{1:h}$.
\begin{definition}
	\label{def:policy_opt_realizability}
	We say that function classes $\cQ_{1:h}$, where
        $\cQ_\ell\subseteq \{Q: \cX\times \cA\rightarrow \bbR_+\}$ for
        $\ell\in[h]$, realize the reward functions $r_{1:h}:\cX\times \cA \rightarrow \bbR_+$ if for all $t\in[h]$ and all $\pi\in \PiNS$,
	\begin{align}
		Q_\ell^{\sMstar,\pi}(\cdot,\cdot;r)\in \cQ_\ell, \quad \text{where}\qquad 	Q^{\sMstar,\pi}_\ell(x,a;r)\coloneqq \En^{\sMstar,\pi}\left[\left.\sum_{\ell'=\ell}^{h} r_{\ell'}(x_{\ell'},a_{\ell'})\ \right| \ x_\ell=x,a_\ell=a\right]. \label{eq:real}
	\end{align}
      \end{definition}

      \begin{lemma}[Main guarantee for \psdp]
        \label{lem:psdp}
        For any $\eps,\delta\in(0,1)$ and reward function
        $\crl*{r_\ell}_{\ell\in\brk{h}}$ with
        $\sum_{\ell=1}^{h}r_\ell\in\brk{0,1}$ that is
        realizable in the sense of
        \cref{def:policy_opt_realizability}, \psdp ensures that with
        probability at least $1-\delta$, the output
        $\pihat=\psdp_h(p_{1:h}, r_{1:h};\eps,\delta,\cQ_{1:h})$
        satisfies
        \begin{align}
          \label{eq:psdp_main}
          \sum_{\ell=1}^{h}\En^{\sMstar,p_\ell}\brk*{
          \max_{a\in\cA}Q_\ell^{\sMstar,\pihat}(x_\ell,a;r)
          -Q_\ell^{\sMstar,\pihat}(x_\ell,\pihat(x_\ell);r)} \leq \eps,
        \end{align}
        and does so using at most
        $\Npsdp(\eps,\delta)=\bigoh\prn*{\frac{H^3\abs*{\cA}\log(\abs{\cQ}H\delta^{-1})}{\eps^2}}$ episodes.
      \end{lemma}
      \begin{proof}[\pfref{lem:psdp}]
        See the proof of Theorem D.1 in \citet{mhammedi2023representation}.
      \end{proof}

\subsection{Proof of \creftitle{thm:model_free}}
\label{sec:mf_proof}

\newcommand{\Delbarweightoff}{\wb{\Delta}_{w, \mathrm{off};h}}
\newcommand{\Delbarweighton}{\wb{\Delta}_{w, \mathrm{on};h}}
\newcommand{\Delbaropt}{\wb{\Delta}_{\mathrm{opt};h}}

\newcommand{\vepsweightoff}{\veps_{w, \mathrm{off};h}}
\newcommand{\vepsweighton}{\veps_{w, \mathrm{on};h}}
\newcommand{\vepsopt}{\veps_{\mathrm{opt};h}}

\modelfree*

\begin{proof}[\pfref{thm:model_free}]
  To keep notation compact, throughout this section we abbreviate
  $d^{\pi}_h\equiv{}d^{\sMstar,\pi}_h$,
  $P_h(\cdot\mid{}\cdot)\equiv{}P_h^{\sMstar}(\cdot\mid{}\cdot)$,
  $\En^{\pi}\brk*{\cdot}\equiv\En^{\sM,\pi}\brk*{\cdot}$, and so on
  when the MDP is clear from context.

  Define
  \begin{align}
    \label{eq:wt}
    w_h\ind{t}(x_h\mid{}x_{h-1},a_{h-1})
    = \frac{P_{h-1}(x_h\mid{}x_{h-1},a_{h-1})}{\sum_{i<t}d_h\ind{\pi\ind{h,i}}(x_h)+P_{h-1}(x_h\mid{}x_{h-1},a_{h-1})}
  \end{align}

  We define two notions of estimation error for the weight function
  estimates produced by the subroutine \estimateweight
  (\cref{alg:estimate_weight}):
  \begin{align}
    \label{eq:vepsweight}
    &(\vepsweightoff\ind{t})^2 =
      \En^{p_{h-1}}\brk*{\prn*{\sqrt{\what_h\ind{t}(x_h\mid{}x_{h-1},a_{h-1})}-\sqrt{w_h\ind{t}(x_h\mid{}x_{h-1},a_{h-1})}}^2},\mathand\\
    &(\vepsweighton\ind{t})^2
      = \frac{1}{t-1}\sum_{i<t}\En^{\pi\ind{h,i}\circ_{h-1}\piunif}\brk*{\prn*{\sqrt{\what_h\ind{t}(x_h\mid{}x_{h-1},a_{h-1})}-\sqrt{w_h\ind{t}(x_h\mid{}x_{h-1},a_{h-1})}}^2}.
  \end{align}
  We define a notion of ``local'' suboptimality for the policies
  $\pi\ind{h,t}$ produced by the subroutine \policyopt as follows:
  \begin{align}
    \label{eq:vepsopt}
    \vepsopt\ind{t}
    = \sum_{\ell=1}^{h-1}\En^{p_\ell}
    \brk*{\max_{a\in\cA}Q_{\ell}^{\pi\ind{h,t}}(x_\ell,a;\what_h\ind{t})
    -Q_{\ell}^{\pi\ind{h,t}}(x_\ell,\pi\ind{h,t}(x_\ell);\what_h\ind{t})}.
  \end{align}

  \begin{lemma}
    \label{lem:outer_level}
    Let $\veps\in(0,1/2)$ and set $T=\frac{1}{\veps}$. Suppose that
    for all $h\geq{}2$ and $t\in\brk{T}$, it holds that
    $\vepsweightoff\ind{t}\leq{}
    c_1(\CpushM[\Mstar]/\abs*{\cA}t)^{1/2}\veps^{1/2}$,
    $\vepsweighton\ind{t}\leq{}c_2(\CpushM[\Mstar]/\abs*{\cA}t)^{1/2}\veps^{1/2}$,
    and $\vepsopt\ind{t}\leq{}c_3\veps^2$ for absolute constants
    $c_1,c_2,c_3>0$. Then for all $h\geq{}2$, we have that
    \begin{align*}
      \pCovM[\Mstar](p_h)   \leq{} 170H\log(\veps^{-1})\cdot{}\CpushM[\Mstar].
    \end{align*}
  \end{lemma}
  Let $\Nweight(t,\eps,\delta)$ denote the number of episodes used by
  \estimateweight to ensure that
  $(\vepsweightoff\ind{t})^2,(\vepsweighton\ind{t})^2\leq{}\eps^2/t$
  with probability at least $1-\delta$ when invoked at iteration
  $t\in\brk{T}$ for layer $h\geq{}2$, and let $\Nopt(\eps,\delta)$ be
  the number of trajectories used by \policyopt to ensures that
  $\vepsopt\ind{t}\leq\eps$ with probability at least $1-\delta$ when
  invoked at iteration $t\in\brk{T}$ for layer $h\geq2$. It follows
  from \cref{lem:outer_level} that with the parameter settings in
  \cref{alg:model_free}, we are guaranteed that with probability at
  least $1-\delta$, for all $h\in\brk{H}$
  \[
    \pCovM[\Mstar](p_h) \leq{}
    170H\log(\veps^{-1})\cdot{}\CpushM[\Mstar].
  \]
  and the total number of episodes used is at most
  \begin{align*}
    N &\leq
        HT\prn*{\Nweight(T,\epsweight,\deltaweight)
        + \Nopt(\epsopt,\deltaopt)} \\
      &\leq HT\prn*{\Nweight(T,c(\CpushM[\Mstar]/\abs*{\cA})^{1/2}\veps^{1/2},\delta/2HT) + \Nopt(c'\veps^2,\delta/2HT)}.
  \end{align*}
  for absolute constants $c,c'>0$. It remains to bound
  $\Nweight(\eps,\delta)$, for which we appeal to the following lemma,
  a corollary of \cref{thm:weight_estimator_log}.
  \begin{lemma}
    \label{lem:weight_function_mf}
    Let $h\geq{}2$ and $t\in\brk{T}$ be given. For any
    $\eps,\delta\in(0,1)$, distribution $p_{h-1}\in\Delta(\PiRNS)$ and
    $\pi\ind{h,1},\ldots,\pi\ind{h,t-1}\in\PiNS$, $\estimateweight$
    ensures that with probability at least $1-\delta$, the output
    $\what\ind{t}_h\gets\estimateweight_{h,t}(p_{h-1},\crl*{\pi\ind{h,i}}_{i<t};\eps,\delta,\cW)$
    satisfies
    \begin{align}
      \label{eq:weight_main1}
      &(\vepsweightoff\ind{t})^2 =
        \En^{\sMstar,p_{h-1}}\brk*{\prn*{\sqrt{\what_h\ind{t}(x_h\mid{}x_{h-1},a_{h-1})}-\sqrt{w_h\ind{t}(x_h\mid{}x_{h-1},a_{h-1})}}^2}
        \leq \eps^2/t,\mathand\\
      &(\vepsweighton\ind{t})^2
        =
        \frac{1}{t-1}\sum_{i<t}\En^{\sMstar,\pi\ind{h,i}\circ_{h-1}\piunif}\brk*{\prn*{\sqrt{\what_h\ind{t}(x_h\mid{}x_{h-1},a_{h-1})}-\sqrt{w_h\ind{t}(x_h\mid{}x_{h-1},a_{h-1})}}^2}
        \leq \eps^2/t,
    \end{align}
    and does so using at most
    $\Nweight(t,\eps,\delta)=80t\frac{\log(\abs{\cW}\delta^{-1})}{\eps^2}$
    episodes.
  \end{lemma}
  Appealing to this result, we conclude that the total number of
  episodes is at most
  \begin{align*}
    N \leq \bigoht\prn*{\frac{H\abs*{\cA}\log(\abs{\cW}\delta^{-1})}{\veps^3}
    +\frac{H}{\veps}\Nopt(c\veps^2,\delta/2HT)
    }.
  \end{align*}
\end{proof}

\subsubsection{Proof of \creftitle{lem:outer_level} (Outer-Level Analysis)}
\newcommand{\Pibara}{\Pibar_{\alpha}}

\paragraph{Extended MDP}
Let $\alpha\geq1$ be a parameter to the proof, whose value will be
chosen at the end as a function of $\veps$. Following
\citet{mhammedi2023representation,mhammedi2023efficient}, we define an \emph{extended MDP} $\Mbar$ by
augmenting $\cAbar=\cA\cup\crl{\term}$ and
$\cXbar=\cX\cup\crl{\term}$. $\Mbar$ has identical dynamics to
$\Mstar$, except that taking action $\term$ causes the
state to transition to $\term$ deterministically; $\term$ is a
self-looping terminal state.

We define $\PibarRNS$ as the set of all randomized non-stationary
policies from $\cXbar$ to $\cAbar$. For $\pi\in\PibarRNS$, we abbreviate
$\dbar_h^{\pi}\equiv{}d^{\sMbar,\pi}_h$,
$\Pbar_h(\cdot\mid{}\cdot)\equiv{}P_h^{\sMbar}(\cdot\mid\cdot)$,
$\bbPbar^{\pi}\brk*{\cdot}\equiv\bbP^{\sMbar,\pi}\brk{\cdot}$,
$\Enbar^{\pi}\brk*{\cdot}\equiv\En^{\sMbar,\pi}\brk{\cdot}$, and so
on. We adopt the convention that
all policies in $\PibarRNS$ select action $\term$ in the terminal state.

\paragraph{Truncated benchmark policy class}
Given the policy class $\Pi=\PiNS$, we inductively define a sequence of policy classes
$\Pibar_{\alpha,1},\ldots,\Pibar_{\alpha,H}$ based on the
extended MDP and the output $p_{1},\ldots,p_{H}$ of the
algorithm as follows.
\begin{itemize}
\item First, $\Pibar_{\alpha,0}=\Pi$.
\item Next, for $h=1,\ldots,H$, we construct $\Pibar_{\alpha,h}$ from
  $\Pibar_{\alpha,h-1}$. For each $\pi\in\Pibar_{\alpha,h-1}$, add a policy $\pi'$ to
    $\Pibar_{\alpha,h}$ defined as follows: For all $h'\neq{}h$,
    $\pi'_{h'}=\pi_{h'}$, and for layer $h$,
    \begin{align}
      \label{eq:truncated}
      \pi'_h(x)
      = \left\{
      \begin{array}{ll}
        \pi_h(x),&\quad \frac{\dbar^{\pi}_h(x)}{\dbar^{p_{h}}_h(x)} \leq \alpha,\\
        \term,&\quad \frac{\dbar^{\pi}_h(x)}{\dbar^{p_{h}}_h(x)} > \alpha.
      \end{array}
                \right.
    \end{align}
  \end{itemize}
  Finally, we adopt the shorthand
$\Pibar_{\alpha}\ldef\Pibar_{\alpha,H}$. 

\paragraph{Main analysis}
Let $h\geq{}2$ be fixed. For the remainder of the proof, we abbreviate $\pi\ind{t}\equiv\pi\ind{t,h}$ to
keep notation compact. Define
\begin{align}
  \label{eq:pcov_bar}
  \pCovbar(p) = \sup_{\pi\in\Pibara}
\Enbar\brk*{\frac{\Pbar_{h-1}(x_h\mid{}x_{h-1},a_{h-1})}{\dbar^{p}_h(x_h)+\veps\cdot{}\Pbar_{h-1}(x_h\mid{}x_{h-1},a_{h-1})}\indic_{x_h\neq\term}}
\end{align}
as a counterpart to the pushforward coverage relaxation for the
extended MDP $\Mbar$. We define three quantities,
  \begin{align*}
   &\Delbarweightoff=    \sum_{t=1}^{T}\sup_{\pi\in\Pibara}\Enbar^{\pi}\brk*{\prn*{\sqrt{\what_h\ind{t}(x_h\mid{}x_{h-1},a_{h-1})}-\sqrt{w_h\ind{t}(x_h\mid{}x_{h-1},a_{h-1})}}^2\indic_{x_h\neq\term}},\\
    &\Delbarweighton=\sum_{t=1}^{T}\Enbar^{\pi\ind{t}}\brk*{\prn*{\sqrt{\what_h\ind{t}(x_h\mid{}x_{h-1},a_{h-1})}-\sqrt{w_h\ind{t}(x_h\mid{}x_{h-1},a_{h-1})}}^2\indic_{x_h\neq\term}},\mathand\\
    &\Delbaropt=   \sum_{t=1}^{T}\sup_{\pi\in\Pibara}\Enbar^{\pi}\brk*{\what_h\ind{t}(x_h\mid{}x_{h-1},a_{h-1})\indic_{x_h\neq\term}}-\Enbar^{\pi\ind{t}}\brk*{\what_h\ind{t}(x_h\mid{}x_{h-1},a_{h-1})\indic_{x_h\neq\term}}.
  \end{align*}
  which measure the quality of the weight function estimates produced
  by \estimateweight and the
  optimization quality of the policies produced by \policyopt in the
  extended MDP.

  We now state three technical lemmas, all proven in the
  sequel. The first lemma shows that we can bound the value
  of $\pCovbar(p_h)$ in terms of the pushforward coverability
  parameter $\CpushhM[\Mstar]$, up to additive error terms depending
  on the quantities defined above.
\begin{restatable}{lemma}{potential}
  \label{lem:potential}
  Let $h\geq2$ be fixed and $\veps\in(0,1/2)$. 
  Then, with $T=\frac{1}{\veps}$, it holds that
  \begin{align}
    \label{eq:potential_bound}
    \pCovbar(p_h)
    \leq{} 72\CpushhM[\Mstar]\log(T)
    + 2\Delbarweightoff + 6\Delbarweighton + 3\Delbaropt.
  \end{align}
\end{restatable}
The next two lemmas relate the weight estimation and policy
optimization errors in $\Mbar$ to their counterparts in the true MDP
$\Mstar$, leveraging key properties of the truncated policy class $\Pibara$.
\begin{restatable}{lemma}{werror}
  \label{lem:w_error}
  The following bounds hold for all $h\geq{}2$, as long as
  $\nrm*{w}_{\infty}\leq{}1$ for all $w\in\cW_h$:
  \begin{align}
    \label{eq:del_weight_off}
    &\Delbarweightoff \leq
    \alpha\abs*{\cA}\sum_{t=1}^{T}(\vepsweightoff\ind{t})^2,\mathand\\
    &\Delbarweighton \leq \sqrt{
        8\abs*{\cA}\CpushM[\Mstar]\log(T)
  \sum_{t=1}^{T}(t-1)(\vepsweighton\ind{t})^2
}
  + 4\CpushM[\Mstar].
  \end{align}
\end{restatable}
\begin{restatable}{lemma}{opterror}
  \label{lem:opt_error}
  The following bound holds for all $h\geq{}2$:
  \begin{align}
    \Delbaropt
    \leq \alpha\sum_{t=1}^{T}\vepsopt\ind{t}.
  \end{align}
\end{restatable}
Appealing to \cref{lem:potential,lem:w_error,lem:opt_error}, we conclude that as long as $\vepsweightoff\ind{t}\leq{}
c_1(\alpha\abs*{\cA}t/\CpushM[\Mstar])^{-1/2}$,
$\vepsweighton\ind{t}\leq{}c_2(\abs*{\cA}Tt/\CpushM[\Mstar])^{-1/2}$,
and $\vepsopt\ind{t}\leq{}c_3(\alpha{}T)^{-1}$ for all $h\geq{}2$,
$t\in\brk{T}$, where $c_1,c_2,c_3>0$ are absolute constants, we are
guaranteed that for all $h$, 
\begin{align}
  \label{eq:push_bar_bound}
    \pCovbar(p_h) 
   \leq{} 85\CpushM[\Mstar]\log(T).
\end{align}

It remains to translate this back to a bound on the \mainobj objective
for the true MDP $\Mstar$. To do so, we start with the following
technical lemma, also proven in the sequel.
\begin{restatable}{lemma}{extendedtranslation}
  \label{lem:extended_translation}
  Consider any reward function $\crl*{r_h}_{h\in\brk{H}}$ with
  $r_h:\cXbar\times\cAbar\to\brk*{0,1}$ such that
  $\sum_{h=1}^{H}r_h(x_h,a_h)\in\brk*{0,1}$ for all sequences $(x_1,a_1),\ldots,(x_H,a_H)$,
  and such that $r_h(\term,a)=0$ and $r_h(x,\term)=0$. It holds that
  \begin{align*}
    \sup_{\pi\in\Pi}\En^{\pi}\brk*{\sum_{h=1}^{H}r_h}
    -
    \sup_{\pi\in\Pibar_{\alpha}}\Enbar^{\pi}\brk*{\sum_{h=1}^{H}r_h}
    \leq
    \sum_{h=1}^{H}\sup_{\pi\in\Pibar_{\alpha}}\bbPbar^{\pi}\brk*{
          \frac{\dbar^{\pi}_h(x_h)}{\dbar^{p_{h}}_h(x_h)} > \alpha,x_h\neq\term
          }.
  \end{align*}
\end{restatable}
Let $h\geq{}2$ be fixed and define a reward function
$\crl{r_{\ell}}_{\ell\leq{}h-1}$  with
$r_\ell:\cXbar\times\cAbar\to\brk{0,\veps^{-1}}$ via
\begin{align*}
  r_{h-1}(x,a) = \En\brk*{\frac{P_{h-1}(x_h\mid{}x_{h-1},a_{h-1})}{d^{p_h}_h(x_h)+\veps\cdot{}P_{h-1}(x_h\mid{}x_{h-1},a_{h-1})}\mid{}x_{h-1}=x,a_{h-1}=a}\indic_{x,a\neq\term}
\end{align*}
and $r_{\ell}=0$ for $\ell<h-1$. Using
\cref{lem:extended_translation}, we have that
\begin{align}
  \label{eq:trn1}
  \pCovM[\Mstar](p_h)
  = \sup_{\pi\in\Pi}\brk*{\sum_{\ell=1}^{h-1}r_{\ell}}
  \leq{}
  \sup_{\pi\in\Pibar_{\alpha}}\Enbar^{\pi}\brk*{\sum_{\ell=1}^{h-1}r_\ell}
+
  \frac{1}{\veps}\sum_{\ell=1}^{h-1}\sup_{\pi\in\Pibar_{\alpha}}\bbPbar^{\pi}\brk*{
          \frac{\dbar^{\pi}_\ell(x_\ell)}{\dbar^{p_{\ell}}_\ell(x_\ell)} > \alpha,x_\ell\neq\term
          }.
\end{align}
To bound the \rhs, we first note that
\begin{align}
  \sup_{\pi\in\Pibar_{\alpha}}\Enbar^{\pi}\brk*{\sum_{\ell=1}^{h-1}r_\ell}
  &=
\sup_{\pi\in\Pibar_{\alpha}}\Enbar^{\pi}\brk*{\En\brk*{\frac{P_{h-1}(x_h\mid{}x_{h-1},a_{h-1})}{d^{p_h}_h(x_h)+\veps\cdot{}P_{h-1}(x_h\mid{}x_{h-1},a_{h-1})}\mid{}x_{h-1},a_{h-1}}\indic_{x_{h-1},a_{h-1}\neq\term}}\notag\\
  &=
\sup_{\pi\in\Pibar_{\alpha}}\Enbar^{\pi}\brk*{\Enbar\brk*{\frac{\Pbar_{h-1}(x_h\mid{}x_{h-1},a_{h-1})}{\dbar^{p_h}_h(x_h)+\veps\cdot{}\Pbar_{h-1}(x_h\mid{}x_{h-1},a_{h-1})}\mid{}x_{h-1},a_{h-1}}\indic_{x_{h-1},a_{h-1}\neq\term}}\notag\\
    &=
      \sup_{\pi\in\Pibar_{\alpha}}\Enbar^{\pi}\brk*{\frac{\Pbar_{h-1}(x_h\mid{}x_{h-1},a_{h-1})}{\dbar^{p_h}_h(x_h)+\veps\cdot{}\Pbar_{h-1}(x_h\mid{}x_{h-1},a_{h-1})}\indic_{x_{h-1},a_{h-1}\neq\term}}\notag\\
      &=
\sup_{\pi\in\Pibar_{\alpha}}\Enbar^{\pi}\brk*{\frac{\Pbar_{h-1}(x_h\mid{}x_{h-1},a_{h-1})}{\dbar^{p_h}_h(x_h)+\veps\cdot{}\Pbar_{h-1}(x_h\mid{}x_{h-1},a_{h-1})}\indic_{x_{h}\neq\term}}
        = \pCovbar(p_h).\label{eq:trn2}
\end{align}
Here, the second equality uses that i) $\Mstar$ and $\Mbar$ have
identical transition dynamics whenever $x_{h-1},a_{h-1}\term$ and ii)
policies in the support of $p_h$ never take the terminal
action. Meanwhile, the second-to-last inequality uses that
$x_h\neq\term$ if and only if $x_{h-1}\neq\term$ and
$a_{h-1}\neq\term$ in $\Mbar$. To bound the second term on the \rhs of
\cref{eq:trn1}, we use a
variant of \cref{prop:pushforward_relaxation}.
\begin{restatable}{lemma}{pushforwardextended}
  \label{lem:pushforward_extended}
  For all $\alpha>0$ and $\ell\geq{}1$, it holds that
  \begin{align}
      \sup_{\pi\in\Pibar_{\alpha}}\bbPbar^{\pi}\brk*{
    \frac{\dbar^{\pi}_\ell(x_\ell)}{\dbar^{p_{\ell}}_\ell(x_\ell)} > \alpha,x_\ell\neq\term
    }
    \leq{} \frac{2}{\alpha}\sup_{\pi\in\Pibara}\Enbar^{\pi}\brk*{
    \frac{\Pbar_{\ell-1}(x_\ell\mid{}x_{\ell-1},a_{\ell-1})}{\dbar_\ell^{p_\ell}(x_\ell)+\alpha^{-1}\cdot{}\Pbar_{\ell-1}(x_\ell\mid{}x_{\ell-1},a_{\ell-1}) }
    \indic_{x_\ell\neq\term}}.\label{eq:extended_relax}
  \end{align}
\end{restatable}
We set $\alpha=\veps^{-1}$, so that combining \cref{eq:trn1} with
\cref{eq:trn2} and \cref{lem:pushforward_extended} yields
\begin{align}
  \label{eq:extended_final}
  \pCovM[\Mstar](p_h)
  \leq{} \pCovbar(p_h)
  +2\sum_{\ell=1}^{h-1}\pCovbarell(p_\ell).
\end{align}
Consequently, for the choice $T=\veps^{-1}$ and $\alpha=\veps^{-1}$,
\cref{eq:push_bar_bound} and \cref{eq:extended_final} imply that for
all $h\geq{}2$,
\begin{align*}
  \pCovM[\Mstar](p_h)   \leq{} 170H\log(\veps^{-1})\cdot{}\CpushM[\Mstar].
\end{align*}
The final approximation requirements for these choices are $\vepsweightoff\ind{t}\leq{}
c_1(\CpushM[\Mstar]/\abs*{\cA}t)^{1/2}\veps^{1/2}$,
$\vepsweighton\ind{t}\leq{}c_2(\CpushM[\Mstar]/\abs*{\cA}t)^{1/2}\veps^{1/2}$,
and $\vepsopt\ind{t}\leq{}c_3\veps^2$ for absolute constants $c_1,c_2,c_3>0$.
\qed

\subsubsection{Proofs for Supporting Lemmas for \creftitle{lem:outer_level}}

\potential*
\begin{proof}[\pfref{lem:potential}]
  This is a slightly modified variant of the proof of \cref{thm:pushforward_optimization}.
  Let $\wt{d}_h^{t}=\sum_{i<t}d_h^{\pi\ind{i}}$ and
  $\dcheck_h\ind{t}=\sum_{i<t}\dbar_h^{\pi\ind{i}}$. Observe
  that for $T=\frac{1}{\veps}$, we have
  \begin{align*}
    \pCovbar(p_h) &= \sup_{\pi\in\Pibara}
\Enbar\brk*{\frac{\Pbar_{h-1}(x_h\mid{}x_{h-1},a_{h-1})}{\dbar^{p_h}_h(x_h)+\veps\cdot{}\Pbar_{h-1}(x_h\mid{}x_{h-1},a_{h-1})}\indic_{x_h\neq\term}}\\
    &= 
T\cdot\sup_{\pi\in\Pibara}\Enbar^{\pi}\brk*{\frac{\Pbar_{h-1}(x_h\mid{}x_{h-1},a_{h-1})}{\dcheck\ind{T+1}_h(x_h)+\Pbar_{h-1}(x_h\mid{}x_{h-1},a_{h-1})}\indic_{x_h\neq\term}},
  \end{align*}
  and hence it suffices to bound the quantity on the right-hand
  side. Next, note that for all $t\in\brk{T}$, we have that
  \begin{align*}
\sup_{\pi\in\Pibara}\Enbar^{\pi}\brk*{\frac{\Pbar_{h-1}(x_h\mid{}x_{h-1},a_{h-1})}{\dcheck\ind{t+1}_h(x_h)+\Pbar_{h-1}(x_h\mid{}x_{h-1},a_{h-1})}\indic_{x_h\neq\term}}
    \leq \sup_{\pi\in\Pibara}\Enbar^{\pi}\brk*{\frac{\Pbar_{h-1}(x_h\mid{}x_{h-1},a_{h-1})}{\dcheck\ind{t+1}_h(x_h)+\Pbar_{h-1}(x_h\mid{}x_{h-1},a_{h-1})}\indic_{x_h\neq\term}}
  \end{align*}
and consequently
\begin{align*}
T\cdot\sup_{\pi\in\Pibara}\Enbar^{\pi}\brk*{\frac{\Pbar_{h-1}(x_h\mid{}x_{h-1},a_{h-1})}{\dcheck\ind{T+1}_h(x_h)+\Pbar_{h-1}(x_h\mid{}x_{h-1},a_{h-1})}\indic_{x_h\neq\term}}
  &\leq
\sum_{t=1}^{T}\sup_{\pi\in\Pibara}\Enbar^{\pi}\brk*{\frac{\Pbar_{h-1}(x_h\mid{}x_{h-1},a_{h-1})}{\dcheck\ind{t}_h(x_h)+\Pbar_{h-1}(x_h\mid{}x_{h-1},a_{h-1})}\indic_{x_h\neq\term}}.
\end{align*}
Now, note that if $x_h\neq\term$, the dynamics of $\Mbar$ imply that
we must have $x_{h-1},a_{h-1}\neq\term$ as well. In this case, we have
\begin{align}
  \label{eq:wbar_equality}
  \frac{\Pbar_{h-1}(x_h\mid{}x_{h-1},a_{h-1})}{\dcheck\ind{t}_h(x_h)+\Pbar_{h-1}(x_h\mid{}x_{h-1},a_{h-1})}
  = w_h\ind{t}(x_h\mid{}x_{h-1},a_{h-1}),
\end{align}
since
$\Pbar_{h-1}(\cdot\mid{}x_{h-1},a_{h-1})=P_{h-1}(\cdot\mid{}x_{h-1},a_{h-1})$
with $x_{h-1},a_{h-1}\neq\term$, and since
$d_h^{\pi\ind{i}}(x_h)=\dbar_h^{\pi\ind{i}}(x_h)$ when $x_h\neq\term$
(as the policies $\pi\ind{1},\ldots,\pi\ind{T}$ never take the
terminal action). As a result, using \cref{lem:mp_weight}, we have
that
\begin{align*}
&\sum_{t=1}^{T}\sup_{\pi\in\Pibara}\Enbar^{\pi}\brk*{\frac{\Pbar_{h-1}(x_h\mid{}x_{h-1},a_{h-1})}{\dcheck\ind{t}_h(x_h)+\Pbar_{h-1}(x_h\mid{}x_{h-1},a_{h-1})}\indic_{x_h\neq\term}}\\
  &=\sum_{t=1}^{T}\sup_{\pi\in\Pibara}\Enbar^{\pi}\brk*{w_h\ind{t}(x_h\mid{}x_{h-1},a_{h-1})\indic_{x_h\neq\term}}\\
  &\leq3\sum_{t=1}^{T}\sup_{\pi\in\Pibara}\Enbar^{\pi}\brk*{\what_h\ind{t}(x_h\mid{}x_{h-1},a_{h-1})\indic_{x_h\neq\term}} \\
    &\qquad + 2 \underbrace{\sum_{t=1}^{T}\sup_{\pi\in\Pibara}\Enbar^{\pi}\brk*{\prn*{\sqrt{\what_h\ind{t}(x_h\mid{}x_{h-1},a_{h-1})}-\sqrt{w_h\ind{t}(x_h\mid{}x_{h-1},a_{h-1})}}^2\indic_{x_h\neq\term}}}_{=\Delbarweightoff}.
\end{align*}
Next, we can bound
\begin{align*}
  \sum_{t=1}^{T}\sup_{\pi\in\Pibara}\Enbar^{\pi}\brk*{\what_h\ind{t}(x_h\mid{}x_{h-1},a_{h-1})\indic_{x_h\neq\term}}
  \leq{}
  \sum_{t=1}^{T}\Enbar^{\pi\ind{t}}\brk*{\what_h\ind{t}(x_h\mid{}x_{h-1},a_{h-1})\indic_{x_h\neq\term}}
  + \Delbaropt
\end{align*}
by definition. Applying \cref{lem:mp_weight} once more, we have that
\begin{align*}
  &\sum_{t=1}^{T}\Enbar^{\pi\ind{t}}\brk*{\what_h\ind{t}(x_h\mid{}x_{h-1},a_{h-1})\indic_{x_h\neq\term}}\\
  &\leq{}
  3\sum_{t=1}^{T}\Enbar^{\pi\ind{t}}\brk*{w_h\ind{t}(x_h\mid{}x_{h-1},a_{h-1})\indic_{x_h\neq\term}}
  + 2 \underbrace{\sum_{t=1}^{T}\Enbar^{\pi\ind{t}}\brk*{\prn*{\sqrt{\what_h\ind{t}(x_h\mid{}x_{h-1},a_{h-1})}-\sqrt{w_h\ind{t}(x_h\mid{}x_{h-1},a_{h-1})}}^2\indic_{x_h\neq\term}}}_{=\Delbarweighton}.
\end{align*}
Finally, note that since $\pi\ind{t}\in\Pi$ never select the terminal
action (and in particular never reach the terminal state), we have
\begin{align*}
  \sum_{t=1}^{T}\Enbar^{\pi\ind{t}}\brk*{w_h\ind{t}(x_h\mid{}x_{h-1},a_{h-1})\indic_{x_h\neq\term}}
  &= \sum_{t=1}^{T}\Enbar^{\pi\ind{t}}\brk*{
    \frac{P_{h-1}(x_h\mid{}x_{h-1},a_{h-1})}{\dtil\ind{t}_h(x_h)+P_{h-1}(x_h\mid{}x_{h-1},a_{h-1})}\indic_{x_h\neq\term}}\\
  &= \sum_{t=1}^{T}\En^{\pi\ind{t}}\brk*{
    \frac{P_{h-1}(x_h\mid{}x_{h-1},a_{h-1})}{\dtil\ind{t}_h(x_h)+P_{h-1}(x_h\mid{}x_{h-1},a_{h-1})}}\\
  &\leq 4\CpushhM[\Mstar]\log(2T),
\end{align*}
where the final bound follows from is
\cref{eq:pushforward_potential}. To simplify the constants, we note
that $\log(2T)\leq{}2\log(T)$ whenever $\veps\leq{}1/2$.

\end{proof}

\werror*
\begin{proof}[\pfref{lem:w_error}]
  We first bound the quantity
  \[
\Delbarweightoff=    \sum_{t=1}^{T}\sup_{\pi\in\Pibara}\Enbar^{\pi}\brk*{\prn*{\sqrt{\what_h\ind{t}(x_h\mid{}x_{h-1},a_{h-1})}-\sqrt{w_h\ind{t}(x_h\mid{}x_{h-1},a_{h-1})}}^2\indic_{x_h\neq\term}}.
\]
Observe that if $x_h\neq\term$, then the dynamics of the MDP $\Mbar$
imply that $x_{h-1}\neq\term, a_{h-1}\neq\term$. Consider an arbitrary
policy $\pi\in\Pibara$. Since $\pi(x_{h-1})\neq\term$, the dynamics in
\cref{eq:truncated} imply that
$\dbar^{\pi}_{h-1}(x_{h-1})/\dbar^{p_{h-1}}_{h-1}(x_{h-1}) \leq
\alpha$. Consequently, for any $t\in\brk{T}$, we can bound
\begin{align*} &\Enbar^{\pi}\brk*{\prn*{\sqrt{\what_h\ind{t}(x_h\mid{}x_{h-1},a_{h-1})}-\sqrt{w_h\ind{t}(x_h\mid{}x_{h-1},a_{h-1})}}^2\indic_{x_h\neq\term}}\\
  &= \sum_{x\in\cX: \pi(x)\neq\term,a\in\cA}\dbar_{h-1}^{\pi}(x,a)
    \Enbar\brk*{\prn*{\sqrt{\what_h\ind{t}(x_h\mid{}x,a)}-\sqrt{w_h\ind{t}(x_h\mid{}x,a)}}^2\mid{}x_{h-1}=x,a_{h-1}=a}\\
               &\leq \sum_{a\in\cA}\sum_{x\in\cX:
                 \pi(x)\neq\term}\dbar_{h-1}^{\pi}(x)
                 \Enbar\brk*{\prn*{\sqrt{\what_h\ind{t}(x_h\mid{}x,a)}-\sqrt{w_h\ind{t}(x_h\mid{}x,a)}}^2\mid{}x_{h-1}=x,a_{h-1}=a}\\
               &\leq \alpha\sum_{a\in\cA}\sum_{x\in\cX}\dbar_{h-1}^{p_{h-1}}(x)
\Enbar\brk*{\prn*{\sqrt{\what_h\ind{t}(x_h\mid{}x,a)}-\sqrt{w_h\ind{t}(x_h\mid{}x,a)}}^2\mid{}x_{h-1}=x,a_{h-1}=a}\\
                 &= \alpha\abs*{\cA}\cdot
                   \Enbar^{p_{h-1}\circ_{h-1}\piunif}\brk*{\prn*{\sqrt{\what_h\ind{t}(x_h\mid{}x_{h-1},a_{h-1})}-\sqrt{w_h\ind{t}(x_h\mid{}x_{h-1},a_{h-1})}}^2}\\
                   &= \alpha\abs*{\cA}\cdot
                     \Enbar^{p_{h-1}}\brk*{\prn*{\sqrt{\what_h\ind{t}(x_h\mid{}x_{h-1},a_{h-1})}-\sqrt{w_h\ind{t}(x_h\mid{}x_{h-1},a_{h-1})}}^2}\\
  &= \alpha\abs*{\cA}\cdot
    \En^{p_{h-1}}\brk*{\prn*{\sqrt{\what_h\ind{t}(x_h\mid{}x_{h-1},a_{h-1})}-\sqrt{w_h\ind{t}(x_h\mid{}x_{h-1},a_{h-1})}}^2}
=      \alpha\abs*{\cA}\cdot(\vepsweightoff\ind{t})^2,
\end{align*}
where the second-to-last equality uses that
$p_{h-1}\circ_{h-1}\piunif=p_{h-1}$ by construction, and the final
equality uses that policies in the support of $p_{h-1}$ never take the
optimal action. Summing over $t$ completes the proof.

We now bound the quantity $\Delbarweighton$. Note that since the
policy $\pi\ind{t}$ never chooses the terminal action $\term$, we can write
\begin{align*}
&\Delbarweighton=\sum_{t=1}^{T}\En^{\pi\ind{t}}\brk*{\prn*{\sqrt{\what_h\ind{t}(x_h\mid{}x_{h-1},a_{h-1})}-\sqrt{w_h\ind{t}(x_h\mid{}x_{h-1},a_{h-1})}}^2}.
\end{align*}
Observe that as a consequence of pushforward coverability, there
exists a distribution $\mu\in\Delta(\cX)$ such that
$d_{h-1}^{\pi}(x)\leq{}\CpushM[\Mstar]\mu(x)$ for all $x\in\cX$,
$\pi\in\PiRNS$. Hence, by
applying \cref{lem:linf_potential} with $g\ind{t}(x) = \En\brk*{\prn*{\sqrt{\what_h\ind{t}(x_h\mid{}x,\pi\ind{t}(x))}-\sqrt{w_h\ind{t}(x_h\mid{}x,\pi\ind{t}(x))}}^2\mid{}x_{h-1}=x,a_{h-1}=\pi\ind{t}(x)}$,
which has $g\ind{t}\in\brk{0,2}$ whenever
$\nrm*{\what\ind{t}}_{\infty},\nrm*{w\ind{t}}_{\infty}\leq{}1$, we
have
\begin{align*}
  &\Delbarweighton\\
  &\leq
  \sqrt{
  2\CpushM[\Mstar]\log(T)
  \sum_{\substack{t\in[T] \\ i<t}}
  \En^{\pi\ind{i}}\brk[\bigg]{
\prn[\Big]{\En\brk[\Big]{\prn[\Big]{\sqrt{\what_h\ind{t}(x_h\mid{}x_{h-1},a_{h-1})}-\sqrt{w_h\ind{t}(x_h\mid{}x_{h-1},a_{h-1})}}^2\mid{}x_{h-1},a_{h-1}=\pi\ind{t}(x_{h-1})}}^2
  }
  } \\
    &\qquad + 4\CpushM[\Mstar]\\
    &\leq
  \sqrt{
  2\CpushM[\Mstar]\log(T)
  \sum_{t=1}^{T}\sum_{i<t}
  \En^{\pi\ind{i}\circ_{h-1}\pi\ind{t}}\brk*{
\prn[\Big]{\sqrt{\what_h\ind{t}(x_h\mid{}x_{h-1},a_{h-1})}-\sqrt{w_h\ind{t}(x_h\mid{}x_{h-1},a_{h-1})}}^4
      }}
      + 4\CpushM[\Mstar]\\
      &\leq
  \sqrt{
  2\CpushM[\Mstar]\abs*{\cA}\log(T)
  \sum_{t=1}^{T}\sum_{i<t}
  \En^{\pi\ind{i}\circ_{h-1}\piunif}\brk*{
\prn[\Big]{\sqrt{\what_h\ind{t}(x_h\mid{}x_{h-1},a_{h-1})}-\sqrt{w_h\ind{t}(x_h\mid{}x_{h-1},a_{h-1})}}^4
      }}
      + 4\CpushM[\Mstar]\\
      &\leq
  \sqrt{
        8\CpushM[\Mstar]\abs*{\cA}\log(T)
  \sum_{t=1}^{T}\sum_{i<t}
  \En^{\pi\ind{i}\circ_{h-1}\piunif}\brk*{
\prn[\Big]{\sqrt{\what_h\ind{t}(x_h\mid{}x_{h-1},a_{h-1})}-\sqrt{w_h\ind{t}(x_h\mid{}x_{h-1},a_{h-1})}}^2
      }}
  + 4\CpushM[\Mstar].
\end{align*}
This proves the result.

\end{proof}

\opterror*
\begin{proof}[\pfref{lem:opt_error}]
Let $t\in\brk{T}$ be fixed, and recall that we can write
\begin{align*} &\sup_{\pi\in\Pibara}\Enbar^{\pi}\brk*{\what_h\ind{t}(x_h\mid{}x_{h-1},a_{h-1})\indic_{x_h\neq\term}}-\Enbar^{\pi\ind{t}}\brk*{\what_h\ind{t}(x_h\mid{}x_{h-1},a_{h-1})\indic_{x_h\neq\term}}\\
  &= \sup_{\pi\in\Pibara}\Enbar^{\pi}\brk*{\what_h\ind{t}(x_h\mid{}x_{h-1},a_{h-1})\indic_{x_{h-1},a_{h-1}\neq\term}}-\Enbar^{\pi\ind{t}}\brk*{\what_h\ind{t}(x_h\mid{}x_{h-1},a_{h-1})\indic_{x_{h-1},a_{h-1}\neq\term}},
\end{align*}
since $x_h\neq\term$ if and only if $x_{h-1},a_{h-1}\neq\term$. If we
define $\Qbar\ind{\pi}$ as the state-action value function for policy
$\pi$ in $\Mbar$ under the reward function given by
$r_{h-1}(x,a)=\Enbar\brk*{\what\ind{t}_h(x_h\mid{}x,a)\mid{}x_{h-1}=x,a_{h-1}=a}\indic_{x,a\neq\term}$
and $r_{h'}(x,a)=0$ for $h'<h-1$, the performance difference lemma
\citep{kakade2003sample} 
implies that
\begin{align*}
&\sup_{\pi\in\Pibara}\Enbar^{\pi}\brk*{\what_h\ind{t}(x_h\mid{}x_{h-1},a_{h-1})\indic_{x_{h-1},a_{h-1}\neq\term}}-\Enbar^{\pi\ind{t}}\brk*{\what_h\ind{t}(x_h\mid{}x_{h-1},a_{h-1})\indic_{x_{h-1},a_{h-1}\neq\term}}\\
  &=   \sup_{\pi\in\Pibara}\sum_{\ell=1}^{h-1}\Enbar^{\pi}\brk*{
    \Qbar_{\ell}^{\pi\ind{t}}(x_{\ell},\pi(x_\ell))
    -\Qbar_\ell^{\pi\ind{t}}(x_\ell,\pi\ind{t}(x_\ell))}.
\end{align*}
Consider an arbitrary policy $\pi\in\Pibara$, and fix
$\ell\in\brk{h-1}$, and write
\begin{align*}
  \Enbar^{\pi}\brk*{
    \Qbar_{\ell}^{\pi\ind{t}}(x_{\ell},\pi(x_\ell))
                 -\Qbar_\ell^{\pi\ind{t}}(x_\ell,\pi\ind{t}(x_\ell))}
  = \sum_{x\in\cX}\dbar_{\ell}^{\pi}(x)
    \prn*{\Qbar_{\ell}^{\pi\ind{t}}(x,\pi(x))
    -\Qbar_\ell^{\pi\ind{t}}(x,\pi\ind{t}(x))}
\end{align*}
For any $x\in\cX$, since $\pi\in\Pibara$, if
$\dbar^{\pi}_\ell(x)/\dbar_\ell^{p_\ell}(x)>\alpha$, then $\pi(x)=\term$,
which implies that
\begin{align*}
  \Qbar_{\ell}^{\pi\ind{t}}(x,\pi(x))
  -\Qbar_\ell^{\pi\ind{t}}(x,\pi\ind{t}(x))\leq -\Qbar_\ell^{\pi\ind{t}}(x,\pi\ind{t}(x))\leq0
\end{align*}
since the reward is non-negative, and since we can receive non-zero
reward only if $\pi(x_\ell)\neq\term$ for all $\ell\leq{}h-1$. Hence,
we can bound
\begin{align*}
  \sum_{x\in\cX}\dbar_{\ell}^{\pi}(x)
    \prn*{\Qbar_{\ell}^{\pi\ind{t}}(x,\pi(x))
  -\Qbar_\ell^{\pi\ind{t}}(x,\pi\ind{t}(x))}
&\leq  \sum_{x\in\cX}\dbar_{\ell}^{\pi}(x)
    \prn*{\Qbar_{\ell}^{\pi\ind{t}}(x,\pi(x))
                                               -\Qbar_\ell^{\pi\ind{t}}(x,\pi\ind{t}(x))}\indic\crl*{\dbar^{\pi}_\ell(x)/\dbar_\ell^{p_\ell}(x)\leq\alpha}\\
  &\leq  \sum_{x\in\cX}\dbar_{\ell}^{\pi}(x)
    \prn*{\max_{a\in\cA}\Qbar_{\ell}^{\pi\ind{t}}(x,a)
    -\Qbar_\ell^{\pi\ind{t}}(x,\pi\ind{t}(x))}\indic\crl*{\dbar^{\pi}_\ell(x)/\dbar_\ell^{p_\ell}(x)\leq\alpha}\\
  &\leq  \alpha\sum_{x\in\cX}\dbar_{\ell}^{p_\ell}(x)
    \prn*{\max_{a\in\cA}\Qbar_{\ell}^{\pi\ind{t}}(x,a)
    -\Qbar_\ell^{\pi\ind{t}}(x,\pi\ind{t}(x))}\indic\crl*{\dbar^{\pi}_\ell(x)/\dbar_\ell^{p_\ell}(x)\leq\alpha}\\
    &\leq  \alpha\sum_{x\in\cX}\dbar_{\ell}^{p_\ell}(x)
    \prn*{\max_{a\in\cA}\Qbar_{\ell}^{\pi\ind{t}}(x,a)
      -\Qbar_\ell^{\pi\ind{t}}(x,\pi\ind{t}(x))}\\
  &=  \alpha\Enbar^{p_\ell}
    \brk*{\max_{a\in\cA}\Qbar_{\ell}^{\pi\ind{t}}(x,a)
      -\Qbar_\ell^{\pi\ind{t}}(x,\pi\ind{t}(x))}.
\end{align*}
Above, the second inequality uses that i)
$\Qbar_\ell^{\pi\ind{t}}(x,a)=0$ for all $a\in\cAbar$ if $x=\term$ and ii) $\pi(x)\in\cA$ if
  $x\neq\term$ but
  $\dbar^{\pi}_\ell(x)/\dbar_\ell^{p_\ell}(x)\leq\alpha$. Finally, we note that
\begin{align*}
  \Enbar^{p_\ell}
    \brk*{\max_{a\in\cA}\Qbar_{\ell}^{\pi\ind{t}}(x_\ell,a)
  -\Qbar_\ell^{\pi\ind{t}}(x_\ell,\pi\ind{t}(x_\ell))}
  = \En^{p_\ell}
    \brk*{\max_{a\in\cA}Q_{\ell}^{\pi\ind{t}}(x_\ell,a;\what_h\ind{t})
  -\Qbar_\ell^{\pi\ind{t}}(x_\ell,\pi\ind{t}(x_\ell);\what_h\ind{t})},
\end{align*}
since i) policies in the support of $p_\ell$ never take the terminal
action, and ii) $\Qbar_\ell^{\pi\ind{t}}(x,a)
= Q_\ell^{\pi\ind{t}}(x,a; \what_h\ind{t})$ whenever $x,a\neq\term$,
since $\pi\ind{t}$ never takes the terminal action.
\end{proof}

\extendedtranslation*

\begin{proof}[\pfref{lem:extended_translation}]
  Since policies $\pi\in\Pi$ never take the terminal action, we can write
  \begin{align*}
&    \sup_{\pi\in\Pi}\En^{\pi}\brk*{\sum_{h=1}^{H}r_h}
    -
                   \sup_{\pi\in\Pibar_{\alpha}}\Enbar^{\pi}\brk*{\sum_{h=1}^{H}r_h}\\
    &=    \sup_{\pi\in\Pibar_{\alpha,0}}\Enbar^{\pi}\brk*{\sum_{h=1}^{H}r_h}
    -
      \sup_{\pi\in\Pibar_{\alpha,H}}\Enbar^{\pi}\brk*{\sum_{h=1}^{H}r_h}\\
          &=   \sum_{\ell=1}^{H}\prn*{\sup_{\pi\in\Pibar_{\alpha,\ell-1}}\Enbar^{\pi}\brk*{\sum_{h=1}^{H}r_h}
    -
    \sup_{\pi\in\Pibar_{\alpha,\ell}}\Enbar^{\pi}\brk*{\sum_{h=1}^{H}r_h}}
  \end{align*}
  by telescoping.
  Fix $\ell\in\brk*{H}$. Let $\pi\in\Pibar_{\alpha,\ell-1}$ be
  arbitrary, and let $\pi'\in\Pibar_{\alpha,\ell}$ denote the
  policy with $\pi'_{h}=\pi_{h}$ for $h\neq\ell$, and with
  \begin{align}
    \label{eq:trunc_reminder}
    \pi'_\ell(x)
      = \left\{
      \begin{array}{ll}
        \pi_\ell(x),&\quad \frac{\dbar^{\pi}_\ell(x)}{\dbar^{p_{\ell}}_\ell(x)} \leq \alpha,\\
        \term,&\quad \frac{\dbar^{\pi}_\ell(x)}{\dbar^{p_{\ell}}_\ell(x)} > \alpha.
      \end{array}
      \right.
      \end{align}
      Let $\Qbar^{\pi}_h(x,a)\ldef{}\Enbar^{\pi}\brk*{\sum_{h'=h}^{H}r_{h'}\mid{}x_h=x,a_h=a}$ denote the Q-function for $\crl*{r_h}$ in
      $\Mbar$. Then by the performance difference lemma
      \citep{kakade2003sample}, we have
      \begin{align*}
        &\Enbar^{\pi}\brk*{\sum_{h=1}^{H}r_h}
        -         \Enbar^{\pi'}\brk*{\sum_{h=1}^{H}r_h} \\
        &= \Enbar^{\pi}\brk*{\sum_{h=1}^{H}\Qbar^{\pi'}_h(x_h,\pi(x_h))
          - \Qbar^{\pi'}_h(x_h,\pi'(x_h))
          }\\
        &= \Enbar^{\pi}\brk*{\Qbar^{\pi'}_\ell(x_\ell,\pi(x_\ell))
          - \Qbar^{\pi'}_\ell(x_\ell,\pi'(x_\ell))
          },
      \end{align*}
      since the policies agree unless $h=\ell$. Since
      $\Qbar^{\pi'}_{\ell}\in\brk*{0,1}$ by the normalization assumption on the
      rewards and $\Qbar^{\pi'}_{\ell}(\term,a)=0$ for all $a\in\cAbar$, we have
      \begin{align*}
        \Enbar^{\pi}\brk*{\Qbar^{\pi'}_\ell(x_\ell,\pi(x_\ell))
          - \Qbar^{\pi'}_\ell(x_\ell,\pi'(x_\ell))
        }
        &\leq \bbPbar^{\pi}\brk*{\pi'(x_\ell)\neq\pi(x_\ell),x_\ell\neq\term}
           \\
        &= \bbPbar^{\pi}\brk*{
          \frac{\dbar^{\pi}_\ell(x_\ell)}{\dbar^{p_{\ell}}_\ell(x_\ell)} > \alpha,x_\ell\neq\term
          },
      \end{align*}
      where the final equality follows from \cref{eq:trunc_reminder}.
      Since this result holds uniformly for all
      $\pi\in\Pibar_{\alpha,\ell-1}$, we conclude that
      \begin{align*}
        &\sup_{\pi\in\Pibar_{\alpha,\ell-1}}\Enbar^{\pi}\brk*{\sum_{h=1}^{H}r_h}
        -
          \sup_{\pi\in\Pibar_{\alpha,\ell}}\Enbar^{\pi}\brk*{\sum_{h=1}^{H}r_h}\\
                &\leq \sup_{\pi\in\Pibar_{\alpha,\ell-1}}\bbPbar^{\pi}\brk*{
          \frac{\dbar^{\pi}_\ell(x_\ell)}{\dbar^{p_{\ell}}_\ell(x_\ell)} > \alpha,x_\ell\neq\term
                  }\\
        &= \sup_{\pi\in\Pibar_{\alpha}}\bbPbar^{\pi}\brk*{
          \frac{\dbar^{\pi}_\ell(x_\ell)}{\dbar^{p_{\ell}}_\ell(x_\ell)} > \alpha,x_\ell\neq\term
          },
      \end{align*}
      where the final equality uses that every policy in
      $\Pibar_{\alpha,\ell-1}$ has a counterpart in
      $\Pibar_{\alpha}=\Pibar_{\alpha,H}$ that takes identical actions
      for layers $1,\ldots,\ell-1$.
  
    \end{proof}

\pushforwardextended*
    
    \begin{proof}[\pfref{lem:pushforward_extended}]
  We follow a proof similar to \cref{prop:pushforward_relaxation}. For any
$\ell$, we can write
\begin{align*}
  &\sup_{\pi\in\Pibar_{\alpha}}\bbPbar^{\pi}\brk*{
          \frac{\dbar^{\pi}_\ell(x_\ell)}{\dbar^{p_{\ell}}_\ell(x_\ell)} > \alpha,x_\ell\neq\term
  }\\
  &=\sup_{\pi\in\Pibar_{\alpha}}\bbPbar^{\pi}\brk*{
          \frac{2\dbar^{\pi}_\ell(x_\ell)}{\dbar^{p_{\ell}}_\ell(x_\ell)+\alpha^{-1}\dbar^{\pi}_\ell(x_\ell)} > \alpha,x_\ell\neq\term
    }\\
    &\leq\frac{2}{\alpha}\sup_{\pi\in\Pibar_{\alpha}}\Enbar^{\pi}\brk*{
\frac{\dbar^{\pi}_\ell(x_\ell)}{\dbar^{p_{\ell}}_\ell(x_\ell)+\alpha^{-1}\dbar_\ell^{\pi}(x_\ell)}\indic_{x_\ell\neq\term}
      }
  =\frac{2}{\alpha}\sup_{\pi\in\Pibar_{\alpha}}\sum_{x\in\cXbar}
\frac{(\dbar^{\pi}_\ell(x))^2}{\dbar^{p_{\ell}}_\ell(x)+\alpha^{-1}\dbar_\ell^{\pi}(x)}\indic_{x\neq\term}.
\end{align*}
  By \cref{lem:convex_f}, the function
  \[
    d\mapsto{}
    \frac{(d)^2}{\dbar_\ell^{p_\ell}(x)+\alpha^{-1}\cdot{}d}
  \]
  is convex for all $x$. Hence, writing
  $\dbar^{\pi}_\ell(x)\indic_{x\neq\term}=\Enbar^{\pi}\brk*{\Pbar_{\ell-1}(x\mid{}x_{\ell-1},a_{\ell-1})\indic_{x\neq\term}}$,
  Jensen's inequality implies that for all $x$,
  \begin{align*}
    \frac{(\dbar^{\pi}_\ell(x))^2}{\dbar_\ell^{p_\ell}(x)+\alpha^{-1}\cdot{}\dbar_{\ell}^{\pi}(x)}\indic_{x\neq\term}
    \leq{}\Enbar^{\pi}\brk*{
    \frac{(\Pbar_{\ell-1}(x\mid{}x_{\ell-1},a_{\ell-1}))^2}{\dbar_\ell^{p_\ell}(x)+\alpha^{-1}\cdot{}\Pbar_{\ell-1}(x\mid{}x_{\ell-1},a_{\ell-1}) }
    \indic_{x\neq\term}}.
  \end{align*}
  We conclude that
  \begin{align*}
    \sup_{\pi\in\Pibar_{\alpha}}\bbPbar^{\pi}\brk*{
    \frac{\dbar^{\pi}_\ell(x_\ell)}{\dbar^{p_{\ell}}_\ell(x_\ell)} > \alpha,x_\ell\neq\term
    }
    \leq{} \frac{2}{\alpha}\sup_{\pi\in\Pibara}\Enbar^{\pi}\brk*{
    \frac{\Pbar_{\ell-1}(x_\ell\mid{}x_{\ell-1},a_{\ell-1})}{\dbar_\ell^{p_\ell}(x_\ell)+\alpha^{-1}\cdot{}\Pbar_{\ell-1}(x_\ell\mid{}x_{\ell-1},a_{\ell-1}) }
    \indic_{x_\ell\neq\term}}.
  \end{align*}
  
\end{proof}

\subsubsection{Proof of \creftitle{lem:weight_function_mf} (Guarantee
  for \estimateweight)}
        \newcommand{\wcheck}{\check{w}}%
        \newcommand{\wbar}{\wb{w}}%
        Let $\wbar_h\ind{t}\ldef{}t\cdot{}w_h\ind{t}$ and let
        $\wcheck_h\ind{t}\ldef{}t\cdot\what_h\ind{t}$. Observe that
        solving the optimization problem in \cref{line:weight_alg} of
        \cref{alg:estimate_weight} is equivalent to solving the
        optimization problem in \cref{eq:weight_estimator_log} over
        the class $t\cdot{}\cW_h$, which has $\nrm*{w'}_{\infty}\leq{}t$
        for all $w'\in\cW_h$. As such, we can appeal to
        \cref{thm:weight_estimator_log}
        (in particular, \cref{rem:contextual_weight}) with
        \[
          \mu(x'\mid{}x,a)=\Pmstar_{h-1}(x'\mid{}x,a),\quad\nu(x'\mid{}x,a)=\frac{1}{t}\prn*{\sum_{i<t}d_h^{\sMstar,\pi\ind{h,i}}(x')
            + \Pmstar_{h-1}(x'\mid{}x,a)},
        \]
        and
        \[
          \omega(x,a) =
          \frac{1}{2}\prn*{d_{h-1}^{\sMstar,p_{h-1}}(x,a) + \frac{1}{t-1}\sum_{i<t}d_{h-1}^{\sMstar,\pi\ind{h,i}\circ_{h-1}\piunif}(x,a)}.
        \]
        Under \cref{ass:weight_realizability_strong}, we have
        \begin{align*}
          \frac{\mu(x'\mid{}x,a)}{\nu(x'\mid{}x,a)}
          = \wbar_h\ind{t}(x'\mid{}x,a) \in{} t\cdot\cW_h,
        \end{align*}
        so \cref{thm:weight_estimator_log} and
        \cref{rem:contextual_weight} imply that
        \begin{align*}
          \En_{(x_{h-1},a_{h-1})\sim{}\omega}
\brk*{\prn*{\sqrt{\wcheck_h\ind{t}(x_h\mid{}x_{h-1},a_{h-1})}-\sqrt{\wbar_h\ind{t}(x_h\mid{}x_{h-1},a_{h-1})}}^2}
          \leq \frac{20t\log(\abs{\cW}\delta^{-1})}{tn}
          =           \frac{20\log(\abs{\cW}\delta^{-1})}{n},
        \end{align*}
        or equivalently,
        \begin{align*}
          \En_{(x_{h-1},a_{h-1})\sim{}\omega}
          \brk*{\prn*{\sqrt{\what_h\ind{t}(x_h\mid{}x_{h-1},a_{h-1})}-\sqrt{w_h\ind{t}(x_h\mid{}x_{h-1},a_{h-1})}}^2}
          \leq{} \frac{20\log(\abs{\cW}\delta^{-1})}{tn}.
        \end{align*}
        From the definition of $\omega$, it follows that setting $n=\nweight(\eps,\delta) =
        \frac{40\log(\abs{\cW}\delta^{-1})}{\eps^2}$ is sufficient to
        achieve the desired bound. The total number of episodes is at
        most $2t\cdot\nweight(\eps,\delta)$.

        \qed

}

\arxiv{
\section{Proofs and Additional Details from \creftitle{sec:model_free}}
\label{app:model_free}

\clearpage

\section{Proofs from \creftitle{sec:structural}}
\label{app:structural}

}

\end{document}